\documentclass[11pt, twoside]{article} 

\usepackage[utf8]{inputenc}
\usepackage[T1]{fontenc}   
\usepackage{longtable}
\usepackage{url}           
\usepackage{booktabs}      
\usepackage{amsfonts}      
\usepackage{nicefrac}      
\usepackage{microtype}     
\usepackage{lipsum}        
\usepackage{graphicx}
\usepackage[numbers]{natbib}
\usepackage{amsmath}
\usepackage{amsthm}
\usepackage{algorithm}
\usepackage{algorithmic}
\usepackage{epsf}
\usepackage{epsfig}
\usepackage{fancyhdr}
\usepackage{graphics}
\usepackage{graphicx}
\usepackage{psfrag}
\usepackage{fullpage}
\usepackage{pdfpages}
\usepackage{color}
\usepackage{amssymb,bbm}
\usepackage{caption}
\usepackage{textcomp}
\usepackage{siunitx}
\usepackage{wrapfig}
\usepackage{multirow}
\usepackage{multicol}
\usepackage[colorlinks,linkcolor=magenta,citecolor=blue, pagebackref=true]{hyperref}
\usepackage{gensymb}
\usepackage{enumitem}
\usepackage{subfig}
\usepackage{float}
\usepackage{wrapfig}
\usepackage{tabularx}
\usepackage{colortbl}
\usepackage{xcolor}
\usepackage{tikz}
\usepackage{authblk}

\usetikzlibrary{shapes.geometric, arrows, positioning}
\tikzset{
    block/.style = {rectangle, draw, text width=5em, text centered, rounded corners, minimum height=4em},
    line/.style = {draw, -latex'}
}
\definecolor{darkgreen}{rgb}{0.0, 0.5, 0.0}
\usepackage{esvect}

\newtheorem{theorem}{Theorem}
\newtheorem{proposition}[theorem]{Proposition}
\newtheorem{lemma}[theorem]{Lemma}
\newtheorem{corollary}[theorem]{Corollary}

\newtheorem{remark}[theorem]{Remark}
\newtheorem{conjecture}[theorem]{Conjecture}

\newcommand{\verteq}{\rotatebox{90}{$\,=$}}
\newcommand{\equalto}[2]{\underset{\scriptstyle\overset{\mkern4mu\verteq}{#2}}{#1}}
\title{Stereographic Spherical Sliced Wasserstein Distances}

\author[1]{Huy Tran*}
\author[1]{Yikun Bai*}
\author[1]{Abihith Kothapalli*}
\author[1]{\\Ashkan Shahbazi}
\author[1]{Xinran Liu}
\author[2]{Rocio Diaz Martin}
\author[1]{Soheil Kolouri}
\affil[1]{Department of Computer Science, Vanderbilt University}
\affil[2]{Department of Mathematics, Vanderbilt University}

\date{}

\begin{document}

\maketitle
\def\thefootnote{*}\footnotetext{These authors contributed equally to this work.}

\vspace{-.25in}
\begin{abstract}
Comparing spherical probability distributions is of great interest in various fields, including geology, medical domains, computer vision, and deep representation learning. The utility of optimal transport-based distances, such as the Wasserstein distance, for comparing probability measures has spurred active research in developing computationally efficient variations of these distances for spherical probability measures. This paper introduces a high-speed and highly parallelizable distance for comparing spherical measures using the stereographic projection and the generalized Radon transform, which we refer to as the Stereographic Spherical Sliced Wasserstein (S3W) distance. We carefully address the distance distortion caused by the stereographic projection and provide an extensive theoretical analysis of our proposed metric and its rotationally invariant variation. Finally, we evaluate the performance of the proposed metrics and compare them with recent baselines in terms of both speed and accuracy through a wide range of numerical studies, including gradient flows and self-supervised learning. Our code is available at \url{https://github.com/mint-vu/s3wd}.
\end{abstract}

\section{Introduction}
\label{sec:intro}

Applications involving distributions defined on a hypersphere are remarkably diverse, highlighting the importance of spherical geometries across various disciplines. These applications include: 1) mapping the distribution of geographic or geological features on celestial bodies, such as stars and planets \citep{jupp2020some,cabella2009statistical,perraudin2019deepsphere}, 2) magnetoencephalography (MEG) imaging \citep{vrba2001signal} in medical domains, 3) spherical image representations and $360\degree$ images \citep{coors2018spherenet,jiang20233d}, such as omnidirectional images in computer vision \citep{khasanova2017graph}, 4) texture mapping in computer graphics \citep{elad2005texture,ayelet2010texture}, and more recently, 5) deep representation learning, where the latent representation is often mapped to a bounded space, commonly a sphere, where cosine similarity is utilized for effective representation learning \citep{chen2020simple,wang2020understanding}.

The analysis of distributions on hyperspheres is traditionally approached through directional statistics, also referred to as circular/spherical statistics \citep{jammalamadaka2001topics,mardia2000directional, ley2017modern, pewsey2021recent}. This specialized field is dedicated to the statistical analysis of directions, orientations, and rotations. More recently, with the growing application of optimal transport theory \citep{villani2008optimal,peyre2018computational} in machine learning, due in part to its favorable statistical, geometrical, and topological properties, there has been an increasing interest in using optimal transport to compare spherical probability measures \cite{cui2019spherical,hamfeldt2022convergence}.

One of the main bottlenecks in optimal transport theory is its high computational cost, generally of cubic complexity. This has sparked extensive research to develop faster solvers \citep{cuturi2013sinkhorn, scetbon2022low, charikar2023fast} or computationally superior equivalent distances \citep{rabin2012wasserstein, bonneel2015sliced, kolouri2019generalized}. Notably, sliced variations of optimal transport distances, such as sliced Wasserstein distances \cite{rabin2012wasserstein} and their various extensions \cite{kolouri2019generalized,le2019tree,nguyen2022hierarchical}, have emerged as effective solutions. These methods use integral geometry and the Radon transform \cite{helgason2011integral} to represent high-dimensional distributions via a set of their one-dimensional marginals. By doing so, they take advantage of more efficient optimal transport solvers designed for one-dimensional probability measures, offering a pragmatic approach to comparing probability measures. Owing to their computational efficiency and implementation simplicity, these methods have been recently adapted for spherical measures, leading to the development of spherical sliced optimal transport methods \citep{bonet2022spherical, quellmalz2023sliced}. This comes as part of a broader effort to extend sliced Wasserstein distances to arbitrary manifolds; in this regard, we also highlight extensions of these methods to measures supported on compact manifolds \cite{Rustamov2020IntrinsicSW}, hyperbolic spaces \cite{bonet2023hyperbolic}, and symmetric positive~definite~matrices~\cite{bonet2023sliced}.

The main challenge in developing spherical sliced optimal transport methods lies in extending the classical Radon transform to its spherical counterparts. In the context of sliced optimal transport, such an extension must: 1) map probability measures on the hypersphere to a family of probability measures on one-dimensional domains (e.g., $\mathbb{R}$ or $\mathbb{S}^1$), and 2) be injective so that a sliced distance can be defined (otherwise, one will obtain only a pseudo-metric). These requirements rule out many of the existing extensions of the Radon transform to the sphere, e.g., the classic Funk-Radon transform, which takes integrals along all great circles \citep{helgason2011integral,quellmalz2020funk}. Recently, \citep{quellmalz2023sliced} proposed using two such spherical extensions for the Radon transform, namely, the vertical slice transform \citep{shepp1994spherical} and a normalized version of the semicircle transform \citep{groemer1998spherical}, to define sliced optimal transport on the sphere. Notably, the semicircle transform was also used in \citep{bonet2022spherical} to define a spherical sliced Wasserstein discrepancy for empirical probability measures. 

The recent works on sliced spherical optimal transport \cite{bonet2022spherical,quellmalz2023sliced} map a distribution defined on a hypersphere into its marginal distributions on a unit circle, thus requiring circular optimal transport to compare these marginals. Importantly, the calculation of optimal transport between two one-dimensional measures defined on a circle is more expensive (requiring an additional binary search) than when the measures are defined on the real line \cite{martin2023lcot,hundrieser2022statistics}. Motivated by this observation, in this paper, we explore a spherical Radon transform aimed at converting a spherical distribution into one-dimensional marginals along the real line. We utilize the stereographic projection, which is a smooth bijection, in composition with an injective (nonlinear) map to transform the hypersphere into a hyperplane, where we then apply the classic Radon transform. The key advantage of the injective map is its ability to manage the distortion of distances introduced by the stereographic projection. We support our proposed method with detailed theoretical analysis and comprehensive numerical studies.

\noindent\textbf{Contributions.} Our specific contributions are as follows: 
% \vspace{-.15in}
\begin{itemize}
\setlength\itemsep{-.05em}
    \item Introducing a computationally efficient transport distance, Stereographic Spherical Sliced Wasserstein distance (S3W), for spherical probability measures.
    \item Providing a rotationally invariant variation of the proposed distance, Rotationally Invariant Stereographic Spherical Sliced Wasserstein distance (RI-S3W).
    \item Offering theoretical analysis of the proposed distances.
    \item Demonstrating the performance, both in terms of speed and accuracy, of the proposed distances in diverse applications, including gradient flows on the sphere, representation learning with Sliced-Wasserstein Auto-Encoders (SWAEs), spherical density estimation via normalizing flows, sliced-Wasserstein variational inference on the sphere, and self-supervised learning.
\end{itemize}

\section{Background}
\label{sec:background}

\subsection{Stereographic Projection}
Stereographic projection is a mathematical technique to map points on a sphere onto a plane. Originating in Greek astronomy, it projects points from a sphere, typically from one of its poles, onto a plane tangential to the opposite pole or onto the ``equator plane.'' This projection is conformal, preserving angles and shapes locally, making it significant in fields like cartography, complex analysis, and computer graphics. It elegantly translates spherical geometry into planar terms, maintaining the intricate relationships between points and angles found on the sphere's surface.

\begin{figure*}[t!]
    \centering
    \includegraphics[width=\linewidth]{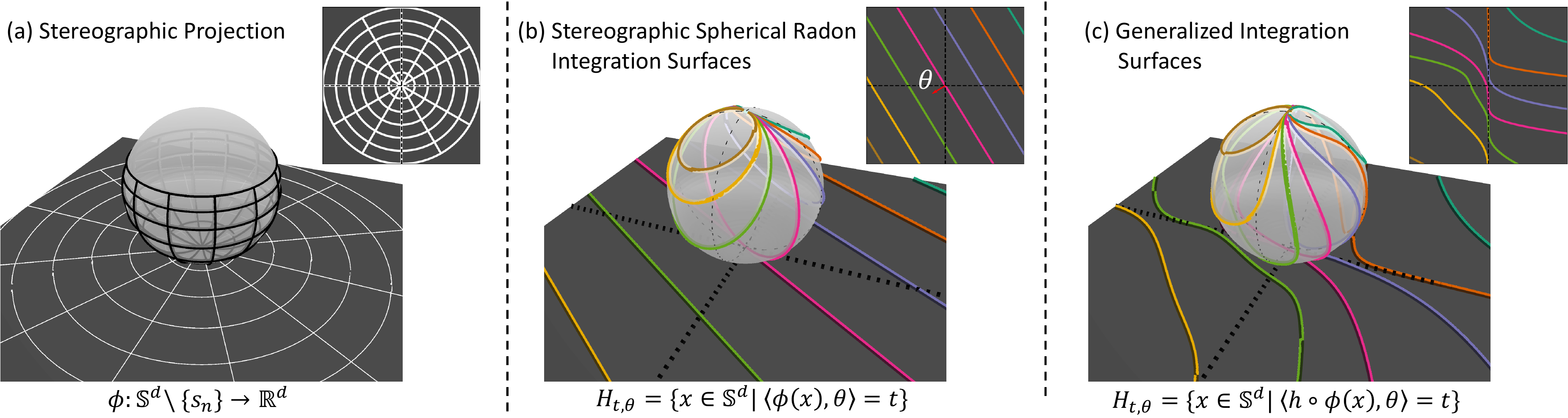}\vspace{-.1in}
    \caption{Depiction of stereographic projection from $\mathbb{S}^2\backslash \{s_n\}$ to $\mathbb{R}^2$ (a), the stereographic Radon transform integration surfaces on the sphere, i.e., the level sets of $\langle \phi(x),\theta\rangle$ for a fixed $\theta\in \mathbb{R}^d$ (b), and the generalized stereographic Radon transform integration surfaces on the sphere, i.e. the level sets of $\langle h\circ\phi(x),\theta\rangle$ for a fixed $\theta\in \mathbb{R}^{d'}$.}
    % \vspace{-.1in}
\end{figure*}

Let \(\mathbb{S}^d\) denote the \(d\)-dimensional sphere in \(\mathbb{R}^{d+1}\) defined as 
\(\mathbb{S}^d := \{s \in \mathbb{R}^{d+1} : \|s\|_2 = 1\}\). The stereographic projection 
\(\phi: \mathbb{S}^d \setminus \{s_n\} \to \mathbb{R}^d\) maps a point \(s\in \mathbb{S}^d\) (excluding the ``North Pole," \(s_n = [0, \ldots, 0, 1]\)) to a point 
\(x\in \mathbb{R}^d\) by the formula:
\begin{equation}\label{eq: stereographic proj}
   [x]_i = \frac{2[s]_i}{1 - [s]_{d+1}}, \quad \text{for } i = 1, 2, \ldots, d.  
\end{equation}
This projection is a bijective and smooth mapping between \(\mathbb{S}^d \setminus \{s_n\}\) and 
\(\mathbb{R}^d\), i.e., the hyperplane tangent to the sphere at the south pole,~$s_{d+1}=-1$, which is commonly used as a way to visualize spherical geometries in an Euclidean space.

\subsection{Radon Transform}

The Radon transform is a fundamental tool in integral geometry, widely used for image reconstruction, particularly in computed tomography (CT). It transforms a function defined in the $d$-dimensional Euclidean space, $\mathbb{R}^d$, into an infinite set of its one-dimensional slices, i.e., its integrals over hyperplanes parametrized with unit vectors $\theta\in\mathbb{S}^{d-1}$. Radon transform is instrumental in medical imaging and has recently attracted ample attention from the machine learning community for its favorable characteristics for measuring distances between probability measures \cite{bonneel2015sliced,kolouri2016radon,kolouri2019generalized} and, more generally, positive measures \cite{bai2023sliced,sejourne2023unbalanced}, and more recently on theoretical analysis of deep neural networks \cite{parhi2023deep,unser2023ridges}. 

More formally, the classic Radon transform, denoted by \( \mathcal{R} \), maps a function \( f \) in the Lebesgue space \( L^1(\mathbb{R}^d) \) — the space of integrable functions on \( \mathbb{R}^d \) — to its integrals over the hyperplanes of \( \mathbb{R}^d \). Formally, it is defined as:
\begin{eqnarray}
    \mathcal{R} f(t,\theta) := \int_{\mathbb{R}^d} f(x)\delta(t-\langle x, \theta \rangle)dx,
\end{eqnarray}
for \( (t, \theta) \in \mathbb{R} \times \mathbb{S}^{d-1} \), where \( \delta(\cdot) \) is the one-dimensional Dirac delta function, and \( \langle \cdot, \cdot \rangle \) denotes the Euclidean inner product. Here, \( \mathcal{R} \) maps from \( L^1(\mathbb{R}^d) \) to \( L^1(\mathbb{R}\times \mathbb{S}^{d-1}) \). Each hyperplane, \( H(t, \theta) = \{x\in \mathbb{R}^d \ |\ \langle x, \theta \rangle = t\} \), corresponds to a level set of a function \( g: \mathbb{R}^d\times\mathbb{S}^{d-1}\rightarrow \mathbb{R} \), defined as \( g(x, \theta) = \langle x, \theta \rangle \). For a fixed \( \theta \), the set of all integrals over hyperplanes orthogonal to \( \theta \) yields the continuous function \( \mathcal{R}f(\cdot,\theta): \mathbb{R} \rightarrow \mathbb{R} \), representing a projection or slice of \( f \).

Notably, the Radon transform $\mathcal{R}:L^1(\mathbb{R}^d) \rightarrow L^1(\mathbb{R}\times \mathbb{S}^{d-1})$ is a linear bijection with a closed-form inversion formula. See the Appendix for details.
% For a function \( f \in L^1(\mathbb{R}^d) \), the Radon transform \( \mathcal{R}f \) can be inverted using the inverse Radon transform, denoted by \( \mathcal{R}^{-1} \). The inversion formula is given by:
% \begin{equation}
%     f(x) = \mathcal{R}^{-1}(\mathcal{R}f)(x) = \int_{\mathbb{S}^{d-1}}  \big(\mathcal{R}f(\cdot,\theta)*\eta(\cdot)\big)(\langle x,\theta\rangle) d\theta,
% \end{equation}
% where $\eta(\cdot)$ is a one-dimensional high-pass filter with corresponding Fourier transform $\mathcal{F}\eta(\omega) =  c|\omega|^{d-1}$, which appears due to the Fourier slice theorem \cite{helgason2011radon}, 
% and `$*$' is the convolution operator. 
% This integral formula effectively reconstructs \( f \) from its Radon transform and, in the medical imaging community, is referred to as `filtered back projection.'

% \rocio{I think the integral in the reconstruction formula is over the half sphere.}

\subsection{Generalized Radon Transform}

The Generalized Radon Transform (GRT) extends the foundational concept of the classic Radon transform, as introduced by Radon in \cite{radon1917uber}, from integration over hyperplanes in $\mathbb{R}^d$ to integration over more complex structures, namely hypersurfaces or $(d-1)$-dimensional manifolds. This broader scope of the GRT has been developed and explored in various works \cite{beylkin1984inversion,denisyuk1994inversion,ehrenpreis2003universality,gel1969differential,kuchment2006generalized,homan2017injectivity}, in many applications from impedance and acoustic tomography 
 \cite{kuchment2006generalized} to machine learning \cite{kolouri2019generalized}.

The GRT of a function \( f \in L^1(\mathbb{R}^d) \) involves the integration of \( f \) over hypersurfaces in $\mathbb{R}^d$. These hypersurfaces are defined as the level sets of a `defining function' \cite{kolouri2019generalized}, \( g: \mathbb{R}^d \times (\mathbb{R}^{d'} \setminus \{0\}) \to \mathbb{R} \), %\footnote{\rocio{More generally, the domain of $g$ could be of the form $\Omega\times \Omega_\theta$, where $\Omega\subseteq \mathbb{R}^d$ and $\Omega_\theta\subseteq \mathbb{R}^{d'}\setminus\{0\}$ are borel subsets.}}, 
characterized by \( H_{t,\theta} = \{ x \in \mathbb{R}^d | \, g(x, \theta) = t \} \). The GRT of \( f \), denoted as \( \mathcal{G}f \), is formally given by:
\begin{equation}
    \mathcal{G}f(t,\theta) := \int_{\mathbb{R}^d} f(x)\delta(t - g(x, \theta)) \, dx, \label{eq: GRT}
\end{equation}
where \( \delta \) represents the Dirac delta function, enabling the integration over the specific level sets defined by \( g \). Note that $g(x,\theta)=\langle x , \theta \rangle$ recovers the classic Radon transform. 

The injectivity of the GRT is essential for defining distances between measures through their generalized slices. The choice of \( g \) identifies whether GRT is injective or not. \cite{kolouri2019generalized} enumerate a set of necessary conditions on \(g\) to construct a bijective GRT. However, these conditions are not sufficient and are limited in practical utility, as they do not provide guidance on crafting specific defining functions that would ensure injectivity.

To achieve a practical injective GRT, \cite{chen2022augmented} recently introduced a variation of the GRT. They posited that by setting \( g(x,\theta) = \langle h(x), \theta \rangle \) for an injective function \( h: \mathbb{R}^d \rightarrow \mathbb{R}^{d'} \), one effectively applies the classic Radon transform, which is bijective, to the image of \( h \). This approach leads to an injective GRT, thus resolving the issue of invertibility. For consistency of notation with \cite{chen2022augmented}, we denote this variation of the GRT as: 
\vspace{-.05in}
\begin{eqnarray}
    \mathcal{H}f(t,\theta) := \int_{\mathbb{R}^d} f(x)\delta(t - \langle h(x), \theta\rangle) \, dx,
    \label{eq:augmented}
\end{eqnarray}
where $\mathcal{H}:L^1(\mathbb{R}^d)\rightarrow L^1(\mathbb{R}\times \mathbb{S}^{d'-1})$. Importantly, a neural network can effectively parametrize the injective function $h$. For instance, $h$ could be a normalizing flow \cite{kobyzev2020normalizing}. Alternatively, to avoid the complexities associated with normalizing flows, one can define $h(x)$ as $[x^T, \rho(x)^T]^T$, where $x^T$ and $\rho(x)^T$ are transposed vectors that are concatenated. Here, $\rho:\mathbb{R}^d\rightarrow \mathbb{R}^{d'-1}$ is any nonlinear function parametrized as a neural network of choice \cite{chen2022augmented}. In our experiments, we will use $\mathcal{H}$ as the default GRT due to its injectivity.

\subsection{GRT of Radon Measures}
\label{sec:grt_measures}
Skipping much of the theoretical details (refer to the appendix), for a Radon measure $\mu\in\mathcal{M}(\mathbb{R}^d)$, its Generalized Radon Transform  $\mathcal{G}(\mu)=\nu$ with respect to the defining function $g$, is defined as the measure $\nu\in \mathcal{M}(\mathbb{R}\times \mathbb{S}^{d'-1})$ such that for each $\psi\in C_0(\mathbb{R}\times\mathbb{S}^{d'-1})$, 
\begin{align}
\int_{\mathbb{R}\times \mathbb{S}^{d'-1}}\psi(t,\theta) \, d\nu(t,\theta)=\int_{\mathbb{R}^d}(\mathcal{G}^*(\psi))(x) \, d\mu(x), \label{eq: R(mu)}
\end{align}
where $\mathcal{G^*}$ is the dual operator (aka adjoint operator), which for any $\psi \in L^\infty(\mathbb{R}\times \mathbb{S}^{d'-1})$, is defined as
\begin{align}
    \mathcal{G}^*(\psi)(x)=\int_{ \mathbb{S}^{d'-1}}\psi(g(x,\theta),\theta) \, d\sigma_{d'}(\theta) \label{eq: GR*} \quad \forall x\in \mathbb{R}^d,
\end{align}
 where $\sigma_{d'}$ is the uniform probability measure defined in $\mathbb{S}^{d'-1}$. Importantly, the dual GRT operator satisfies,
$$\mathcal{G}(\mu)(\psi)=\mu(\mathcal{G}^*(\psi)).$$ 
Lastly, with a slight abuse of notation, we denote the corresponding slice for $\theta\in\mathbb{S}^{d'-1}$ as $\mathcal{G}(\mu)_\theta=g(\cdot,\theta)_\# \mu \in \mathcal{M}(\mathbb{R})$. Note that $f_\#\mu$ denotes the pushforward of measure $\mu$ with respect to $f$, defined as $f_\#\mu(A)=\mu(f^{-1}(A))$. When $\mu$ is a positive or a probability measure in $\mathbb{R}^d$, then $\mathcal{G}(\mu)_\theta$ is a positive/probability measure in $\mathbb{R}$.

% Lastly, the GRT of an absolutely continuous Radon measure \(\nu\) defined in \(\mathbb{R}^d\) can be expressed in terms of the Radon-Nikodym derivative of \(\nu\) with respect to the Lebesgue measure. Let's assume that \(\nu\) is absolutely continuous with respect to the Lebesgue measure, denoted by \(m\), and let \(f\) be the Radon-Nikodym derivative such that \(d\nu = f \, dm\), where \(f\in L^1(\mathbb{R}^d)\). The GRT of $\nu$ is defined as:
% \begin{align}
% \mathcal{H}\nu(t,\theta) &= \int_{\mathbb{R}^d} \delta(t-\langle h(x), \theta \rangle) \, d\nu(x), \end{align}
% which is equal to Eq. \eqref{eq:augmented}

\subsection{Wasserstein and Sliced Wasserstein Distances}

Let $M$ denote a Riemannian manifold endowed with the distance $d(\cdot,\cdot):M\times M \rightarrow \mathbb{R}_+$. For $1\leq p<\infty$, let $\mu, \nu\in \mathcal{P}_p(M):=\{\mu \in \mathcal{P}(M)| \, \int_M d^p(x,x_0) \, d\mu(x)<\infty~\text{for some}~ x_0\in M\}$ be two probability measures defined on manifold $M$ with a finite $p$-th moment. Then, the optimal transport (OT) problem \cite{villani2008optimal} seeks to transport the mass in $\mu$ into $\nu$ such that the expected traversed distance is minimized. This leads to the $p$-Wasserstein distance:
\begin{align}
    W_p^p(\mu,\nu):=\inf_{\gamma\in \Gamma(\mu,\nu)} \int_{M\times M} d^p(x,y)d\gamma(x,y),
    \label{eq:ot}
\end{align}
where $\Gamma(\mu,\nu)$ %$:=\{\gamma\in \mathcal{P}(M\times M)| \forall A \subseteq M, \gamma(A\times M)=\mu(A)~\text{and}~ \gamma(M\times A)=\nu(A)\}$ 
%denotes the Kantorovich transport couplings between $\mu$ and $\nu$ defined as the set of all 
denotes the joint probability measures $\gamma\in\mathcal{P}(M\times M)$ with marginals $\mu$ and $\nu$. 
%As can be seen, calculating the $p$-Wasserstein distance requires solving a linear programming problem. 
Unfortunately, for discrete probability measures with $N$ particles, solving \eqref{eq:ot} generally has a $\mathcal{O}(N^3 \log N)$ complexity. However, for $\mu, \nu \in \mathcal{P}_p(\mathbb{R})$ the problem can be solved in $\mathcal{O}(N\log N)$: 
\begin{align}
    W_p^p(\mu,\nu)=\int_{0}^1 \|F_{\mu}^{-1}(t)-F_{\nu}^{-1}(t)\|^p dt
\end{align}
where $F^{-1}_{\mu}$ and $F^{-1}_{\nu}$ are the quantile functions of $\mu$ and $\nu$. Notably, similar efficient solvers are developed for when $\mu,\nu\in \mathcal{P}_p(\mathbb{S}^{1})$ \cite{delon2010fast,hundrieser2022statistics,bonet2022spherical}.  For an injective generalized Radon transform $\mathcal{G}$, such that for $\mu\in\mathcal{P}_p(M)$
%with $M=\mathbb{S}^{d'-1}$, 
we have $\mathcal{G}(\mu)_\theta\in \mathcal{P}_p(\mathbb{R})$, the generalized Sliced-Wasserstein distance \cite{rabin2012wasserstein,kolouri2019generalized} can be defined as: 
\begin{align}    SW_{\mathcal{G},p}^p(\mu,\nu):=\int_{\mathbb{S}^{d'-1}} W_p^p(\mathcal{G}(\mu)_\theta,\mathcal{G}(\nu)_\theta)d\sigma_{d'}(\theta),
\end{align}
where $\sigma_{d'}\in\mathcal{P}(\mathbb{S}^{d'-1})$ is a probability measure possessing a non-zero density on the sphere $\mathbb{S}^{d'-1}$, often simply the uniform measure. Recently, \cite{bonet2022spherical} introduced the concept of Spherical Sliced Wasserstein distance, which involves projecting spherical measures onto great circles, resulting in $\mathcal{G}(\mu)_\theta\in \mathcal{P}_p(\mathbb{S}^1)$, and utilizing circular OT to measure distances between these slices. Notably, circular OT still necessitates solving an optimization problem to register cumulative distribution functions on a circle (i.e., finding an optimal cut). This results in a slower solver compared to OT on $\mathbb{R}$.

\section{Method}
\label{sec:method}
Based on the extended definition of GRT for probability measures in \ref{sec:grt_measures}, here we formally introduce the ``Stereographic Spherical Radon Transform'' and the corresponding sliced Wasserstein distance, ``Stereographic Spherical Sliced Wasserstein Distance,'' for spherical probability measures. 

\subsection{Stereographic Spherical Radon Transform}

Let \(\mu\in \mathcal{M}(\mathbb{S}^d)\) denote a Radon measure defined on \(\mathbb{S}^{d}\) that does not assign mass to the North Pole, i.e., $\mu(\{s_n\})=0$. We denote the stereographic projection as $\phi:\mathbb{S}^{d}\setminus \{s_n\}\to \mathbb{R}^d$, which is a bijection, and we have that $\phi_\#\mu$ is a Radon measure defined in $\mathbb{R}^d$. The \textbf{Stereographic Spherical Radon transform}, of $\mu$ is defined as 
\begin{align}
    \mathcal{S}_{\mathcal{R}}(\mu):=\mathcal{R}(\phi_\#\mu)\in \mathcal{M}(\mathbb{R}\times \mathbb{S}^{d-1}), \label{eq: SR} 
\end{align}
and we also define its generalized version as
\begin{align}
    \mathcal{S}_{\mathcal{G}}(\mu):=\mathcal{G}(\phi_\#\mu)\in \mathcal{M}(\mathbb{R}\times \mathbb{S}^{d'-1}). \label{eq: SG}
\end{align}
Similarly we define $\mathcal{S}_\mathcal{H}(\mu):= \mathcal{H}(\phi_\#\mu)$, which provides an invertible transformation from $\mathcal{M}(\mathbb{S}^d)$ to $\mathcal{M}(\mathbb{R}\times\mathbb{S}^{d'-1})$.

\begin{proposition}\label{pro: SSRT ori}
For $\mu\in \mathcal{M}(\mathbb{S}^{d})$ that does not give mass to the North Pole $\{s_n\}$
the Stereographic Spherical Radon transforms $\mathcal{S}_\mathcal{G}$ and $\mathcal{S}_\mathcal{H}$ satisfy the following properties: 
\begin{enumerate}
    \item[(1)]   
    $\mathcal{S}_\mathcal{G}(\mu),\mathcal{S}_\mathcal{H}(\mu)\in\mathcal{M}(\mathbb{R}\times \mathbb{S}^{d'-1})$. 
    In addition $\mathcal{S}_\mathcal{G}$ and $\mathcal{S}_\mathcal{H}$ preserves mass, and if  $\mu$ is a positive measure, then $\mathcal{S}_\mathcal{G}(\mu),\mathcal{S}_\mathcal{H}(\mu)$ are positive measures too. Finally, if $\mu\in \mathcal{P}(\mathbb{S}^{d}\setminus\{s_n\})$, then  
    $\mathcal{S}_\mathcal{G}(\mu),\mathcal{S}_\mathcal{H}(\mu)$ are probability measures defined on $\mathbb{R}\times \mathbb{S}^{d'-1}$.
    \item[(2)] The disintegration theorem gives a unique $\mathcal{S}_\mathcal{G}(\mu)-$a.s. set of measures $(\mathcal{S}_\mathcal{G}(\mu)_\theta)_{\theta\in \mathbb{S}^{d'-1}}\subset \mathcal{M}(\mathbb{R})$, and a unique $\mathcal{S}_\mathcal{H}(\mu)-$a.s. set of measures $(\mathcal{S}_\mathcal{H}(\mu)_\theta)_{\theta\in \mathbb{S}^{d'-1}}\subset \mathcal{M}(\mathbb{R})$  such that for any $\psi \in C_0(\mathbb{R}\times\mathbb{S}^{d'-1})$,
    \begin{align}
\int_{\mathbb{R}\times\mathbb{S}^{d'-1}}\psi(t,\theta) \, d\mathcal{S}_\mathcal{G}(\mu)(t,\theta) =\int_{\mathbb{S}^{d'-1}}\int_{\mathbb{R}}\psi(t,\theta) \, d\mathcal{S}_\mathcal{G}(\mu)_\theta(t) \, d\sigma_{d'}(\theta), \qquad 
\end{align}
    and similarly for $\mathcal{S}_{\mathcal{H}}$. Then, it holds that
    \begin{equation}\label{eq: slice general Sg and Sh}
      \mathcal{S}_\mathcal{G}(\mu)_\theta=(g(\cdot,\theta)\circ \phi)_\#\mu,
    \end{equation}
    and the same holds for $\mathcal{S}_\mathcal{H}$.
    \item[(3)] 
    $\mathcal{S}_\mathcal{H}$ is invertible. 
\end{enumerate}
\end{proposition}
The proof of this proposition is included in the appendix Section \ref{sec:supp_ssr}. Being equipped with the proposed stereographic spherical Radon transform, we are now ready to define our proposed distance. 

\subsection{Stereographic Spherical Sliced Wasserstein }

Let $\mu,\nu\in \mathcal{P}_p(\mathbb{S}^{d})$ denote two probability measures on the unit sphere in $\mathbb{R}^{d+1}$. We introduce the formal definition of the Stereographic Spherical Sliced Wasserstein ($S3W$) distances as follows: 
\begin{align*}
S3W_{\mathcal{G},p}^p(\mu,\nu)&:=\int_{\mathbb{S}^{d'-1}}W_p^p(\mathcal{S}_\mathcal{G}(\mu)_\theta,\mathcal{S}_\mathcal{G}(\nu)_\theta) 
 \, d\sigma_{d'}(\theta).
\end{align*}
In this context, $\sigma_{d'}\in\mathcal{P}(\mathbb{S}^{d'-1})$ typically represents a probability measure possessing a non-zero density on the sphere $\mathbb{S}^{d'-1}\subset\mathbb{R}^{d'}$. However, for the sake of simplicity and in line with common practice, we opt to consider the uniform measure on $\mathbb{S}^{d'-1}$. Similarly, we can define $S3W_{\mathcal{H},p}^p$ by integrating over $W_p^p(\mathcal{S}_\mathcal{H}(\mu)_\theta,\mathcal{S}_\mathcal{H}(\nu)_\theta)$.

\begin{theorem}\label{th:main} The proposed 
 $S3W_{\mathcal{G},p}(\cdot,\cdot)$ and $S3W_{\mathcal{H},p}(\cdot,\cdot)$ are well-defined. 
Furthermore, $S3W_{\mathcal{G},p}(\cdot,\cdot)$ is generally a pseudo-metric in $\mathcal{P}_p(\mathbb{S}^{d}\setminus \{s_n\})$, i.e., it is non-negative, symmetric and satisfies triangular inequality. In addition, $S3W_{\mathcal{H},p}(\cdot,\cdot)$ defines a metric in $\mathcal{P}_p(\mathbb{S}^{d}\setminus\{s_n\})$. 
%\rocio{Do we need $\mathcal{P}$ or $\mathcal{P}_p$?}
\end{theorem}

The proof is in the appendix Section \ref{sec:supp_ssr}.

% We define the Stereographic Spherical Sliced Wasserstein distance (S3WD) between them as: 
% \begin{align}
%     S3W_p^p(\mu,\nu) := \int_{\mathbb{S}^{d'-1}} W_p^p(\phi^\theta_\#\mu,\phi^\theta_\#\nu) \, d\sigma_{d'}(\theta)
% \end{align}
% where $\sigma_{d'}$ is the uniform measure on $\mathbb{S}^{d'-1}$.

% \soheil{I think we should drop the max-S3W from the main paper, especially given that we won't report results on it.}

% Lastly, and similar to the max-sliced Wasserstein distance, here we define the max-S3W distance as: 
% \begin{align}    \text{Max-}S3W_{\mathcal{G},p}^p(\mu_1,\mu_2) := \sup_{\theta\in \mathbb{S}^{d'-1}} ~~W_p^p(\mathcal{S}_\mathcal{G}(\mu_1)_\theta,\mathcal{S}_\mathcal{G}(\mu_2)_\theta),
% \end{align}
% and similarly for $\mathcal{H}$ we define $\text{Max-}S3W_{\mathcal{H},p}^p$.

\subsection{Distance Distortion}\label{sec:dist_distortion}

% \huy{I propose change the North Pole notation to $s_n$ and pseudo northpole $p'_N$ since $n$ is often overloaded and could be confusing}

The Stereographic Projection, while being conformal (i.e., preserving angles), severely distorts distances. To demonstrate the extent of this distortion, consider points $s=(\epsilon,0,\ldots,0, \sqrt{1-\epsilon^2})$ and $s'=(-\epsilon,0,\ldots, 0, \sqrt{1-\epsilon^2})$. Then, as $\epsilon\rightarrow 0$, we have $\arccos(\langle s, s'\rangle)\to 0$, while $\|\phi(s)-\phi(s')\|\to \infty$! This distortion implies that the transportation cost after the stereographic projection would be significantly different from the transportation cost on the sphere, making the naive application of this projection with optimal transport unsuitable. 

% In this section, we aim to design an injective function, $h$, such that the Euclidean distance induced by $h$ approximates the arclength, i.e.,
% $$\|h(\phi(s))-h(\phi(s'))\|\approx \arccos(\langle s, s'\rangle),~\forall s,s'\in\mathbb{S}^d.$$ To this end, we propose a good candidate for such an $h$ and analyze it theoretically. Then, we provide a training framework to learn such an $h$ that leads to a nearly-isometric embedding.

In this section, we aim to construct an injective function $h$ that closely approximates the arclength on the sphere $\mathbb{S}^d$ with the Euclidean distance in the embedded space. Specifically, we seek to satisfy:
$$\|h(\phi(s))-h(\phi(s'))\|\approx \arccos(\langle s, s'\rangle),~\forall s,s'\in\mathbb{S}^d.$$ We introduce two variants of $h$: an analytical function $h_1(\cdot)$ and a neural network-based learnable function $h_{NN}(\cdot)$, both mapping to a nearly-isometric embedding from $\phi(s)$.

We start by defining the analytic function:
\begin{align}
    h_1(x):= \arccos\left(\frac{\|x\|^2-1}{\|x\|^2+1}\right)\frac{x}{\|x\|}, \quad \forall x\in\mathbb{R}^d
    %\arccos\left(\frac{1-\|x\|}{1+\|x\|}\right)\frac{x}{\|x\|}, \qquad \forall x\in\mathbb{R}^d
    \label{eq: h_1}
\end{align}
and provide the following proposition. 

\begin{proposition}
\label{pro:distortion}
For $s_0,s,s'\in \mathbb{S}^{d}$, where $s_0$ denotes the South Pole, the stereographic projection $\phi$, and $h_1$ as defined in Eq. \eqref{eq: h_1} we have: 
\begin{itemize}
\item $h_1(\phi(s))=\angle(s,s_0)\frac{s[1:d]}{\|s[1:d]\|}$. Thus, we have
%the Euclidean distance in embedding space provides an upper bound for the arclength, and it is upper bounded by $2\pi$, i.e., 
    $$ \|h_1(\phi(s))-h_1(\phi(s'))\|\leq 2\pi$$
    \item The following inequality holds,
    \begin{align*}
        \arccos(\langle s,& s'\rangle) \leq \|h_1(\phi(s))-h_1(\phi(s'))\|+\epsilon(s,s')
        %, \\ &2\pi-\|h_1(\phi(s))-h_1(\phi(s'))\|)+
    \end{align*}
    where the equality holds when $s$, $s'$, and $s_0$ are in the same great circle, 
    and $\epsilon(s,s')\to 0$ as $d_{\mathbb{S}^d}(s, s')\to 0$. 
\end{itemize}
\end{proposition}
The proof is in the appendix Section \ref{sec:supp_dist}.

Next, we consider training a neural network to obtain a nearly-isometric Euclidean embedding. To do so, we define 
\begin{equation}
    h_{NN}(x):=[h^T_1(x)/C,\rho^T(x)]^T \label{eq:h_NN}
\end{equation}
where $\rho:\mathbb{R}^{d}\to\mathbb{R}^{d'-d}$ is a neural network, and $C\geq 2\pi$ a constant. We first note that such $h_{NN}:\mathbb{R}^d\rightarrow \mathbb{R}^{d'}$ is injective, then train $\rho$ by minimizing: 
\begin{align}
    \mathcal{L}(\rho)=\mathbb{E}_{s,s'} \Big[&(\arccos(\langle s, s'\rangle) \nonumber \\
    &~~- \|h_{NN}(\phi(s))-h_{NN}(\phi(s'))\|)^2\Big],
    \label{eq: loss_rho}
\end{align}
where $s$ and $s'$ are sampled according to the uniform distribution in the sphere $\mathbb{S}^d\subset\mathbb{R}^{d+1}$, i.e., $(s,s')\sim \sigma_{d+1}\times \sigma_{d+1}$. Figure \ref{fig:distortion} illustrates the arclength versus the distance in the embedding for random pairs of samples $s,s'\in\mathbb{S}^{d}$ with respect to the various scenarios proposed in this section. It is evident that an injective function $h_{NN}$ parameterized with a neural network can yield a nearly-isometric embedding. We discuss the use of this neural-network based $h_{NN}$ further in the appendix Section \ref{sec:h_nn_appendix}.

\begin{figure}[H]
    \centering
    \includegraphics[width=0.7\columnwidth]{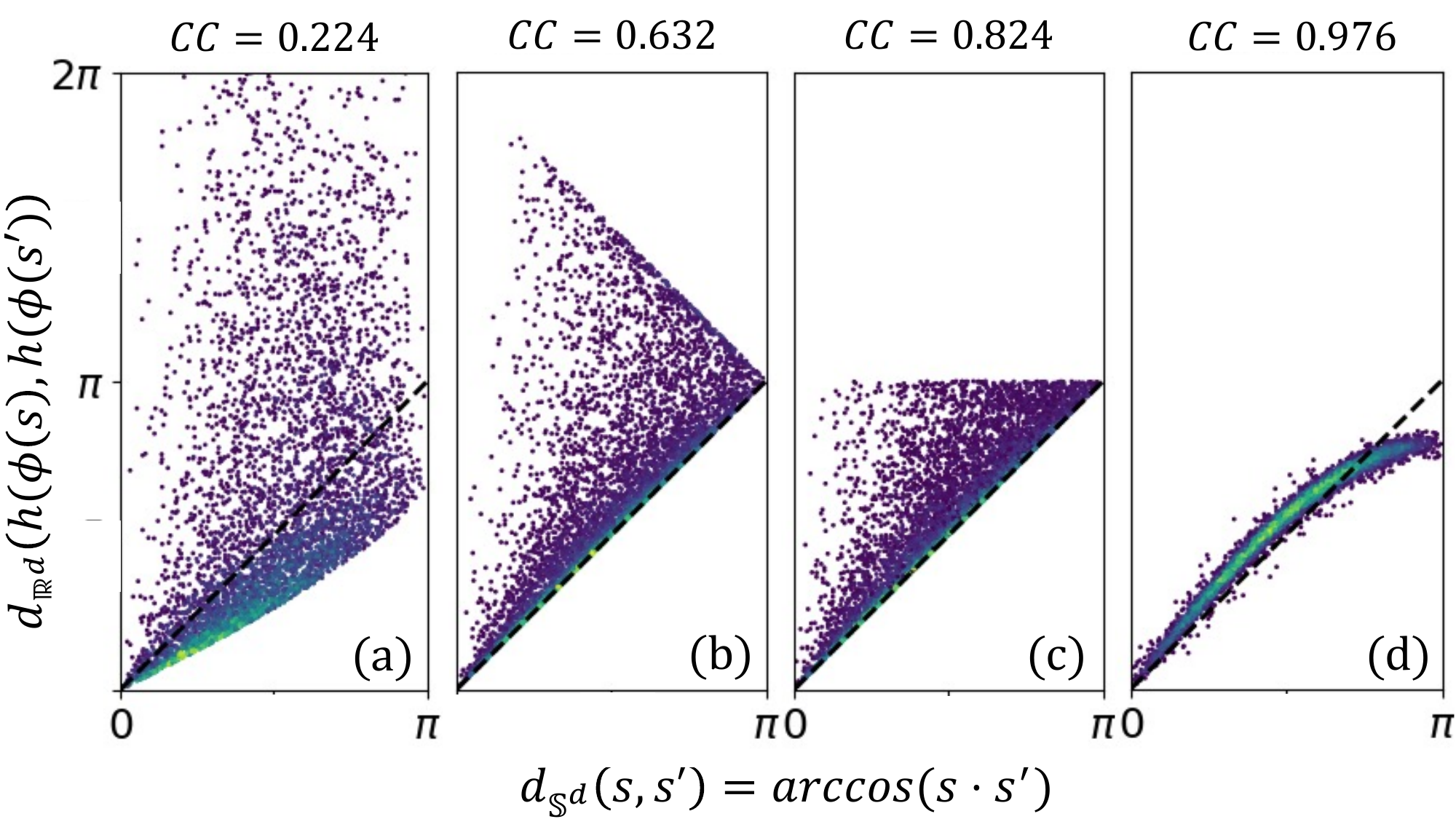}
    % \captionsetup{skip=0pt}
    \caption{ Spherical distance (i.e., the arclength) versus the distance after stereographic projection, where CC denotes Pearson's correlation coefficient.  From left to right, when the injective function $h=id$, and the distance is $\|\phi(s)-\phi(s')\|$ (a),  when $h(x)=h_1(x)$ (see Eq. \eqref{eq: h_1}) and the distance is $\|h(\phi(s))-h(\phi(s'))\|$ (b),  when $h(x)=h_1(x)$ and the distance is $\min(\|h(\phi(s))-h(\phi(s'))\|,2\pi-\|h(\phi(s))-h(\phi(s'))\|)$ (c), and finally when $h(x)=h_{NN}(x)$ (see Eq. \eqref{eq:h_NN}) where $\rho(x)$ is a trained neural network minimizing Eq. \eqref{eq: loss_rho} and $C\geq 2\pi$ (d).}
    \label{fig:distortion}
\end{figure}

To avoid any potential confusion, we emphasize that irrespective of the specific choice of $h$, as long as it maintains injectivity, $S3W_{\mathcal{H},p}$ remains a valid metric in $\mathcal{P}_p(\mathbb{S}^d)$. The discussion in this section is, however, significant, particularly when we aim to ensure that the transportation cost in the embedding closely resembles the spherical distance.

\subsection{Rotationally Invariant Extension of $S3W$}

%Rotational symmetry is an important property of probability metrics on the sphere. 
Rotational symmetry is a significant property when dealing with probability metrics on the sphere. This symmetry implies that the probability distribution remains unchanged under rotations, making it a key consideration in spherical statistics and related applications. It simplifies calculations and enhances our understanding of the underlying processes on a spherical surface.
Notably, the spherical OT leads to a rotationally symmetric metric on $\mathcal{P}_p(\mathbb{S}^d)$. However, our proposed Stereographic Spherical Sliced Wasserstein ($S3W$) metric is not rotationally invariant. Here, we propose a rotationally invariant variation of $S3W$, which leads to a robust and easy-to-implement rotationally invariant metric. 

Let $\mathrm{SO}(d+1)$ denote the special orthogonal group in $\mathbb{R}^{d+1}$, and let $R\in \mathrm{SO}(d+1)$ denote a rotation matrix. For $\mu\in \mathcal{P}_p(\mathbb{S}^d)$ we denote the rotated measure as $R_\# \mu$. Given probability measures $\mu, \nu \in \mathcal{P}_p(\mathbb{S}^d)$, we define the rotationally invariant extension of $S3W$ as: 
\begin{align}
RI\text{-}S3W_{\mathcal{G},p}(\mu,\nu) := \mathbb{E}_{R\sim \omega} [S3W_{\mathcal{G},p}(R_\#\mu, R_\#\nu)], \label{eq: ri-s3w} 
\end{align}
where $\omega$ denotes the normalized Haar measure on $\mathrm{SO}(d+1)$ (i.e., the uniform probability measure which is left-rotation-invariant\footnote{The Haar measure on a locally compact group $G$ is the unique, up to positive constants, left-translation-invariant regular Borel measure on $G$. If $G=\mathrm{SO}(d+1)$, the group translations are rotations. Here, we normalize it to obtain a probability measure.}). We similarly define $\text{RI-S3W}_{\mathcal{H},p}$ using $\text{S3W}_{\mathcal{H},p}$.

% It is straightforward to check that for $\mu, \nu \in \mathcal{P}_p(\mathbb{S}^d)$: 
% \begin{align}
% S3W_{\mathcal{G},p}(R_\#\mu,R_\#\nu)\neq S3W_{\mathcal{G},p}(\mu,\nu),
% \end{align}
% This is while the spherical Wasserstein distance is invariant to rotation. 

% \huy{Suggestion:}

% The rotational symmetry of $\mathbb{S}^d$ implies that a meaningful distance metric in this space should be invariant to rotational transformation. However, while being a conformal bijection, the stereographic mapping introduces an asymmetry through the projection point and distorts distances based on the orientation of the data distributions. It is straightforward to verify that our proposed $S3W$ metric is not invariant to rotations. To overcome the challenge, we introduce its Rotationally Invariant extension, aptly regarded as RI-S3W.
% and consider its action on $\mathbb{S}^d$\

% Let $\mathrm{SO}(d+1)$ denote the special orthogonal group in $\mathbb{R}^{d+1}$. For a rotation $R \in \mathrm{SO}(d+1)$, we specify the rotated measure $R_\# \mu$ as the pushforward of $\mu$ under $R$. Given probability measures $\mu, \nu \in \mathcal{P}(\mathbb{S}^d)$, the $RI\text{-}S3W$ distance is defined by integrating the $S3W$ distance between the rotated measures over all rotations in $\mathrm{SO}(d+1)$. That is,

% \begin{align}
% RI\text{-}S3W_{\mathcal{G},p}(\mu,\nu) := \int_{\mathrm{SO}(d+1)} S3W_{\mathcal{G},p}(R_\#\mu, R_\#\nu) \, d\omega(R),
% \end{align}

% where $\omega$ denotes the Haar measure on $\mathrm{SO}(d+1)$.

\begin{theorem}\label{thm: ri}
$RI\text{-}S3W_{\mathcal{G},p}(\cdot,\cdot)$ and $RI\text{-}S3W_{\mathcal{H},p}(\cdot,\cdot)$  are well-defined.  Furthermore, $RI\text{-}S3W_{\mathcal{G},p}(\cdot,\cdot)$ is generally a pseudo-metric in $\mathcal{P}_p(\mathbb{S}^{d})$, i.e., it is non-negative, symmetric and satisfies triangular inequality. In addition, $RI\text{-}S3W_{\mathcal{H},p}(\cdot,\cdot)$ defines a metric in $\mathcal{P}_p(\mathbb{S}^{d})$. 
\end{theorem}

The proof is included in the appendix Section \ref{sec:ri-ssr}.

\section{Numerical Implementation Details}
% Based on the theoretical framework outlined in Section \ref{sec:method}, in what follows, we detail the numerical implementation of the $S3W$ and $RI\text{-}S3W$ distances.

%is a key component in our approach to map points from $\mathbb{S}^{d}$ to $\mathbb{R}^{d}$. 
\subsection*{Stereographic Projection (SP)} 
A key issue in implementing the SP for $S3W$ concerns the numerical handling of points near the North Pole. To ensure numerical stability, we introduce an \(\epsilon\)-cap around the North Pole, which serves to effectively establish an upper bound for the norm of the projected points. Specifically, any point where \(x_{d+1} > 1 - \epsilon\) is initially mapped to the circle \(x_{d+1} = 1 - \epsilon\), and then projected using SP. We refer to this modified SP as \(\phi_\epsilon\). We discuss the stability of this $\epsilon$-cap in the appendix Section \ref{sec:eps_stability}.

% we avoid mapping points in that region to extremely large values in $\mathbb{R}^d$. For implementation, we simply
% set $x_d = \min(x_d, 1-\epsilon)$, project it back onto $\mathbb{S}^d$, and then perform the Stereo Projection.
% $For implementation, we simply restrict values in the last dimension to ensure they do not exceed $(1-\epsilon)$ before performing the Stereo Projection, i.e. by setting $x_d = min(x_d, 1-\epsilon$) and then projecting it back onto $\mathbb{S}^d$. 

%We select $\epsilon=0.01$ but in practice it is a function of batch size. If the batch is very large then $\epsilon$ also needs to be larger to avoid overflow in loss calculation if data happens to concentrate at the pole.

\subsection*{The $S3W$ Distances} 
Let $\phi_{\epsilon}$ denote the Stereographic Projection operator that excludes the $\epsilon$-cap around $s_n\in\mathbb{S}^{d}$, and $h:\mathbb{R}^d\rightarrow\mathbb{R}^{d'}$ an injective defining function, $L$ the number of slices, and $\theta_l\in\mathbb{S}^{d'-1}$ a slicing direction. For simplicity, we initially consider two data distributions with an equal number of samples and uniform mass distribution (the general case involving different numbers of samples and non-uniform mass distribution is discussed in the appendix). 
Let $\hat{\mu}=\frac{1}{M}\Sigma_{m=1}^M\delta_{x_m}$ and $\hat{\nu}=\frac{1}{M}\Sigma_{m=1}^M\delta_{y_m}$ denote the empirical distributions, where $\delta_{x_m}$ denotes a Dirac measure centered at $x_m$. Then, $S3W_{\mathcal{H},p}^p(\hat{\mu},\hat{\nu})$ can be approximated via the following Monte Carlo estimator:
\begin{align*}
\frac{1}{L} \sum_{l=1}^L \sum_{m=1}^M \left| \langle h(\phi_{\epsilon}(x_{\pi_l[m]})), \theta_l \rangle - \langle h(\phi_{\epsilon}(y_{\pi'_l[m]})), \theta_l \rangle \right|^p
\end{align*}
where $\pi_l[m]$ and $\pi'_l[m]$ are the sorted indices of the projected samples on $\theta_l$ (see Algorithm \ref{alg:s3w} for the procedure).

% projected onto the generalized slices

\begin{algorithm}[tb]
   \caption{$S3W$}
   \label{alg:s3w}
    \begin{algorithmic}
       \STATE {\bfseries Input:} $\{x_i\}_{i=1}^M \sim \mu$, $\{y_j\}_{j=1}^M \sim \nu$, $L$ projections, \\~~~~~~~~~~~~~$p$-th order, $\epsilon$ for excluding the $\epsilon$-cap around $s_n$.
       \STATE Initialize: $h$ (injective map), $\{\theta_l\}_{l=1}^L$ (projections)
       % \STATE Compute  $\{ \hat{x}_{i} = \phi_{\epsilon}(x_i)\}$ and $\{ \hat{y}_{j} = \phi_{\epsilon}(y_j)\}$ 
       % \STATE Compute $\{h_{x_i} = h(\hat{x}_{i})\}$ and $\{h_{y_j} = h(\hat{x}_{i})\}$
       \STATE Compute $\{u_i=h(\phi_\epsilon(x_i))\}$ and $\{v_j=h(\phi_\epsilon(y_j))\}$
       \STATE Initialize distance $d = 0$
       \FOR{$l=1$ {\bfseries to} $L$}
        \STATE Compute ${u^l_{i}} = { \langle u_i, \theta_l\rangle }$, ${v^l_{j}} = { \langle v_j, \theta_l \rangle }$
        \STATE Sort $\{u^l_i\}$, $\{v_j^l\}$, s.t  $u^l_{\pi_l[i]} \leq u^l_{\pi_l[i+1]}$, $v^l_{\pi_l'[j]} \leq v^l_{\pi_l'[j+1]}$
        \STATE $d = d + \frac{1}{L}\sum_{i=1}^M |u^l_{\pi_l[i]} - v^l_{\pi_l'[i]}|^p$
       \ENDFOR
    \STATE \textbf{Return}~$d^{\frac{1}{p}}$
    \end{algorithmic}
    % \vspace{-.1in}
\end{algorithm}

\subsection*{The $RI\text{-}S3W$ Distances} 
The Monte Carlo approximation of $RI\text{-}S3W$ distances can be written as:
% \begin{align}
% RI\text{-}S3W_p(\hat{\mu}, \hat{\nu}) \approx \frac{1}{K} \sum_{k=1}^K S3W_p((R_k)_{\#}\hat{\mu},(R_k)_{\#}\hat{\nu}).
% \end{align}
% \vspace{-5mm} 
\begin{align*}
RI\text{-}S3W_p(\hat{\mu}, \hat{\nu}) \approx \frac{1}{N_R} \sum_{n=1}^{N_R} S3W_p((R_n)_{\#}\hat{\mu},(R_n)_{\#}\hat{\nu}),
\end{align*}
where $\{R_n\}_{n=1}^{N_R}\subset \mathrm{SO}(d+1)$ are random rotation matrices. To generate $R_n$, we adopt the GeoTorch library \cite{lezcano2019trivializations}, which provides a direct method to sample from $\mathrm{SO}(d+1)$. Our approach is efficient with vectorization and parallel processing on GPU. We note that generating $R_n$ is generally $\mathcal{O}(N_R \cdot d^3)$, which could become expensive for high-dimensional data or a large number of rotations. Instead, we could amortize this cost by presampling a rotation pool which could then be subsampled for every distance calculation. We denote this implementation as $ARI\text{-}S3W$, which involves a trade-off in memory and potential increased bias. In practice, we observe highly favorable performance in a variety of settings. We discuss this further in the appendix Section \ref{sec:numerical}.

%adding reasonable overhead for less than $20$ rotations

%When used as a loss for gradient descent optimization, it suffices to set the number of random rotations $N_R$ to be $1$. We provide an ablation study w.r.t $N_R$ in the experiment section, and theoretical analysis in the Supplementary Materials section \ref{sec:supp_ssr}.

% \vspace{-.1in}
\subsection*{Computational Complexity} 
Stereographic projection requires $\mathcal{O}(Nd)$ where $N=n+m$ is the total number of data points from the source and target, and $(d+1)$ is the data dimensionality. If $n \gg m$, then we let $N=n$, and the same analysis holds. Applying $h(\cdot)=h_1(\cdot)$ to the projected data also requires $\mathcal{O}(Nd)$. Thereafter, slicing the data is done in $\mathcal{O}(LNd)$, where $L$ is the number of projections; sorting is done in $\mathcal{O}(LN\log N)$; and finally, the distance calculation requires $\mathcal{O}(LN)$. The overall time complexity for $S3W$ is therefore $\mathcal{O}(LN(d + \log N))$. For $RI\text{-}S3W$, the cost of calculating $S3W$ for all rotations is $\mathcal{O}(N_R L N(d + \log N))$. Generating the random rotations takes an overhead of $\mathcal{O}(N_R d^3 )$ (which can be amortized), and applying these rotations to the data takes $\mathcal{O}(N_R N d^2)$. The total complexity is $\mathcal{O}(N_R (d^3+ N d^2+ LN(d + \log N))$. If we amortize generating the rotation matrices, then the per-operation cost becomes $\mathcal{O}(N_R N(d^2 + Ld + L \log N)).$

%$O(N_R \cdot N \cdot d^2) + O(N_R \cdot L \cdot N\cdot (d + \log N))$. 

% In practice, if we have $N \gg d$ and $N \gg N_R$ (i.e. low-dimensional data optimization with large batch size and few rotations), the complexity for $S3W$, RI-S3W, and Amortized $S3W$ is simply $O(LN \log N)$.

% For the $RI\text{-}S3W$ method, generating $N_R$ random rotation matrices requires $O(N_R \cdot d^3)$ time, and applying these rotations to the $N$ data points then requires $O(N_R \cdot N \cdot d^2)$. The $S3W$ procedure is then performed on each rotation, requiring $O(N_R \cdot L \cdot N\cdot (d + \log N))$. Therefore, assuming $N \gg d$ as often is the case, the overall time complexity of $RI\text{-}S3W$ is $O(N_R \cdot L \cdot N(d + \log N))$.

\subsection*{Runtime Comparison} 
We compare the runtime to compute different distances between the uniform distribution and a von Mises-Fisher distribution on $\mathbb{S}^{100}$. The results in Figure \ref{fig:runtime_main} are averaged over $50$ iterations for varying sample sizes of each distribution. We use $L=200$ projections for all sliced methods, $N_R=10$ rotations for $ARI\text{-}S3W_2$ and $RI\text{-}S3W_2$, and a pool size of $100$ for $ARI\text{-}S3W_2$. We do not include the time required to generate the rotation pool in our $ARI\text{-}S3W_2$ measurements.  Results for Wasserstein and Sinkhorn are based on the Python OT library \cite{flamary2021pot}.

\begin{figure}[H]
    \centering
    \includegraphics[width=0.8\columnwidth]{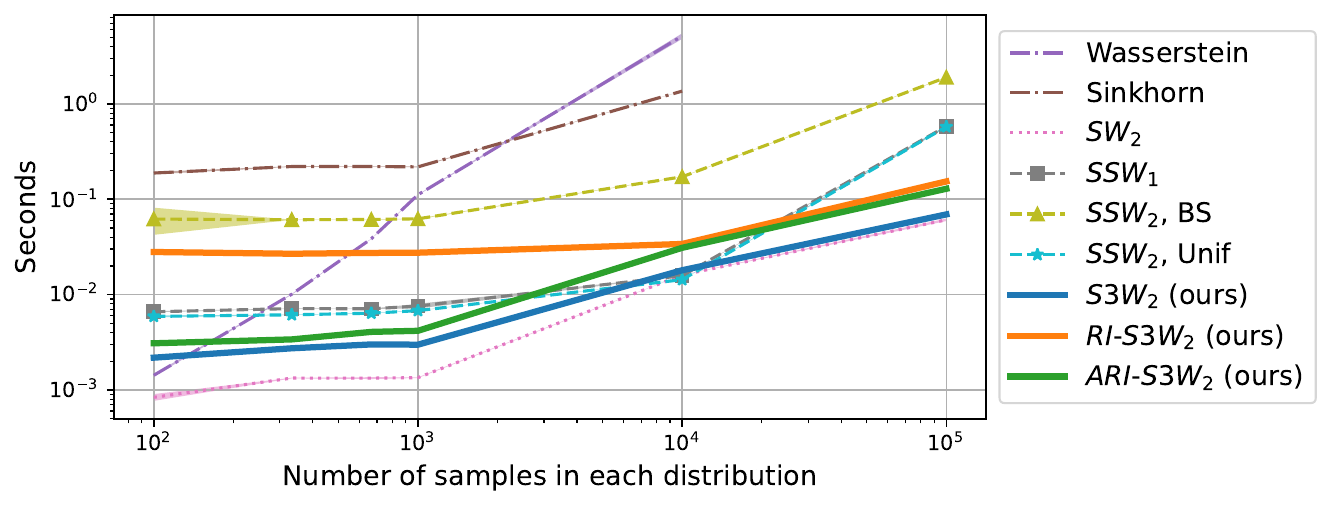}
    % \vspace{-.3in}
    \caption{Runtime comparison for Wasserstein distance, Sinkhorn distance \cite{cuturi2013sinkhorn} with geodesic distance as cost function, $SW_2$ (sliced Wasserstein) distance, $SSW_1$ distance (using level median formula) \cite{bonet2022spherical}, $SSW_2$ distance with binary search (BS) and antipodal closed form (for uniform distribution) \cite{bonet2022spherical}, $S3W_2$ distance (ours), $RI\text{-}S3W_2$ distance (ours), and $ARI\text{-}S3W_2$ distance (ours).}
    \label{fig:runtime_main}
    % \vspace*{-.3in}
\end{figure}

\section{Experiments}
Here, we present key results from our numerical experiments. We defer much of the visualization, discussions, and further results to the appendix Section \ref{sec:numerical}. All our experiments were executed on a Linux server with an AMD EPYC 7713 64-Core Processor, 8 $\times$ 32GB DIMM DDR4, 3200 MHz, and a NVIDIA RTX A6000 GPU.

\subsection{Gradient Flow On The Sphere}
\label{subsec:gf}
Similar to the work of \cite{bonet2022spherical}, we apply our proposed distances as a loss function for the gradient flow problem. We consider a challenging target probability measure $\nu$ with 12 von Mises-Fisher distributions (vMFs), and aim to solve $\text{argmin}_{\mu} S3W(\mu,\nu)$. Suppose we have access to the target measure only via i.i.d. samples $\{y_j\}_{j=1}^M$, i.e., $\hat{\nu}=\frac{1}{M}\sum_{j=1}^M \delta_{y_j}$ where $M=2400$. We initialize $2400$ particles uniformly sampled on $\mathbb{S}^2$, and directly optimize these particles with full-batch projected gradient descent on the surface of the sphere. Figure \ref{fig: gf_main} shows the converged loss curves after $500$ iterations and reports the runtime, negative log-likelihood (NLL), and the logarithm of the 2-Wasserstein distance between the distributions for $SSW$, $S3W$, $RI\text{-}S3W$ with $N_R\in\{1,5\}$, and $ARI\text{-}S3W$ with $N_R=30$ and pool size of $1000$. We observe that the proposed distances provide on-par or better performance while being significantly faster than $SSW$ (up to $20$X). Additionally, we provide mini-batch results in the appendix Section \ref{section:gf}. When the target is only known up to a constant, we use sliced-Wasserstein variational inference \cite{Yi2022SlicedWV} (see the appendix, Section \ref{section:vi}). 

\begin{figure}[H]
    \centering
    \subfloat{\includegraphics[width=0.7\columnwidth]{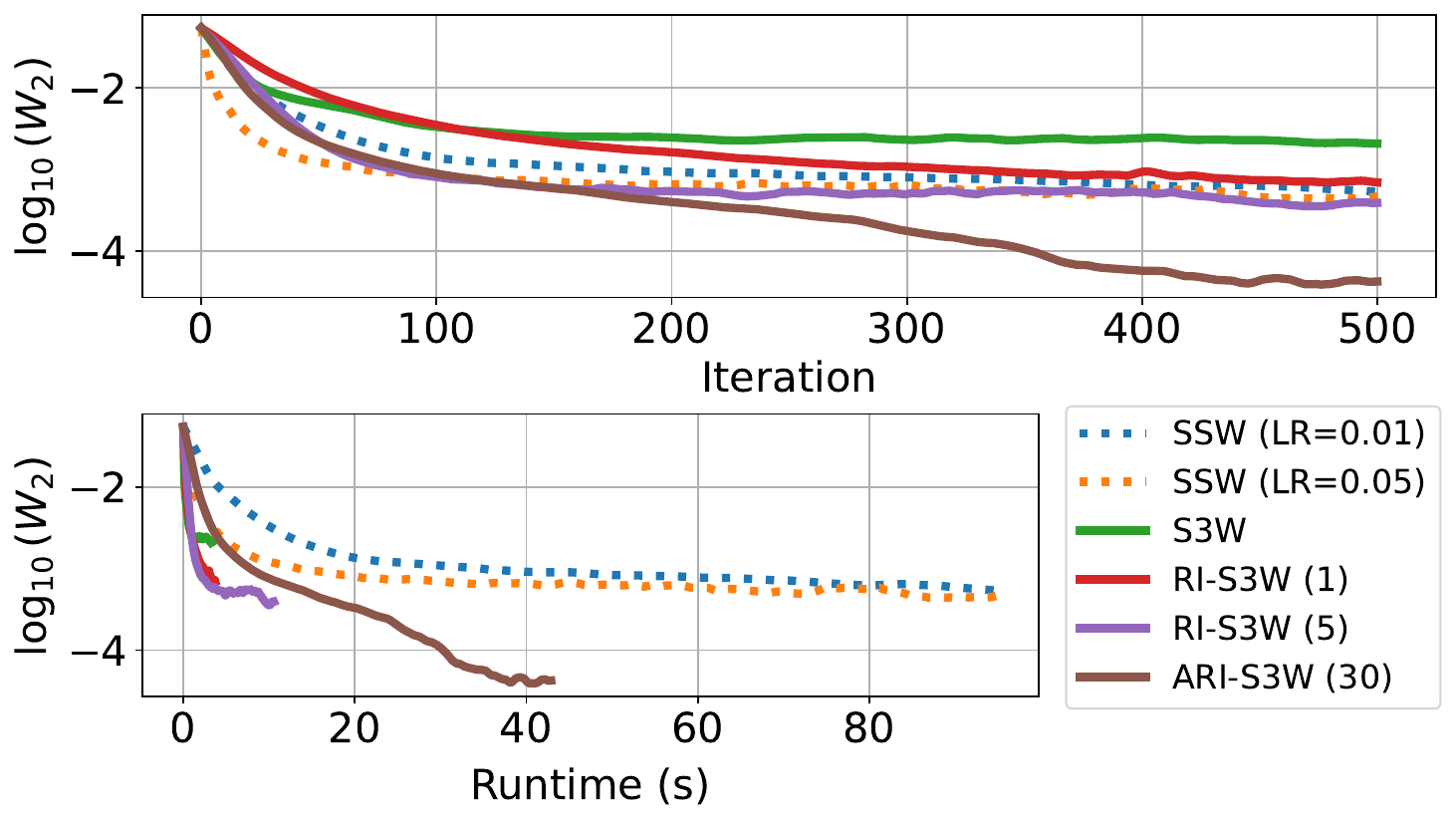}} \\
    {\footnotesize
    \begin{tabularx}{0.7\columnwidth}{|X|>{\centering\arraybackslash}X|>{\centering\arraybackslash}X|>{\centering\arraybackslash}X|}
    \hline
    \rowcolor[gray]{0.99}
    Method & Runtime(s) & NLL $\downarrow$ & $\log W_2$ $\downarrow$ \\
    \hline
    $SSW$ (LR=$0.01$) & 94.97 $\pm$ 0.96 & -4974.50 $\pm$ 2.27 & -3.27 $\pm$ 0.17 \\
    $SSW$ (LR=$0.05$) & 94.98 $\pm$ 0.27 & -4974.46 $\pm$ 3.50 & -3.32 $\pm$ 0.13 \\ 
    \hline
    $S3W$ & \textbf{3.32 $\pm$ 0.62} & -4753.81 $\pm$ 83.70 & -2.57 $\pm$ 0.19 \\
    $RI\text{-}S3W$ (1) & 4.00 $\pm$ 0.83 & -4957.22 $\pm$ 34.70 & -3.15 $\pm$ 0.26 \\
    $RI\text{-}S3W$ (5) & 10.58 $\pm$ 0.27 & -4983.56 $\pm$ 6.62 & -3.50 $\pm$ 0.17 \\
    $ARI\text{-}S3W$ (30) & 42.88 $\pm$ 0.05 & \textbf{-5025.37} $\pm$ 5.57 & \textbf{-4.37 $\pm$ 0.21} \\
    \hline
    \end{tabularx}
    }

    \caption[Convergence comparison]{Learning a mixture of $12$ vMFs. $ARI\text{-}S3W$ (30) has $30$ rotations, pool size of $1000$. $S3W$ variants use $\text{LR}=0.01$. $SSW$ has an additional $\text{LR}=0.05$ for better comparison. The plots show convergence of different distances w.r.t. iterations and runtime. The table summarizes numerical results for $10$ independent runs. We provide more details of the plots in the appendix Section \ref{subsec:evo_loss_curve}.}
    \label{fig: gf_main}
\end{figure}

\subsection{Self-Supervised Learning (SSL)}

We now show that our method can be an effective loss for contrastive SSL on the sphere. We adopt the contrastive objective proposed in \cite{wang2020understanding}, composed of alignment and uniformity loss terms, and replace the Gaussian kernel uniformity loss with an $S3W$-based loss:
\begin{equation*}
    \mathcal{L}= \frac{1}{n}\sum_{i=1}^n \lVert z^A_i - z^B_i \rVert^2_2 + \frac{\lambda}{2} \left(\text{S3W}_2(z^A, \nu) + \text{S3W}_2(z^B, \nu)\right),
\end{equation*}
where $z^A,z^B\in\mathbb R^{n\times (d+1)}$ are two encoded views of the same images, $\nu =\text{Unif}(\mathbb{S}^{d})$ is the uniform distribution on $\mathbb{S}^{d}$ and $\lambda>0$ is the regularization coefficient. We run our experiments on CIFAR-10 using a ResNet18 encoder. Figure \ref{fig:ssl_dim_3_main} visualizes the learned embeddings when $d=2$, and Table \ref{table:ssl_dim_10} assesses the quality of the learned embeddings for $d=9$ using the standard linear classifier evaluation. The details of these experiments are included in the appendix Section \ref{section:ssl}. We can see that the proposed metrics consistently lead to embeddings that are competitive both in terms of performance and runtime.

\begin{figure}[H]
    \centering
    \includegraphics[width=0.7 \columnwidth]{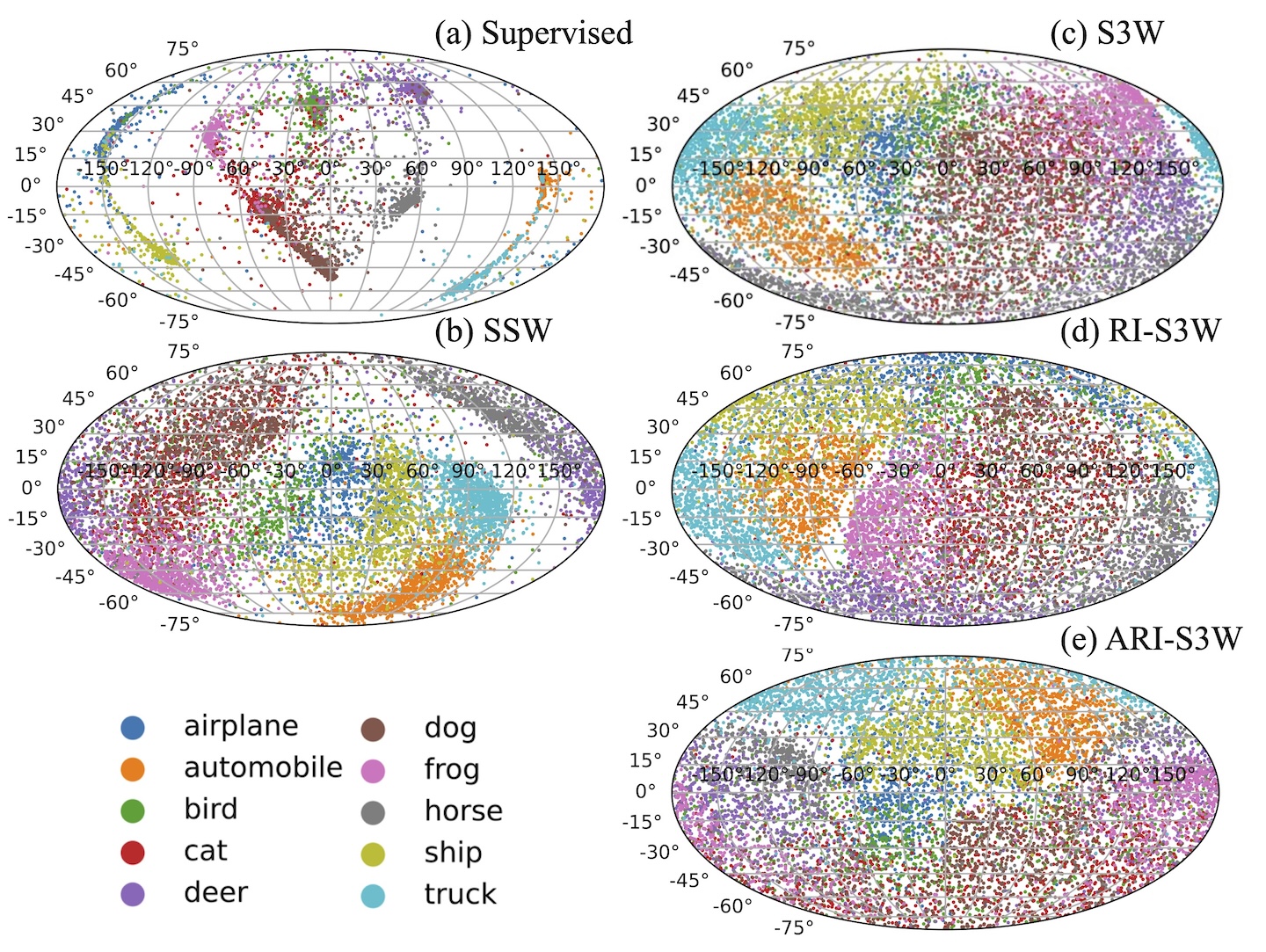}
    \caption{Projected features on $\mathbb{S}^2$ for CIFAR-10.}    \label{fig:ssl_dim_3_main}
\end{figure}

\begin{table}[H]
\centering
\footnotesize
\begin{tabularx}{0.7\columnwidth}{|>{\raggedright\arraybackslash}X|>{\centering\arraybackslash}X|>{\centering\arraybackslash}X|}
\hline
\rowcolor[gray]{0.99}
Method & Acc.(\%) E/P $\uparrow$ & Time(s/ep.) \\
\hline
Supervised & 92.38 / 91.77 & ----- \\
\hline \hline
Hypersphere \cite{wang2020understanding} & 79.76 / 74.57 & 24.28 \\
SimCLR \cite{chen2020simple} & 79.69 / 72.78 & \textbf{20.94} \\
\hline
$SSW$ \cite{bonet2022spherical} & 70.46 / 64.52 & 33.14 \\
$SW$ & 74.45 / 68.35 & 21.09 \\
\hline
$S3W$ & 78.54 / 73.84 & 21.36 \\
$RI\text{-}S3W$ $(5)$& \textbf{79.97} / 74.27 & 21.59 \\
$ARI\text{-}S3W$ $(5)$& 79.92 / \textbf{75.07} & 21.51 \\
\hline
\end{tabularx}
\caption{Standard linear evaluation on CIFAR-10. Accuracy is for encoded (E) features and projected (P) features on $\mathbb{S}^9$ (encoder feature dimension is 10). Time reported is per epoch of pretraining.}
\label{table:ssl_dim_10}
\end{table}

\subsection{Sliced-Wasserstein Autoencoder (SWAE)}
\label{subsec:ssw}

We now adopt the SWAE framework proposed by \cite{kolouri2018sliced} and demonstrate the application of $S3W$ distances in generative modeling. Let $\varphi: \mathcal{X} \rightarrow \mathbb{S}^{d}$ denote an encoder network and $\psi: \mathbb{S}^d \rightarrow \mathcal{X}$ denote a decoder network. The SWAE framework aims to enforce that the encoded data follow a prior distribution in the latent space. In our experiments, we use a mixture of vMF distributions with $10$ components on $\mathbb{S}^2$ as our prior distribution, which we denote as $q$. Then, the training objective for the modified SWAE is:
\begin{equation*}
\min_{\varphi, \psi}\mathbb{E}_{x \sim p} [c(x, \psi(\varphi (x))] + \lambda \cdot S3W(\varphi_{\#}p, q)
\end{equation*}
where $\lambda$ is the regularization coefficient, $c(\cdot, \cdot)$ the reconstruction loss, for which we use the standard Binary Cross Entropy (BCE) loss, and $p$ denotes the data distribution. Details about the network architectures and results on the MNIST benchmark can be found in the appendix Section \ref{section:swae}.

\begin{table}[H]
\centering
\setlength{\tabcolsep}{2pt}
\begin{tabular}{|p{3.0cm}|>{\centering\arraybackslash}p{2.5cm}|>{\centering\arraybackslash}p{2.5cm}|>{\centering\arraybackslash}p{2.5cm}|>
{\centering\arraybackslash}p{2.5cm}|}
\hline
\rowcolor[gray]{0.99}
Method & $\log W_2$ $\downarrow$ & NLL $\downarrow$ & BCE $\downarrow$ & Time (s/ep.)\\ 
\hline
Supervised & -0.5132 & 0.0060 & 0.6319 & \textbf{5.3243} \\ 
\hline\hline
$SSW$ & -2.1949 & 0.0052 & 0.6323 & 15.4651 \\ 
$SW$ & -3.3229 & -0.0007 & 0.6348 & 5.4661 \\ 
\hline
$S3W$ & -3.3381 & 0.0025 & \textbf{0.6318} & 5.7511 \\
$RI$-$S3W$ $(5)$ & -3.1424 & \textbf{-0.0043} & 0.6376 & 7.5443 \\
$ARI$-$S3W$ $(5)$ & \textbf{-3.3853} & 0.0028 & 0.6332 & 5.8316 \\
\hline
\end{tabular}
\caption{CIFAR-10 results for SWAE. We evaluate the latent regularization loss ($\log(W_2)$ and NLL), along with the BCE reconstruction loss on the test data.}
\label{table:main_swae}
\end{table}

\subsection{Earth Density Estimation} 
We extend our $S3W$ distances to the task of density estimation with normalizing flows on $\mathbb{S}^2$. Our focus is on the three datasets introduced by \cite{mathieu2020riemannian}, representing the Earth's surface as a perfect spherical manifold: Earthquake \cite{earthquakedataset}, Flood \cite{Brakenridge2017FloodArchive}, and Fire \cite{firedataset}. Similar to \cite{bonet2022spherical}, we use an exponential map normalizing flow model \cite{rezende2020normalizing} (see \ref{section:de}) optimizing $\min_T S3W(T_{\#}\nu, q)$, where $T$ is the transformation induced by the model, $\nu$ is the data distribution known via samples $\{y_j\}_{j=1}^M$, and $q$ is the prior distribution on $\mathbb{S}^2$. The learned density at $y\in\mathbb{S}^2$ can then be approximated as $f_{\nu}(y)=q(T(y))|\operatorname{det}J_T(y)|$ where $J_T(y)$ denotes the Jacobian of $T$ at $y$.

\begin{table}[ht]
\centering
\begin{tabular}{|p{3.0cm}|>{\centering\arraybackslash}p{2.5cm}|>{\centering\arraybackslash}p{2.5cm}|>{\centering\arraybackslash}p{2.5cm}|}
\hline
\rowcolor[gray]{0.99}
Method & Quake $\downarrow$ & Flood $\downarrow$ & Fire $\downarrow$ \\ \hline
Stereo & 2.04 ± 0.19 & 1.85 ± 0.03 & 1.34 ± 0.11 \\ \hline \hline
$SW$ & 1.12 ± 0.07 & 1.58 ± 0.02 & 0.55 ± 0.18 \\
$SSW$ & 0.84 ± 0.05 & 1.26 ± 0.03 & 0.24 ± 0.18 \\ \hline
$S3W$ & 0.88 ± 0.09  & 1.33 ± 0.05 & 0.36 ± 0.04 \\
$RI\text{-}S3W (1)$ & 0.79 ± 0.07 & 1.25 ± 0.02 & 0.15 ± 0.06 \\
$ARI\text{-}S3W (50)$ & \textbf{0.78 ± 0.06} & \textbf{1.24 ± 0.04} & \textbf{0.10 ± 0.04} \\
\hline
\end{tabular}
\caption{Earth datasets results. We use a pool size of $100$ for $ARI\text{-}S3W$. Stereo denotes the approach introduced by \cite{gemici2016normalizing} which involves stereographically projecting the sphere $\mathbb{S}^d$ onto $\mathbb{R}^{d}$ and then performing RealNVP. The results are compared with methods cited from \cite{dinh2016density}.}
\label{table:earth_density_nll}
% \vspace{-5mm}
\end{table}

\section{Conclusion}
We introduced a new class of sliced Wasserstein ($SW$) distances for spherical data using the stereographic projection (SP). We rigorously addressed the distortion issue caused by SP and presented several high-speed, high-performing variants of our approach. $S3W$ maps data to a generalized hypersurface with minimal distortion (via composing SP with a novel injective function) and efficiently computes $SW$ over an order of magnitude faster than existing baselines in many settings. $RI\text{-}S3W$ encodes rotation invariance into $S3W$, further boosting its performance. Although $RI\text{-}S3W$ is highly parallelizable, we achieved additional efficiency with our implementation of amortization (denoted as $ARI\text{-}S3W$). Given that our approach uses stereographic projection to map the hypersphere into Euclidean space, there may exist novel extensions of our slicing framework to unbalanced settings by leveraging recent advancements in unbalanced and partial OT on $\mathbb{R}$ \cite{bai2023sliced, sejourne2023unbalanced}. Moreover, our approach opens up interesting directions to bridging the spherical manifold and classical $SW$ literature via pushforwards to appropriate hypersurfaces. Lastly, we highlight recent advancements in improving the projection complexity of sliced Wasserstein distances \cite{deshpande2019max,nguyen2020distributional,nguyen2022amortized,nguyen2022hierarchical,nguyen2024energy,nguyen2024markovian}. These strategies are compatible with our proposed $S3W$ distances and can be integrated with our method to further enhance the projection efficiency of our metric in comparing high-dimensional spherical measures.

\section*{Acknowledgements}
This research was supported by the NSF CAREER Award No. 2339898. We are grateful to Clément Bonet for providing insights into the SSW implementation.

\clearpage
\bibliography{genesis}
\bibliographystyle{unsrt}

\newpage
\clearpage
\onecolumn
\appendix

\section{Notation}
\begin{itemize}
    \item $(\mathbb{R}^{d},\|\cdot\|)$: d-dimensional Euclidean space, where $\|\cdot\|$ is the Euclidean norm (or $2-$norm): given $x=[x_1,\dots,x_d]\in\mathbb{R}^d$, $ \|x\|=\sqrt{x_1^2+\dots+x_d^2}$.  Sometimes we will write $\|\cdot\|_2$ to stress that we are considering this ``$2-$norm''.
    \item $\langle \cdot,\cdot\rangle$: canonical inner product in Euclidean spaces.
    \item $L^p(\Omega)$, where $p\ge 1$ and $\Omega\subseteq \mathbb{R}^d$: functional space defined by 
    $$L^p(\Omega):=\{f:\Omega\to\mathbb{R}| \, \int_{\Omega}|f|^p<\infty\}$$
    endowed with the norm 
    $$\|f\|_p=\left(\int_{\Omega}|f|^p\right)^{\frac{1}{p}}.$$
    In the extreme case  $p=\infty$, we have 
    \begin{align}
    L^\infty(\Omega):=\{f:\Omega\to\mathbb{R}| \, \sup|f|<\infty\}.
    \end{align} 
    In our case, $\Omega$ will be $\mathbb{R}^d$, $\mathbb{R}\times\mathbb{S}^{d-1}$, $\mathbb{R}\times\mathbb{R}^{d'}\setminus\{0\}$ or $\mathbb{S}^{d}$.
 %   We have $L^p\subset L^{p+1}$  and $L^p\subset L^\infty ,\forall p\geq 1$.
    \item $C(X)$: space of real-valued continuous functions defined on the space $X$.
    \item $C_0(\mathbb{R}^d)$: set of continuous functions ``that vanish at infinity''.
    \item $C^\infty$: class of infinite differentiable functions. 
    \item $\mathbb{S}^{d-1}$: the unit sphere in $\mathbb{R}^d$, defined as $\mathbb{S}^{d-1}:=\{x\in\mathbb{R}^d: \|x\|^2=1\}.$
        \item $\sigma_{d}$: the uniform  probability measure defined in the sphere $\mathbb{S}^{d-1}\subset\mathbb{R}^d$.
    \item $s_n=[0,\dots,0,1]\in\mathbb{R}^{d+1}$ is the North Pole and $s_0=[0,\dots,0,-1]\in\mathbb{R}^{d+1}$ is the South Pole in $\mathbb{S}^{d}\subset\mathbb{R}^{d+1}$.
    \item $\phi$: stereographic projection (SP).
    \item $\phi_{\epsilon}$: stereographic projection operator whose domain excludes the $\epsilon$-cap around $s_n$ in the sphere $\mathbb{S}^{d}$.
    \item $v^T$: transpose vector. 
    \item \( H(t, \theta) = \{x\in \mathbb{R}^d \ |\ \langle x, \theta \rangle = t\} \) hyperplane.
    \item \( H_{t,\theta} = \{ x \in \mathbb{R}^d | \, g(x, \theta) = t \} \) ``level set'' of $g:\mathbb{R}^d\times \mathbb{S}^{d-1}\to \mathbb{R}$ at level $t$ with fixed spherical variable $\theta$.
    \item $\mathcal{R}(\cdot)$: Radon transform with $\mathcal{R}: L^1(\mathbb{R}^d)\to L^1(\mathbb{R}\times\mathbb{S}^{d-1})$. 
  \item   $\mathcal{G}(\cdot)$:  Generalized Radon Transform (GRT) with 
$$\mathcal{G}(f)(t,\theta)=\int_{\mathbb{R}^{d}}f(x)\delta(t-g(x,\theta)) \, dx.$$
  In this formulation, $d'\ge d$ in general, and $g: \mathbb{R}^d\times (\mathbb{R}^{d'}\setminus \{0\})\to \mathbb{R}$ is a function which satisfies the following:
    \begin{enumerate}
        \item[(H.1)] $g(x,\theta)$ is $C^\infty$ function on $\mathbb{R}^d\times(\mathbb{R}^{d'}\setminus\{0\})$. 
        \item[(H.2)] $g(x,\theta)$ is homogeneous of degree one in $\theta$, i.e., 
        $g(x,\lambda\theta)=\lambda g(x,\theta)$ for all $\lambda\in \mathbb{R}$.
        \item[(H.3)] $g(x,\theta)$ is non-degenerate with respect to $x$ in the sense that 
        $\forall (x,\theta)\in \mathbb{R}^d\times(\mathbb{R}^{d'}\setminus\{0\}), \, \frac{\partial g}{\partial x}(x,\theta)\neq 0$.
        \item[(H.4)] The mixed Hessian of $g$ is strictly positive, i.e., 
        $\textrm{det}\left(\frac{\partial^2 g}{\partial x^i\partial \theta^j}\right)>0$.  
     \end{enumerate}
    By property (H.2), $\mathcal{G}:L^1(\mathbb{R}^d)\to L(\mathbb{R}\times\mathbb{S}^{d'-1})$. 
    We refer \cite{beylkin1984inversion,denisyuk1994inversion,gel1969differential,ehrenpreis2003universality,kuchment2006generalized,homan2017injectivity,kolouri2019generalized} for more details. 
    
    \item $\mathcal{H}(\cdot)$: a variant and simplified version of the Generalized Radon Transform defined by $\mathcal{H}:L^1(\mathbb{R}^d)\to L^1(\mathbb{R}\times \mathbb{S}^{d'-1})$,
    $$\mathcal{H}(f)(t,\theta)=\int_{\mathbb{R}^d}f(x)\delta(t-\langle h(x),\theta\rangle )dx,$$
    where $h:\mathbb{R}^d\to \mathbb{R}^{d'}$ is injective. %In particular, we consider $h:\mathbb{R}^d\to \mathbb{R}^d$, and $\mathcal{H}:L^1(\mathbb{R}^d)\to L^1(\mathbb{R}\times \mathbb{S}^{d-1})$.
    \item $h_1$: function defined by \eqref{eq: h_1}.
    \item $(t,\theta)$: the first and second inputs for (generalized) Radon transform, where $t\in \mathbb{R}$ and variable $\theta$ lies in a pre-defined sphere: $\theta\in \mathbb{S}^{d-1}$ for the classic Radon transform ($\mathcal{R}$), and $\theta \in \mathbb{S}^{d'-1}$ for the generalized cases $\mathcal{G}$ and $\mathcal{H}$. 
    %In the setting of $\mathcal{H}$, we have $\theta\in \mathbb{S}^{d'-1}$. 
    \item $\mathcal{R}^*,\mathcal{G}^*,\mathcal{H}^*$: the Hermitian adjoint operators for $\mathcal{R}, \mathcal{G}, \mathcal{H}$. In particular, $\mathcal{R}^*:L^\infty(\mathbb{R}\times\mathbb{S}^{d-1})\to L^\infty(\mathbb{R}^d)$,  and $\mathcal{G}^*,\mathcal{H}^*:L^\infty(\mathbb{R}\times\mathbb{S}^{d'-1})\to L^\infty(\mathbb{R}^d)$ with 
    \begin{align}
      &\mathcal{R}^*(\psi)(x)=\int_{\mathbb{S}^{d-1}}\psi(\langle x,\theta\rangle,\theta)d\sigma_{d}(\theta) \nonumber \\ 
      &\mathcal{G}^*(\psi)(x)=\int_{\mathbb{S}^{d'-1}}\psi(g(x,\theta),\theta)d\sigma_{d'}(\theta) \nonumber \\
      &\mathcal{H}^*(\psi)(x)=\int_{\mathbb{S}^{d'-1}}\psi(\langle h(x),\theta\rangle,\theta)d\sigma_{d'}(\theta) \nonumber
    \end{align}

\item $\mathcal{M}(\Omega)$: Set of all real Radon measures (finite regular Borel measures not necessarily positive, that is, it includes signed measures) defined on $\Omega$. In this article $\Omega$ can be $\mathbb{R}^d$, $\mathbb{R}\times\mathbb{S}^{d-1}$, or $\mathbb{R}\times\mathbb{S}^{d'-1}$. 

$\mathcal{M}(\Omega)$ is endowed with the total variation norm
$$\|\mu\|_{TV}=\mu^+(\Omega)+\mu^-(\Omega)$$
where $\mu^{\pm}$ are the positive and negative parts of $\mu$.

\item $\mathcal{M}_+(\Omega)$: Set of all positive Radon measures defined on $\Omega$. It is endowed with the total variation norm
$\|\mu\|_{TV}=\mu(\Omega)$.

\item $\mathcal{P}(\Omega)$: Set of all probability measures defined on $\Omega$ (i.e., positive measures with $\|\mu\|_{TV}=\mu(\Omega)=1$). 
\item $\mathcal{P}_p(\Omega)$: Set of all probability measures defined on $\Omega$ with $p-$th finite moment, that is,
$$\mathcal{P}_p(\Omega)=\{\mu\in\mathcal{P}(\Omega): \, \int_{\Omega}\|x\|^p d\mu<\infty\}$$
We have the following relation: 
$$\mathcal{P}_p(\Omega)\subseteq\mathcal{P}(\Omega)\subseteq \mathcal{M}_+(\Omega)\subseteq \mathcal{M}(\Omega).$$
\item $T_\#\mu$: push-forward of the measure $\mu\in\mathcal{M}(X)$ by the measurable map $T:X\to Y$, which defines a measure on $\mathcal{M}(Y)$ such that $T_\#\mu(B)=\mu\left(\{x\in X:\, T(x)\in B\}\right)$ for every measurable set $B\subseteq Y$.
\item $W_p$: $p$-Wasserstein distance.
\item $\hat{\mu}=\frac{1}{M}\Sigma_{m=1}^M\delta_{x_m}$: empirical distribution.
\item $\text{Unif}(X),\sigma(X)$: Uniform distribution on a measure space $X$.
\item $F_\mu$, $F_\mu^{-1}$: cumulative distribution function (CDF) and quantile function of the measure $\mu$, respectively.
\item $\mathcal{S}_\mathcal{R}$, $\mathcal{S}_\mathcal{G}$, and $\mathcal{S}_\mathcal{H}$: Stereographic Spherical Radon Transforms (SSRT),
$$\mathcal{S}_\mathcal{R}(\mu)=\mathcal{R}(\phi_\#\mu), \qquad \mathcal{S}_\mathcal{G}(\mu)=\mathcal{G}(\phi_\#\mu),\qquad \mathcal{S}_\mathcal{H}(\mu)=\mathcal{H}(\phi_\#\mu). $$
\item $S3W_{\mathcal{R},p}$,  $S3W_{\mathcal{G},p}$, and $S3W_{\mathcal{H},p}$: Stereographic Spherical Sliced Wasserstein (S3W) distances
$$S3W_{\mathcal{R},p}^p(\mu,\nu)=\int_{\mathbb{S}^{d-1}}W_p^p(\mathcal{S}_\mathcal{R}(\mu)_\theta,\mathcal{S}_\mathcal{R}(\nu)_\theta) 
 \, d\sigma_{d}(\theta),$$
$$S3W_{\mathcal{G},p}^p(\mu,\nu)=\int_{\mathbb{S}^{d'-1}}W_p^p(\mathcal{S}_\mathcal{G}(\mu)_\theta,\mathcal{S}_\mathcal{G}(\nu)_\theta) 
 \, d\sigma_{d'}(\theta), $$
$$S3W_{\mathcal{H},p}^p(\mu,\nu)=\int_{\mathbb{S}^{d'-1}}W_p^p(\mathcal{S}_\mathcal{H}(\mu)_\theta,\mathcal{S}_\mathcal{H}(\nu)_\theta) 
 \, d\sigma_{d'}(\theta).$$
\item $\mathbb{E}$: expected value.
\item $\mathrm{O}(n)$, $\mathrm{SO}(n)$: orthogonal group (rotations and reflections) and special orthogonal group (rotations) of $n\times n$ real matrices, respectively.
\item $\omega$: Haar probability measure on the compact group $\mathrm{SO}(d+1)$.
\item $\approx$: approximately.
\item $x\sim \mu$: element $x$ sampled from the distribution $\mu$.
\item $\nabla_x$: gradient with respect to the variable $x$.
\item $RI-S3W_{\mathcal{G},p}$: Rotationally Invariant Extension of S3W defined by \eqref{eq: ri-s3w}.
\item $ARI-S3W$: Amortized Rotationally Invariant Extension of S3W used in our experiments. 
\item $d_{\mathbb{S}^d}(\cdot,\cdot)$: The great circle distance on the sphere  $\mathbb{S}^d$ defined by 
$d_{\mathbb{S}^d}(s_1,s_2)= \arccos( \langle s_1 , s_2 \rangle)$ for all $s_1,s_2\in\mathbb{S}^d$.
\item $\angle(s_1,s_2)$: the angle between the points $s_1$ and $s_2$ defined by $\angle(s_1,s_2)=\arccos(\langle s_1,s_2\rangle){\in[0,\pi]}$.
\end{itemize}

\section{Radon Transform for Radon Measures}
This section will review the classic Radon transform in the general Radon measure setting. In the next section, we introduce the Generalized Radon Transform (GRT) for general Radon measures. 

As we discussed in the main section, given a function $f\in L^1(\mathbb{R}^d)$, the Radon transform, $\mathcal{R}(f)$, is a function in $L^1(\mathbb{R}\times \mathbb{S}^{d-1})$. In particular, for each $(t,\theta)\in\mathbb{R}\times \mathbb{S}^{d-1}$, we have: 
\begin{align}
  \mathcal{R}(f)(t,\theta)&=\int_{\mathbb{R}^d}f(x)\delta(t-\langle x,\theta\rangle) \, dx\nonumber\\   
  &=\int_{\mathbb{R}^{d-1}}f(t\theta+U_\theta \xi) \, d\xi \nonumber
\end{align}
where $U_\theta\in\mathbb{R}^{d\times (d-1)}$ is any matrix such that its columns are formed by an orthonormal basis of $\theta^\perp$ (the subspace which is perpendicular to the vector $\theta$), see equation (22) in \cite{bonneel2015sliced}. It is straightforward to verify the above definition is well-defined (i.e. $\mathcal{R}(f)$ is independent of the choice of $U_\theta$). 

Notably, the Radon transform $\mathcal{R}:L^1(\mathbb{R}^d) \rightarrow L^1(\mathbb{R}\times \mathbb{S}^{d-1})$ is a linear bijection with a closed-form inversion formula. For a function \( f \in L^1(\mathbb{R}^d) \), the Radon transform \( \mathcal{R}f \) can be inverted using the inverse Radon transform, denoted by \( \mathcal{R}^{-1} \). The inversion formula is given by:
\begin{equation}
    f(x) = \mathcal{R}^{-1}(\mathcal{R}f)(x) = \int_{\mathbb{S}^{d-1}}  \big(\mathcal{R}f(\cdot,\theta)*\eta(\cdot)\big)(\langle x,\theta\rangle) \, d\sigma_{d}(\theta),
\end{equation}
where $\eta(\cdot)$ is a one-dimensional high-pass filter with corresponding Fourier transform $\mathcal{F}\eta(\omega) =  c|\omega|^{d-1}$, which appears due to the Fourier slice theorem \cite{helgason2011integral}, 
and `$*$' is the convolution operator. 
This integral formula effectively reconstructs \( f \) from its Radon transform and, in the medical imaging community, is referred to as `filtered back projection.'
%\rocio{I think the integral in the reconstruction formula is over the half sphere. I'll write my sketch later.}

Besides, the fact that the Radon transform $\mathcal{R}$ is a bounded linear operator from $L^1(\mathbb{R}^d)$ to 
$L^1(\mathbb{R}\times \mathbb{S}^{d-1})$, 
%Furthermore, since $\forall f_1,f_2\in L^1(\mathbb{R}^d)$, we have: 
%$$\langle f_1,f_2\rangle:=\int_{\mathbb{R}^d}f_1f_2 dx_1dx_2 \leq \sqrt{\|f_1\|_2\|f_2\|_2}\leq \sqrt{\|f_1\|_1\|f_2\|_1} <\infty,$$
%and similar to $L^1(\mathbb{R}\times \mathbb{S}^{d-1})$.
%Thus, the Euclidean inner product operator can be defined in $L^1(\mathbb{R}^d)$ and $L^1(\mathbb{S}^d)$. 
yields to the definition of the dual operator \footnote{Some references call the back projection as adjoint (aka Hermitian conjugate) operator e.g. \cite{bonneel2015sliced,unser2023ridges}.}, denoted as $\mathcal{R}^*$, which is generally called ``back projection'' \cite{helgason2011integral,rubin2015introduction}.
By definition of the dual operator and Riesz representation theorem that identifies the dual space of $L^1$ with $L^\infty$, for each $\psi\in L^\infty(\mathbb{R}\times\mathbb{S}^{d-1})$, we have 
\begin{align}
\int_{\mathbb{R}\times\mathbb{S}^{d-1}}\psi(t,\theta) \, \mathcal{R}(f)(t,\theta) \, dt \, d\sigma_{d}(\theta)=\int_{\mathbb{R}^d}f(x)\,  \mathcal{R}^*(\psi)(x) \, dx \label{eq: R and R*}.
\end{align}
Moreover, $\mathcal{R}^*:L^\infty(\mathbb{R}\times\mathbb{S}^{d-1})\to L^\infty(\mathbb{R}^d)$ has closed form: For each $x\in\mathbb{R}^d$ and $\psi\in L^\infty(\mathbb{R}\times \mathbb{S}^{d-1})$, 
\begin{align}
  \mathcal{R}^*(\psi)(x)=\int_{\mathbb{S}^{d-1}}\psi(\langle x,\theta\rangle,\theta) \, d\sigma_d(\theta), \label{eq: R*} 
\end{align}
where $\sigma_d$ is the uniform probability measure defined in the sphere $\mathbb{S}^{d-1}\subset \mathbb{R}^d$.

As discussed in the main section, the classic Radon transform can describe the distribution in projected 1D space. Indeed, suppose $f$ is the density of a probability measure $\mu$, and $\theta\in \mathbb{S}^{d-1}$, then function $\mathcal{R}(f)(\cdot,\theta): \mathbb{R}\to \mathbb{R}_{\geq 0}$ is exactly the density of the projected probability measure of $\mu$ into the space spanned by $\theta$. We denote this projection $\theta_\# \mu$ by using pushforward notation where we identify $\theta \in \mathbb{S}^{d-1}$ with the function 
\begin{gather}
    \theta: \mathbb{R}^d\to\mathbb{R}^d\nonumber \\
   \qquad  x\mapsto\langle x,\theta\rangle \, \theta, \label{eq: theta as function}
\end{gather}
which range lies of the line generated by the direction $\theta$.
Thus, we write
\begin{equation}\label{eq: radon and theta}
  \mathcal{R}(f)(\cdot,\theta)=\theta_\#\mu\in\mathcal{M}(\mathbb{R}).  
\end{equation}
However, when $\mu$ is not a continuous measure, continuous (e.g. empirical distribution), then $\theta_\#\mu$ does not admit a density function. 
To address this limitation and extend $\mathcal{R}$ to any king of finite measures, we follow a measure-theoretic approach of the Radon transform as in \cite{bonneel2015sliced} by using equation \eqref{eq: R and R*}. We refer the reader also to \cite{helgason2011integral}[Ch.1, Sec.5]). 

For this extension, we first recall the classical Riesz-Markov representation theorem and the lemma for the identity between Radon measures. 

\begin{theorem}[Riesz-Markov Representation theorem]\label{thm: Riesz}
Suppose $\Omega$ is a separable and locally compact space, %(e.g. $\Omega\subset (\mathbb{R}^d,\|\cdot\|)$), 
then  the dual space of the Banach space $(C_0(\Omega), \|\cdot\|_\infty)$, denoted as $(C_0(\Omega)^*,\|\cdot\|_*)$, is isomorphic to $(\mathcal{M}(\Omega),\|\cdot\|_{TV})$. In particular, for each bounded linear functional $\xi\in C_0(\Omega)^*$, there exists a unique $\nu\in \mathcal{M}(\Omega)$ such that 
$$\xi(\psi)=\int_{\Omega}\psi(x) \, d\nu(x), \qquad \text{ and } \qquad \|\xi\|_*=\|\nu\|_{TV}.$$

\end{theorem}
\begin{lemma} [Identity of Radon measures]
Given $\mu_1,\mu_2\in \mathcal{M}(\Omega)$ where $\Omega\subset \mathbb{R}^d$, the following are equivalent: 
\begin{enumerate}
    \item[(1)] $\nu_1=\nu_2$ 
    \item[(2)] $\int_{\Omega}\psi(x)\, d\nu_1(x)=\int_{\Omega}  \psi(x) \,d\nu_2(x),\qquad 
\forall \psi\in C_0(\Omega)$
\item[(3)] $\int_{\Omega}\psi(x)\ d\nu_1(x)=\int_\Omega \psi(x)\ d\nu_2(x),\qquad \forall \psi\in L^{\infty}(\Omega)$
\end{enumerate}
\end{lemma}
We will refer to the spaces $C_0(\Omega)$ and $L^{\infty}(\Omega)$ as the spaces of \textit{test functions} for $\mathcal{M}(\Omega)$.

\begin{proof}
If $\nu_1=\nu_2$, we directly have $(2),(3)$. 
Since $C_0(\Omega)\subset L^\infty(\Omega)$, then $(3)$ implies (2). All we need to prove is $(2)\Rightarrow (1)$. By the Riesz-Markov Representation theorem \ref{thm: Riesz}, $\mathcal{M}(\Omega)$ is isometric to $(C_0(\Omega)^*,\|\cdot\|_*)$. By the identity in the dual space, we have $\nu_1=\nu_2$ given (2). 
\end{proof}

Based on the above theorem, for $\mu\in\mathcal{M}(\mathbb{R}^d)$, its Radon transform  $\mathcal{R}(\mu)=\nu$ is defined as the measure $\nu\in \mathcal{M}(\mathbb{R}\times \mathbb{S}^{d-1})$ such that for each $\psi\in C_0(\mathbb{R}\times\mathbb{S}^{d-1})$, 
\begin{align}
\int_{\mathbb{R}\times \mathbb{S}^{d-1}}\psi(t,\theta) \, d\nu(t,\theta)=\int_{\mathbb{R}^d}(\mathcal{R}^*(\psi))(x) \, d\mu(x), \label{eq: R(mu) 2}
\end{align}
where the dual operator $\mathcal{R^*}$ is defined in \eqref{eq: R*}. Synthetically, one usually writes
$$\mathcal{R}(\mu)(\psi)=\mu(\mathcal{R}^*(\psi)).$$

%ADD STH SAYING THAT WE CAN REPLACE C0 WITH L INFTY.
Note that, by the above lemma, we can choose the $L^\infty(\mathbb{R}\times\mathbb{S}^{d-1})$ space of test functions. In particular, by setting $\psi=1_A$ for any Borel set $A\subset \mathbb{R}\times \mathbb{S}^{d-1}$, we can verify that $\nu$ is Radon measure if $\mu$ is Radon measure. Similarly, $\nu$ is a positive Radon measure if $\mu$ is positive. 
In addition, %by setting $\psi=1_{\Omega}$, 
one can obtain  $\|\nu\|_{TV}=\|\mu\|_{TV}$. Thus, if $\mu$ is a probability measure on $\mathbb{R}^d$, then
$\nu=\mathcal{R}(\mu)$ is a probability measure defined on $\mathbb{R}\times \mathbb{S}^{d-1}$. 

By the disintegration theorem in classic measure theory, there exists a $\nu-$a.s. unique set of measures $(\nu_\theta)_{\theta\in \mathbb{S}^{d-1}}\subset \mathcal{M}(\mathbb{R})$ such that for any $\psi \in C_0(\mathbb{R}\times\mathbb{S}^{d-1})$, we have 
\begin{align}
    \int_{\mathbb{R}\times\mathbb{S}^{d-1}}\psi(t,\theta) \, d\nu(t,\theta)=\int_{\mathbb{S}^{d-1}}\int_{\mathbb{R}}\psi(t,\theta) \, d\nu_\theta(t) \, d\sigma_{d}(\theta) \label{eq: nu_theta}.
\end{align}
Similar to \eqref{eq: radon and theta} for the classic Radon measure, given a fixed $\theta\in\mathbb{S}^{d-1}$, $\mathcal{R}(\mu)_\theta$ describes the projected measure of $\mu$ into the 1D space spanned by $\theta$ \cite{bonneel2015sliced}[Proposition 6], that is: 
$$\mathcal{R}(\mu)_\theta=\nu_\theta=\theta_\# \mu\in\mathcal{M}(\mathbb{R}).$$

\section{Generalized Radon Transform for Probability Measures}
Extending the Radon transform as in \eqref{eq: GRT} and \eqref{eq:augmented} is natural. In this section, we review those definitions. 

Given $ g: \mathbb{R}^d \times (\mathbb{R}^{d'} \setminus \{0\}) \to \mathbb{R}$, such that  $g$ satisfies the technical assumptions [H.1--H.4] listed in the Notation section and given in \cite {beylkin1984inversion} and \cite{kolouri2019generalized}, the Generalized Radon Transform (GRT) is defined as 
$$\mathcal{G}(f)(t,\theta)=\int_{\mathbb{R}^d}f(x)\delta(t-g(x,\theta)) \, dx, \qquad \forall (t,\theta)\in \mathbb{R}\times(\mathbb{R}^{d'}\setminus \{0\}).$$
Note, for each $(t,\theta)\in\mathbb{R}\times (\mathbb{R}^{d'}\setminus\{0\})$, we set the corresponded hyper-surface as $H_{t,\theta}=\{x\in \mathbb{R}^d\mid \,  g(x,\theta)=t\}$. By the homogeneity property (H.2), $H_{\lambda t,\lambda \theta}=H_{t,\theta}, \, \forall \lambda\in \mathbb{R}$.
Then, $$\mathcal{G}f(\lambda t,\lambda \theta)=\mathcal{G}f( t, \theta), \qquad \forall (t,\theta,\lambda)\in \mathbb{R}\times(\mathbb{R}^{d'}\setminus\{0\})\times\mathbb{R}.$$
Thus, the hypersurface can be parameterized by $(t,\theta)\in \mathbb{R}\times\mathbb{S}^{d-1}$, and we can %implicitly 
redefine 
$\mathcal{G}$ as a mapping from space $L^1(\mathbb{R}^d)$ to $L^{1}(\mathbb{R}\times\mathbb{S}^{d'-1})$. 

Furthermore, as $\mathcal{G}$ is an linear operator from $L^1(\mathbb{R}^d)$ to $L^1(\mathbb{R}\times\mathbb{S}^{d'-1})$, we can define the dual operator $\mathcal{G}^*$  of $\mathcal{G}$ as in  \cite{homan2017injectivity} (also called the \textit{adjoint} operator). In analogy with the Radon transform, this is done by the duality between $L^1$ and $L^\infty$: For each $\psi \in L^\infty(\mathbb{R}\times \mathbb{S}^{d'-1})$,
\begin{align}
    \mathcal{G}^*(\psi)(x)=\int_{ \mathbb{S}^{d'-1}}\psi(g(x,\theta),\theta) \, d\sigma_{d'}(\theta) \label{eq: GR* 2}, \qquad \forall x\in \mathbb{R}^d,
\end{align}
where $\sigma_{d'}$ is the uniform probability measure defined in $\mathbb{S}^{d'-1}$.

In the specific case, $g(x,\theta)=\langle h(x),\theta\rangle$, where $h: \mathbb{R}^d\to \mathbb{R}^{d'}$ is a injective function, the induced GRT becomes
$$
\mathcal{H}(f)(t,\theta)=\int_{\mathbb{S}^{d'-1}}f(x)\delta(t-\langle h(x),\theta\rangle) \, dx \label{eq: Hf}, \qquad \forall (t,\theta)\in\mathbb{R}\times \mathbb{S}^{d'-1},
$$
with corresponding dual operator
\begin{align}
\mathcal{H}^*(\psi)(x)=\int_{\mathbb{S}^{d'-1}}\psi(\langle h(x),\theta \rangle, \theta) \, d\sigma_{d'}(\theta), \label{eq: Hf*} \qquad \forall\psi\in L^\infty(\mathbb{R}\times \mathbb{S}^{d'-1}), \, x\in \mathbb{R}^d.
\end{align}
In this case, that is $g(x,\theta)=\langle h(x),\theta\rangle$ for $h=(h_1,\dots, h_{d'}):\mathbb{R}^d\to\mathbb{R}^{d'}$, properties [H.1-H.4] read as follows: First, notice that the real inner product in $\mathbb{R}^d$ defines, in each variable, a $C^\infty$ function which is homogeneous of degree one. In particular, (H.2) is trivially satisfied.  Then, $h$ must satisfy:
\begin{enumerate}
    \item[(a)] In place of (H.1): $h\in C^\infty$.
    \item[(b)] In place of (H.3): For all $(x,\theta)\in\mathbb{R}^d\times\mathbb{S}^{d'-1}$, 
    $\left(\langle \frac{\partial h}{\partial x^1}(x),\theta\rangle,\dots, \langle \frac{\partial h}{\partial x^d}(x),\theta\rangle\right)\not=0$.  
    \item[(c)] In place of (H.4): $\mathrm{det}\left(\left[\frac{\partial h_i}{\partial x^j}\right]_{1\leq i\leq 1, 1\leq j\leq d'}\right)>0$.    
\end{enumerate}
In particular, if $d=d'$, let $\mathbf{J}_h$ denote the Jacobian of $h=(h_1,\dots,h_d):\mathbb{R}^d\to\mathbb{R}^{d}$, that is, the $d'\times d$ matrix $\mathbf{J}_h=\left[\frac{\partial h_i}{\partial x^j}\right]_{ij}$. Then, condition (b) can be written as the matrix-vector multiplication $ \mathbf{J}_h(x)\cdot \theta\not=0$, and condition (c) is equivalent $\mathrm{det}(\mathbf{J}_h(x))>0$ for all $x\in\mathbb{R}^d$, and so in this case (b) is redundant. 
Therefore, when $d=d'$, $h:\mathbb{R}^d\to\mathbb{R}^d$ must satisfy:
\begin{equation*}
  h\in C^\infty \qquad \text{ and } \qquad \mathrm{det}\left(\mathbf{J}_h\right)>0.  
\end{equation*}

%\rocio{CASE $h$ INJECTIVE .... }

Inspired by the framework of classic Radon transform for Radon measure, we can extend the idea into the GRT setting. In particular,  given $\mu\in \mathcal{M}(\mathbb{R}^d)$, we define the measures $\mathcal{G}(\mu)$ and $\mathcal{H}(\mu)$ as follows: 
For each test function $\psi\in C_0(\mathbb{R}\times\mathbb{S}^{d'-1})$,
\begin{align}
&\int_{\mathbb{R}\times\mathbb{S}^{d'-1}} \psi (t,\theta) \, d\mathcal{G}(\mu)(t,\theta)=\int_{\mathbb{R}^d}\mathcal{G}^*(\psi)(x) \, d\mu(x)\label{eq: GR(mu)} \\ 
&\int_{\mathbb{R}\times\mathbb{S}^{d'-1}} \psi (t,\theta) \, d\mathcal{H}(\mu)(t,\theta)=\int_{\mathbb{R}^d}\mathcal{H}^*(\psi)(x) \,  d\mu(x) \label{eq: HR(mu)} 
\end{align}
%We have the following properties.
\begin{proposition}\label{pro: GRT}
Given $\mu\in \mathcal{M}(\mathbb{R}^d)$, the measures $\mathcal{G}(\mu),\mathcal{H}(\mu)$ is defined in \eqref{eq: GR(mu)} and \eqref{eq: HR(mu)} have the following properties:
\begin{enumerate}[]
    \item [(1)] $\mathcal{G}(\mu)$ is well-defined, i.e., $\mathcal{G}(\mu)\in \mathcal{M}(\mathbb{R}\times \mathbb{S}^{d'-1})$.
    \item [(2)] If $\mu \in \mathcal{M}_+(\mathbb{R}^d)$, then  $\mathcal{G}(\mu)\in \mathcal{M}_+(\mathbb{R}\times\mathbb{S}^{d'-1})$. Similarly, if $\mu\in\mathcal{P}(\mathbb{R}^d)$, then  $\mathcal{G}(\mu)\in\mathcal{P}(\mathbb{R}\times\mathbb{S}^{d'-1})$. 
    \item [(3)] $\mathcal{G}$ preserves the the total variation norm, i.e.,  $\|\mu\|_{TV}=\|\mathcal{G}(\mu)\|_{TV}$. In particular, $\mathcal{G}$ preserves mass. 
    \item [(4)] If $\mu\in\mathcal{P}(\mathbb{R}^d)$, then  $\mathcal{G}(\mu)\in\mathcal{P}(\mathbb{R}\times\mathbb{S}^{d'-1})$. 
    \item [(5)] 
    Given $\mu\in\mathcal{M}(\mathbb{R}^d)$, the disintegration theorem gives a unique $\mathcal{G}(\mu)-$a.s. set of measures $(\mathcal{G}(\mu)_\theta)_{\theta\in \mathbb{S}^{d'-1}}\subset \mathcal{M}(\mathbb{R})$ such that for any $\psi \in C_0(\mathbb{R}\times\mathbb{S}^{d'-1})$,
\begin{align*}
\int_{\mathbb{R}\times\mathbb{S}^{d'-1}}\psi(t,\theta) \, d\mathcal{G}(\mu)(t,\theta)=\int_{\mathbb{S}^{d'-1}}\int_{\mathbb{R}}\psi(t,\theta) \, d\mathcal{G}(\mu)_\theta(t) \, d\sigma_{d'}(\theta) %\label{eq: nu_theta}.
\end{align*}
Then, it holds that   
    %Fix $\theta\in \mathbb{S}^{d'-1}$,  we have 
    $$\mathcal{G}(\mu)_\theta=g(\cdot,\theta)_\# \mu.$$

    \item [(6)] Properties $(1)-(5)$ hold true for $\mathcal{H}$ in place of $\mathcal{G}$.
    \item [(7)]  Given $\mu\in\mathcal{M}(\mathbb{R}^d)$, the disintegration theorem gives a unique $\mathcal{H}(\mu)-$a.s. set of measures $(\mathcal{H}(\mu)_\theta)_{\theta\in \mathbb{S}^{d'-1}}\subset \mathcal{M}(\mathbb{R})$ such that for any $\psi \in C_0(\mathbb{R}\times\mathbb{S}^{d'-1})$,
\begin{align*}
\int_{\mathbb{R}\times\mathbb{S}^{d'-1}}\psi(t,\theta) \, d\mathcal{H}(\mu)(t,\theta)=\int_{\mathbb{S}^{d'-1}}\int_{\mathbb{R}}\psi(t,\theta) \, d\mathcal{H}(\mu)_\theta(t) \, d\sigma_{d'}(\theta) %\label{eq: nu_theta}.
\end{align*}
Then, it holds that   
    %Fix $\theta\in \mathbb{S}^{d'-1}$,  we have 
    $$\mathcal{H}(\mu)_\theta=\theta_\#h_\# \mu,$$
or equivalently, $\mathcal{H}(\mu)_\theta=(\theta\circ h) _\# \mu$, where the function $\theta$ is defined as in \eqref{eq: theta as function}.
    
    \item [(8)] $\mathcal{H}(\mu)=\mathcal{R}(h_\#\mu)$ for every $\mu\in\mathcal{M}(\mathbb{R}^d)$.
    \item [(9)] When $h$ is injective, $\mathcal{H}$ is invertible. %\rocio{We have to discuss this later.}
\end{enumerate}
\begin{proof}$ $
\begin{enumerate}
    \item[(1)] Let $\mu\in\mathcal{M}(\mathbb{R}^d)$. 
By \cite{beylkin1984inversion,homan2017injectivity}, we have that \eqref{eq: GR*} is the dual operator of $\mathcal{G}$. In fact, by functional analysis theory, $L^\infty$ is the dual space of Banach space $L^1$, and  thus, $\mathcal{G^*}: L^\infty (\mathbb{R}\times\mathbb{S}^{d'-1}) \to L^\infty(\mathbb{R}^d)$ is well-defined. In particular,  $\int_{\mathbb{R}^d}\mathcal{G}^*(\psi)(x)\, d\mu(x)$ is finite for all $\psi\in C_0(\mathbb{R}\times\mathbb{S}^{d'-1})\subset L^\infty(\mathbb{R}\times\mathbb{S}^{d'-1})$. Then, expression \eqref{eq: GR(mu)} is finite for all $\psi\in C_0(\mathbb{R}\times\mathbb{S}^{d'-1})$, and moreover
defines $\mathcal{G}(\mu)$ as an element in the dual space of $C_0(\mathbb{R}\times \mathbb{S}^{d'-1})$. By Theorem \ref{thm: Riesz}, we have $\mathcal{G}(\mu)\in\mathcal{M}(\mathbb{R}\times \mathbb{S}^{d'-1})$.

\item[(2)] Let $\mu\in\mathcal{M}_+(\mathbb{R}^d)$ and $\psi\in C_0(\mathbb{R}\times\mathbb{S}^{d'-1})$ with $\psi(t,\theta)\ge 0,\forall t,\theta$. 
Thus, by definition \eqref{eq: GR*}, 
$\mathcal{G}^*\psi(x)\ge 0, \forall x\in\mathbb{R}^d$. 
Therefore, 
\begin{align*}
  \int_{\mathbb{R}\times \mathbb{S}^{d'-1}}\psi(t,\theta) \, d\mathcal{G}\mu(t,\theta) =\int_{\mathbb{R}^d}\mathcal{G}^*(\psi)(x)  \, d\mu(x) \ge 0 
\end{align*}
which follows from the facts $\mathcal{G}^*(\psi)\ge 0$ and $\mu\in\mathcal{M}_+(\mathbb{R}^d)$. 
Since the above property holds for all non-negative test function $\psi\in C_0(\mathbb{R}\times \mathbb{S}^{d'-1})$, we obtain that $\mathcal{G}(\mu)\in \mathcal{M}_+(\mathbb{R}\times\mathbb{S}^{d'-1})$.

\item [(3)] Set $\psi=1_{\mathbb{R}\times \mathbb{S}^{d'-1}}\in L^\infty(\mathbb{R}\times \mathbb{S}^{d'-1})$. By definition of $\mathcal{G}^*$, we have 
$$\mathcal{G}^*(\psi)(x)=\int_{\mathbb{S}^{d'-1}}1 \, d\sigma_{d'}(\theta)=1, \qquad \forall x\in \mathbb{R}^d.$$
That is, $\mathcal{G}^*(1_{\mathbb{R}\times \mathbb{S}^{d'-1}})=1_{\mathbb{R}^d}$. Now, let $\mu\in\mathcal{M}(R^d)$. We have that
\begin{align}
\mu(\mathbb{R}^d) &=\int_{\mathbb{R}^d}1_{\mathbb{R}^d} \, d\mu=\int_{\mathbb{R}^d}\mathcal{G}^*(1_{\mathbb{R}\times\mathbb{S}^{d'-1}}) \, d\mu=\int_{\mathbb{R}\times\mathbb{S}^{d'-1}}1_{\mathbb{R}\times \mathbb{S}^{d'-1}} \, d\mathcal{G}(\mu) =\mathcal{G}(\mu)(\mathbb{R}\times \mathbb{S}^{d'-1}).
\end{align}
Let $\mu^+$ and $\mu^-$ are the positive and negative parts of $\mu$, that is, $\mu=\mu^+-\mu^-$, where $\mu^{\pm}\in\mathcal{M}_+(\mathbb{R}^d)$. 
By property (2) we have that $\mathcal{G}(\mu^\pm)\in\mathcal{M}_{+}(\mathbb{R}\times\mathbb{S}^{d'-1})$. By linearity of $\mathcal{G}$, we have $$\mathcal{G}(\mu)=\mathcal{G}(\mu^+-\mu^-)=\underbrace{\mathcal{G}(\mu^+)}_{\mathcal{M}_+}-\underbrace{\mathcal{G}(\mu^-)}_{\mathcal{M}_+},$$
which by the uniqueness of the Hahn-Jordan decomposition theorem implies that the positive and negative parts of the measure $\mathcal{G}(\mu)$ are $\mathcal{G}(\mu^+)$ and $\mathcal{G}(\mu^-)$, respectively. That is,
$$\mathcal{G}(\mu)^\pm=\mathcal{G}(\mu^\pm).$$
Then, 
\begin{align*}
    \|\mathcal{G}(\mu)\|_{TV}&=\mathcal{G}(\mu)^+(\mathbb{R}\times \mathbb{S}^{d'-1})+\mathcal{G}(\mu)^-(\mathbb{R}\times \mathbb{S}^{d'-1})\\
    &=\mathcal{G}(\mu^+)(\mathbb{R}\times \mathbb{S}^{d'-1})+\mathcal{G}(\mu^-)(\mathbb{R}\times \mathbb{S}^{d'-1})\\
    &=\mu^+(\mathbb{R}\times \mathbb{S}^{d'-1})+\mu^-(\mathbb{R}\times \mathbb{S}^{d'-1})=
\|\mu\|_{TV},
\end{align*}
that is, $\|\mathcal{G}(\mu)\|_{TV}=\|\mu\|_{TV}$. 

\item[(4)] If $\mu\in\mathcal{P}(\mathbb{R}^d)$, 
combining the property $\mathcal{G}(\mu)\in\mathcal{M}_+(\mathbb{R}\times \mathbb{S}^{d'-1})$ with the property $\|\mathcal{G}(\mu)\|_{TV}=\|\mu\|_{TV}=1$, we obtain that $\mathcal{G}(\mu)\in\mathcal{P}(\mathbb{R}\times \mathbb{S}^{d'-1})$.

\item [(5)] We will follow the ideas in \cite{bonneel2015sliced}[Proposition 6]. For any test function $\psi\in C_0(\mathbb{R}\times\mathbb{S}^{d'-1})$ we have 
\begin{align}
\int_{\mathbb{S}^{d'-1}}\int_{\mathbb{R}}\psi(t,\theta) \ d\mathcal{G}(\mu)_\theta(t) \, d\sigma_{d'}(\theta)
&=\int_{\mathbb{R}\times\mathbb{S}^{d'-1}}\psi(t,\theta) \ d\mathcal{G}(\mu)(t,\theta) \nonumber \qquad (\textit{disintegration theorem})\\
&=\int_{\mathbb{R}^d}\mathcal{G}^*(\psi)(x) \, d\mu(x)\nonumber \qquad (\textit{definition of } \mathcal{G}(\mu))\\ 
&=\int_{\mathbb{R}^d}\int_{\mathbb{S}^{d'-1}}\psi(g(x,\theta),\theta) \, d\sigma_{d'}(\theta) \, d\mu(x)\nonumber \qquad (\textit{definition of } \mathcal{G}^*)\\
&=\int_{\mathbb{S}^{d'-1}}\int_{\mathbb{R}^d}\psi(g(x,\theta),\theta) \, d\mu(x) \, d\sigma_{d'}( \theta) \nonumber \\ 
&=\int_{\mathbb{S}^{d'-1}}\int_{\mathbb{R}}\psi(t,\theta) \, d(g(\cdot,\theta)_\# \mu)(t) \, d\sigma_{d'}(\theta), \nonumber 
\end{align}
where the fourth equality follows from Fubini theorem and the fact $\psi$ is bounded and $\mu,\sigma_{d'}$ are finite measures; in the fifth equality, $t=g(x,\theta)$ and we use the definition of pushforward. %thus $d\mu(x)=d(g(\cdot,\theta)_\#\mu)(t)$. 
By uniqueness of the disintegration of measures we have $\mathcal{G}(\mu)_\theta=g(\cdot,\theta)_\# \mu$. This also can be shown by choosing an arbitrary
$\psi_0\in C_0(\mathbb{R})$, defining $\psi(t,\theta):=\psi_0(t), \, \forall (t,\theta)\in\mathbb{R}\times\mathbb{S}^{d'-1}$, so $\psi\in C_0(\mathbb{R}\times\mathbb{S}^{d'-1})$, and therefore we have
\begin{align}
\int_{\mathbb{R}}\psi_0(t) 
\, d\mathcal{G}(\mu)_\theta (t)
&=\int_{\mathbb{S}^{d'-1}}\int_{\mathbb{R}}\psi(t,\theta) \, d\mathcal{G}(\mu)_\theta (t) \, d\sigma_{d'}(\theta) \nonumber \\ 
&=\int_{\mathbb{S}^{d'-1}}\int_{\mathbb{R}}\psi(t,\theta) \, d(g(\cdot,\theta)_\# \mu)(t) \ d\sigma_{d'}(\theta)\nonumber\\ 
&=\int_{\mathbb{R}}\psi_0(t) \, d(g(\cdot,\theta)_\#\mu)(t).\nonumber 
\end{align}

% It holds for all $\psi\in C_0(\mathbb{R}\times\mathbb{S}^{d-1})$. 
% Choose $\psi_0\in C_0(\mathbb{R})$ and define $\psi$ with $\psi(x,\theta)=\psi_0(x),\forall x$. Thus $\phi\in C_0(\mathbb{R}\times\mathbb{S}^{d-1})$ and we have 
% \begin{align}
% &\int_{\mathbb{R}}\psi_0(t) 
% d\mathcal{G}(\mu)_\theta (t)\nonumber\\ 
% &=\int_{\mathbb{S}^{d'-1}}\int_{\mathbb{R}}\psi(t,\theta) d\mathcal{G}(\mu)_\theta (t)d\sigma_{d'}(\theta) \nonumber \\ 
% &=\int_{\mathbb{S}^{d'-1}}\int_{\mathbb{R}}\psi(t,\theta) d(g(\cdot,\theta)_\# \mu)d\sigma_{d'}(\theta)\nonumber\\ 
% &=\int_{\mathbb{R}}\psi_0(t)d(g(\cdot,\theta)_\#\mu)(t)\nonumber 
% \end{align}
% Therefore, 
% $\mathcal{G}(\mu)_\theta=g(\cdot,\theta)_\# \mu$. 

\item[(6)] For $\mathcal{H}$, properties $(1)-(5)$ can be derived similarly. 

\item[(7)] As in $(5)$, we will follow the ideas in \cite{bonneel2015sliced}[Proposition 6]. For any test function $\psi\in C_0(\mathbb{R}\times\mathbb{S}^{d'-1})$ we have 
\begin{align}
\int_{\mathbb{S}^{d'-1}}\int_{\mathbb{R}}\psi(t,\theta) \ d\mathcal{H}(\mu)_\theta(t) \, d\sigma_{d'}(\theta)
&=\int_{\mathbb{R}\times\mathbb{S}^{d'-1}}\psi(t,\theta) \ d\mathcal{H}(\mu)(t,\theta) \nonumber \qquad (\textit{disintegration theorem})\\
&=\int_{\mathbb{R}^d}\mathcal{H}^*(\psi)(x) \, d\mu(x)\nonumber \qquad (\textit{definition of } \mathcal{H}(\mu))\\ 
&=\int_{\mathbb{R}^d}\int_{\mathbb{S}^{d'-1}}\psi(g(x,\theta),\theta) \, d\sigma_{d'}(\theta) \, d\mu(x)\nonumber \qquad (\textit{definition of } \mathcal{H}^*)\\
&=\int_{\mathbb{S}^{d'-1}}\int_{\mathbb{R}^d}\psi(\langle h(x),\theta\rangle,\theta) \, d\mu(x) \, d\sigma_{d'}( \theta) \nonumber \qquad (\textit{Fubini's Theorem})\\ 
&=\int_{\mathbb{S}^{d'-1}}\int_{\mathbb{R}}\psi(\langle y,\theta\rangle,\theta) \, d(h_\# \mu)(y) \, d\sigma_{d'}(\theta) \nonumber \qquad (\textit{definition of pushforward})\\
&=\int_{\mathbb{S}^{d'-1}}\int_{\mathbb{R}}\psi(t,\theta) \, d(\theta_\#h_\# \mu)(t) \, d\sigma_{d'}(\theta) \nonumber \qquad (\textit{definition of pushforward}). 
\end{align}
Thus, by the uniqueness of the disintegration of measures, we have $\mathcal{H}(\mu)_\theta=\theta_\#h_\# \mu=(\theta\circ h)_\#\mu$. 

\item[(8)] Let $\mu\in \mathcal{M}(\mathbb{R}^d)$, and let  $\psi\in C_0(\mathbb{R}\times \mathbb{S}^{d'-1})$ be any test function. Following the ideas from $(7)$, we have
\begin{align}
\int_{\mathbb{R}\times\mathbb{S}^{d'-1}}\psi(t,\theta) \, d\mathcal{H}(\mu)(t,\theta)
&=\int_{\mathbb{R}^d}\mathcal{H}^*(\psi)(x) \, d\mu(x)\nonumber \\ 
&=\int_{\mathbb{R}^d}\int_{\mathbb{S}^{d'-1}}\psi(\langle h(x),\theta\rangle ,\theta) \, d\sigma_{d'}(\theta) \, d\mu(x)\nonumber\\
&=\int_{\mathbb{R}^d}\int_{\mathbb{S}^{d'-1}}\psi(\langle y,\theta\rangle,\theta) \, d\sigma_{d'}(\theta) \, dh_\#\mu(y)\nonumber \\ 
&=\int_{\mathbb{R}^d}(\mathcal{R}^*(\psi))(y) \, dh_\#\mu(y)\nonumber\\ 
&=\int_{\mathbb{R}\times \mathbb{S}^{d'-1}}\psi(t,\theta) \, d\mathcal{R}(h_\# \mu)(t,\theta) \nonumber ,
\end{align}
where in the third equality $y=h(x)$ and we use the definition of pushforward; %and thus $d\mu(x)=dh_\# \mu(y)$; 
the first and fifth equation follows from the equations \eqref{eq: GR(mu)} and \eqref{eq: R(mu)}, respectively.
Therefore, $\mathcal{H}(\mu)=\mathcal{R}(h_\#\mu)$. 

\item [(9)] By Proposition 7 in \cite{bonneel2015sliced}, we have the mapping
\begin{align}
%\mathcal{M}(\mathbb{R}^d)&\to \mathcal{M}(\mathbb{R}\times\mathbb{S}^{d'-1}) \\ 
h_\# (\mu)&\mapsto \mathcal{R}(h_\#(\mu))=\mathcal{H}(\mu)  
\end{align}
is invertible. If $h$ is injective, we have $\mu \mapsto h_\# (\mu)$ is invertible. Thus, $\mu\mapsto \mathcal{H}(\mu)$ is invertible. 
%?? I do not use other properties of $g$. \rocio{Yes, I'm confused. Let's discuss this later.}
\end{enumerate}

\end{proof}

\end{proposition}

\section{Stereographic Spherical Radon Transform and Stereographic Spherical Sliced Wasserstein Distance for Probability Measures}\label{sec:supp_ssr}
Based on the extended definition of GRT for probability measures, in this section, we formally introduce the ``Stereographic Spherical Radon Transform'' and the corresponding sliced Wasserstein distance, ``Stereographic Spherical Sliced Wasserstein distance'' for general probability measures.

Given $\mu\in \mathcal{M}(\mathbb{S}^d)$ and the stereographic projection $\phi:\mathbb{S}^{d}\setminus \{s_n\}\to \mathbb{R}^d$, which is a bijection, we have that $\phi_\#\mu$ is a measure defined in $\mathbb{R}^d$. The \textbf{generalized Stereopgraphic Spherical Radon transform}, 
 of $\mu$ has two versions $\mathcal{S}_\mathcal{G}(\mu)$ and $\mathcal{S}_\mathcal{H}(\mu)$ defined as follows:
 \begin{equation}\label{eq: def of S}
      \mathcal{S}_\mathcal{G}(\mu):=\mathcal{G}(\phi_\#\mu)\qquad \text{ and } \qquad \mathcal{S}_\mathcal{H}(\mu):= \mathcal{H}(\phi_\#\mu).
 \end{equation}
 That is, from \eqref{eq: GR(mu)} and \eqref{eq: HR(mu)}, we have that for each test function $\psi\in C_0(\mathbb{R}\times \mathbb{S}^{d'-1})$
\begin{align}
&\int_{\mathbb{R}\times \mathbb{S}^{d'-1}}\psi(t,\theta) \, d\mathcal{S}_\mathcal{G}(\mu)(t,\theta)= \int_{\mathbb{R}^d}\mathcal{G}^*(\psi)(x) \, d\phi_\#\mu(x)=\int_{\mathbb{S}^{d}\setminus\{s_n\}}\mathcal{G}^*(\psi)(\phi(s))d\mu(s) \label{eq: GSSRT mu}, \\   
&\int_{\mathbb{R}\times \mathbb{S}^{d'-1}}\psi(t,\theta) \, d\mathcal{S}_\mathcal{H}(\mu)(t,\theta)= \int_{\mathbb{R}^d}\mathcal{H}^*(\psi)(x) \, d\phi_\#\mu(x)=\int_{\mathbb{S}^d\setminus\{s_n\}}\mathcal{H}^*(\psi)(\phi(s)) \, d\mu(s) \label{eq: HSSRT mu}.  
\end{align}

%It is straightforward to verify: 
\begin{proposition}\label{pro: SSRT}
%Given $\theta\in \mathbb{S}^{d'-1}$, and $\mu\in \mathcal{P}(\mathbb{S}^{d}\setminus\{s_n\})$, let  $\mathcal{S}_\mathcal{G}(\mu),\mathcal{S}_\mathcal{H}(\mu)$ be the Stereographic Spherical Radon transform of $\mu$, as defined in \eqref{eq: GSSRT mu} and \eqref{eq: HSSRT mu}, then we have the following: 
Let $\mu\in \mathcal{M}(\mathbb{S}^{d})$ that does not give mass to the North Pole $\{s_n\}$. The 
Stereographic Spherical Radon transforms defined in \eqref{eq: def of S} have the following properties: 
\begin{enumerate}
    \item[(1)]   
    $\mathcal{S}_\mathcal{G}(\mu),\mathcal{S}_\mathcal{H}(\mu)\in\mathcal{M}(\mathbb{R}\times \mathbb{S}^{d'-1})$. 
    In addition $\mathcal{S}_\mathcal{G}$ and $\mathcal{S}_\mathcal{H}$ preserves mass, and if  $\mu$ is a positive measure, then $\mathcal{S}_\mathcal{G}(\mu),\mathcal{S}_\mathcal{H}(\mu)$ are positive measures too. Finally, if $\mu\in \mathcal{P}(\mathbb{S}^{d}\setminus\{s_n\})$, then  
    $\mathcal{S}_\mathcal{G}(\mu),\mathcal{S}_\mathcal{H}(\mu)$ are probability measures defined on $\mathbb{R}\times \mathbb{S}^{d'-1}$.
    \item[(2)] The disintegration theorem gives a unique $\mathcal{S}_\mathcal{G}(\mu)-$a.s. set of measures $(\mathcal{S}_\mathcal{G}(\mu)_\theta)_{\theta\in \mathbb{S}^{d'-1}}\subset \mathcal{M}(\mathbb{R})$, and a unique $\mathcal{S}_\mathcal{H}(\mu)-$a.s. set of measures $(\mathcal{S}_\mathcal{H}(\mu)_\theta)_{\theta\in \mathbb{S}^{d'-1}}\subset \mathcal{M}(\mathbb{R})$  such that for any $\psi \in C_0(\mathbb{R}\times\mathbb{S}^{d'-1})$,
    \begin{align}
\int_{\mathbb{R}\times\mathbb{S}^{d'-1}}\psi(t,\theta) \, d\mathcal{S}_\mathcal{G}(\mu)(t,\theta)&=\int_{\mathbb{S}^{d'-1}}\int_{\mathbb{R}}\psi(t,\theta) \, d\mathcal{S}_\mathcal{G}(\mu)_\theta(t) \, d\sigma_{d'}(\theta), \qquad \text{and} \nonumber\\ 
\int_{\mathbb{R}\times\mathbb{S}^{d'-1}}\psi(t,\theta) \, d\mathcal{S}_\mathcal{H}(\mu)(t,\theta)&=\int_{\mathbb{S}^{d'-1}}\int_{\mathbb{R}}\psi(t,\theta) \, d\mathcal{S}_\mathcal{H}(\mu)_\theta(t) \, d\sigma_{d'}(\theta). \label{eq: disintegration Sh}
\end{align}
    Then, it holds that
    \begin{equation}
      \mathcal{S}_\mathcal{G}(\mu)_\theta=(g(\cdot,\theta)\circ \phi)_\#\mu, \qquad \text{ and } \qquad \mathcal{S}_\mathcal{H}(\mu)_\theta=%\langle(h\circ \phi)(\cdot),\theta\rangle_\# \mu
    (\theta\circ h\circ \phi)_\# \mu,  
    \end{equation}
    where we recall that the function $\theta$ is defined as in \eqref{eq: theta as function}  %$\langle,\theta\rangle_\#$
    \item[(3)] %The above properties holds for $\mathcal{S}_\mathcal{H}$ similarly. 
    %In addition, 
    $\mathcal{S}_\mathcal{H}$ is invertible. 
\end{enumerate}
\end{proposition}
\begin{proof}$ $
\begin{enumerate}
    \item[(1)] If $\mu\in \mathcal{M}(\mathbb{S}^{d})$ is such that does not give mass to $\{s_n\}$, then $\phi_\# \mu$ is a  measure defined in $\mathbb{R}^{d}$ with the same total mass. Thus, since $\mathcal{S}_\mathcal{G}(\mu)=\mathcal{G}(\phi_\#\mu)$, 
    the conclusion follows from by (1)--(4) in Proposition \eqref{pro: GRT}. Similar to $\mathcal{S}_\mathcal{H}(\mu)$.
    
    %If $\mu\in \mathcal{P}(\mathbb{S}^{d}\setminus\{s_n\})$ is a probability measure, then $\phi_\# \mu$ is a probability measure defined in $\mathbb{R}^{d}$. Thus $\mathcal{S}_\mathcal{G}(\mu)=\mathcal{G}(\phi_\#\mu)$ is a probability measure defined in $\mathbb{R}\times\mathbb{S}^{d'-1}$ by  (3) in Proposition \eqref{pro: GRT}. Similar to $\mathcal{S}_\mathcal{H}(\mu)$.

    \item[(2)] By definition of $\mathcal{S}_\mathcal{G}$ and (5) in Proposition \eqref{pro: GRT} we have 
$$\mathcal{S}_\mathcal{G}(\mu)_\theta=g(\cdot,\theta)_\# (\phi_\#\mu)=g(\phi(\cdot),\theta)_\#\mu.%(g(\cdot,\theta)\circ \phi)_\#\mu.
$$  
Similarly, by definition of $\mathcal{S}_\mathcal{H}$ and (7) in Proposition \eqref{pro: GRT} we have
$$\mathcal{S}_\mathcal{H}(\mu)_\theta=(\theta\circ h)_\#\phi_\# \mu =(\theta\circ h\circ \phi)_\# \mu.$$
% $$\mathcal{S}_\mathcal{H}(\mu)_\theta=\langle h(\cdot) ,\theta\rangle _\# (\phi_\#\mu)=(\langle h(\cdot),\theta\rangle \circ \phi)_\#\mu=\langle  (h\circ \phi) (\cdot),\theta\rangle_\# \mu .$$

    \item[(3)] Finally, since the stenographic projection $\phi:\mathbb{S}^{d}\setminus\{s_n\}\to\mathbb{R}^d$ is invertible, then following composition 
\begin{align*}
 \mathcal{M}(\mathbb{S}^d\setminus\{s_n\})&\to \mathcal{M}(\mathbb{R}^d)\to \mathcal{M}(\mathbb{R}\times\mathbb{S}^{d'-1}) \nonumber\\
\mu&\longmapsto \phi_\#\mu\ \ \longmapsto \mathcal{H}(\phi_\# \mu)=\mathcal{S}_\mathcal{H}(\mu)
\end{align*}
is invertible as each mapping is invertible.
Similarly,
\begin{align*}
 \mathcal{P}(\mathbb{S}^d\setminus\{s_n\})&\to \mathcal{P}(\mathbb{R}^d)\to \mathcal{M}(\mathbb{R}\times\mathbb{S}^{d'-1}) \nonumber\\
\mu&\longmapsto \phi_\#\mu\ \ \longmapsto \mathcal{H}(\phi_\# \mu)=\mathcal{S}_\mathcal{H}(\mu)
\end{align*}
is invertible.
Thus $\mu\mapsto \mathcal{S}_\mathcal{H}(\mu)$ is an invertible mapping either general measures or for probability measures. 
\end{enumerate}
\end{proof}

\begin{remark}
    If $g(x,\theta)=\langle x,\theta\rangle$, then the generalized Radon transform $\mathcal{G}$ coincides with the classical Radon transform $\mathcal{R}$, and the corresponding \textit{Stereographic Spherical Radon transform} coincides with the one defined in \eqref{eq: SR}:
    $$\mathcal{S}_\mathcal{R}(\mu)=\mathcal{R}(\phi_\#\mu)$$
for any Radon measure $\mu\in \mathcal{M}(\mathbb{S}^{d})$ that does not assign mass to $\{s_n\}$
\end{remark}

We introduce the notation $\phi_g^\theta, \phi_h^\theta:\mathbb{S}^d\backslash\{s_n\}\rightarrow \mathbb{R}^{d'}$ 
\begin{equation*}
    \phi_g^\theta(z):=g(\phi(z),\theta)\theta,\qquad \phi_h^\theta(z):=\langle h(\phi(z) ,\theta\rangle\theta, 
\end{equation*} 
%\rocio{the image is a line in the direction of $\theta$, then we discuss how we write it}
having range in the line generated by $\theta\in\mathbb{S}^{d'-1}$.
The expressions \eqref{eq: slice general Sg and Sh} in (2) in the above Proposition can be written as
\begin{equation*}
    \mathcal{S}_\mathcal{G}(\mu)_\theta = {\phi_g^\theta}_\# \mu, \qquad   \mathcal{S}_\mathcal{H}(\mu)_\theta = {\phi_h^\theta}_\# \mu
\end{equation*}
and we call them \textit{slices} of the SSRT of $\mu$. 

Given two measures $\mu_1,\mu_2\in \mathcal{P}(\mathbb{S}^{d})$, we introduce the formal definition of Stereographic Spherical Sliced Wasserstein (S3W) distance as follows: 
\begin{align}
S3W_{\mathcal{G},p}^p(\mu_1,\mu_2)&:=\int_{\mathbb{S}^{d'-1}}W_p^p(\mathcal{S}_\mathcal{G}(\mu_1)_\theta,\mathcal{S}_\mathcal{G}(\mu_2)_\theta) 
 \, d\sigma_{d'}(\theta)\nonumber\\ 
&=\int_{\mathbb{S}^{d'-1}}W_p^p({\phi_g^\theta}_\#  \mu_1,{\phi_g^\theta}_\# \mu_2) 
 \, d\sigma_{d'}(\theta) \label{eq: $S3W$ G}\\
S3W_{\mathcal{H},p}^p(\mu_1,\mu_2)
&:=\int_{\mathbb{S}^{d'-1}}W_p^p(\mathcal{S}_\mathcal{H}(\mu_1)_\theta,\mathcal{S}_\mathcal{H}(\mu_2)_\theta) \,  d\sigma_{d'}(\theta)\nonumber\\ 
&=\int_{\mathbb{S}^{d'}}W_p^p( {\phi_h^\theta}_\#  \mu_1, {\phi_h^\theta}_\# \mu_2) \, d\sigma_{d'}(\theta) \label{eq: $S3W$ H},
\end{align}
%where $\alpha\in\mathcal{P}(\mathbb{S}^{d'-1})$. 
where $\sigma_{d'}\in\mathcal{P}(\mathbb{S}^{d'-1})$ is the uniform measure on $\mathbb{S}^{d'-1}$. 

%Based on the SSRT transform, given two probability measures $\mu_1,\mu_2\in \mathcal{P}(\mathbb{S}^{d})$, we introduce the formal definition of Stereographic Spherical Sliced Wasserstein distance (S3WD) as follows: 

% \begin{align}
% S3W_{\mathcal{G},p}^p(\mu_1,\mu_2)&:=\int_{\mathbb{S}^{d'-1}}W_p^p(\mathcal{S}_\mathcal{G}(\mu_1)_\theta,\mathcal{S}_\mathcal{G}(\mu_2)_\theta) d\alpha(\theta)\nonumber\\ 
% &=\int_{\mathbb{S}^{d'-1}}W_p^p((g(\cdot,\theta)\circ \phi)_\# \mu_1,(g(\cdot,\theta)\circ \phi )_\# \mu_2) d\alpha(\theta) \label{eq: $S3W$ G}\\
% S3W_{\mathcal{H},p}^p(\mu_1,\mu_2)
% &:=\int_{\mathbb{S}^{d'-1}}W_p^p(\mathcal{S}_\mathcal{H}(\mu_1)_\theta,\mathcal{S}_\mathcal{H}(\mu_2)_\theta) d\alpha(\theta)\nonumber\\ 
% &=\int_{\mathbb{S}^{d'}}W_p^p\langle (h\circ \phi)(\cdot),\theta\rangle_\# \mu_1,\langle (h\circ \phi)(\cdot),\theta\rangle_\#\mu_2) d\sigma(\theta) \label{eq: $S3W$ H},
% \end{align}
% where $\alpha\in\mathcal{P}(\mathbb{S}^{d'-1})$. 

%\rocio{Later, I'll recall the def of Sliced Wass and say that the above is a generalization}

\begin{theorem}
\label{thm:s3w}
The above definitions $S3W_{\mathcal{G},p}(\cdot,\cdot)$ and $S3W_{\mathcal{H},p}(\cdot,\cdot)$ are well-defined. 
Furthermore, $S3W_{\mathcal{G},p}(\cdot,\cdot)$ is a pseudo-metric in $\mathcal{P}(\mathbb{S}^{d}\setminus \{s_n\})$, i.e., it is non-negative, symmetric and satisfies triangular inequality. 
In addition, %if $\mathrm{supp}(\alpha)=\mathbb{S}^{d'-1}$, 
$S3W_{\mathcal{H},p}(\cdot,\cdot)$ defines a metric in $\mathcal{P}(\mathbb{S}^{d}\setminus\{s_n\})$. 
\end{theorem}
\begin{proof}
Given probability measures $\mu_1,\mu_2\in\mathcal{P}(\mathbb{S}^{d}\setminus\{s_n\})$, by (1) in Proposition \ref{pro: SSRT} we have that $\mathcal{S}_\mathcal{G}(\mu_1),\mathcal{S}_\mathcal{G}(\mu_2)$ are probability measures defined on $\mathbb{R}\times\mathbb{S}^{d'-1}$. Thus, by (2) in Proposition \ref{pro: SSRT}, %for each $\theta$, 
$\mathcal{S}_\mathcal{G}(\mu_1)_\theta,\mathcal{S}_\mathcal{G}(\mu_2)_\theta$ are probability measure on $\mathbb{R}$. Then,  
$W_p^p(\mathcal{S}_\mathcal{G}(\mu_1)_\theta,\mathcal{S}_\mathcal{G}(\mu_2)_\theta)$ is well-defined and thus 
$S3W_{\mathcal{G},p}^p(\cdot,\cdot)$ is well defined. Analogously, $S3W_{\mathcal{H},p}^p(\cdot,\cdot)$ is well-defined.

Since the Wasserstein distance $W_p(\cdot,\cdot)$ in $\mathcal{P}(\mathbb{R})$ is non-negative and symmetric and $\sigma_{d'}$ is a positive measure
%and $\alpha$ is a positive measure,  %$W_p^p(\mathcal{S}_\mathcal{G}(\mu_1)_\theta,\mathcal{S}_\mathcal{G}(\mu_2)_\theta)$ is non-negative and symmetric, 
we have that $S3W_{\mathcal{G},p}(\cdot,\cdot)$ and $S3W_{\mathcal{H},p}(\cdot,\cdot)$ are non-negative and symmetric. 

For triangular inequality, let $\mu_1,\mu_2,\mu_3\in\mathcal{P}(\mathbb{S}^{d}\setminus\{s_n\})$,  
\begin{align}
S3W_{\mathcal{G},p}(\mu_1,\mu_3)
&=\left(\int_{\mathbb{S}^{d'-1}}W_p^p(\mathcal{S}_\mathcal{G}(\mu_1)_\theta,\mathcal{S}_\mathcal{G}(\mu_3)_\theta)d\sigma_{d'}(\theta)\right)^{1/p}\nonumber\\ 
&\leq \left(\int_{\mathbb{S}^{d'-1}}\left(W_p(\mathcal{S}_\mathcal{G}(\mu_1)_\theta,\mathcal{S}_\mathcal{G}(\mu_2)_\theta)+W_p(\mathcal{S}_\mathcal{G}(\mu_2)_\theta,\mathcal{S}_\mathcal{G}(\mu_3)_\theta)\right)^{p}d\sigma_{d'}(\theta)\right)^{1/p}\nonumber\\
&\leq \left(\int_{\mathbb{S}^{d'-1}}W_p^p(\mathcal{S}_\mathcal{G}(\mu_1)_\theta,\mathcal{S}_\mathcal{G}(\mu_2)_\theta)\right)^{1/p}+\left(\int_{\mathbb{S}^{d'-1}}W_p^p(\mathcal{S}_\mathcal{G}(\mu_2)_\theta,\mathcal{S}_\mathcal{G}(\mu_3)_\theta)\right)^{1/p}\nonumber\\ 
&=S3W_{\mathcal{G},p}(\mu_1,\mu_2)+S3W_{\mathcal{G},p}(\mu_2,\mu_3) \nonumber 
\end{align}
where the first inequality holds from the fact $W_p(\cdot,\cdot)$ in $\mathcal{P}(\mathbb{R})$ is a metric; the second inequality follows from Minkowski inequality in $L^p(\mathbb{S}^{d'-1})$.
Then, $S3W_{\mathcal{G},p}(\cdot,\cdot)$ is a Pseudo-metric. Analogously for $S3W_{\mathcal{H},p}(\cdot,\cdot)$. 

In addition, suppose $S3W_{\mathcal{H},p}(\mu_1,\mu_2)=0$, we have 
%for each $\theta\in\mathbb{S}^{d'-1}$, 
$\mathcal{S}_\mathcal{H}(\mu_1)_\theta=\mathcal{S}_\mathcal{H}(\mu_2)_\theta$
where the equality holds $\sigma_{d'}(\theta)-$a.s. Then, from \eqref{eq: disintegration Sh} we can deduce that $\mathcal{S}_\mathcal{H}(\mu_1)=\mathcal{S}_\mathcal{H}(\mu_2)$. Since $\mathcal{S}_\mathcal{H}$ is invertible by Proposition \eqref{pro: SSRT}, we have $\mu_1=\mu_2$. Thus, $S3W_{\mathcal{H},p}(\cdot,\cdot)$ is a metric. 

\end{proof}

% \begin{theorem}\label{thm: ri}
% $S3W_{\mathcal{G},p}(\cdot,\cdot)$ and $S3W_{\mathcal{H},p}(\cdot,\cdot)$  are well-defined. Importantly, both define a metric in $\mathcal{P}_p(\mathbb{S}^d)$. In other words, they are non-negative, symmetric, and satisfy the triangle inequality as well as the identity of the indiscernibles.
% \end{theorem}

% \begin{proof}

% \end{proof}

Lastly, and similar to the max-sliced Wasserstein distance, here we define the max-S3W distance as: 
\begin{align}    \text{Max-}S3W_{\mathcal{G},p}^p(\mu_1,\mu_2) := \sup_{\theta\in \mathbb{S}^{d'-1}} ~~W_p^p(\mathcal{S}_\mathcal{G}(\mu_1)_\theta,\mathcal{S}_\mathcal{G}(\mu_2)_\theta),
\end{align}
and similarly for $\mathcal{H}$ we define $\text{Max-}S3W_{\mathcal{H},p}^p$.

\section{Rotationally Invariant Stereographic Spherical Sliced Wasserstein Distances (RI-S3W)}
\label{sec:ri-ssr}
In this section, we discuss the Rotationally Invariant $S3W$ distance \eqref{eq: ri-s3w}: 
$$RI\text{-}S3W_{\mathcal{G},p}(\mu,\nu) := \mathbb{E}_{R\sim \omega} [S3W_{\mathcal{G},p}(R_\#\mu, R_\#\nu)].$$

First, we describe the \textit{uniform (Haar) distribution} in $\mathrm{O}(d+1)$ denoted as $\omega$. Precisely, we will describe how a $(d+1)\times (d+1)$ orthonormal matrix $R$ with rows $v_1,\dots,v_{d+1}$ can be randomly generated as a sample of $\omega$. The procedure relies on the Gram-Schmidt algorithm: 

First, we randomly select $v_1\in\mathbb{S}^d$, in particular, we denote it as $v_1\sim \sigma_{d+1}$, that is, we sample $v_1$ according to $\sigma_{d+1}$ which is the uniform distribution in $\mathbb{S}^d$. 

Next, $v_2$ should selected from  $\text{span}(v_1)^\perp=H(v_1,0)=\{v: \langle v_1, v\rangle=0\}$ and satisfy $\|v_2\|=1$, where $\text{span}(v_1)$ is the 1D sub-space in $\mathbb{R}^d$ spanned by $v_1$. Denote  $\mathbb{S}^{d-1}(v_1):=\mathbb{S}^{d}\cap H(v_1,0)$. Note, it is essentially a $d-1$ sphere. We randomly select  $v_2\in \mathbb{S}^{d-1}(v_1)$, denoted as $v_2\sim \sigma_{d}$ given $v_1$, where $\sigma_{d}$ denotes the uniform distribution in sphere $\mathbb{S}^{d-1}(v_1)$. 

In the third step, similarly, let $\mathbb{S}^{d-2}(v_1,v_2):=\mathbb{S}^{d}\cap H(v_1,0)\cap H(v_2,0)$, and sample $v_3\sim \sigma_{d-1}$ given $v_1,v_2$, 
where $\sigma_{d-1}$ is the uniform distribution in the $d-2$ sphere $\mathbb{S}^{d-2}(v_1,v_2)$. 

Continue this process recursively until step $d+1$. Note that at the last iteration, we have an orthonormal set of vectors  $\{v_1,\ldots, v_d\}$, and $\mathbb{S}^{0}(v_1,\ldots v_d)$ is a $0$-dimensional sphere, that is, it is a set of two points $\{u_1,u_2\}$. We randomly select $v_{d+1}$ as one of those two points in $\mathbb{S}^{0}(v_1,\ldots v_d)$, denoting $v_{d+1}\sim \sigma_{1}$, where $\sigma_1$ is the uniform distribution in $\mathbb{S}^{0}(v_1,\ldots, v_d)$ (that is, $\sigma_1=0.5\delta_{u_{1}}+0.5\delta_{u_2}$).

Thus, the uniform probability measure $\omega$ in $\mathrm{O}(d+1)$ can be described in the following way: Consider an arbitrary Borel set $A\subset \mathrm{O}(d+1)$, and let $A_i:=\{v\in \mathbb{S}^{d}: \, v\text{ is the $i-$th row of }R, \,  R\in A\}$ for $i=1,2,\ldots d+1$, then
\begin{align}
\omega(A):=\int_{v_1\in A_1,\ldots,v_{d+1}\in A_{d+1}} d\sigma_{1}(v_{d+1}) \ d\sigma_{2}(v_{d}) \ldots d\sigma_{d}(v_2) \,  d\sigma_{d+1}(v_1).
 % \omega(A)=\int_{v_1\in A_1,\ldots,v_d\in A_d} d\sigma_{\mathbb{S}^0(v_1,\ldots v_d)}(v_{d+1})\ldots d\sigma_{\mathbb{S}^d}(v_1).
 \label{eq: omega}
\end{align}
where we are using the above notation, that is,   $\sigma_{d+1}$ is the uniform distribution on $\mathbb{S}^d$, and $\sigma_{d+1-k}$ is the uniform distribution on $\mathbb{S}^d(v_1,\dots, v_k)$ for $|\leq k\leq d$
(i.e., $\sigma_{d}$ is the uniform distribution on $\mathbb{S}^d(v_1)$, $\ldots$ , $\sigma_{2}$ is the uniform distribution on $\mathbb{S}^d(v_1,\dots, v_{d-1})$, $\sigma_{1}$ is the uniform distribution on $\mathbb{S}^d(v_1,\dots, v_{d})$). 
In particular, $\omega$ is constant, in the sense that its density function is constant:
Indeed, let $f_{\sigma_{d+1}},\dots,f_{\sigma_2}$ denote the density functions of $\sigma_{d+1},\dots, \sigma_2$ (which are all constant equal to the reciprocal of the ``surface area'' of the corresponding sphere). Then, let $f_{\omega}$ denote the density function of probability measure $\omega$, and so, for each  $R=[v_1^T,v_2^T,\ldots, v_{d+1}^T]\in \mathrm{O}(d+1)$
\begin{align}
f_\omega(R)=f_{\sigma_{d+1}}(v_1)\cdot \dots\cdot f_{\sigma_2}(v_d)\cdot{\sigma_1}(v_{d+1})=\frac{1}{2}\prod_{i=1}^d \frac{1}{\|\mathbb{S}^{i}\|_i}
%\begin{cases}
%     f_{\sigma_{d+1}}(v_1)\cdot \dots\cdot f_{\sigma_2}(v_d)\cdot{\sigma_1}(v_{d+1})=\frac{1}{2}\prod_{i=1}^d \frac{1}{\|\mathbb{S}^{i}\|_i} &\text{if }R\in \mathrm{O}(d+1) \\
%     0  &\text{elsewhere}
% \end{cases} \label{eq: f_omega}
\end{align}
Thus we call $\omega$ uniform distribution. We mention that it is also called a Haar measure for the special orthogonal group $\mathrm{SO}(d+1)$. 
Finally, we can normalize it so that we have a probability distribution. 

Next, we discuss the proof of Theorem \ref{thm: ri}. We first introduce the following lemma. 
\begin{lemma}\label{lem: omega}
Consider $s_n=[0,\dots,0,1]\in\mathbb{R}^{d+1}$ the North Pole of  $\mathbb{S}^d$. Let $s_1\in\mathbb{S}^d$ and   $A=\{R\in \mathrm{O}(d+1): \, Rs_1=s_n\}$. Then $A=\{R\in \mathrm{O}(d+1): \text{ the $(d+1)-$th row of $R$ is } s_1\}$. Furthermore, $\omega(A)=0$. 
\end{lemma}
\begin{proof}
Let $B=\{R\in \mathrm{O}(d+1): \, \text{the $(d+1)-$th row of $R$ is } s_1\}$. It is straightforward to verify $B\subset A$. For the other direction, for any $R=[v_1^T,\ldots, v_{d+1}^T]\in A$ we have
\begin{align}
 v_i^Ts_1=0 &\text{ if }i\neq d+1 \nonumber\\
 v_{d+1}^Ts_{1}=1  &\text{ if }i=d+1  \nonumber 
\end{align}
Thus, $Rs_1=s_n$, and we have $A\subset B$.  

Next, we will show $\omega(A)=0$. Let $A'=\{R\in \mathrm{O}(d+1): \text{ 1st row of $R$ is }s_1\}$. We have $\omega(A')=\omega(A)$. 
Then $A'_1=\{s_1\}$. 
and \eqref{eq: omega}, we have 
\begin{align*}
\omega(A')&\leq \int_{v_1=s_1}\left(\int_{v_2,\ldots v_d \in \mathbb{S}^d}d\sigma_{1}(v_{d+1}) \ldots d\sigma_{d}(v_2)\right) d\sigma_{d+1}(v_1)=0,
\end{align*}
which completes the proof. 
\end{proof}

\begin{proof} [Proof of Theorem \ref{thm: ri}]$ $

We first claim that \eqref{eq: ri-s3w} is well-defined. 

Let 
$S(\mu,\nu)=\{s\in\mathbb{S}^d: \min(\mu(\{s\}),\nu(\{s\}))>0\}$. 
Since $\mu,\nu$ are probability measures, we have that $S(\mu,\nu)$ is a finite set. 
For each $s\in S$, let  $A(s)=\{R\in \mathrm{O}(d+1): \, Rs=s_n\}$. Note that $S3W_{\mathcal{G},p}(R_\#\mu,R_\#\nu)$ is well-defined for each $R\in \mathrm{O}(d+1)\setminus A(\mu,\nu)$, where $A(\mu,\nu)=\bigcup_{s\in S(\mu,\nu)}A(s)$. By lemma \ref{lem: omega}, 
$$\omega(A(\mu,\nu))=\omega\left(\bigcup_{s\in S(\mu,\nu)}A(s)\right)\leq \sum_{s\in S(\mu,\nu)}\omega(A(s))=0.$$ 
Thus, we have that
\begin{align}
RI-S3W_{\mathcal{G},p}(\mu,\nu)
:= \mathbb{E}_{R\sim \omega} [S3W_{\mathcal{G},p}(R_\#\mu, R_\#\nu)]=\int_{R\in \mathrm{O}(d+1)\setminus A(\mu,\nu)}
 S3W_{\mathcal{G},p}(R_\#\mu, R_\#\nu) \, d\omega(R) \nonumber
\end{align}
is well-defined. 

Now, we will verify its properties. 

By Theorem \ref{thm:s3w}, $S3W_{\mathcal{G},p}$ is non-negative, and symmetric. Thus $RI\text{-}S3W_{\mathcal{G},p}(\cdot,\cdot)$ is also non-negative and symmetric.

For the triangle inequality, let $\mu_1,\mu_2,\mu_3\in \mathcal{P}(\mathbb{S}^{d})$.
Similarly to before, let $A(\mu_1,\mu_2,\mu_3)=\bigcup_{s\in S(\mu_1,\mu_2,\mu_3)}{A(s)}$. Since $\omega(A(\mu_1,\mu_2,\mu_3))=0$, we have 
\begin{align}
&RI-S3W_{\mathcal{G},p}(\mu_1,\mu_3)=\int_{\mathrm{O}(d+1)}S3W_{\mathcal{G},p}(R_\# \mu,R_\# \nu) \, d\omega(R) =\int_{\mathrm{SO}(d+1)\setminus A(\mu_1,\mu_2,\mu_3)}S3W_{\mathcal{G},p}(R_\# \mu,R_\# \nu) \, d\omega(R) \nonumber\\
&\leq \int_{\mathrm{SO}(d+1)\setminus A(\mu_1,\mu_2,\mu_3)}\left(S3W_{\mathcal{G},p}(R_\# \mu_1,R_\# \mu_2)+S3W_{\mathcal{G},p}(R_\# \mu_2,R_\# \mu_3)\right) \, d\omega(R) \qquad\text{(Triangular inequality)} \nonumber\\
&=\int_{\mathrm{SO}(d+1)\setminus A(\mu_1,\mu_2,\mu_3)}S3W_{\mathcal{G},p}(R_\# \mu_1,R_\# \mu_2) \, d\omega(R)+\int_{\mathrm{SO}(d+1)\setminus A(\mu_1,\mu_2,\mu_3)}S3W_{\mathcal{G},p}(R_\# \mu_2,R_\# \mu_3) \, d\omega(R) \nonumber\\
&=\int_{\mathrm{SO}(d+1)}S3W_{\mathcal{G},p}(R_\# \mu_1,R_\# \mu_2) \, d\omega(R)+\int_{\mathrm{SO}(d+1)}S3W_{\mathcal{G},p}(R_\# \mu_2,R_\# \mu_3) \, d\omega(R) \nonumber\\
&=RI-S3W_{\mathcal{G},p}(\mu_1,\mu_2)+RI-S3W_{\mathcal{G},p}(\mu_2,\mu_3) \nonumber
\end{align}
where the second and previous to last equalities hold from the fact $\omega(A(\mu_1,\mu_2,\mu_3))=0$. 
Thus, $RI\text{-}S3W_{\mathcal{G},p}(\cdot,\cdot)$ is a well-defined pseudo-metric. 

Similarly, we can prove $RI\text{-}S3W_{\mathcal{H},p}(\cdot,\cdot)$ is a well-defined pseudo-metric. 

Finally, it remains to show that $\mu=\nu$ if $RI\text{-}S3W_{\mathcal{H},p}(\mu,\nu)=0$. If that holds, then for all $R\in \mathrm{SO}(d+1)\setminus A(\mu,\nu)$, we have $S3W_{\mathcal{H},p}(R_{\#}\mu,R_\#\nu)=0$. Pick one of such rotations $R$. From Theorem \ref{thm:s3w}, and the fact $R_\# \mu, R_\# \nu\in \mathcal{P}(\mathbb{S}^{d}\setminus \{s_n\})$, we have $R_\#\mu=R_\# \nu$.  
Since 
$$
\mathbb{R}^{d+1}\ni v\mapsto Rv\in \mathbb{R}^{d+1}
%\mathrm{SO}(d+1)\ni v\mapsto Rv\in \mathrm{SO}(d+1)
$$
is a one-by-one mapping, we have 
$\mu=\nu$.

\end{proof}

\section{Funk Radon Transform and Stereographic Spherical Radon Transform}
The Funk-Radon Transform \cite{quinto1982null,helgason2011integral,quellmalz2017generalization,quellmalz2020funk}, which is also known as the Minkowski-Funk transform or the spherical Radon transform is a classical and significant mathematical tool used in integral geometry, with profound applications in image reconstruction and analysis. It plays a crucial role in fields like tomography, allowing for the reconstruction of images from projection data. In this section, we discuss the relation between the Funk Transform and the Stereographic Spherical Radon Transform. 

First, we review some basic concepts about the Funk transform. In the sphere $\mathbb{S}^{d}\subset \mathbb{R}^{d+1}$ centered at the origin with radius $1$, every $(d-1)-$dimensional sub-sphere (``circle'') is the intersection of $\mathbb{S}^{d}$ with a 
%d-1$
$d-$dimensional hyperplane, that is, 
$$S(s,t)=\{s'\in\mathbb{S}^d\mid \,  \langle s, s'\rangle =t\},$$
where $s\in\mathbb{S}^d$ is normal to the hyperplane and $t\in[-1,1]$ is the signed distance of the hyperplane to origin. For example, if $t=\pm 1$, $S(s,t)$ consists of only the singleton $\{\pm s\}$. As another example, when $t=0$, for any $s\in\mathbb{S}^d$, the sub-sphere $S(s,0)$ is a ``great circle'' (a sub-sphere of dimension $d-1$ in $\mathbb{R}^{d+1}$ centered at the origin and with radius $1$). 

The \textit{Spherical Transform} or \textit{Spherical Mean Operator} \cite{quellmalz2017generalization} of a function defines on  $\mathbb{S}^d$ is formally defined as a function on $\mathbb{S}^d\times [-1,1]$ 
by the surface integral
\begin{align}\label{eq: sph mean operator}
\mathcal{U}f(s,t):=\frac{1}{\|S(s,t)\|_{d-1}}\int_{S(s,t)}f \, dS(s,t)=\frac{1}{\|S(s,t)\|_{d-1}}\int_{\mathbb{R}^{d+1}}f(s')\delta(t-\langle s,s'\rangle) \, ds', \quad (s,t)\in \mathbb{S}^d\times [-1,1],
\end{align}
where $\|S(s,t)\|_{d-1}$ denotes the ``surface area'' of the sub-sphere  $S(s,t)$, and $dS(s,t)$ is the induced volume form on the surface $S(s,t)$.
In the special case $d=2$, given $f$ defined on $\mathbb{S}^2\subset \mathbb{R}^3$, the integral in \eqref{eq: sph mean operator} is carried out with respect to the arclength $dS$ of the circle
$S(s,t)$, that is,
\begin{align*}
\mathcal{U}f(s,t)=\frac{1}{2\pi \cdot \text{radius}(S(s,t))}\int_{S(s,t)}f \, dS(s,t):=\frac{1}{2\pi\cdot \text{radius}(S(s,t))}\int_a^b f(r(u)) \|r'(u)\|\, du,
\end{align*}
where $\text{radius}({S(s,t)})$ denotes the radius of the circle $S(s,t)$, $r(u)$ for $a\leq u\leq b$ is a parametrization of $S(s,t)$, and $\|r'(u)\|$ is the magnitude of the tangent vector $r'(u)$ to $S(s,t)$. Besides, in the special case $d=3$, given $f$ defined on $\mathbb{S}^3\subset \mathbb{R}^4$, we have that $S(s,t)$ is $2$-dimensional sphere and we get 
\begin{align*}
\mathcal{U}f(s,t)=\frac{1}{\frac{4}{3}\pi \cdot \text{radius}(S(s,t))^3}\int_{S(s,t)}f \, dS(s,t):=\frac{1}{\frac{4}{3}\pi\cdot \text{radius}(S(s,t))^3}\int\int f(r(u,v)) \|r_u\times r_v\|\, du \, dv,
\end{align*}
where $\text{radius}({S(s,t)})$ denotes the radius of the sphere $S(s,t)$, $r(u,v)$ is a parametrization of $S(s,t)$, and $\|r_u\times r_v\|$ is the magnitude of the cross-product between the partial derivatives $r_u$ and $r_v$, which is a normal vector to $S(s,t)$.  

The inversion of this Spherical Mean Operator is an over-determined problem (for instance, $\mathcal{U}f(s,t)=f(s)$ for all $s\in\mathbb{S}^{d}$). In application, one has access restrictions of $\mathcal{U}f(s,t)$ to specific sub-spheres $S(s,t)$, and this yields different transforms.

One important example arises when we restrict $\mathcal{U}f(s,t)$ to ``great circles'' $S(s,0)$ getting the classical Funk-Radon transform $\mathcal{F}_1$,  that is, 
\begin{align}
\mathcal{F}_1f(s):=\mathcal{U}f(s,0)\label{eq: Funk-Radon transform}, \qquad s\in \mathbb{S}^d.
\end{align}

Another example, studied in \cite{Abouelaz1993, gindikin1994spherical, helgason2011integral}
arises when considering the restriction to the family of sub-spheres containing the North Pole
$s_n=[0,\dots,0,1]\in\mathbb{R}^{d+1}$.

Similarly as in \cite{quellmalz2017generalization,quellmalz2020funk}, we will interested in a generalized Funck-Radon transform given by the following restriction of $\mathcal{U}$.  Let $s=[s_1,\dots,s_{d+1}]\in\mathbb{S}^d\subset \mathbb{R}^{d+1}$, and $\xi\in [-1,1]$. Consider the point 
$[0,\ldots,0, \xi s_{d+1}]\in\mathbb{R}^{d+1}$ located in the same positive axis as the North Pole $s_n=[0,\dots,0,1]$ and inside the sphere $\mathbb{S}^d$\footnote{In \cite{quellmalz2017generalization,quellmalz2020funk}, $\xi\in [0,1)$ and so $[0,\ldots,0, \xi s_{d+1}]$ lies strictly inside the sphere.}.
%$\xi e^{d+1}=[0,\ldots,0, \xi]^T$. 
Then, we restrict $\mathcal{U}f$ to the integrals over
the intersections of the sphere $\mathbb{S}^d$ with hyperplanes containing the point $[0,\ldots,0, \xi s_{d+1}]$, that is,  
%The hyperplane perpendicular to $s$ that contains 
% $[0,\ldots,0, \xi s_{d+1}]$
% can be described as 
% $$\{x\in \mathbb{R}^{d+1}: \langle s, x\rangle=\xi s_{d+1} \},$$
% and the corresponding sub-sphere is 
$$\zeta_\xi^s:=S(s, \xi s_{d+1})=\{s'\in\mathbb{S}^{d}| \, \langle s,s' \rangle=\xi 
s_{d+1}\}, \qquad \xi\in [-1,1], \, s\in \mathbb{S}^d.
$$
Then, the generalized Funk-Radon transform is defined by
% \begin{align}
% U_tf(s):= Uf(s,\xi s)=\frac{1}{V(\zeta_\xi^s)}\int_{\zeta_\xi^s}f(s')d
% \zeta_\xi^s=\frac{1}{V(\zeta_\xi^s)}\int_{\mathbb{R}^{d+1}}\delta(\xi s_{d+1}-\langle s',s\rangle)f(s')ds' \label{eq: Spherical transform}, 
% \end{align}
\begin{align}
\mathcal{U}_\xi f(s):= \mathcal{U}f(s,\xi s_{d+1})=\frac{1}{\|\zeta_\xi^s\|_{d-1}}\int_{\zeta_\xi^s}f \,  d\zeta_\xi^s
=\frac{1}{\|\zeta_\xi^s\|_{d-1}}\int_{\mathbb{R}^{d+1}}f(s')\delta(\xi s_{d+1}-\langle s,s'\rangle) \, ds' \label{eq: Spherical transform}, 
\end{align}
where $\|\zeta_\xi^s\|_{d-1}$ is the volume (``surface area'') of sub-sphere (circle, if $d=2$) $\zeta_\xi^s$ and $d\zeta_\xi^s$ is the integration over the surface $\zeta_\xi^s$ (path, if $d=2$).\\
\begin{remark}
Note, the generalized Funk-Radon transform has been further generalized by \cite{rubin2022spherical} under the name ``spherical slice transform.'' In particular, choose $\alpha\in \mathbb{R}^{d+1}$ and let $\mathcal{T}(\alpha,k)$ denote the set of all $k$-dimensional affine planes that pass through $\alpha$ and intersect $\mathbb{S}^n$:  
$$\mathcal{T}(\alpha,k)=\{H(a)=\alpha+H: H \text{ is hyperplane}, \operatorname{dim}(H)=k, H(a)\cap \mathbb{S}^{d}\neq \emptyset\}.$$
Pick $H\in \mathcal{T}(\alpha,k)$, then $H\cap \mathbb{S}^d$ is a $k-1$ dimensional sub-circle (or a point). The spherical slice transform of function $f\in L_1(\mathbb{S}^d)$ is defined as: 
\begin{align}
\mathcal{SS}_\alpha f(H)=\int_{H\cap \mathbb{S}^d}f(s) dH\cap \mathbb{S}^d, \qquad \forall H\in \mathcal{T}(\alpha,k) \label{eq:SSf}.
\end{align}

When $k=d$, pick $s\in \mathbb{S}^d$ and $\xi\in[-1,1]$, we define affine plane
$$H(s,\xi s_n)=\xi s_n+s^\perp,$$
we have $H\in \mathcal{T}(\xi s_n,d)$ and $\zeta_xi^s=H(s)\cap \mathbb{S}^d$. Thus 
$$\mathcal{SS}f(H(s,\xi s_n))=U_\xi(f)(s).$$
That is, the generalized Funk transform \eqref{eq: Funk-Radon transform} is a special case of the spherical slice transform \eqref{eq:SSf}. 
\end{remark}

As said before, the goal of this section is to discuss the %potential 
relationship between a generalized Funk-Radon transform \eqref{eq: Spherical transform} and the Stereographic Radon transform. For doing that let us first consider the \textit{stereographic projection }
\(\phi: \mathbb{S}^d \setminus \{s_n\} \to \mathbb{R}^d\) given by the formula
\begin{equation}\label{eq: stereographic proj without 2}
  \phi(s)=\frac{s[1:d]}{1-s_{d+1}}%=\frac{s[1:d]}{\|s[1:d]\|}
  , \qquad \forall s=[s_1,\dots,s_d,s_{d+1}]\in\mathbb{S}^d,  
\end{equation}
where $s[1:d]:=[s_1,\dots,s_d]\in\mathbb{R}^d$,
which corresponds to the intersection with the equator plane $x_{d+1}=0$, rather than \eqref{eq: stereographic proj}
which corresponds to projection onto the plane $x_{d+1}=-1$ (which has a factor of $2$ in \eqref{eq: stereographic proj without 2}).

It is straightforward to verify that the inverse of $\phi$, denoted as $\phi^{-1}$, is as follows:  
\begin{equation}\label{eq: inv stereo}
    \phi^{-1}(x)=[(1-s_{d+1})x,s_{d+1}], \quad \text{where } s_{d+1}=\frac{\|x\|^2-1}{\|x\|^2+1}, \qquad \forall x\in \mathbb{R}^d.
\end{equation}
See Figures \ref{fig: cartoon stereo d=1}, \ref{fig: stereo 2} and \ref{fig: stereo3}. 

\begin{figure}[H]
    \centering
\includegraphics[width=0.6\linewidth]{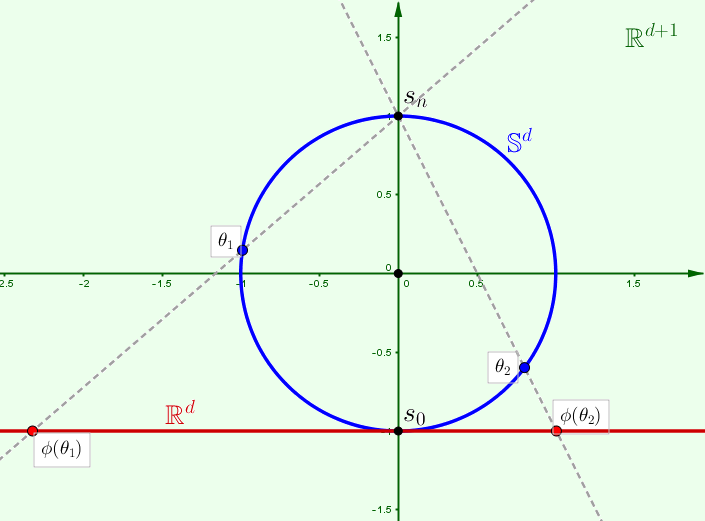}
    \caption{Visualization for $d=1$ of the \textit{stereographic projection} $\phi$ defined by the formula \eqref{eq: stereographic proj}. The ambient space $\mathbb{R}^{d+1}$ is depicted in green, while the unit sphere (circle) $\mathbb{S}^{d}$ is in blue. The North and South Poles are labeled as $s_n$ and $s_0$, respectively. The projected space $\mathbb{R}^d=\phi(\mathbb{S}^d\setminus{s_n})$ is highlighted in red, corresponding to the plane (line) in $\mathbb{R}^{d+1}$ defined by all points with the last coordinate equal to $-1$. Two points, $\theta_1$ and $\theta_2$ (blue dots), from $\mathbb{S}^d$ are projected through $\phi$ obtaining the points $\phi(\theta_1)$ and $\phi(\theta_2)$ (red dots).}
    \label{fig: cartoon stereo d=1}
\end{figure}
\vspace{-0.15in}
\begin{figure}[H]
    \centering
\includegraphics[width=0.6\linewidth]{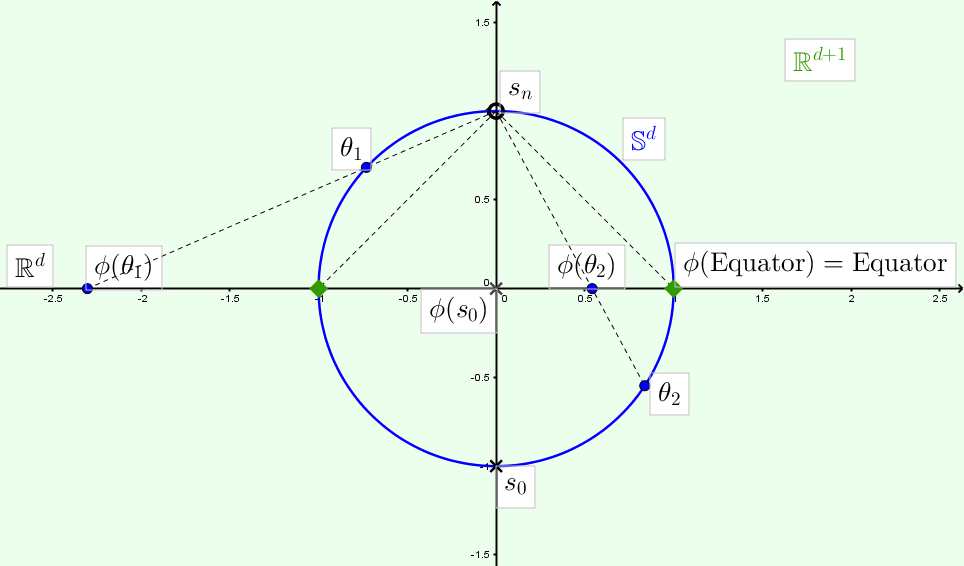}
    \caption{Visualization for $d=1$ of the \textit{stereographic projection} $\phi$ defined by the formula \eqref{eq: stereographic proj without 2}. The ambient space $\mathbb{R}^{d+1}$ is depicted in green, while the unit sphere (circle) $\mathbb{S}^{d}$ is in blue. The North and South Poles are labeled as $s_n$ and $s_0$, respectively. The projected space $\mathbb{R}^d=\phi(\mathbb{S}^d\setminus{s_n})$ coincides with the hyperplane $x_{d+1}=0$, which in the figure is nothing but the horizontal axis. Two points, $\theta_1$ and $\theta_2$ (blue dots), from $\mathbb{S}^d$ are projected through $\phi$ obtaining the points $\phi(\theta_1)$ and $\phi(\theta_2)$ (red dots). The South Pole is projected to the origin in $\mathbb{R}^d$ and the points in the Equator are fixed points for this stereographic projection.}
    \label{fig: stereo 2}
\end{figure}
\vspace{-0.15in}
\begin{figure}[H]
    \centering    \includegraphics[width=0.6\linewidth]{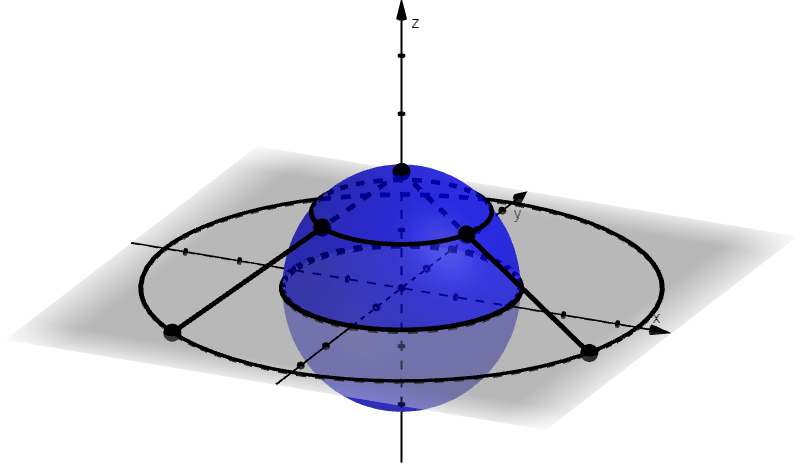}
    \caption{Visualization for $d=2$ of the \textit{stereographic projection} $\phi$ defined by the formula \eqref{eq: stereographic proj without 2}. The ambient space is the 3D space $\mathbb{R}^3$. The sphere $\mathbb{S}^{2}$ is depicted in blue. The projected space $\mathbb{R}^2=\phi(\mathbb{S}^2\setminus{s_n})$ is depicted in grey. The points in the Equator are fixed points for this stereographic projection, and circles parallel to the Equator are mapped to circles in the plane as shown in the plot: If the height of the circle is positive, its projection is a circle with radius greater than 1 as the case depicted in the figure; if the height is negative, its projection is a circle with radius smaller than 1.}
    \label{fig: stereo3}
\end{figure}

When we set $\xi=1$ we have the following relations between the sub-spheres $\zeta_1^s$ and the hyperplanes $H(t,\theta)$ through the stereographic projection $\phi$.

\begin{proposition}\label{pro: hyperplane and circle}
Consider the sub-sphere (circle) $\zeta_1^s=\{s'\in \mathbb{S}^d| \, \langle s',s\rangle=s_{d+1} \}$, where $s=[s_1,\dots,s_d,s_{d+1}]\in \mathbb{S}^d$ is such that $s\neq s_n=[0,\ldots,0,1]$. 
Then, $\phi(\zeta_1^s)$ is a hyperplane in $\mathbb{R}^d$, precisely, 
$$\phi(\zeta_1^s)=\left\{x\in\mathbb{R}^d\mid \,  \langle x,\theta \rangle=t\right\},$$
where $\theta=\frac{s[1:d]}{\|s[1:d]\|}$ and $t=\frac{s_{d+1}}{\sqrt{1-s_{d+1}^2}}$. 

Conversely, given hyperplane $H(t,\theta)=\{x\in\mathbb{R}^d \mid \, \langle x,\theta\rangle=t\}$ for $\theta\in\mathbb{S}^{d-1}$ and $t\in\mathbb{R}$, we have 
$$\phi^{-1}(H(t,\theta))=\zeta_1^s,$$ 
where $s=[\sqrt{1-s_{d+1}^2}\theta,s_{d+1}]$ with $s_{d+1}=\text{sign}(t)\frac{t^2}{t^2+1}$.
\end{proposition}

\begin{proof}
Let $s\in\mathbb{S}^d$, and consider $s'\in \zeta_1^s$. Let $x'=\phi(s')=\frac{s'[1:d]}{1-s'_{d+1}}$. Then, we can write $s'=[(1-s_{d+1}')x',s'_{d+1}]$.  
Besides, since
$$s_1^2+\dots+s_{d}^2+s_{d+1}^2=1,$$
by defining $\theta=\frac{s[1:d]}{\|s[1:d]\|}$, we can write $s=\left[\sqrt{1-s_{d+1}^2}\theta,s_{d+1}\right]$. 
Then, we have that 
\begin{align}\label{eq: proof stereo proj of circle}
s_{d+1}&=\langle s',s\rangle\nonumber\\ 
&=s'_{d+1}s_{d+1}+\langle (1-s_{d+1}')x',\sqrt{1-s_{d+1}^2} \theta\rangle\nonumber\\ 
&=s'_{d+1}s_{d+1}+(1-s'_{d+1})\sqrt{1-s^2_{d+1}}\langle x',\theta\rangle.
\end{align}
That is, 
$$\langle x',\theta\rangle=\frac{s_{d+1}}{\sqrt{1-s_{d+1}^2}}.$$
Thus, $$\phi(\zeta_1^s)=\left\{x\in\mathbb{R}^d\mid \,  \langle x, \theta\rangle=\frac{s_{d+1}}{\sqrt{1-s_{d+1}^2}}\right\}.$$

For the converse, given $\theta\in\mathbb{S}^{d-1}$ and $t\in\mathbb{R}$, consider $x'\in H(t,\theta)$.
By using \eqref{eq: inv stereo}, let 
$$s'=    \phi^{-1}(x)=[(1-s'_{d+1})x,s'_{d+1}], \qquad \forall x\in \mathbb{R}^d$$
where $s'_{d+1}=\frac{\|x'\|^2-1}{\|x'\|^2+1}$.  

We set $s=[\sqrt{1-s_{d+1}^2}\theta,s_{d+1}]$ with  
$s_{d+1}=\text{sign}(t)\sqrt{\frac{t^2}{t^2+1}}$. So, we have 
that $s\in\mathbb{S}^d$ because
$$(1-s_{d+1}^2)\underbrace{\theta^2}_{=1}+s_{d+1}^2=1,$$
and also we have that
$t=\frac{s_{d+1}}{\sqrt{1-s_{d+1}^2}}$.
Thus, we have that
$$\langle x',\theta\rangle=t=\frac{s_{d+1}}{\sqrt{1-s_{d+1}^2}}$$
and by reversing the above proof (see \eqref{eq: proof stereo proj of circle}), we have that
it implies that
$$s_{d+1}=\langle s',s\rangle,$$
and so $s'\in \zeta_1^s.$
Combining with the fact $\phi$ is invertible, we have 
$$\phi^{-1}(H(t,\theta))=\zeta_1^s.$$
\end{proof}

\begin{lemma}\label{lem: s and theta,t}
Let $s\in\mathbb{S}^d$ different from the North Pole $s_n=[0,\dots,0, 1]$, the South Pole $s_0=[0,\dots, 0,-1]$, and consider $S_0$ the unique %half 
great circle (\textit{meridian}) that passes through the North Pole $s_n$, the South Pole $s_0$, and the point $s$.
For each $\hat s\in S_0$, we define the functions $$\theta(\hat s):=\frac{\hat s[1:d]}{\|\hat s[1:d]\|}, \qquad t(\hat s):=\frac{s_{d+1}}{\sqrt{1-s_{d+1}^2}}.$$  
For all $\hat s\in S_0$, 
$\theta(\hat{s})=\theta(s)$. 
In addition $t: S_0\setminus\{s_n\}\to \mathbb{R}$ is a bijection. 
\end{lemma}
\begin{proof} %[Proof of Lemma \ref{lem: s and theta,t}]
Without loss of generality, we can suppose $s=[\cos\alpha,0,\ldots,0,\sin\alpha]$ for some $\alpha\in(-\frac{3}{2}\pi,\frac{1}{2}\pi)$. Thus, $\theta(s)=[1,0,\ldots 0]$, and 
$$S_0=\left\{\hat{s}_\alpha=[\cos\hat\alpha,0,\ldots,0,\sin\hat\alpha]: \, \hat\alpha\in[-\frac{3}{2}\pi,\frac{1}{2}\pi]\right\}.$$
Pick $\hat s_\alpha\in S_0\setminus\{s_n\}$, we have 
$\theta(\hat s_\alpha)=[1,0,\ldots,0]$. 
In addition, $t(\hat{s}_\alpha)=\frac{\sin\alpha}{|\cos\alpha|}$ is a one by one mapping. 
\end{proof}

Notice that for $\mathbb{S}^2\subset{\mathbb{R}^3}$, if we visualize $\mathbb{S}^2$ as the Earth and we pick the point $s\in \mathbb{S}^2$ to be the  Royal Observatory in the town of Greenwich, London, England, then $S_0$ is the \textit{Greenwich meridian} (or ``prime meridian'').

\bigskip

Finally, we discuss the relation between the variant Funk transform defined in \eqref{eq: Spherical transform} with $\xi=1$ (i.e., $\mathcal{U}_1$) and the \textit{Stereographic Spherical Radon transform} defined in \eqref{eq: SR} as
$$\mathcal{S}_\mathcal{R}(\mu)=\mathcal{R}(\phi_\#\mu)$$
for any Radon measure $\mu\in \mathcal{M}(\mathbb{S}^{d})$ that does not assign mass to $\{s_n\}$. In particular, we can define the Stereographic Spherical Radon transform for $L^1$ functions $f:\mathbb{S}^d\setminus\{s_n\}\to \mathbb{R}$, as they can be viewed as densities of continuous measures. 

Indeed, let $\mu$ be the corresponding measure for $f$, that is,
$$\int_{\mathbb{S}^d\setminus\{s_n\}} \psi_0(s)f(s) \, d\sigma_{d+1}(s)=\int_{\mathbb{S}^d\setminus\{s_n\}}\psi_0(s) \, d\mu(s),\qquad \forall \psi_0\in C(\mathbb{S}^{d}),$$ 
where $\sigma_{d+1}$ is the uniform measure on $\mathbb{S}^d\subset \mathbb{R}^{d+1}$.

Let $\hat{f}$ denote the density of $\phi_\#\mu$. By definition of push forward measure, we have:
\begin{align}
  \int_{\mathbb{S}^d\setminus\{s_n\}}\psi(\phi(s)) f(s)d\sigma_{d+1}(s)=\int_{\mathbb{S}^d\setminus\{s_n\}}\psi(\phi(s))d\mu(s)=\int_{\mathbb{R}^d}\psi(x)d\phi_\#\mu(x)=\int_{\mathbb{R}^d}\psi(x)\hat{f}(x)dx, \label{eq: f_hat}  
\end{align}
for all test functions $\psi\in C_0(\mathbb{R}^d)$ that decay to zero. 
In short, we usually write $d\mu=f d\sigma_{d+1}$, and $d\phi_\#\mu=\hat f dx$, where $dx$ represents the Lebesgue measure in $\mathbb{R}^d$.

Thus, $\mathcal{S}_{\mathcal{R}}(f)$ is defined as 
\begin{align}
  \mathcal{S}_\mathcal{R}(f)(t,\theta)=\mathcal{R}(\hat{f})(t,\theta)=\int_{\mathbb{R}^d}\hat{f}(x)\delta(t-\langle x,\theta\rangle) \, dx=\int_{H(t,\theta)} \hat{f}\  dH(t,\theta),
  \label{eq: SR 2}  
\end{align}
where $dH(t,\theta)$ is the surface/line integral over the hyperplane/plane/line
 $H(t,\theta)$.

\begin{proposition}
Let $f\in L^1(\mathbb{S}^d\setminus\{s_n\})$, %$f\in L^1(\mathbb{S}^d\setminus\{0\})$, 
and  let $s\in \mathbb{S}^d$.
Then, 
$$\|\zeta^s_1\|_{d-1} \, \mathcal{U}_1f(s)=\mathcal{S}_\mathcal{R}f(t,\theta),
$$
where
$$\theta=\frac{s[1:d]}{\|s[1:d]\|}, \qquad t=\frac{s_{d+1}}{\sqrt{1-s_{d+1}^2}}.$$
\end{proposition}
\begin{proof}
For convenience, we can extend $f$ to $\mathbb{S}^d$ by setting $f(s_n)=0$.

Without loss of generality, we suppose $s=[\cos\theta, 0,\ldots,0,\sin\theta]$ where $\theta\in(-\frac{3}{2}\pi,\frac{1}{2}\pi)$.

As before, let $\mu$ be the measure having density $f$ and let $\hat{f}$ denote the density of $\phi_\#\mu$. Since $\phi$ is invertible, it is straightforward to verify that 
$(\phi^{-1})_\#(\phi_\#)(\mu)=\mu$, that is, 
\begin{align}
\int_{\mathbb{R}^d}\psi_0(\phi^{-1}(x))\hat{f}(x)dx=\int_{\mathbb{S}^d}\psi_0(s)f(s)d\sigma_{d+1}(s), \label{eq: f_hat inv} \qquad \forall\psi_0\in C(\mathbb{S}^{d}) 
\end{align}
%for all $\psi_0\in C(\mathbb{S}^d)$. %or, equivalently, for all $\psi_0\in L^1(\mathbb{S}^d)$.
In short, $d\mu=fd\sigma_{d+1}$ and $d\phi_\#\mu=\hat f dx$.

Let $S_0$ denote the unique great circle (\textit{meridian}) that passes through the North Pole $s_n=[0,\dots,0, 1]$, the South Pole $s_0=[0,\dots, 0,-1]$ and the point $s$. It is straightforward to verify 
$$\mathbb{S}^d=\bigcup_{\hat s\in S_0}\zeta_1^{\hat s}$$
and that for any distinct $\hat s_1,\hat s_2\in S_0$, $\zeta_1^{\hat s_1}\cap \zeta_1^{\hat s_2}=\{s_n\}$. 

Similarly, $$\mathbb{R}^d=\bigcup_{\hat{t}\in \mathbb{R}}H({\hat{t},\theta})$$
and for any distinct $\hat t_1,\hat t_2\in \mathbb{R}$, $H({\hat{t}_1,\theta})\cap H({\hat{t}_2,\theta})=\emptyset$.
%Since $f(s_n)=0$, we have 

Thus, we can ``slice'' and integrate in the following way:
\begin{align}\label{eq: int over Sd}
\int_{\mathbb{S}^d}\psi_0 \, d\sigma_{d+1}=\int_{\hat{s}\in S_0}\int_{\zeta_1^{\hat s}} \psi_0 \, d\zeta_1^{\hat s} \, dS_0(\hat{s}),
\end{align}
where $d\zeta_1^{\hat s}$ denotes the surface/path integral over $\zeta_1^{\hat s}$, and $dS_0$ denotes the surface/path integral over $S_0$. 

Also, by using Proposition \ref{pro: hyperplane and circle}, we can ``slice'' and integrate in the following way:
\begin{align}\label{eq: int over Rd}
\int_{\mathbb{R}^d}\psi(x) \, dx=\int_{\hat s\in \mathbb{R}}\int_{H(t(\hat s),\theta(\hat s))} \psi \, dH(t(\hat s),\theta(\hat s)) \, d\hat s,
\end{align}
where 
for each $\hat s\in S_0$ we define $\theta(\hat s):=\frac{\hat s[1:d]}{\|\hat s[1:d]\|},t(\hat s):=\frac{\hat s_{d+1}}{\sqrt{1-\hat s_{d+1}^2}}$,   
and where
$dH(t(\hat s),\theta(\hat s))$ denotes the surface/line integral over $H(t(\hat s), \theta(\hat s))$, and $d\hat s$ denotes the integral over $\mathbb{R}=\phi(S_0\setminus\{s_n\})$.
(Notice that if $d=2$, \eqref{eq: int over Rd} is nothing but the Fubini-Tonelli expression $\int_{\mathbb{R^2}}\phi(z) \, dz=\int_\mathbb{R}\int_\mathbb{R}\psi(x,y) \, dx \, dy$, where $dz$ is the Lebesgue measure on $\mathbb{R}^2$ and $dx$, $dy$ the Lebesgue measure on $\mathbb{R}$.)

Refer to Figure \ref{fig: slides circles and lines} for an illustrative visualization. 
\begin{figure}[H]
    \centering
    \includegraphics[width=0.75\linewidth]{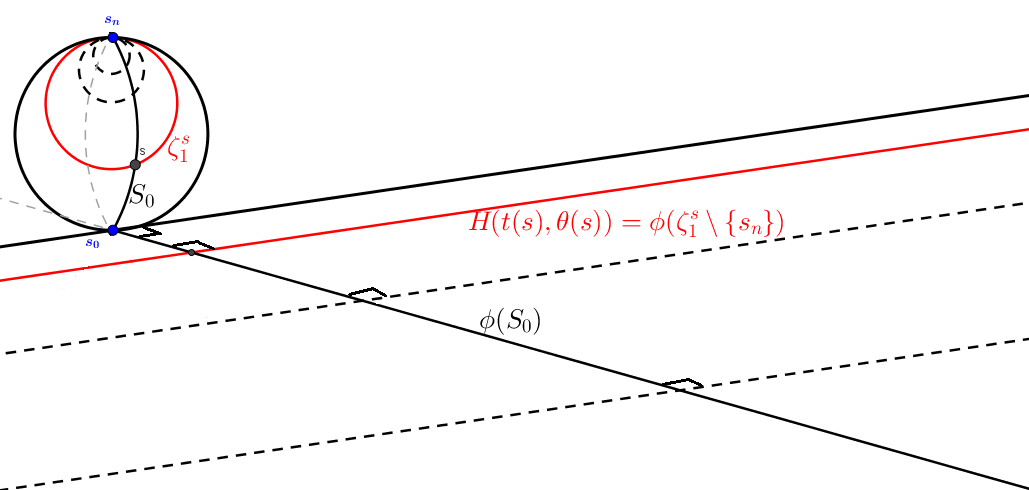}
    \caption{Representation of a unit sphere $\mathbb{S}^2$ in $\mathbb{R}^3$, featuring the North and South Poles ($s_n$ and $s_0$ in blue). The sketch includes five great circles passing through a point $s$: four intersecting $s_n$ and one meridian labeled as $S_0$. The stereographic projection of these five great circles is depicted as straight lines, which lie on the same plane. For the stereographic projection $\phi$ we used formula \eqref{eq: stereographic proj}. The four circles passing through both $s_n$ and $s$ project onto parallel lines. One such circle is highlighted in red and labeled as $\zeta_1^s$, with its corresponding projection also in red, labeled as $H(t(s),\theta(s))$. These four parallel lines are perpendicular to the line labeled as $\phi(S_0)$, representing the stereographic projection of the meridian $S_0$.}
    \label{fig: slides circles and lines}
\end{figure}

By using Lemma \ref{lem: s and theta,t} combined with identity \eqref{eq: f_hat inv}, together with with \eqref{eq: int over Sd} and \eqref{eq: int over Rd}, we have
% \begin{align}
% \int_{S_0} \int \psi_0(\phi^{-1}(x))\hat{f}dH({t(\hat{s})}, \theta) d\hat{s}&=\int_{-\infty}^\infty \int \psi_0(\phi^{-1}(x))\hat{f} dH(t,{\theta}) dt\nonumber\\
% &=\int_{\mathbb{R}^d}\psi_0(\phi^{-1}(x))\hat{f}(x) dx\nonumber\\
% &=\int_{\mathbb{S}^d}\psi_0(s)f(s)d s\nonumber\\
% &=\int_{S_0}\int \psi_0(\hat s) f(\hat s) d\mathcal{S}_{\hat s,1}d\hat{s}.\nonumber 
% \end{align}
\begin{align*}
    \equalto{\int_{\mathbb{R}^d}\psi_0(\phi^{-1}(x))\hat{f}(x)dx}{\int_{\hat s\in \mathbb{R}}\int_{H(t(\hat s),\theta(\hat s))} \psi_0\circ\phi^{-1} \, \hat f \, dH(t(\hat s),\theta(\hat s)) \, d\hat s} = \equalto{\int_{\mathbb{S}^d}\psi_0(s)f(s)d\sigma_{d+1}(s)}{\int_{\hat{s}\in S_0}\int_{\zeta_1^{\hat s}} \psi_0 \, f \, d\zeta_1^{\hat s} \, dS_0(\hat{s})} 
\end{align*}
Since the above equality holds for all $\psi_0\in C(\mathbb{S}^d)$, by setting $\psi_0\equiv 1$ we obtain 
%$\psi_0\in L_1(s)$, by setting $\psi_0(\hat{s})=1(s=\hat{s})$, we obtain 
\begin{align*}
S_{\mathcal{R}}f(t,\theta)=\int_{H{(t,\theta)}} \hat f \,  dH{(t,\theta)}=\int_{\zeta^s_1} f \,  d\zeta_1^s=\|\zeta^s_1\|_{d-1} \, \mathcal{U}_1f(s)
\end{align*}
and we complete the proof. 
%\rocio{So far I think the proof is good. It is not that rigorous on the test functions $\psi_0$ and on the ``slice'' integration in \eqref{eq: int over Sd} and \eqref{eq: int over Rd}, but that's ok. Essentially, we are claiming something like  ``$d\zeta_1^s=d\phi^{-1}_\#\phi_\#\zeta_1^s=d\phi^{-1}_\#H(t,\theta)$''.}
\end{proof}

% That is 
% \begin{align}
% \int \psi_0(\phi^{-1}(x))\hat{f}(x)d \mathcal{H}_{t,\theta}=
% \end{align}

% Let $\mu$ denote the corresponding radon measure $f$ defined, i.e. 
% $$\int_{\mathbb{S}^d} \psi d\mu=\int_{\mathbb{S}^d} \psi fds,$$ 
% for all test function $\psi$. 
% we have 
% \begin{align}
% U_1f(s)&=\frac{1}{V(\zeta_1^s)}\int_{\zeta_1^s}f(s')d\zeta_1^s\nonumber\\
% &=\frac{1}{V(\zeta_1^s)}\int_{ \mathbb{S}^d}f(\phi^{-1}(x))dH_{\theta,t}\nonumber\\
% &=\frac{1}{V(\zeta_1^s)}\int_{\mathbb{R}^d} \delta(s_{d+1}=\langle \phi^{-1}(x'),s\rangle)d \phi_\# \mu(x') \nonumber\\
% % &=\frac{1}{V(\zeta_1^s)}\int_{\mathbb{R}^d}\delta(t=\langle \theta, x' \rangle)d\phi_\#\mu(x') \nonumber\\ 
% % &=\frac{1}{V(\zeta_1^s)}\mathcal{S}f(\theta,t) \nonumber
% % \end{align}
% where in the fourth equation, $x'=\phi^{-1}(s')$ and it follows by changing variables; the fifth equation holds by proposition \ref{pro: hyperplane and circle}. 

\section{Related Work: Vertical Sliced Transform, Semi-Circle Transform, and Related Wasserstein Distances}
This section introduces the \textbf{vertical sliced transform} \cite{shepp1994spherical}, \textbf{semi-circle transform} \cite{groemer1998spherical}, and the related optimal transport distances \cite{quellmalz2023sliced}. 
\subsection{Vertical Slice Transform and Vertical Sliced Wasserstein Distance}
\label{sec:vsw}
We first introduce the vertical slice transform technique introduced in \cite{quellmalz2023sliced}, which can only be defined in $\mathbb{S}^2$ space. Next, we extend it into $\mathbb{S}^d$ for $d\ge 2$. Finally, we discuss the vertical slice of Wasserstein distance and its relation to the classical slice of Wasserstein distance. 

\subsubsection*{Parametrization in $\mathbb{S}^2$.}
For each point $s\in \mathbb{S}^2$, we can parameterize it as 
\begin{align}
s=[\cos\alpha\sin\beta,\sin\alpha\sin\beta,\cos \beta]\label{eq:s},
\end{align}
where $\alpha\in[0,2\pi),\beta\in[0,\pi]$. 

When $s$ is the north (or south) pole, i.e., 
$s=[0,0,1]$ (or $s=[0,0,-1]$), the above parameterization in particular is not uniquely determined. In this case, we set $\alpha=0,\beta=0$ (or $\alpha=0, \beta=\pi$). Otherwise, $\alpha,\beta$ are uniquely determined. 

In particular, $\alpha$ is called the azimuth angle, denoted as $\alpha=\operatorname{azi}(s)$, and $\beta$ is called the zenith angle, denoted as $\beta = \operatorname{zen}(s)$. 

In summary, the following mapping is bijective: 
\begin{align}
\Phi:\left((0,2\pi)\times (0,\pi)\right) \cup \left(\{0\}\times\{0,\pi\}\right)&\to \mathbb{S}^2\nonumber\\
 (\alpha,\beta)&\mapsto \Phi(\alpha,\beta)=s=[\cos\alpha\sin\beta,\sin\alpha\sin\beta,\cos\beta]\label{eq:Phi}.
\end{align}

\subsubsection*{Vertical Slice Transform and Vertical Sliced Wasserstein Distance.}

The equator in $\mathbb{S}^2$ can be defined by 
\begin{align}
\mathcal{E}:=\{\Phi(\alpha,0)=[\cos\alpha,\sin\alpha,0]:\alpha\in[0,2\pi)\}\label{eq:E_S2}.
\end{align}

Each $\theta=\Phi(\alpha,0)\in \mathcal{E}$ and $t\in[-1,1]$ can determine a ``vertical'' circle
\begin{align}
\mathcal{VS}(\alpha,t):=\mathcal{VS}(\theta,t):=\{s\in \mathbb{S}^2: \langle \theta,s\rangle=t\},
\end{align}
characterizing a circle centered at $t\theta$ with radius $\sqrt{1-t^2}$. Thus $\|\mathcal{VS}(\theta,t)\|_{1}=2\pi \sqrt{1-t^2}$. 

Given $f\in \mathcal{L}_1(\mathbb{S}^2)$, the \textbf{vertical slice transform} $\mathcal{V}$ of $f$, denoted as $\mathcal{V}f$, is a $L_1$ function defined on $\mathcal{E}\times [-1,1]$: 
\begin{align}
\mathcal{V}f(\theta,t):=
\begin{cases}
\frac{1}{\|\mathcal{VS}(\theta,t)\|_1}\int_{\mathcal{VS}(\theta,t)}f(s)d\mathcal{VS}(\theta,t)(s)
=\frac{1}{2\pi \sqrt{1-t^2}}\int_{\mathbb{S}^2}f(s)\delta(t-\langle \theta,s\rangle )ds  &\text{if }t\in(-1,1)\\ 
f(\theta) &\text{if }t=1\\
f(-\theta) &\text{if }t=-1
\end{cases}
\label{eq:Vf}
\end{align}

Its adjoint operator $\mathcal{V}^*: L^1(\mathcal{E}\times[-1,1])\to L^1(\mathbb{S}^2)$ has closed form: 
\begin{align}
\mathcal{V}^*g(s)=\frac{1}{2\pi}\int_{\mathcal{E}}g(\theta,\langle \theta,s\rangle) d\mathcal{E} \label{eq:V*}.
\end{align}
Thus, we can extend the vertical transform to any (positive Radon) measure as follows.

For each $\mu\in\mathcal{M}_+(\mathbb{S}^2)$, $\mathcal{V}(\mu)$ is defined by:  

\begin{align}
\int_{\mathcal{E}\times[-1,1]} \psi(\theta,t) d\mathcal{V}\mu(\theta,t)=\int_{\mathbb{S}^2}\mathcal{V^*}\psi(s)(s) d\mu(s), \qquad \forall \psi\in C_0(\mathbb{E}\times [-1,1]) \label{eq:V(mu)}.
\end{align}

Note, $\mathcal{V}(\mu)$ can be equivalently defined as 
\begin{align}
\int \psi(\theta,t)d\mathcal{V}(\mu)(\theta,t)=\int_{\mathcal{E}}\int_{-1}^1 \psi(\theta,t)d \mathcal{V}(\mu)_\theta(t) d\sigma_\mathcal{E}(\theta)\label{eq:Vmu_2}
\end{align}
where $\sigma_\mathcal{E}$ is the uniformly probability measure defined on $\mathcal{E}$ and  $$\mathcal{V}(\mu)_\theta=\langle \theta,\cdot \rangle_\#\mu,\forall \theta\in\mathcal{E}.$$ is the integration of $\mathcal{V}\mu$ with respect to $\sigma_{\mathcal{E}}$.  

In the discrete case, $\mu=\sum_{i=1}^np_i\delta_{x_i}$, the above definition becomes 
$
\mathcal{V}(\mu)_\theta=\sum_{i=1}^np_i\delta_{\langle \theta,x_i \rangle} \label{eq:V(mu)_2_discrete}. $

Given $\mu,\nu\in \mathcal{P}(\mathbb{S}^2)$, 
the \textbf{vertical sliced Wasserstein distance} is defined by 
\begin{align}
VSW_p^p(\mu,\nu)&:=\int_{\mathcal{E}} W_p^p(\mathcal{V}(\mu)_\theta,\mathcal{V}(\nu)_\theta) d\sigma_{\mathcal{E}}(\theta)\label{eq:vsw}\\
&\approx \frac{1}{N}\sum_{t=1}^NW_p^p(\langle\theta,\cdot \rangle_\#\mu,\langle\theta,\cdot \rangle_\#\nu) \label{eq:vsw_approx}
\end{align}
where \eqref{eq:vsw_approx} is the Monte Carlo approximation and $\theta_1,\ldots \theta_N$ are uniformly selected from $\mathcal{E}$. 

By [Theorem 3.7, \cite{quellmalz2023sliced}], the vertical transform \eqref{eq:V(mu)} (or \eqref{eq:Vmu_2}) is injective. Thus, the above VSW problem defines a metric.

\subsubsection*{Relation with Classical Radon Transform.}

We refer to \eqref{eq: R(mu)} and \eqref{eq: R(mu) 2} for the classical Radon transform for a measure $\mu\in \mathcal{P}(\mathbb{R}^3)$. 

In addition, we have 
\begin{align}
\mathcal{R}(\mu)_\theta=\mathcal{V}(\mu)_\theta=\langle\theta,\cdot\rangle_\#\mu,\qquad \forall \theta\in \mathcal{E}\subset \mathbb{S}^2. \nonumber 
\end{align}
That is, the restricted transformed measure of $\mu$ under $\mathcal{R}$ and $\mathcal{V}$ coincide when $\theta\in\mathcal{E}$. 

Recall that the Sliced Wasserstein (SW) distance is defined by 
\begin{align}
SW_p^p(\mu,\nu)&:=\int_{\mathbb{S}^2} W_p^p(\mathcal{R}(\mu)_\theta,\mathcal{R}(\nu)_\theta)d\sigma_{\mathbb{S}^2}(\theta)\label{eq:sw}\\
&\approx \frac{1}{N}\sum_{i=1}^NW_p^p(\langle\theta,\cdot\rangle_\#\mu, \langle\theta,\cdot\rangle_\#\nu)\label{eq:sw_approx},
\end{align}
where \eqref{eq:sw_approx} is the Monte Carlo approximation of the SW problem and $\theta_1,\ldots \theta_N$ are uniformly sampled from $\mathbb{S}^2$. 

Comparing \eqref{eq:vsw_approx} and \eqref{eq:sw_approx}, we observe that when $\theta_1,\ldots \theta_N$ are sampled from $\mathcal{E}$, \eqref{eq:vsw_approx} and \eqref{eq:sw_approx} coincide. 

\subsubsection*{Generalization of Vertical Sliced Wasserstein in $\mathbb{S}^d$.}

It is natural to extend the vertical sliced transform given by \eqref{eq:Vf}, \eqref{eq:V(mu)}, \eqref{eq:Vmu_2} to $\mathbb{S}^{d}$ for $d\ge 2$.

Let $$\mathcal{E}=\{s\in \mathbb{S}^{d}: s_{d+1}=0\},$$
denote the equator. 

Choose $(\theta,t)\in \mathcal{E}\times [-1,1]$, 
the $d-1$ dimensional vertical sphere $\mathcal{VS}(\theta,t)$ is defined as 
$$\mathcal{VS}(\theta,t):=\{s\in \mathbb{S}^d: \langle\theta,s\rangle=t\}.$$

Choose $f\in L_1(\mathbb{S}^{d})$, $\mathcal{V}f$ is defined by 
\begin{align}
\mathcal{V}f(\theta,t):=\frac{1}{\|\mathcal{E}\|_{d-1}}\int_{VS(\theta,t)}f(s)dVS(\theta,t)=\frac{1}{\|\mathcal{E}\|_{d-1}}\int_{\mathbb{S}^{d}}f(s)\delta(\langle \theta,s\rangle-t)ds\label{eq:Vf_d}. \end{align}
It is straightforward to verify 
$\mathcal{V}f\in L_1(\mathcal{E}\times[-1,1])$. 

In addition, its adjoint operator  $\mathcal{V}^*$ becomes 
\begin{align}
\mathcal{V}^*g(s)=\frac{1}{\|\mathcal{E}\|_{d-1}}\int_{\mathcal{E}}g(\theta,\langle \theta,s\rangle) d\theta \label{eq:V*_d},\forall g\in L_1(\mathcal{E}\times[-1,1]).
\end{align}
And similarly, for each $\mu\in \mathcal{M}(\mathbb{S}^d)$, vertical transformed measure $\mathcal{V}\mu$ can be defined by \eqref{eq:V(mu)}, or equivalently, 
\begin{align}
\int_{\mathcal{E}\times[-1,1]} g(\theta,t) d \mathcal{V}\mu(\theta,t) =\int_{\mathcal{E}}
\int_{[-1,1]} g(\theta,t)d\mathcal{V}(\mu)_\theta(t)d\sigma_{\mathcal{E}}(\theta) \label{eq:V(mu)_2_d},
\end{align}
where $\mathcal{V}(\mu)_\theta=\langle \theta,\cdot\rangle_\#\mu$ is the integration of $d\mathcal{V}\mu$ with respect to $\sigma_\mathcal{E}$.

Thus, vertical sliced Wasserstein \eqref{eq:vsw}, \eqref{eq:vsw_approx} can be extended to the space $\mathbb{S}^d$. 

We provide additional results in Appendix \ref{section:gf}.

\subsection{Semi-Circle Transform and Spherical Sliced Wasserstein Distance}

We first introduce the semi-circle transform and then discuss its relation to spherical sliced Wasserstein (SSW). 

The semi-circle can be defined by the following three equivalent formulations. 

\noindent\textbf{Formulation 1 \cite{quellmalz2023sliced}}: 
In $\mathbb{S}^2$, we define 
\begin{align}
\Psi(\alpha,\beta,\gamma):=
\begin{bmatrix} 
\cos\alpha & -\sin\alpha &0\\
\sin\alpha &  \cos\alpha &0\\
0          &  0          &1
\end{bmatrix}
\begin{bmatrix} 
\cos\beta & 0 &-\sin\beta\\
0         & 1 &0\\
\sin\beta & 0 &\cos\beta
\end{bmatrix}
\begin{bmatrix} 
\cos\gamma & -\sin\gamma &0\\
\sin\gamma &  \cos\gamma &0\\
0          &  0          &1
\end{bmatrix}.
\end{align}

Choose $s^*=\Phi(\alpha^*,\beta^*)$ (defined in \eqref{eq:Phi}), $\xi\in [0,2\pi)$, 
the semi-circle is defined by 
\begin{align}
\mathcal{SC}(s^*,\xi)=\{s\in \mathbb{S}^2: \operatorname{azi}(\Psi(\alpha^*,\beta^*,0)^\top s)=\xi\}\cup \{\pm \Phi(\alpha^*,\beta^*)\}.\label{eq:SC_1}
\end{align}

\noindent\textbf{Formulation 2 \cite{groemer1998spherical}}: 

In $\mathbb{S}^{d}$, where $d\ge 2$, given $u,v\in \mathbb{S}^d$ with $u\perp v$, the semi-circle is
defined by 
\begin{align}
\mathcal{SC}(u,v):=\{s\in \mathbb{S}^d: s\perp u, \langle s,v\rangle\ge 0\}. \label{eq:SC_2}
\end{align}

\noindent\textbf{Formulation 3 \cite{bonet2023sliced}}:

In $\mathbb{S}^d$, choose $U\in V_2(d):=\{U'\in \mathbb{R}^{d\times 2}: (U')^\top U'=I_2\}$, and let $U^\perp$ denote the perpendicular space of $\text{Range}(U)$.
In addition, $U$ can determine a great circle: 
$$\mathcal{S}_U:=\{s\in \mathbb{R}^{d+1}:U^Ts=0\}\cap \mathbb{S}^d.$$

The projection from $\mathbb{S}^d$ to $\mathcal{S}_U$ admits the following closed form: 
\begin{align}
P_U(s)=\arg\min_{s'\in \mathcal{S}_U}d_\mathbb{S}(s,s')=\frac{UU^Ts}{\|U^Ts\|},\forall s\in \mathbb{S}^d\setminus U^\perp. \label{eq:proj_1}
\end{align}

The great circle $\mathcal{S}_U$ can be regarded as a $\mathbb{S}^{1}$ circle, and we can use $t\in [0,2\pi)$ to represent each point in $\mathcal{S}_U$. Thus, the above projection mapping can be rewritten as 
$s\mapsto \frac{U^Ts}{\|U^Ts\|}$. The new formulation induces the semi-circle: 
\begin{align}
\mathcal{SC}(U,t)=\left\{s\in \mathbb{S}^d: \frac{U^Ts}{\|U^Ts\|}=t\right\}\cup U^\perp 
\end{align}

When $d=2$, the above three formulations of semi-circle are equivalent. When $d\ge 2$, the second and third formulations are equivalent. For convenience, we select formulation 3.

\subsubsection*{Semi-Circle Transform.}
Given $f\in L_1(\mathbb{S}^d)$, we will introduce the semi-circle transform of $f$. Note, in \cite{quellmalz2023sliced,hielscher2018svd}, this transform is called the \textit{unnormalized semi-circle transform}; in \cite{bonet2023sliced}, this transform is called the \textit{spherical Radon transform}; and in \cite{groemer1998spherical}, it is called the \textit{hemispherical transform}: 

\begin{align}
\mathcal{SC}(f)(U,t):=\int_{\mathcal{SC}(U,t)}f(s) d\mathcal{SC}(U,t)\label{eq:SC(f)}. 
\end{align}

Similar to the Radon transform \eqref{eq: R(mu)} and the vertical slice transform \eqref{eq:V(mu)}, the above definition can be extended into $\mathcal{M}(\mathbb{S}^d)$. The corresponding Wasserstein problem is the spherical sliced Wasserstein problem \cite{bonet2022spherical}: 
\begin{align}
SSW(\mu,\nu):=\int_{V_2(d)} W_p^p(\mathcal{SC}(\mu)_U,\mathcal{SC}(\nu)_U)d\sigma_{V_2(d)}(U) \label{eq:ssw},
\end{align}
where $\mathcal{SC}(\mu)_U$ is the integration of $\mathcal{SC}(\mu)$ with respect to the uniformly measure $\sigma_{V_2(d)}$. In addition, with $\sigma_{V_2(d)}-a.s.$, we have $$\mathcal{SC}(\mu)_U=\left(\frac{U^T\cdot}{\|U^T\cdot\|}\right)_\#\mu.$$

\section{Distance Distortion}
\label{sec:supp_dist}
% In $\mathbb{S}^{d}$, the geodesic distance is defined as $d_{\mathbb{S}^d}(s_1,s_2)=\arccos(\langle s_1,s_2\rangle)$. Via stereographic projection $\phi$, we project $\mathbb{S}^d\setminus\{s_n\}$ into $\mathbb{R}^d$. However, the Euclidean distance in $\mathbb{R}^d$ is not equivalent to that in $\mathbb{S}^d$. For example, we consider points $z_1=(\epsilon,\ldots \sqrt{1-\epsilon^2}),z_2=[-\epsilon,\ldots, \sqrt{1-\epsilon^2}]$, then 

In $\mathbb{R}^d$ we consider the Euclidean distance and in $\mathbb{S}^d$ the \textit{great circle distance} 
$d_{\mathbb{S}^d}(\cdot,\cdot)$. While in Euclidean spaces the distance between two points is the length of a straight segment between them, in the sphere, we measure the distance between two points as the length of the shortest path that lies in the sphere and connects them. In particular, for two angles in the unit circle $\theta_1,\theta_2\in[0,2\pi)$ 
$$d_{\mathbb{S}^1}(\theta_1,\theta_2)=\min\{|\theta_1-\theta_2|,2\pi-|\theta_1-\theta_2|\}.$$
In general, the distance between $s_1,s_2\in\mathbb{S}^d$ given by the arclength of the shortest path can be expressed as 
$$d_{\mathbb{S}^d}(s_1,s_2) := \arccos( \langle s_1 , s_2 \rangle).$$

We point out that in the stereographic projection, significant distortion can occur. For instance,
\begin{align}
\text{as } d_{\mathbb{S}^d}(s_1,s_2)\longrightarrow 0,  \text{ we might have } \|\phi(s_1)-\phi(s_2)\|\longrightarrow \infty\nonumber.
\end{align}
See Figure \ref{fig: distortion} for a visualization.

\begin{figure}[H]
    \centering
\includegraphics[width=1\linewidth]{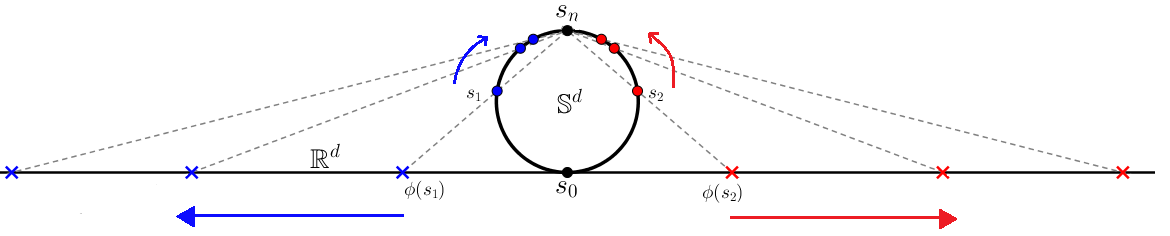}
    \caption{Illustration of the distortion phenomenon in one-dimensional space ($d=1$) as a result of applying the stereographic projection with formula \eqref{eq: stereographic proj}. As the points $s_1$ (blue) and $s_2$ (red) are brought closer within the sphere (circle), their respective stereographic projections, depicted as crosses, diverge from each other.}
    \label{fig: distortion}
\end{figure}
%as $\epsilon\to 0$. 

In this section, we aim to construct an injective function $h: \mathbb{R}^d\to\mathbb{R}^d$ such that the following holds:

% the Euclidean distance induced by $h$, 

$$\|h(\phi(s_1))-h(\phi(s_2))\|\approx d_{\mathbb{S}^d}(s_1,s_2).$$

% First, it is straightforward to verify that: 
% $$\phi(\mathbb{S}_\epsilon^d)=\mathbb{D}_{\epsilon}^d:=\left\{x: \|x\|^2\leq \frac{1-\epsilon^2}{\epsilon^2}\right\}.$$

We define $h_1: \mathbb{R}^d\to \mathbb{R}^d$ by 
\begin{equation}\label{eq: special h_1}
  h_1(x):=\arccos(-s_{d+1})\frac{x}{\|x\|}, \quad \text{where } s_{d+1}=\frac{\|x\|^2-1}{\|x\|^2+1}. 
\end{equation}

\begin{proposition}\label{prop: h prop1}
Let $s, s_1,s_2\in \mathbb{S}^d\setminus\{s_n\}$, and let $s_0=[0,\ldots,0,-1]^T\in \mathbb{S}^d$ denote the South Pole. Consider $\phi:\mathbb{S}^{d}\setminus\{s_n\}\to\mathbb{R}^d$ the stereographic projection as in \eqref{eq: stereographic proj without 2} with inverse \eqref{eq: inv stereo}, and also consider the function $h_1:\mathbb{R}^d\to\mathbb{R}^d$ defined by \eqref{eq: special h_1}. 
\begin{enumerate}
\item[(1)]
$h_1(\phi(s))=\angle(s,s_0)\frac{s[1:d]}{\|s[1:d]\|}$, and thus $\|h_1(\phi(s_1))-h_1(\phi(s_2))\|\leq 2\pi$, where $\angle(s,s_0)=\arccos(\langle s,s_0\rangle){\in[0,\pi]}$ denotes the angle between $s$ and $s_0$. 

\item[(2)]  If $s_1,s_2,s_0$ are in the same great circle, denoted as $S$, let  $d_S$ denote the circle distance in $S$, then 
    \begin{align}
      d_{\mathbb{S}^d}(s_1,s_2)=d_S(s_1,s_2)=\min\{\|h_1(\phi(s_1))-h_1(\phi(s_2)\|,2\pi-\|h_1(\phi(s_1))-h_1(\phi(s_2)\|\}\label{eq: d_h and d_S case1}
    \end{align} 
 %   \item[(3)] Let $\mathrm{O}(d+1)$ denote the set of all $(d+1)\times (d+1)$ orthonormal matrices, then 
  %  \begin{align}
   % d_{\mathbb{S}^d}(s_1,s_2)&=\min_{R\in \mathrm{O}(d+1)}\min\{\|h_1(\phi(R s_1))-h_1(\phi(R s_2))\|,2\pi-\|h_1(\phi(R s_1))-h_1(\phi(R s_2))\| \}\nonumber\\ 
   % &=\min_{R\in \mathrm{O}(d+1)}\|h_1(\phi(R s_1))-h_1(\phi(R s_2))\|    
   % \end{align}
    
%\item[(4)] In $\mathbb{S}^d_\epsilon$, $\|h_1(\phi(\cdot_1))-h_1(\phi(\cdot_2))\|$, $d_{\mathbb{S}^{d}}$ are equivalent. 
\end{enumerate}

\end{proposition}

\begin{proof}$ $
\begin{itemize}
    \item[(1)] Note that 
\begin{align*}
h_1(\phi(s))&=\arccos(-s_{d+1})\frac{\phi(s)}{\|\phi(s)\|} =\arccos (\langle s, s_0\rangle) \frac{s[1:d]}{\|s[1:d]\|}=\angle(s,s_0) \underbrace{\frac{s[1:d]}{\|s[1:d]\|} }_{\leq 1},
\end{align*}
for all $ s\in\mathbb{S}^d\setminus\{s_n\}\subset\mathbb{R}^{d+1}.$

Thus, we have that
$$\|h_1(\phi(s_1))-h_1(\phi(s_2))\|\leq \angle(s_1,s_0)+\angle(s_2,s_0)\leq \pi+\pi=2\pi.$$ 

\item[(2)] Without loss of generality, we suppose that the great circle $S$ (meridian) is of the form $$S=\left\{s=[\cos \theta,0,\ldots,0,\sin\theta]: \,  \theta\in[-\frac{3}{2}\pi,\frac{1}{2}\pi]\right\}$$
(where we go through the circle clockwise having $-\pi/2$ in the middle). 
Thus, we assume that $S$ passes through $s_0=[0,\dots,0,-1]=[\cos(-\frac{\pi}{2}),0,\dots,0,\sin(-\frac{\pi}{2})]$, $s_1=[\cos \theta_1,0,\ldots 0,\sin\theta_1]$, and $s_2=[\cos\theta_2,0,\ldots,0,\sin\theta_2]$ in $S\setminus \{s_n\}$ with $\theta_1,\theta_2\in (-\frac{3}{2}\pi,\frac{1}{2}\pi)$. 
It follows that 
\begin{align}
h_1(\phi(s_1))&=\angle(s_1,s_0)\frac{s_1[1:d]}{\|s_1[1:d]\|}
=(\theta_1+\frac{\pi}{2})[1,0,\ldots,0] \nonumber 
\end{align}
% \rocio{To me this should be
% $h(\phi(s_1))=h(\left[\frac{\cos\theta_1}{1-\sin\theta_1},0,\ldots, 0\right])$ and so this item is not correct.} 
and similarly, $h_1(\phi(s_2))=(\theta_2+\frac{\pi}{2})[1,0,\ldots,0]$.
Thus, the third term in \eqref{eq: d_h and d_S case1} becomes 
$$\min\{\|h_1(\phi(s_1))-h_1(\phi(s_2))\|,2\pi-\|h_1(\phi(s_1))-h_1(\phi(s_2))\|\}=\min\{|\theta_1-\theta_2|,2\pi-|\theta_1-\theta_2|\}.$$ 

For the first term in \eqref{eq: d_h and d_S case1}, we have 
$$d_{\mathbb{S}^d}(s_1,s_2)=\arccos(\cos\theta_1\cos\theta_2+\sin\theta_1\sin\theta_2)=\arccos(\cos(\theta_1-\theta_2))=\min\{|\theta_1-\theta_2|,2\pi-|\theta_1-\theta_2|\}.$$

For the second term, we have 
$d_{S}(s_1,s_2)=\min(|\theta_1-\theta_2|,2\pi-|\theta_1-\theta_2|)$. 

% Thus, item (2) holds.
And we are done.

\end{itemize}
\end{proof}

\begin{remark}
As a corollary from part 1 in Proposition \ref{prop: h prop1}, we notice that $h_1$ preserves angles with the vertex at the origin in the projected space. (See angle $\beta$ in Figure \ref{fig: stereo for proof2}.) 

In addition, when $d=2$, we observe that $h_1(\phi)$ is similar to the  
\textbf{Lambert azimuthal equal-area projection} \cite{bradawl2003statistics}: 
\begin{align}
LP(s)=\sqrt{\frac{2}{1-s[3]}}s[1:2]\label{eq:LP}.
\end{align}
The function $LP$ projects points in $\mathbb{S}^2$ into $\mathbb{R}^2$, while preserving the area in $\mathbb{S}^2$, i.e. 
$$\|A\|_2=\|LP(A)\|_2.$$
Meanwhile, $h_1\circ \phi$ aims to preserve distance. 

In addition, if we write $s=[\cos\alpha\cos\beta,\sin\alpha\cos\beta,\sin\beta]$, then  
\begin{align}
h_1(\phi(s))&=\angle (s,s_0) \frac{s[1:2]}{\|s[1:2]\|}\nonumber\\
&=(\beta+\frac{\pi}{2})[\cos \alpha,\sin\alpha]\nonumber \\
&=-\underbrace{(\beta+\frac{\pi}{2})[\sin( \alpha-\pi/2),-\cos(\alpha-\pi/2)]}_A
\end{align}
where $A$ is exactly the \textbf{azimuthal equidistant projection} \cite{snyder1997flattening} of $s$ centered at south pole $[0,0,-1]$. 

Thus, $h_1\circ \phi$ and the azimuthal equidistance projection are equivalent. 
\end{remark}

\begin{conjecture}\label{prop: ds and dh}
Let $s, s_1,s_2\in \mathbb{S}^d\setminus\{s_n\}$, and let $s_0=[0,0,\ldots,-1]^T\in \mathbb{S}^d$ denote the South Pole. Consider $\phi:\mathbb{S}^{d}\setminus\{s_n\}\to\mathbb{R}^d$ the stereographic projection as in \eqref{eq: stereographic proj without 2} with inverse \eqref{eq: inv stereo}, and also consider the function $h_1:\mathbb{R}^d\to\mathbb{R}^d$ defined by \eqref{eq: special h_1}. 
Let $\mathrm{O}(d+1)$ denote the set of all $(d+1)\times (d+1)$ orthonormal matrices, then 
    \begin{align}
    d_{\mathbb{S}^d}(s_1,s_2)&=\min_{R\in \mathrm{O}(d+1)}\min\{\|h_1(\phi(R s_1))-h_1(\phi(R s_2))\|,2\pi-\|h_1(\phi(R s_1))-h_1(\phi(R s_2))\| \}\nonumber\\ 
    &=\min_{R\in \mathrm{O}(d+1)}\|h_1(\phi(R s_1))-h_1(\phi(R s_2))\|    
    \end{align}
\end{conjecture}
Numerically, we have shown that statement holds (see Figure \ref{fig: ds dh}). We will leave the theoretical proof for a future study. 

Here we provide the intermediate results:  Proposition \ref{pro: h1 lower_bound} and Proposition \ref{pro: h1 upper_bound}. 

\begin{proposition}\label{pro: h1 lower_bound}
Let $s, s_1,s_2\in \mathbb{S}^d\setminus\{s_n\}$, and let $s_0=[0,0,\ldots,-1]^T\in \mathbb{S}^d$ denote the South Pole. Consider $\phi:\mathbb{S}^{d}\setminus\{s_n\}\to\mathbb{R}^d$ the stereographic projection as in \eqref{eq: stereographic proj without 2} with inverse \eqref{eq: inv stereo}, and also consider the function $h_1:\mathbb{R}^d\to\mathbb{R}^d$ defined by \eqref{eq: special h_1}. 
There exists $R^*\in \mathrm{SO}(d+1)$ such that
    \begin{align}
    d_{\mathbb{S}^d}(s_1,s_2)&=\min\{\|h_1(\phi(R^* s_1))-h_1(\phi(R^* s_2))\|,2\pi-\|h_1(\phi(R^* s_1))-h_1(\phi(R^* s_2))\| \}\nonumber\\
    &\ge\min_{R\in\mathrm{O}(d+1)} \min\{\|h_1(\phi(R s_1))-h_1(\phi(R s_2))\|,2\pi-\|h_1(\phi(R s_1))-h_1(\phi(R s_2))\| \}\nonumber
    \end{align}
\end{proposition}

\begin{proof}

Choose $s_1,s_2\in \mathbb{S}^{d}$ two different points. Let $\zeta$ denote the shortest path from $s_1$ to $s_2$. 

Then $\zeta$ can be parametrized as a function. In particular, there exists a mapping $\zeta:[0,1]\to \mathbb{S}^{d}$,  
$\zeta(0)=s_1,\zeta(1)=s_2$, $\zeta$ is differentiable, $\|\zeta'(t)\|$ is constant, and the range of $\zeta$ is exactly the shortest path. 
By definition of the great circle distance $d_{\mathbb{S}^d}$, it is exactly the length of $\zeta$.
That is, $$d_{\mathbb{S}^d}(s_1,s_2)=|\zeta|=\int_0^1\sqrt{(\zeta'(t))^2}dt.$$
Furthermore, $\zeta$ lies in a great circle, denoted as $S$. 

Let $s_0'=\zeta(t_0)$ denote the middle point between $s_1,s_2$ (for an appropriate $t_0\in (0,1)$),  

By the Gram-Schmidt algorithm, there exists $R_0\in \text{SO}(d+1)$ such that $R_0s_0'=s_0$ (for example,  first set $r_1=-s_0'$, then construct orthonormal vectors $r_2,\ldots, r_{d+1}$, and finally define the matrix $R_0$ having rows $r_1,r_2,\dots, r_{d+1}$). %$R=[r_2^T;\ldots;r_{d+1}^T;r_1^T]$.) 

It holds that 
$$d_{\mathbb{S}^d}(s_1,s_2)=d_{\mathbb{S}^d}(R_0s_1,R_0s_2),$$
i.e., $d_{\mathbb{S}^d}$ is invariant under rotations and reflections. 

Besides, note that $R_0s_1,R_0s_2,R_0s_0'$ are in the same great circle. Thus, by using that $d_{\mathbb{S}^2}$ in invariant under orthogonal transformations and by using part (2) in Proposition \ref{prop: h prop1}, we have
\begin{align*}
d_{\mathbb{S}^d}(s_1,s_2)=d_{\mathbb{S}^d}(R_0s_1,R_0s_2)    &=\min\{\|h_1(\phi(R_0s_1))-h_1(\phi(R_0s_2))\|,2\pi-\|h_1(\phi(R_0s_1))-h_1(\phi(R_0s_2))\|\}\\
&=\|h_1(\phi(R_0s_1))-h_1(\phi(R_0s_2))\|
\end{align*}
where the last equality holds because $s_0=R_0s_0'$ is a point between $R_0s_1,R_0s_2$.

In particular, 
\begin{align}
\min_{R\in \mathrm{O}(d+1)}\min\{\|h_1(\phi(Rs_1))-h_1(\phi(Rs_2)\|,2\pi-\|h_1(\phi(Rs_1))-h_1(\phi(Rs_2)\|\} \leq
%&\leq \min_{R\in O(d+1)}\|h(\phi(Rs_1))-h(\phi(Rs_2)\|\nonumber\\
%&\leq d_{\mathbb{S}^d}(Rs_1,Rs_2)=
d_{\mathbb{S}^d}(s_1,s_2)\nonumber
\end{align}
\end{proof}

For the other direction, we first recall that $d_{\mathbb{S}^d}$ is invariant under orthogonal transformations:
$$d_{\mathbb{S}^d}(s_1,s_2)=d_{\mathbb{S}^d}(Rs_1,Rs_2) \qquad \forall R\in \mathrm{O}(d+1).$$
In particular, every $R'\in \mathrm{O}(d+1)$ is a minimizer of: 
 $d_{\mathbb{S}^d}(s_1, s_1)=d_{\mathbb{S}^d}(R' s_1, R' s_1)=\inf_{R\in \mathrm{O}(d+1)} d_{\mathbb{S}^d}(R s_1, R s_1).$
By using part (2) in Proposition \ref{prop: h prop1} we have that for all  $R^*$ such that $R^*s_1$, $R^*s_2$, and  $R^*s_0$ are in the same great circle it holds
\begin{equation}\label{eq: min on restricted set of rot}
  d_{\mathbb{S}^d}(s_1,s_2)=d_{\mathbb{S}^d}(R^*s_1,R^*s_2)=\min\{\|h_1(\phi(R^*s_1))-h_1(\phi(R^*s_2))\|,2\pi-\|h_!(\phi(R^*s_1))-h_1(\phi(R^*s_2))\|\}  
\end{equation}
However, we want to minimize over all the possible orthogonal matrices. At this point, the proof is more involved, and to prove this part we will need a series of auxiliary lemmas:
%%%%%%%%%%% proof by Yikun %%%%%%%%%%%
%To show the other direction, we need the following lemma: 
\begin{lemma}\label{lem: same s_d+1}
Given $s_1,s_2\in \mathbb{S}^{d}$, with $(s_1)_{d+1}=(s_2)_{d+1}$ (i.e., $s_1$ and $s_2$ are in the same \text{latitude}), then 
$$d_{\mathbb{S}^d}(s_1,s_2)\leq\min(2\pi-\|h(\phi(s_1))-h_1(\phi(s_2))\|,\|h_1(\phi(s_1))-h_1(\phi(s_2))\|).$$     
\end{lemma}

\begin{proof} 
Suppose $(s_1)_{d+1}=(s_2)_{d+1}=\sin\alpha$ for some $\alpha\in [-\frac{\pi}{2},\frac{\pi}{2})$. Let $\beta=\angle(s_1[1:d],s_2[1:d])\in [0,{\pi}]$. 
We have 
\begin{align}
d_{\mathbb{S}^d}(s_1,s_2)&=\arccos(\langle s_1[1:d],s_2[1:d]\rangle+(s_1)_{d+1}(s_2)_{d+1})\nonumber\\
&=\arccos\left(\cos^2(\alpha)\cos(\beta)+\sin^2(\alpha)\right)\nonumber\\
\|h_1(\phi(s_1))-h_1(\phi(s_2))\|&=\sqrt{\|h_1(\phi(s_1)\|^2+\|h_1(\phi(s_2)\|^2-2\|h_1(\phi(s_1))\|\|h_1(\phi(s_2))\|\cos(\angle(s_1[1:d],s_2[1:d]))}\nonumber\\
&=\left(\alpha+\frac{\pi}{2}\right)\sqrt{2-2\cos(\beta)}\nonumber\\
&=\left(\alpha+\frac{\pi}{2}\right)2\sin\left(\frac{\beta}{2}\right)\nonumber
\end{align}
Let \begin{align}
 L(\alpha,\beta)&=\|h_1(\phi(s_1))-h_1(\phi(s_2))\|- d_{\mathbb{S}^d}(s_1,s_2)\nonumber\\
 &=2(\alpha+\frac{\pi}{2})\sin\left(\frac{\beta}{2}\right)-\arccos\left(\cos^2(\alpha)\cos(\beta)+\sin^2(\alpha)\right)\label{pf: L beta}
\end{align}
In what follows we will show that $$L(\alpha,\beta)\ge 0 \qquad \text{if } \alpha\in[-\frac{\pi}{2},\frac{\pi}{2}), \, \beta\in[0,\pi].$$ 

First, if $\alpha=-\frac{\pi}{2}$ or $\beta=0$, we have $L=0$.

It remains to consider the case $\alpha\in(-\frac{\pi}{2},\frac{\pi}{2})$, $\beta\in(0, \pi]$. 

The procedure will be the following: We fix $\alpha\in(-\frac{\pi}{2},\frac{\pi}{2})$ and consider the function $L(\alpha,\cdot)$ on the variable $\beta$. 

We have: 
\begin{align}
&\frac{d}{d\beta}L=\left(\alpha+\frac{\pi}{2}\right)\cos\left(\frac{\beta}{2}\right)+\frac{-\cos^2(\alpha)\sin(\beta)}{\sqrt{1-(\cos^2(\alpha)\cos(\beta)+\sin^2(\alpha))^2}}\label{pf: dL beta}
\end{align}
We set $\frac{d}{d\beta}L=0$, and obtain
\begin{align}\label{eq: deriv = 0}
\left(\alpha+\frac{\pi}{2}\right)^2\cos^2\left(\frac{\beta}{2}\right)=\frac{\cos^4(\alpha)\sin^2(\beta)}{1-(\cos^2(\alpha)\cos(\beta)+\sin^2(\alpha))^2}
\end{align}
Let $A=\cos (\beta)=2\cos^2\left(\frac{\beta}{2}\right)-1=1-2\sin^2(\frac{\beta}{2})$, $B=\cos^2(\alpha)=\sin^2 (\alpha')=1-\sin^2(\alpha)$, $C=(\alpha+\frac{\pi}{2})^2=(\alpha')^2$, where $\alpha'=\alpha+\frac{\pi}{2}\in(0,\pi)$.
The above equation \eqref{eq: deriv = 0} becomes: 
\begin{align*}
    \frac{1}{2}C(A+1)&=\frac{B^2(1-A^2)}{1-(AB+1-B)^2}\\
    &=\frac{B^2(1-A)(1+A)}{1-(A^2B^2 +2AB-2AB^2+1-2B+B^2)}\\
        &=\frac{B^2(1-A)(1+A)}{-A^2B^2 -2AB+2AB^2+2B-B^2}\\
    &=\frac{B(1-A)(1+A)}{-A^2B -2A+2AB+2-B}    
\end{align*}
thus, $C(-A^2B -2A+2AB+2-B)=2B(1-A)$, or equivalently, 
\begin{align}\label{eq: quadratic in A}
    (-CB)A^2+2(BC-C+B)A+ (2C-2B-BC)=0
\end{align}
% \begin{align}
% B^2A^2+2((1-B)B-\frac{1}{C})A+((1-B)^2-1+\frac{2}{C})=0\nonumber 
% \end{align}
Denote by $a=-CB,
b=2(BC-C+B),c=(2C-2B-BC)$ the coefficients of the above quadratic expression \eqref{eq: quadratic in A} as function of $A$: We have 
$$\Delta^2=b^2-4ac=4(B-C)^2\ge 0$$
and thus, we obtained real roots. 
Consider  the following cases: 
\begin{itemize}
    \item Case 1: $a=0$. Thus, $B=0$ or $C=0$, and so $\alpha=-\pi/2$, which was already analyzed.
    \item Case 2: $a\not=0$. Since $B,C\geq 0$, we have that $a<0$. 
\end{itemize}

Besides, the two roots $A_1,A_2$ of \eqref{eq: quadratic in A} are given by
\begin{align}
A_1&=\frac{BC+2(B-C)}{BC}  \qquad \text{ and } \qquad
A_2=1 \nonumber    
\end{align}
Since $A=\cos (\beta)\in(-1,1)$ as $\beta\in(0,\pi)$, once we fix $\alpha$, there exists at most one $\beta^*$ such that $\frac{d}{d\beta}(\alpha,\beta^*)L=0$. In fact, such $\beta^*$ exists if and only if
$-1<A_1<1$. 

In addition, we have that
\begin{align}
\frac{d}{d\beta}L(\alpha,0)&=\lim_{\beta\searrow 0} \frac{d}{d\beta}L=(\alpha+\frac{\pi}{2})>0 \nonumber\\
L(\alpha,0)&=0\nonumber \\
L(\alpha,\pi)&=
\begin{cases} 0 &\text{if } \alpha \in(-\frac{\pi}{2},0]\\
4\alpha &\text{if }\alpha\in[0,\frac{\pi}{2})
\end{cases}\ge 0 \nonumber 
\end{align}
We have the following: 

If $\cos(\beta^*)=A_1\in(-1,1)$, for some $\beta^*\in(0,\pi)$, we have $\frac{d}{d\beta}L(\alpha,\beta)\ge 0,\forall \beta\in[0, \beta^*]$, thus $L(\alpha,\beta)$ is an increasing function on $[0,\beta^*]$. Since $L(\alpha,0)>0$, we have $L(\alpha,\beta)\ge 0$ for all $\beta\in[0,\beta^*]$. 

On $(\beta^*,\pi]$, $\frac{dL}{d\beta}\neq 0$. By the Intermediate Value Theorem, we have $\frac{dL}{d\beta}<0$. 
Thus, $L$ is a decreasing function here. Since $L(\alpha,{\pi})\ge 0$, we have $L(\alpha,\beta)\ge 0$ on $(\beta^*,\pi]$. 

If $A_1\notin (-1,1)$, there is no $\beta^*\in(0,\pi)$ such that $\frac{d}{d\beta}L(\alpha,\beta)=0$. Combining this with the fact  $\frac{d}{d\beta}L(\alpha,0)>0$, we have $\frac{d}{d\beta}L(\alpha,\beta)> 0$ on $(0,{\pi})$. Thus, $L$ is an increasing function here, and thus $L(\alpha,\beta)\ge 0$.

Therefore, $L\geq 0$.

Through a similar process, one can show that
\begin{align}
2\pi-\|h_1(\phi(s_1))-h_1(\phi(s_2))\|-d_{\mathbb{S}^d}(s_1,s_2)\geq 0\nonumber    
\end{align}
and thus we complete the proof. 
\end{proof}

\begin{figure}[H]
    \centering
\includegraphics[width=0.85\linewidth]{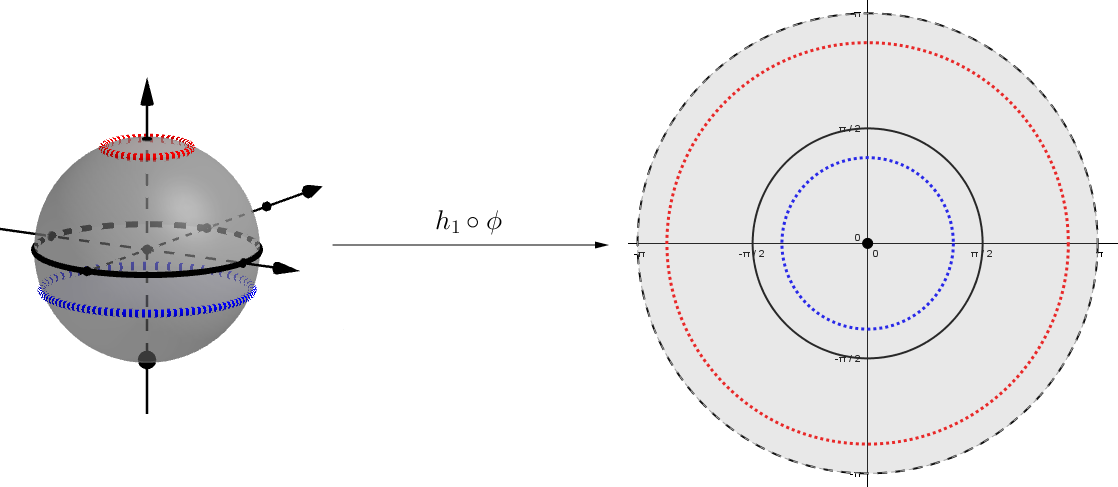}
    \caption{Depiction of the map $h_1\circ\phi:\mathbb{S}^2\setminus\{s_n\}\to\mathbb{R}^2$. The domain is depicted on the left.  Its image, on the right, is inside the open disk of radius $\pi$ centered at the origin. The Equator (black circle on the left) goes to the circle of radius $\pi/2$ (black circle on the right). The South Pole (black big dot on the left) goes to the origin of coordinates of the plane on the right (black dot). Dotted blue ``Parallel'' on the left at negative height (in the South Hemisphere), goes to the small blue dotted circle on the right. Dotted red ``Parallel'' on the left at positive height (in the North Hemisphere), goes to the big red dotted circle on the right. We recall that ``parallels'' are circles on the sphere that are parallel to the Equator.}
    \label{fig: h_phi}
\end{figure}

\begin{lemma}\label{lem: circles to circles}
    In 3D, the function $h_1\circ \phi:\mathbb{S}^2\setminus\{s_n\}\to\mathbb{R}^2$, where $\phi$ defined by \eqref{eq: stereographic proj without 2} and $h_1$ by  \eqref{eq: h_1} with $d=2$ has range contained in the circle of radius $\pi$ centered at the origin. Moreover,  $h_1\circ\phi$ maps circles parallel to the Equator to circles in the plane centered at the origin. 
    See Figure \ref{fig: h_phi} for a visualization.
    
    This can be generalized to the dimensions $d>2$, where we consider the function $h_1\circ \phi:\mathbb{S}^d\setminus\{s_n\}\to\mathbb{R}^d$, and in the preceding statement we have to change the word ``circle'' by ``sphere''. The Equator in $\mathbb{S}^{d}$ is the sub-sphere intersecting the hyperplane in $\mathbb{R}^{d+1}$ with equation $x_{d+1}=0$.
\end{lemma}
\begin{proof}
This is a consequence of part 1 in Proposition \ref{prop: h prop1}, that is,
$h_1(\phi(s))=\angle(s,s_0)\frac{s[1:d]}{\|s[1:d]\|}
$ for every $s\in\mathbb{S}^d\setminus\{s_n\}$, and the fact that the stereographic projection $\phi$ with formula given by \eqref{eq: stereographic proj without 2}
maps circles (resp., sub-spheres of dimension $d-1$ in $\mathbb{R}^{d+1}$) parallel to the Equator to circles (resp., spheres of dimension $d-1$) in $\mathbb{R}^2$ (resp., $\mathbb{R}^d$). See Figure \ref{fig: stereo3}.

Given a circle $S_\alpha$ (sub-sphere) which is the intersection between the sphere $\mathbb{S}^d$ and the (hyper)plane $x_{d+1}=\sin(\alpha)$ for some 
$\alpha\in[-\pi/2,\pi/2)$
(and so, $-1\leq \sin(\alpha)< 1$), i.e.,
\begin{equation}\label{eq: S_alpha}
  S_\alpha=\mathbb{S}^d\cap\left\{x\in\mathbb{R}^{d+1}:\ x_{d+1}=\sin(\alpha)\right\},  
\end{equation}
we have that
$$\angle(s,s_0)=\angle(s',s_0)=\alpha+\frac{\pi}{2} \qquad \forall s,s'\in S_\alpha.$$
Thus, if for example we consider a representative $s=[0,\dots,0,\cos(\alpha),\sin(\alpha)]$ in $S_\alpha$,
 we have that
$h_1(\phi(s))=\angle(s,s_0)=\alpha+\frac{\pi}{2}$.
In general, $$h_1(\phi(s'))=(\alpha+\frac{\pi}{2})\frac{s'[1:d]}{\|s'[1:d]\|} \qquad \forall s'\in S_\alpha.$$
Since the points of the form $\frac{s'[1:d]}{\|s'[1:d]\|}$ belong to the unit ``circle'' in $\mathbb{R}^d$,  the function $h_1\circ \phi$ maps the ``circle'' $S_\alpha$ (which is parallel to the Equator at height $\sin(\alpha)$) to the ``circle'' centered at the origin with radius $\alpha+\frac{\pi}{2}$:
$$h_1(\phi(S_{\alpha}))=\left\{x\in\mathbb{R}^d:\, \|x\|=\alpha+\frac{\pi}{2}\right\}.$$
In particular, we have the following:
 \begin{itemize}
\item The Equator corresponds to $\alpha=0$ and so 
$$h_1(\phi(\text{Equator}))=\left\{x\in \mathbb{R}^d: \, \sqrt{x_1^2+\dots +x_d^2}=\frac{\pi}{2}\right\}.$$
\item The South Pole corresponds to $\alpha=-\pi/2$, i.e., $s_0=[0,\dots,0,-1]=[0,\dots,\cos(-\pi/2),\sin(-\pi/2)]$, and so $h_1(\phi(s_0))=[0,\dots,0]\in\mathbb{R}^{d}$. 
\item For small numbers $\epsilon>0$ we have that the ``circle'' $S^\epsilon$ parallel to the Equator at height $\pi/2-\epsilon$ defined by  $$S^\epsilon=\mathbb{S}^d\cap\left\{x\in\mathbb{R}^{d+1}: \, x_{d+1}=\frac{\pi}{2}-\epsilon\right\}$$
is sent to 
$$h_1(\phi(S^\epsilon))=\left\{ x\in\mathbb{R}^d: \, \sqrt{x_1^2+\dots +x_d^2}=\pi-\epsilon\right\}.$$
Thus, as $\epsilon\to 0$ we obtain that the range of $h_1\circ \phi$ is contained on the open sphere/circle in $\mathbb{R}^{d}$ of radius $\pi$ centered at the origin.

\item ``Circles'' at positive height, i.e., $S_\alpha$ as in \eqref{eq: S_alpha} with $0\leq\alpha<\pi/2$, are sent through $h_1\circ\phi$ to spheres/circles in $\mathbb{R}^d$ with radius grater than $\pi/2$ and smaller that $\pi$ centered at the origin. ``Circles'' at at negative height, i.e., $S_\alpha$ as in \eqref{eq: S_alpha} with $-\pi/2\leq\alpha<0$, are sent through $h_1\circ\phi$ to spheres/circles in $\mathbb{R}^d$ with radius between $0$ and $\pi/2$ centered at the origin. 
 \end{itemize}
\end{proof}

\begin{corollary}\label{coro: pi 2pi}
Given two arbitrary points $s_1,s_2\in\mathbb{S}^2\setminus\{s_n\}$, we have the following bounds
\begin{equation*}
    d_{\mathbb{S}^d}(s_1,s_2)\leq \pi \qquad \text{ and } \qquad\|h_1(\phi(s_1))-h_1(\phi(s_2))\|\leq 2\pi.
\end{equation*}
\end{corollary}
\begin{proof}
    In general $d_{\mathbb{S}^d}(s_1,s_2)$ is the minimum arc length between $s_1$ and $s_2$ taken among all possible paths on the sphere connecting those points, and so it is at most $\pi$. We recall that $d_{\mathbb{S}^d}(s_1,s_2)=\arccos(\langle s_1,s_2\rangle)\in [0,\pi]$.

    The other inequality is a consequence of Lemma \ref{lem: circles to circles}: The higher Ecudidean distance between points in a circle (sphere) with diameter $2\pi$ (radius $\pi$) is, in fact, $2\pi$.
\end{proof}

\bigskip

The above results (identity \eqref{eq: min on restricted set of rot}, Lemma \ref{lem: same s_d+1} and Corollary \ref{coro: pi 2pi}) suggest that given two points $s_1,s_2$ on the sphere we would have that $d_{\mathbb{S}^d}(s_1,s_2)\leq\|h_1(\phi(s_1))-h_1(\phi(s_2))\|$. This is shown experimentally in Figure \ref{fig: ds dh}.  We will need some more lemmas to then sketch the proof of this fact theoretically. 
\begin{figure}[H]
    \centering
    \includegraphics[width=0.6\linewidth]{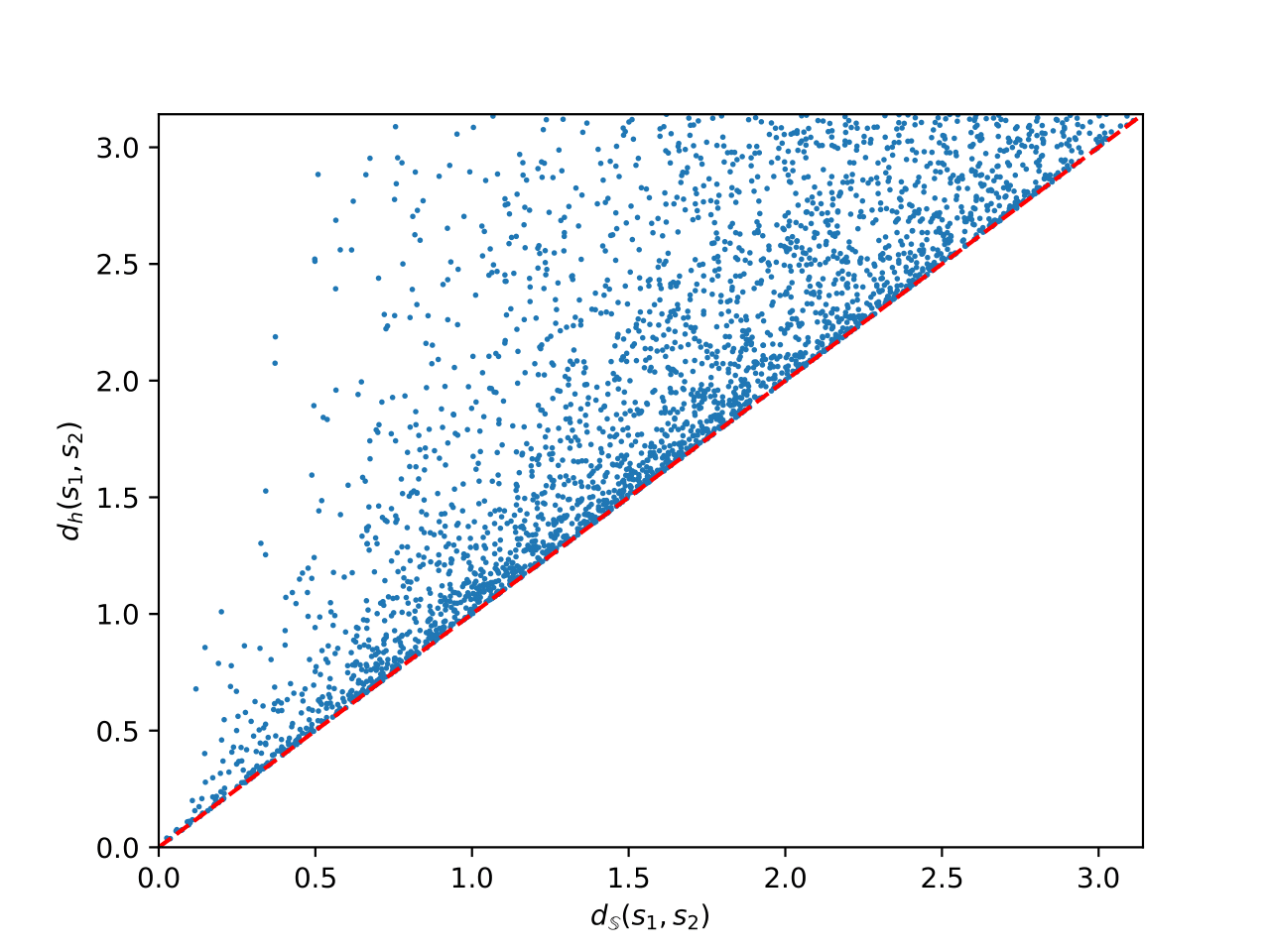}
    \caption{Experimental results for the relation $d_{\mathbb{S}^2}(s_1,s_2)\leq d_h(s_1,s_2):=\min\{\|h_1(\phi(s_1))-h_1(\phi(s_2))\|, 2\pi-\|h_1(\phi(s_1))-h_1(\phi(s_2))\|\}$. There were taken 
    $1000$ pairs $(s_1,s_2)$ uniformly at random on the unit sphere in $\mathbb{R}^3$.
    Additionally, we calculate the correlation between $d_{S^2}(s_1,s_2)$ and $d_h(s_1,s_2)$, resulting in an average correlation coefficient of $0.83$.
    }
    \label{fig: ds dh}
\end{figure}

\begin{lemma}\label{lem: Ch}
Consider $s_1,s_2\in \mathbb{S}^d$, with $(s_1)_{d+1}\leq (s_2)_{d+1}$, let $s_2'$ be the unique point such that 
\begin{equation}\label{eq: s_2'}
  (s_2')_{d+1}=(s_1)_{d+1} \qquad \text{ and } \qquad\angle(s_2'[1:d],s_2[1:d])=0.  
\end{equation}
Let $\beta= \angle(s_1[1:d], s_2[1:d])\in[0,\pi]$, and 
let $C_H$ denote the angle between line segments $[h_1(\phi(s_1)),h_1(\phi(s_2'))],[h_1(\phi(s_2')),h_1(\phi(s_2))]$.
Then, 
\begin{equation}\label{eq: Ch}
  C_H=\frac{\beta}{2}+\frac{\pi}{2}  
\end{equation}
\end{lemma}
\begin{proof}
We suggest the reader to look at Figure \ref{fig: stereo for proof2} for visualization purposes. 

Consider $C_H$ the angle between line segments $[h_1(\phi(s_1)),h_1(\phi(s_2'))]$ and $[h_1(\phi(s_2)),h_1(\phi(s_2'))]$ in the projected space. The way we have chosen $s_2'$ guarantees that the segment $[0, h_1(\phi(s_2'))]$ is contained in the segment $[0, h_1(\phi(s_2))]$. 

As an application of Lemma \eqref{lem: circles to circles}, we have that the segments  $[0, h_1(\phi(s_2'))]$ and  $[0, h_1(\phi(s_1))]$ have the same length as $s_2'$ and $s_1$ are at the same latitude. Then, the triangle in the projected space with vertices the origin or coordinates $O=[0,\dots,0]\in\mathbb{R}^d$, $s_1$ and $s_2'$ is isosceles. The angle of that triangle having vertex at $O$ is
\begin{align}\label{eq: beta}
 \beta=\angle(s_1[1:d], s_2[1:d])&=\angle(s_1[1:d], s_2[1:d]) =\angle(s_1[1:d], s_2'[1:d])\notag\\
 &=\angle(\phi(s_1), \phi(s_2'))=\angle(h_1(\phi(s_1)), h_1(\phi(s_2')))
\end{align}
and the other two angles are equal. Thus, since $C_H$ is adjacent to one of those two equal angles, and since the sum of the internal angles of a triangle is $180°$, we have $\pi=\beta+2(\pi-C_H)$ or, equivalently, $C_H=\frac{\beta}{2}+\frac{\pi}{2}$.
\end{proof}

\begin{lemma}\label{lem: angle bound}
Consider $s_1,s_2\in \mathbb{S}^d$, with $(s_1)_{d+1}\leq (s_2)_{d+1}$, let $s_2'$ be the unique point satisfying \eqref{eq: s_2'}. Let $C\in[0,\pi]$ denote the spherical angle between   $\text{arc}(s_1,s_2'),\text{arc}(s_2',s_2)$, and $C_H$ denote the angle between line segments $[h(\phi(s_1)),h(\phi(s_2')],[h_1(\phi(s_2')),h_1(\phi(s_2))]$.
Then, 
\begin{equation}\label{eq: c < ch}
  C\leq C_H.  
\end{equation}
\end{lemma}
\begin{proof}
%We start from a simple version of the proof for the case $d=2$. Then we discuss the general case. 
%In the rest of the analysis, %we will assume $s_1,s_2$ in the positive hemisphere (the case where both are on the negative hemisphere will be analogous). Also, 
For better visualization and an easier way of parametrizing the points, we will work on the 3-dimensional space (see Figure \ref{fig: angle_sphere_+}).
%, i.e., $\mathbb{S}^2\subset\mathbb{R}^3$, and $h_1\circ\phi:\mathbb{S}^2\setminus\{s_n\}\to\mathbb{R}^2$.
%The case that remains to consider is that of when $s_2$ is in the positive hemisphere (i.e., $(s_2)_{d+1}\geq 0$) and $s_1$ is in the negative hemisphere (i.e., $(s_1)_{d+1}< 0$).
Thus, assuming $d=2$, by using spherical coordinates let $\alpha\in[-\pi/2,\pi/2]$, $\beta= \angle(s_1[1:d], s_2[1:d])\in[0,\pi]$, and $\alpha'\in[\alpha,\frac{\pi}{2})$ and, without loss of generality, let us parametrize
\begin{equation}\label{eq: parametriz s1 s2 s2'}
    \begin{cases}
  &s_1=[\cos(\alpha),0,\sin(\alpha)]\\
  &s_2=[\cos(\alpha')\cos 2\beta,\cos(\alpha')\sin(\beta),\sin(\alpha)]\\
  &s_2'=[\cos\alpha\cos 2\beta,\cos(\alpha)\sin(\beta),\sin(\alpha)]   
    \end{cases}
\end{equation}
% If $d=2$, without loss of generality, we suppose $s_1=[\cos\alpha,0, 0, \sin\alpha]^T$ where $\alpha\in[-\frac{\pi}{2},\frac{\pi}{2})$. 
% We set $\beta=\frac{1}{2}\angle(s_1[1:d],s_2[1:d])$, thus $\beta\in[0,\frac{\pi}{2}]$ and $s_2,s_2'$ can be parametrized as follows: 
% \begin{align}
% s_2'&=[\cos\alpha\cos 2\beta,\cos\alpha\sin2\beta,\sin\alpha]\nonumber\\
% s_2&=[\cos\alpha'\cos 2\beta,\cos\alpha'\sin2\beta,\sin\alpha]
% \end{align}

If $\beta=0$, $s_2'=s_1$, $C=C_H=0$ and there is nothing else to prove.

If $\beta>0$, the angle $C$ is characterized by
$$\cos(C)=\frac{\langle n_{12'}, n_{22}\rangle}{\|n_{12'}\|\|n_{22'}\|},$$
where $n_{12'}$ is a vector normal (perpendicular) to the plane containing the vectors $\vv{O s_1}$ and $\vv{O s_2'}$, and $n_{22'}$ is a vector normal to the plane containing the vectors $\vv{O s_2}$ and $\vv{O s_2'}$, where $O$ denotes the origin of coordinates $O=[0,\dots,0]\in\mathbb{R}^{d+1}$. 

Since $\arccos(\cdot)$ is a decreasing function, want to show that 
\begin{equation}\label{eq: cos C > cos Ch}
    \cos(C)\geq\cos(C_H).
\end{equation}
By using the parametrizations in \eqref{eq: parametriz s1 s2 s2'}, we can compute 
\begin{equation*}
    \begin{cases}
        n_{22'}=[\sin(\beta),-\cos(\beta),0] &\qquad \|n_{22'}\|=1\\
        n_{12'}=[\sin(\alpha)\sin(\beta),\sin(\alpha)(1-\cos(\beta)),\cos(\alpha)\sin(\beta)] & \qquad \|n_{12'}\|=\sqrt{\sin^2(\beta)+\sin^2(\alpha)(1-\cos(\beta))^2}\\
        \langle n_{12'},n_{22'}\rangle=\sin(\alpha)(1-\cos(\beta))
    \end{cases}
\end{equation*}
Thus,
\begin{equation}\label{eq: cos C}
    \cos(C)=\frac{\sin(\alpha)(1-\cos(\beta))}{\sqrt{\sin^2(\beta)+\sin^2(\alpha)(1-\cos(\beta))^2}}
\end{equation}
Notice that in the special case where $\alpha\in[0,\pi/2]$, that is,  $(s_1)_{d+1}\geq 0$, and so both $s_1,s_2$ are in the same ``positive hemisphere'' (as in Figure \ref{fig: angle_sphere_+}), we have that $\sin(\alpha)\geq 0$, so, 
$\cos(C)\geq 0$, and therefore
$0\leq C\leq \pi/2$. Thus, in this particular case, 
\begin{equation}\label{eq: c<pi/2<ch}
 C\leq \pi/2\leq C_H \qquad \text{if } 0\leq\alpha< \frac{\pi}{2}.   
\end{equation}
(See Figure \ref{fig: angle_sphere_+} for a visualization of this case.)

In general, by using \eqref{eq: cos C} and \eqref{eq: Ch}, the inequality \eqref{eq: cos C > cos Ch} reads as
\begin{equation*}
    \cos(C)=\frac{\sin(\alpha)(1-\cos(\beta))}{\sqrt{\sin^2(\beta)+\sin^2(\alpha)(1-\cos(\beta))^2}}\geq \cos\left(\frac{\beta}{2}+\frac{\pi}{2}\right)=-\sin\left(\frac{\beta}{2}\right)
\end{equation*}
By algebraic manipulations and by using trigonometric identities \eqref{eq: cos C > cos Ch} holds if and only if 
\begin{equation}\label{eq: positive for c ch}
    \sin\left(\frac{\beta}{2}\right)\left(1+\frac{\sin(\alpha)\sin\left(\frac{\beta}{2}\right)}{\sqrt{\cos^2\left(\frac{\beta}{2}\right)+\sin^2(\alpha)\sin^2\left(\frac{\beta}{2}\right)}}\right)\geq 0 \qquad \text{if } 0\leq \beta\leq \pi, \, \frac{-\pi}{2}\leq \alpha\leq \frac{\pi}{2}
\end{equation}
%In fact, \eqref{eq: positive for c ch} is satisfied: See Figure \ref{fig: c ch}. For an analytical proof, 
Let us rewrite the inequality in \eqref{eq: positive for c ch} (by using that $\sin(\beta/2)\geq 0$ in our range) as
\begin{equation}\label{eq: C analytical proof}
    -\sin(\alpha)\leq \sqrt{\cos^2\left(\frac{\beta}{2}\right)+\sin^2(\alpha)\sin^2\left(\frac{\beta}{2}\right)} 
\end{equation}
The parameter $\alpha$ is constraint to the interval $[-\pi/2,\pi/2)$. We have already proved that for $\alpha\geq 0$ everything works (see \eqref{eq: c<pi/2<ch}). Thus, let us suppose that $\alpha<0$: In this case, $\sin(\alpha)<0$ and so, both sides in \eqref{eq: C analytical proof} are positive. Therefore, it is equivalent to show that
$ \sin^2(\alpha)\leq \cos^2\left(\frac{\beta}{2}\right)+\sin^2(\alpha)\sin^2\left(\frac{\beta}{2}\right) 
$
%or, equivalently, that $    \sin^2(\alpha)\cos^2\left(\frac{\beta}{2}\right)\leq \cos^2\left(\frac{\beta}{2}\right)$, 
or, equivalently, that  
%\begin{equation*}
$\sin^2(\alpha)\leq 1$,
%\end{equation*}
which is always true.

\end{proof}

    \begin{figure}[H]
    \centering
    \includegraphics[width=0.7\linewidth]{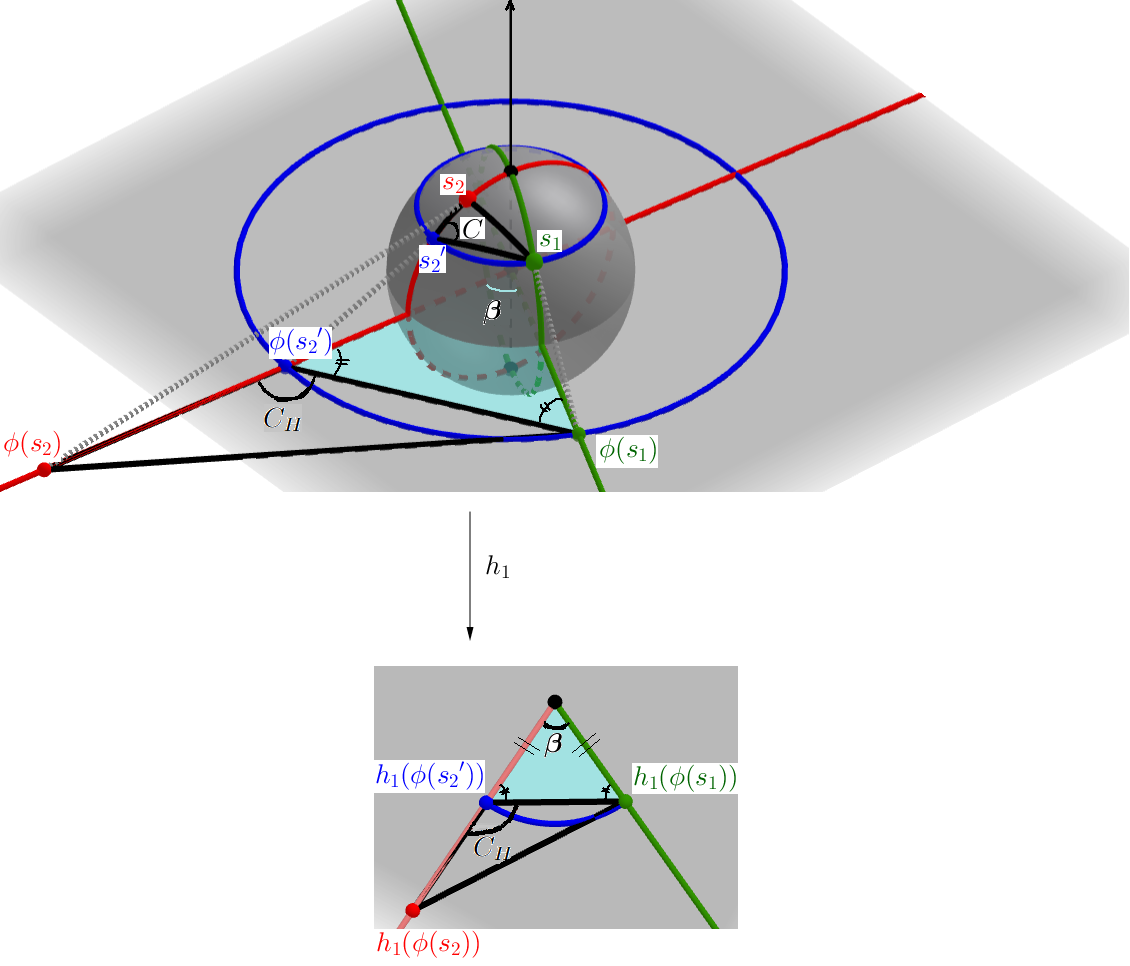}
    \caption{On top we have the sphere in the 3-dimensional space. Three points $s_1$, $s_2$, $s_2'$ are considered on the sphere. The point $s_2'$ satisfies \eqref{eq: s_2'}, that is, $s_2',s_1$ are at the same latitude, and $s_2',s_2$ are in the same meridian (and among the points that satisfy these two conditions always choose the $s_2'$ which is ``closer'' to $s_2$). The corresponding projections $\phi(s_1)$, $\phi(s_2)$, $\phi(s_2')$ are plotted on the 2-dimensional plane. The circle on the sphere parallel to the Equator that passes through $s_1$ and $s_2'$ and its projected circle by $\phi$ are depicted in blue. The circle on the sphere that passes through $s_2$, $s_2'$, and the poles (meridian) and its projection by $\phi$ (which is a line) are depicted in red. Similarly, the circle on the sphere that passes through $s_1$, and the poles (meridian) and its projection by $\phi$ (which is a line) are depicted in green. On the bottom, we sketch the image by $h_1$ of the projected space.
    The angles $\beta$ given by \eqref{eq: beta}, $C$ (spherical angle between the arcs $\text{arc}(s_1,s_2'),\text{arc}(s_2',s_2)$), and $C_H$ (defined by the line segments 
    $[\phi(s_1),\phi(s_2')]$ and $[\phi(s_2),\phi(s_2')]$ or, 
    equivalently, by the segments $[h_1(\phi(s_1)),h_1(\phi(s_2'))]$ and $[h_1(\phi(s_2)),h_1(\phi(s_2'))]$) are depicted. When applying the function $h_1$ after $\phi$, it preserves angles with vertex at the origin in the projected space, preserves the lines that pass through the origin (red and green lines), maps circles centered at the origin to circles centered at the origin (blue arc). Roughly speaking, the effect of $h_1$ is that it ``compresses'' the projected plane.}
    \label{fig: stereo for proof2}
\end{figure}

\begin{figure}[H]
    \centering
    \includegraphics[width=0.45\linewidth]{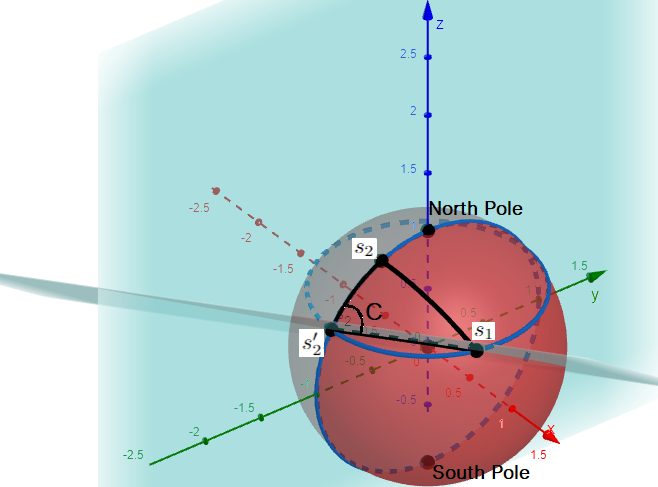}
    \caption{Given two arbitrary points $s_1,s_2\in\mathbb{S}^2\subset\mathbb{R}^3$ in the ``positive'' hemisphere. The point $s_1'$ is plotted such that it has the same latitude as $s_1$ and it is in the same meridian as $s_2$.  The angle $C$ is depicted as the angle between the two planes containing the arcs $arc (s_1, s_2')$ and $arc (s_2, s_2')$. A spherical triangle connecting the points $s_1,s_2',s_2$ is depicted in black.}
    \label{fig: angle_sphere_+}
\end{figure}

\begin{proposition}
\label{pro: h1 upper_bound}
Let $s, s_1,s_2\in \mathbb{S}^d\setminus\{s_n\}$, and let $s_0=[0,0,\ldots,-1]^T\in \mathbb{S}^d$ denote the South Pole. Consider $\phi:\mathbb{S}^{d}\setminus\{s_n\}\to\mathbb{R}^d$ the stereographic projection as in \eqref{eq: stereographic proj without 2} with inverse \eqref{eq: inv stereo}, and also consider the function $h_1:\mathbb{R}^d\to\mathbb{R}^d$ defined by \eqref{eq: special h_1}. Then
    \begin{align}
    d_{\mathbb{S}^d}(s_1,s_2)
    &\leq \min_{R\in\mathrm{O}(d+1)}\|h_1(\phi(R s_1))-h_1(\phi(R s_2))\|+\epsilon(s_1,s_2)\nonumber
    \end{align}
    where $\epsilon(s_1,s_2)\to0$ as $d_{\mathbb{S}^d}(s_1,s_2)\to 0$.
\end{proposition}

 \begin{proof}
 %[Sketch proof of the other direction of Proposition \ref{prop: ds and dh}: $d_{\mathbb{S}^d}(s_1,s_2)\leq \min_{R\in \mathrm{O}(d+1)}\|h_1(\phi(R s_1))-h_1(\phi(R s_2))\|$.]

Let $s_1$, $s_2$ in $\mathbb{S}^d$. Assume, without loss of generality, that
with $(s_2)_{d+1}\geq (s_1)_{d+1}$.

%If $(s_1)_{d+1}\geq 0$ (i.e., both $s_1$ and $s_2$  are on the same positive ``hemisphere''), 
Define $s_2'$ as in Lemma \ref{lem: angle bound}: That is, such that $(s_2')_{d+1}=(s_1)_{d+1}$ (same latitude) and $s_2'$, $s_2$, $s_0$ are in the same great circle (same meridian).
Since there are two possible choices for $s_2'$, choose it not only such that $(s_2')_{d+1}=(s_1)_{d+1}$ but also such that
$\angle(s_2'[1:d],s_2[1:d])=0$. (This means we always choose the $s_2'$ which is ``closer'' to $s_2$.)

Let $C$ be the \textit{spherical angle} angle between the arcs $arc (s_1, s_2')$ and $arc (s_2, s_2')$ given by the smallest portions of the great circles connecting the corresponding points $s_1$ with $s_2'$ and $s_2$ with $s_2'$, with vertex the origin of coordinates $[0,\dots,0]\in\mathbb{R}^{d+1}$. We recall that the angle $C$ is measured by the angle between the planes containing those two arcs. 

See Figure \ref{fig: angle_sphere_+} for a visualization in $\mathbb{R}^3$.

%Analogously, if $(s_2)_{d+1}\leq 0$ (i.e., both $s_1$ and $s_2$  are on the same negative ``hemisphere''), 
%define $s_1'$ satisfying $(s_1')_{d+1}=(s_2)_{d+1}$ and which is closer to $s_1$  (same latitude of $s_2$) and such that $s_1'$, $s_1$, $s_0$ are in the same great circle (same meridian of $s_1$); and let
%$\tilde C$ be the \textit{spherical angle} angle between the arcs $arc (s_2, O, s_1')$ and $arc (s_1,O, s_1')$.

 A \textit{spherical triangle} on the surface of the unit sphere is defined by the great circles connecting three points on the sphere, in our case consider $s_1,s_2,s_2'$.  
The \textit{spherical law of cosines} reads as follows:
$$\cos(d_{\mathbb{S}^d}(s_1,s_2))=\cos(d_{\mathbb{S}^d}(s_1,s_2'))\cos(d_{\mathbb{S}^d}(s_2,s_2'))+\sin(d_{\mathbb{S}^d}(s_1,s_2'))\sin(d_{\mathbb{S}^d}(s_2,s_2'))\cos(C).$$
As a consequence, by using Taylor expansions, 
\begin{align*}
  d_{\mathbb{S}^d}(s_1,s_2)^2&= d_{\mathbb{S}^d}(s_1,s_2')^2+d_{\mathbb{S}^d}(s_2,s_2')^2-2d_{\mathbb{S}^d}(s_1,s_2')d_{\mathbb{S}^d}(s_2,s_2')\cos( C )\\
  &\qquad +\mathcal{O}(d_{\mathbb{S}^d}^4(s_1,s_2))+\mathcal{O}(d_{\mathbb{S}^d}^4(s_1,s_2'))+\mathcal{O}(d_{\mathbb{S}^d}^4(s_2,s_2'))  
\end{align*}  
Thus, for ``small'' spherical triangles, i.e., when the lengths $d_{\mathbb{S}^d}(s_1,s_2)$, $d_{\mathbb{S}^d}(s_1,s_2')$, $d_{\mathbb{S}^d}(s_2,s_2')$ are small, we can use the ``planar'' approximation 
$$d_{\mathbb{S}^d}(s_1,s_2)^2\approx d_{\mathbb{S}^d}(s_1,s_2')^2+d_{\mathbb{S}^d}(s_2,s_2')^2-2 d_{\mathbb{S}^d}(s_1,s_2')d_{\mathbb{S}^d}(s_2,s_2')\cos (C).$$

Now, as in the above lemmas, consider $C_H$ the angle between line segments $[h_1(\phi(s_1)),h_1(\phi(s_2'))]$ and $[h(\phi(s_2)),h(\phi(s_2'))]$ in the projected space.

%The case that remains to consider is that of when $s_2$ is in the positive hemisphere (i.e., $(s_2)_{d+1}\geq 0$) and $s_1$ is in the negative hemisphere (i.e., $(s_1)_{d+1}< 0$).

 By Lemma \ref{lem: same s_d+1} and by using $C\leq C_H$ together with fact that $\cos$ is a decreasing function on $[0,\pi]$, we have:
\begin{align*}
  &d_{\mathbb{S}^d}^2(s_1,s_2')+d_{\mathbb{S}^d}^2(s_2,s_2')-2 d_{\mathbb{S}^d}(s_1,s_2')d_{\mathbb{S}^d}(s_2,s_2')\cos(C) \\
  &\leq \|h_1(\phi(s_1))-h_1(\phi(s_2'))\|^2+\|h_1(\phi(s_2))-h_1(\phi(s_2'))\|^2\\
  &\quad-2\|h_1(\phi(s_1))-h_1(\phi(s_2'))\|\|h_1(\phi(s_2))-h_1(\phi(s_2'))\|\cos(C_H)\\
  &\quad \qquad=\|h_1(\phi(s_1))-h_1(\phi(s_2))\|^2
\end{align*}
where in the last step we used the classical law of cosines.

Therefore, by using the approximation given by the spherical law of cosines (assuming $d_{\mathbb{S}^d}(s_1,s_2)$ small), we obtain
$$d_{\mathbb{S}^d}(s_1,s_2)\leq \|h_1(\phi(s_1))-h_1(\phi(s_2))\|.$$
In particular,
$$d_{\mathbb{S}^d}(s_1,s_2)=d_{\mathbb{S}^d}(Rs_1,Rs_2)\leq \|h_1(\phi(Rs_1))-h_1(\phi(Rs_2))\| \qquad \forall R\in\mathrm{O}(d+1)$$
and so
$$d_{\mathbb{S}^d}(s_1,s_2)\leq\min_{R\in \mathrm{O}(d+1)}\|h_1(\phi(Rs_1))-h_1(\phi(Rs_2))\|.$$

%Moreover, by Lemma \ref{lem: same s_d+1} we actually have
%\begin{align*}
%d_{\mathbb{S}^d}(s_1,s_2')\leq\min(2\pi-\|h(\phi(s_1))-h_1(\phi(s_2'))\|,\|h_1(\phi(s_1))-h_1(\phi(s_2'))\|).    
%\end{align*}
%and 
%\begin{align*}
%d_{\mathbb{S}^d}(s_2,s_2')\leq\min(2\pi-\|h(\phi(s_2))-h_1(\phi(s_2'))\|,\|h_1(\phi(s_2))-h_1(\phi(s_2'))\|).    
%\end{align*}
%So, by using the same arguments as before, we obtain
%$$d_{\mathbb{S}^d}(s_1,s_2)\leq \min\{\|h_1(\phi(R s_1))-h_1(\phi(R s_2))\|,2\pi-\|h_1(\phi(R s_1))-h_1(\phi(R s_2))\| \},$$
%and, by invariance,
%$$d_{\mathbb{S}^d}(s_1,s_2)\leq\min_{R\in \mathrm{O}(d+1)}\min\{\|h_1(\phi(R s_1))-h_1(\phi(R s_2))\|,2\pi-\|h_1(\phi(R s_1))-h_1(\phi(R s_2))\| \}.$$
\end{proof}

% \section{Rotationally Invariant R3W}
% \section{Relation between SSRT and Funk transform}

% \begin{align}
% \mathcal{F}f(t,\theta)=\int_{\langle\theta,x\rangle=t} f(x)dx
% \end{align}

\subsection{Neural Network-Based Embedding} \label{sec:h_nn_appendix}

While any choice for $h$ which maintains injectivity will ensure $S3W_{\mathcal{H},p}$ remains a valid metric in $\mathcal{P}_p(\mathbb{S}^d)$, the particular choice for $h$ is significant in ensuring that the transportation cost in the embedding space resembles the spherical distance. As discussed in Section \ref{sec:dist_distortion}, as an alternative to the proposed analytic function $h_1(\cdot)$ (as given in Eq. \eqref{eq: h_1}), we may consider training a neural network to obtain a nearly-isometric Euclidean embedding and correct for the distance distortion caused by stereographic projection. 

Here, we consider the injective function defined by $h_{NN}(x):=[h^T_1(x)/C,\rho^T(x)]^T$ where $\rho:\mathbb{R}^{2}\to\mathbb{R}^{3}$ is a neural network, and $C\geq 2\pi$ a constant. We parameterize $\rho$ using a multi-layer perceptron (MLP) with two hidden layers, each consisting of 128 neurons, and train $\rho$ by minimizing
\begin{align}
    \mathcal{L}(\rho)=\mathbb{E}_{s,s'} \left[(\arccos(\langle s, s'\rangle) - \|h(\phi(s))-h(\phi(s'))\|)^2\right],
    \label{eq: loss_rho_2}
\end{align}
where $s$ and $s'$ are sampled according to the uniform distribution in the sphere $\mathbb{S}^d\subset\mathbb{R}^{d+1}$, i.e., $(s,s')\sim \sigma_{d+1}\times \sigma_{d+1}$. As shown in Figure \ref{fig:distortion}, this choice for $h$ can yield a nearly-isometric embedding. 

We then conduct two different gradient flow experiments to study the distinct behaviors of $h_1(\cdot)$ and $h_{NN}(\cdot)$. The visualizations demonstrating these gradient flows can be found at the following \href{https://github.com/mint-vu/s3wd/blob/c701ac49d63bd6b524da4f0d5681feab48b338ce/src/demos/gradient_flows/README.md}{GitHub URL}. We observe that the use of $h_{NN}(\cdot)$ provides more ``geodesic-like'' paths for the particles as compared to $h_1(\cdot)$, though the final outcomes are comparable.

We note that in our experiments, we primarily use only the analytic $h_1$ for its computational efficiency as opposed to the neural network-based learnable function $h_{NN}$. Specifically, in our experiments, the proposed distances are used in gradient-based optimization frameworks where the distance is computed in each iteration of (stochastic) gradient descent. Utilizing $h_{NN}$ in these optimization applications increases the computational cost of calculating the distance. More precisely, the computational cost would have the following components: 1) stereographic projection, 2) evaluation of $h_{NN}$, and 3) slicing, with a neural network becoming the computational bottleneck in the evaluation of $h_{NN}$.

\section{Numerical Experiments}
\label{sec:numerical}
In this section, we aim to demonstrate our proposed $S3W$ as an efficient metric suitable for various Machine Learning tasks. In \ref{section:gf}, we show that $S3W$ and its variants provide a high-speed alternative for $SSW$ as an effective loss for gradient-based optimization on the spherical manifold. In \ref{section:evolution}, we study the evolution of our proposed distances w.r.t. varying parameters. In \ref{section:runtime}, we discuss the computational efficiency of our method via runtime study. \ref{section:de}, \ref{section:ssl}, \ref{section:vi}, \ref{section:swae} provide experiments in practical ML settings. Throughout our experiments, we use $h(x) = h_1(x):= \arccos\left(\frac{\|x\|^2-1}{\|x\|^2+1}\right)\frac{x}{\|x\|}.$

\subsection{Additional Algorithmic and Implementation Details}
Let $\{ \alpha_i, x_i\}_{i=1}^M \sim \hat{\mu}$ and $\{\beta_j, y_j\}_{j=1}^N \sim \hat{\nu}$ be iid samples from the empirical measures $\hat{\mu}$ and $\hat{\nu}$, respectively, where $\sum_{i=1}^M \alpha_i = 1$ and $\sum_{j=1}^N \beta_j = 1$. We provide below a general formulation for computing $S3W_p$ between $\hat{\mu}$ and $\hat{\nu}$. We then define the procedure to compute $RI\text{-}S3W_p$ and $ARI\text{-}S3W_p$. By abuse of notation, we denote these Algorithms as $S3W$, $RI\text{-}S3W$, and $ARI\text{-}S3W$, respectively.

\clearpage
\subsubsection{Alternative Implementation to Compute $S3W_p$}

\begin{algorithm}[ht]
   \caption{Stereographic Spherical Sliced-Wasserstein ($S3W$)}
   \label{alg:s3w_alt}
    \begin{algorithmic}
       \STATE {\bfseries Input:} $\{\alpha_i, x_i\}_{i=1}^M \sim \hat{\mu}$, $\{\beta_j, y_j\}_{j=1}^N \sim \hat{\nu}$, $L$ projections, $p$-th order, $\epsilon$
       \STATE Initialize: $h$ (injective map), $\{\theta_l\}_{l=1}^L$ (projection directions)
       \STATE Initialize $d = 0$
       \STATE Calculate $\{u_i = h(\phi_{\epsilon}(x_i))\}_{i=1}^M$ and $\{v_j = h(\phi_{\epsilon}(y_j))\}_{j=1}^N$
       \FOR{$\ell=1$ {\bfseries to} $L$}
       \STATE Project onto $\theta_l$: $u_i^l = \langle u_i, \theta_l \rangle$ and $v_j^l = \langle v_j, \theta_l \rangle$
        \STATE Sort $\{u^l_i\}$, $\{v_j^l\}$, s.t  $u^l_{\pi_l[i]} \leq u^l_{\pi_l[i+1]}$, $v^l_{\pi_l'[j]} \leq v^l_{\pi_l'[j+1]}$
       % \STATE Compute the (weighted) CDFs $\{\tilde{u}_i = F_{u^l}(u_{\pi_l[i]}^l) = \sum_{k=1}^i\alpha_{\pi_{l[k]}}\}_{i=1}^M $, $\{\tilde{v}_j = F_{v^l}(v_{\pi'_l[j]}^l) = \sum_{k=1}^j  \beta_{\pi'_{l[k]}}\}_{j=1}^N$
       \STATE Compute the (weighted) CDFs $\{\tilde{u}^l_i = \tilde{F}_{u^l}(u_{\pi_l[i]}^l)\}_{i=1}^M $, $\{\tilde{v}^l_j = \tilde{F}_{v^l}(v_{\pi_l[j]}^l)\}_{j=1}^N$
       
       \STATE Merge and sort the $\{\tilde{u}^l_i\}_i$ and $\{\tilde{v}^l_j\}_j$ into $\{z_k^l\}_{k=1}^{M+N}$ where $z^l_k \leq z^l_{k+1}$
        \STATE For each $z_k^l$, compute $\tilde{F}^{-1}_{u^l}(z_k^l)$ and $\tilde{F}^{-1}_{v^l}(z_k^l)$  by setting $\tilde{F}^{-1}_{u^l}(z_k^l)$ to the value of $u_i^l$ such that $\tilde{F}_{u^l}(u_i^l) $ is nearest to $z_k^l$. Proceed similarly for $\tilde{F}^{-1}_{v^l}(z_k^l)$.
        % \STATE \hspace*{2em} If $z_k^l < \min(\{u_i^l\})$, set $\tilde{F}^{-1}_{u^l}(z_k^l) = 0$ and similarly for $\tilde{F}^{-1}_{v^l}(z_k^l)$.
        % \STATE \hspace*{2em} If $z_k^l > \max(\{u_i^l\})$, set $\tilde{F}_{u^l}(z_k^l) = 1$ and similarly for $\tilde{F}_{v^l}(z_k^l)$.
        % \STATE \hspace*{2em}
        
        % \STATE $\tilde{F}^{-1}_{u^l}(z) = \text{Clamp}(z, \text{min}(u^l), \text{max}(u^l)), \ \tilde{F}^{-1}_{v^l}(z) = \text{Clamp}(z, \text{min}(v^l), \text{max}(v^l))$
        
       % \STATE Compute empirical CDF $F_{\tilde{X}_\ell}(t) = \frac{1}{M}\sum_{m=1}^{M} \mathbf{1}_{\tilde{x}_{i_\ell[m]} \leq t}$ and $F_{\tilde{Y}_\ell}(t) = \frac{1}{N}\sum_{n=1}^{N} \mathbf{1}_{\tilde{y}_{j_\ell[n]} \leq t}$
       \STATE $d_\ell = \sum_{k=1}^{M+N} \left| \tilde{F}^{-1}_{u^l}(z^l_k) - \tilde{F}^{-1}_{v^l}(z^l_k) \right|^p(z_k^l - z_{k-1}^l)$
       %where $z_k \in \{\tilde{x}_{i_\ell[\cdot]}\} \cup \{\tilde{y}_{j_\ell[\cdot]}\}$
       \STATE Update total distance: $d = d + \frac{1}{L}d_\ell$
       \ENDFOR
       \STATE Return $\text{S3W}_p(\hat{\mu}, \hat{\nu}) \approx d^{\frac{1}{p}}$
    \end{algorithmic}
\end{algorithm}

\subsubsection{Implementation to Compute the $RI\text{-}S3W_p$}

\begin{algorithm}[ht]
   \caption{Rotationally Invariant $S3W$ ($RI\text{-}S3W$)}
   \label{alg:ris3w}
    \begin{algorithmic}
       \STATE {\bfseries Input:} $\{\alpha_i, x_i\}_{i=1}^M \sim \hat{\mu}$, $\{\beta_j, y_j\}_{j=1}^N \sim \hat{\nu}$, $L$ projections, $p$-th order, $\epsilon$, $N_R$ rotations
       \STATE Sample $\{R_r\}_{r=1}^{N_R} \in \mathrm{SO}(d+1)$
       \STATE Initialize $d = 0$
       \FOR{$r=1$ {\bfseries to} $N_R$}
         \STATE Apply rotation $R_r$ to obtain $X_r =  \{\alpha_i R_r(x_i)\}_{i=1}^M$ and $Y_r =  \{\beta_jR_r(y_j)\}_{j=1}^N$ 
         \STATE Use Algorithm \ref{alg:s3w_alt} for inputs $X_r, Y_r$  to compute $d_r$
         \STATE Update total distance: $d = d + \frac{1}{N_R}d_r$
       \ENDFOR
       \STATE Return $RI\text{-}S3W_p(\hat{\mu}, \hat{\nu}) \approx d$
    \end{algorithmic}
\end{algorithm}

\subsubsection{Implementation to Compute the $ARI\text{-}S3W_p$}

Although $RI\text{-}S3W$ is highly parallelizable via batch matrix multiplication, generating random rotation matrices can still be computationally expensive. That is especially true for high-dimensional data and training setups with small batch size where $RI\text{-}S3W$ gets called repeatedly. One way to mitigate this issue is by pregenerating a batch of $N$ rotation matrices, $\{R_i\}_{i=1}^N \subset \mathrm{SO}(d+1)$, and subsample from this batch to amortize the generation cost. We introduce the Amortized Rotationally Invariant $S3W$ ($ARI\text{-}S3W$) as algorithm \ref{alg:aris3w}. This process can be made seamless by managing the rotation matrices in a stateful manner. That is, we maintain a Rotation manager class with static variables to ensure that once a set of matrices is generated, they can be efficiently reused. We describe the abstract procedure to compute $ARI\text{-}S3W$ in Algorithm \ref{alg:aris3w}.

\begin{algorithm}[t]
\caption{Amortized Rotationally Invariant $S3W$ ($ARI\text{-}S3W$)}
\label{alg:aris3w}
\begin{algorithmic}
\STATE {\bfseries Input:} $\{\alpha_i, x_i\}_{i=1}^M \sim \hat{\mu}$, $\{\beta_j, y_j\}_{j=1}^N \sim \hat{\nu}$, $L$ projections, $p$-th order, pool size $N_{\text{total}}$, $N_R$ rotations, pregenerated rotations $\mathcal{R} = \{R_k\}_{k=1}^{N_{total}} \subset \mathrm{SO}(d+1)$ where $N_{total} \geq N_R$
\STATE Initialize $d = 0$
\STATE Randomly subsample $N_R$ rotations $\tilde{\mathcal{R}} = \{\tilde{R}_{r} \sim \mathcal{R}\}_{r=1}^{N_R}$ 
\FOR{each r in $N_R$}
\STATE Apply $\tilde{R}_r$ to obtain $X_r = \{\alpha_i \tilde{R}_r(x_i)\}_{i=1}^M$ and $Y_r =  \{\beta_j\tilde{R}_r(y_j)\}_{j=1}^N$
\STATE Use Algorithm \ref{alg:s3w_alt} where inputs are $\{\tilde{R}_r(x_i^{r})\}_{i=1}^N$ and $\{\tilde{R}_r(y_j^{r})\}_{j=1}^N$ to compute $d_{r}$
\STATE Update total distance: $d = d + \frac{1}{N_R}d_{r}$
\ENDFOR
\STATE Return $ARI\text{-}S3W_p(\hat{\mu}, \hat{\nu}) \approx d$
\end{algorithmic}
\end{algorithm}

\subsection{Study: Gradient Flow on the Sphere}
\label{section:gf}
\subsubsection{Background Overview}
\textbf{The von Mises-Fisher distribution} is a generalization of a Gaussian from $\mathbb{R}^d$ to $\mathbb{S}^d$. Its parametric form can be characterized by the mean direction $\mu \in \mathbb{S}^d$  and the concentration parameter $\kappa > 0$. That is

\begin{equation}
f(x; \mu, \kappa) = C_d(\kappa) \exp(\kappa \mu^T x)
\end{equation}

where  $C_d(\kappa) = \frac{\kappa^{d/2-1}}{(2\pi)^{d/2} I_{d/2-1}(\kappa)}$ is the normalization constant. In our experiment setting, we mostly work with $d=2$. The pdf is
\begin{equation}
f(x; \mu, \kappa) = \frac{\kappa}{4\pi \sinh(\kappa)} \exp(\kappa \mu^T x).
\end{equation}

\textbf{Gradient descent on the sphere.}
\label{subsubsec:gd}
The sphere $\mathbb{S}^d$ is a Riemannian manifold where Euclidean geometry holds only locally. This local linearization is sufficient to adapt gradient descent to follow geodesic paths, which replace straight lines in Euclidean space. For a function $f: \mathbb{R}^m \rightarrow \mathbb{R}$, the classical Euclidean gradient update rule is given by
\begin{equation}
    x_{t+1} = x_t - \gamma \nabla_{\text{Euc}} f(x_t),
\end{equation}

where $\nabla_{\text{Euc}}f(x_t)$ denotes the Euclidean gradient at $x_t$, and from now on we simply write $\nabla f(x_t)$. On manifolds, where straight lines are generalized to geodesic curves, we replace $\nabla f(x_t)$ with the Riemannian gradient $\nabla_{\text{Riemannian}} f(x) = \nabla f(x) - \langle \nabla f(x), x \rangle x$ where $\langle \cdot, \cdot \rangle$ denotes the inner product. While this gives local linearized direction, it does not specify how to move along the manifold itself. To address this, one could use a retraction\footnote{a mapping that approximates geodesics and projects the updated point in the tangent space back onto the manifold.}. When that retraction is normalization, the update rule becomes
\begin{equation}
\label{eqn:retracted_riemannian}
\begin{split}
x_{t+1} &= \text{Retr}_{x_t}\left(-\gamma \nabla_{\text{Riemannian}} f(x_t)\right) \\
&= \frac{x_t - \gamma \left(\nabla f(x_t) - \langle \nabla f(x_t), x_t \rangle x_t\right)}{\left\| x_t - \gamma \left(\nabla f(x_t) - \langle \nabla f(x_t), x_t \rangle x_t\right) \right\|}.
\end{split}
\end{equation}

If we use $\nabla f(x_t)$ instead of $\nabla_{\text{Riemannian}} f(x_t)$, then we have the familiar Projected Gradient Descent. Another option for retraction is using the exponential map (see Eq. \eqref{eqn:exp_map}) to project $\nabla_{\text{Riemannian}}f(x_t)$ from the tangent space $T_x \mathbb{R}^d$ back onto the manifold along the geodesic. The gradient update is

\begin{equation}
\begin{split}
x_{t+1} &= \text{exp}_{x_t}\left(-\gamma \nabla_{\text{Riemannian}} f(x_t)\right) \\
&= \cos\left(\left| - \gamma(\nabla f (x_t) - \langle \nabla f(x_t), x_t \rangle x_t) \right|\right)x_t \\
&\quad + \sin\left(\left| - \gamma (\nabla f (x_t) - \langle \nabla f(x_t), x_t \rangle x_t) \right|\right) \times \frac{- \gamma (\nabla f (x_t) - \langle \nabla f(x_t), x_t \rangle x_t)}{\left| - \gamma (\nabla f (x_t) - \langle \nabla f(x_t), x_t \rangle x_t) \right|}.
\end{split}
\end{equation}

This variant is Riemannian Gradient Descent \cite{absil2008optimization}. For simplicity, we use Projected Gradient Descent in our gradient-based optimization, similar to \cite{bonet2022spherical}.

\subsubsection{Experiment: Learning a Mixture of $12$ von Mises-Fisher Distributions}

Assuming we only have access to the target measure via its samples $\{y_j \in \hat{\nu_m}\}_{j=1}^M$ Our objective is to iteratively minimize $\text{argmin}_{\mu} d (\hat{\mu}_i, \hat{\nu}_{m_i})$. The Projected Gradient Descent algorithm for $S3W$ gives the following update rule

\begin{align}
\mathbf{x}'_{i, k+1} &= \mathbf{x}_{i, k} - \gamma \nabla_{\mathbf{x}_{i,k}} \text{S3W}_p^p (\hat{\mu}_k, \hat{\nu}_{m_i}), \label{eq:mini_batch_update} \\
\mathbf{x}_{i, k+1} &= \frac{\mathbf{x}'_{i, k+1}}{\|\mathbf{x}'_{i, k+1}\|_2}, \label{eq:mini_batch_normalization}
\end{align}

where $\gamma$ denotes the learning rate, $i$ indexes the mini-batches, and $k$ the gradient step.

\begin{figure}[H]
    \centering
    \includegraphics[width=0.6\linewidth]{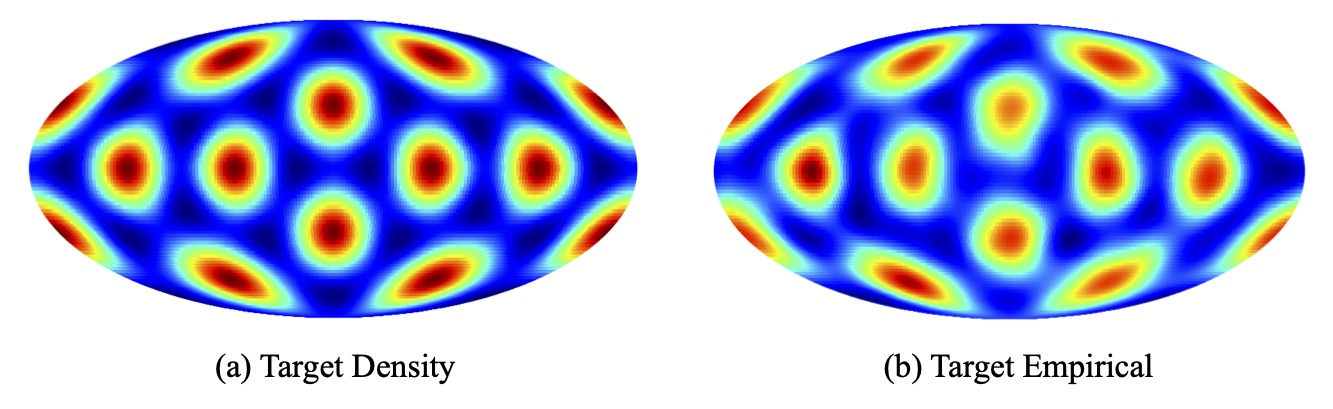}
    % \caption{}
    \label{fig: 12vmfs}
    \vspace{-5mm}
\end{figure}

\textbf{Implementation.}
We consider a mixture of $12$ vMFs as our target on $\mathbb{S}^2$, with each vMF centered on the vertices of an icosahedron determined by the golden ratio $\phi = (1 + \sqrt{5}) / 2$. We assume access to $2400$ samples from the mixture ($200$ per vMF), each with $\kappa = 50$, and set batch sizes to $|m_i|=\{200, 2400\}$. We fix the number of projecions to $1000$ for all distances. We optimize over $500$ gradient steps using the Adam Optimizer \cite{kingma2014adam} (learning rate of $\gamma=0.01$ for full-batch and $0.001$ for mini-batch) for $10$ in dependent runs. We report both qualitative and quantitative results (NLL, $W_2$); the latter comes with mean and standard deviation.

\textbf{Full-batch results.} In addition to the numerical results provided in the main paper (refer to Figure \ref{fig: gf_main}), we provide the Molleweide projections in Figure \ref{fig:gf_full_appendix} to show that all distances work comparatively well to learn the target distribution. In this setting, we also include the evaluation of the vertical sliced Wasserstein (VSW) distance \cite{quellmalz2023sliced} in Figure \ref{fig:gf_full_appendix} and Table \ref{table:gf_comparison} (see Section \ref{sec:vsw}).

\textbf{Mini-batch results.} We provide below our mini-batch results for all distances. In addition to the numerical results reported in Figure \ref{table:gf_comparison}, we also show the Mollweide projections for the learned distributions as well as particle scatter-plots at epochs $\{0, 100, 300, 500\}$. In Table \ref{table:gf_comparison}, $S3W$ shows the fastest runtime for both full-batch and mini-batch optimization. It is almost an order of magnitude faster than $SSW$. $RI\text{-}S3W (1)$ adds negligible computational overhead to $S3W$ with the single added rotation but performs almost on par with SSW. $RI\text{-}S3W$ (5) and $RI\text{-}S3W$ (10) are both high-performing candidates, exceeding $SSW$ albeit with some additional runtime compared to $S3W$. $ARI\text{-}S3W (30/1000)$ shows significant performance gain while being almost as fast as $S3W$. We remark that the smaller the batch size, the more frequently the distance function gets called, the bigger advantage $ARI\text{-}S3W$ will have over $RI\text{-}S3W$ since the rotation generation cost gets amortized over a large number of function calls and becomes negligible. 

Both $RI\text{-}S3W$ and $ARI\text{-}S3W$ can achieve additional efficiency with rotation scheduler. In particular, one could start from few rotations and gradually (i.e. linearly) increase the number of rotations over time for fine-tuning. We provide additional visualization and numerical results for $RI\text{-}S3W$ and $ARI\text{-}S3W$ with a linear schedule from $1$ to $30$ rotations over $500$ epochs. We observe that $ARI\text{-}S3W (1\text{-}30)$ outperforms $ARI\text{-}S3W (30)$ while being twice as fast.

\begin{figure}[H]
    % \vspace{-5mm}
    \centering
    \subfloat[$SSW$]{\label{target_ps_2}\includegraphics[width=0.2\linewidth]{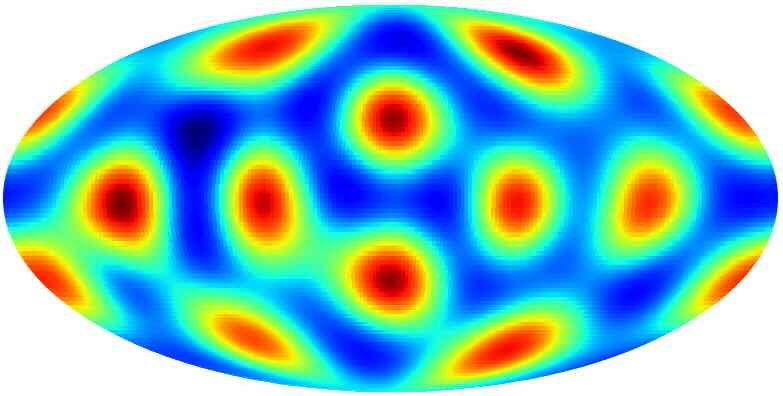}}\hspace{5mm}  % Adjust 5mm as needed
    \subfloat[$S3W$]{\label{evo_runtime_s3w_2}\includegraphics[width=0.2\linewidth]{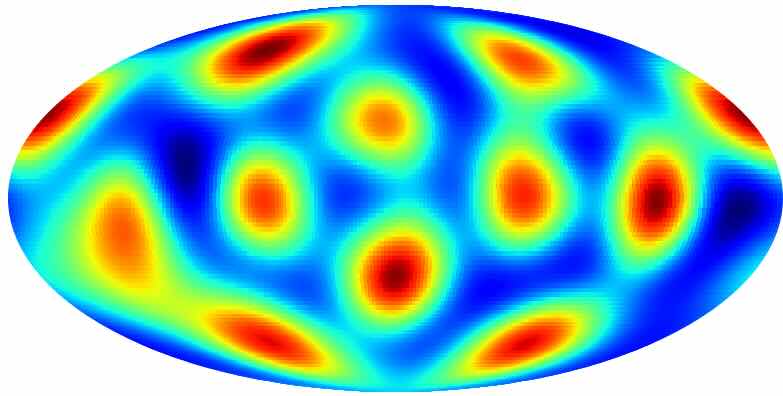}}\hspace{5mm}  % Adjust 5mm as needed
    \subfloat[$RI\text{-}S3W (1)$]
    {\label{evo_runtime_ri1_ap}\includegraphics[width=0.2\linewidth]{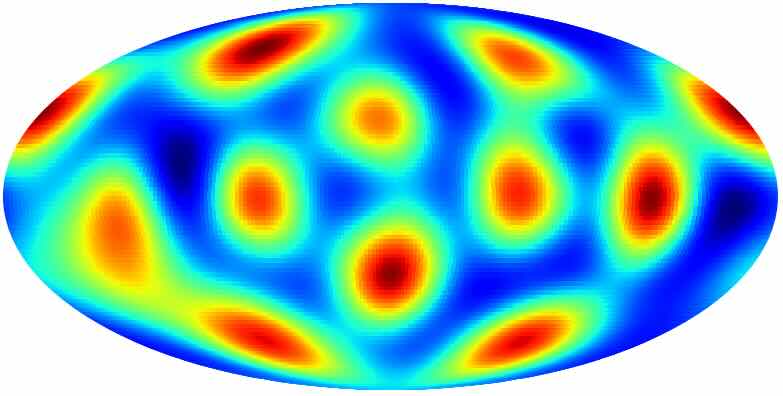}}\hspace{5mm}  % Adjust 5mm as needed
    \subfloat[$RI\text{-}S3W (5)$]{\label{evo_runtime_ri5_ap}\includegraphics[width=0.2\linewidth]{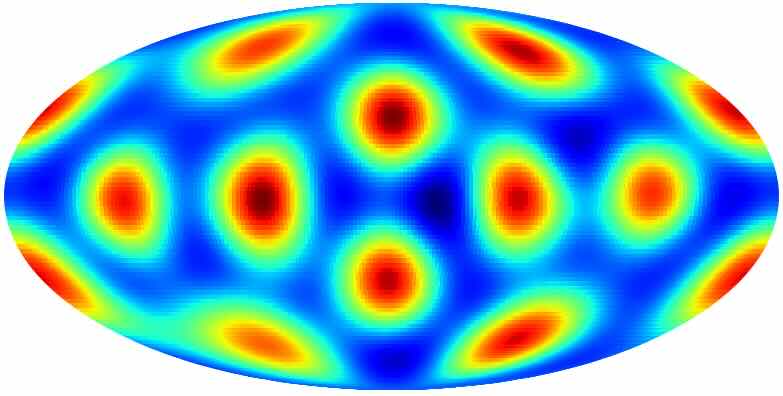}}\\
    
    \subfloat[$RI\text{-}S3W (10)$]{\label{evo_runtime_ris3w10_1}\includegraphics[width=0.2\linewidth]{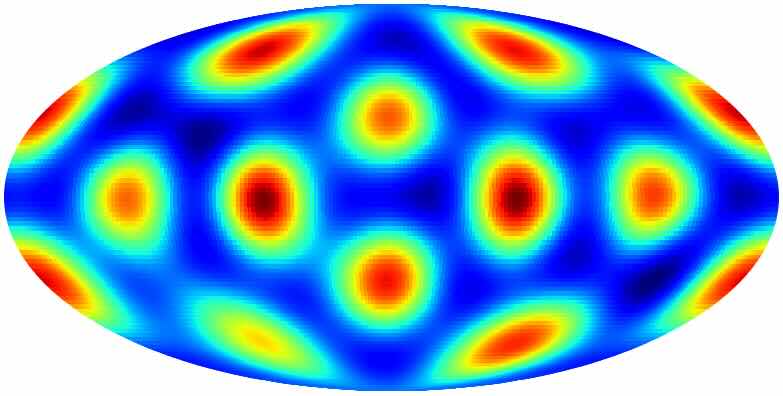}}\hspace{5mm}  % Adjust 5mm as needed
    \subfloat[$ARI\text{-}S3W (30)$]{\label{evo_runtime_ari30_ap}\includegraphics[width=0.2\linewidth]{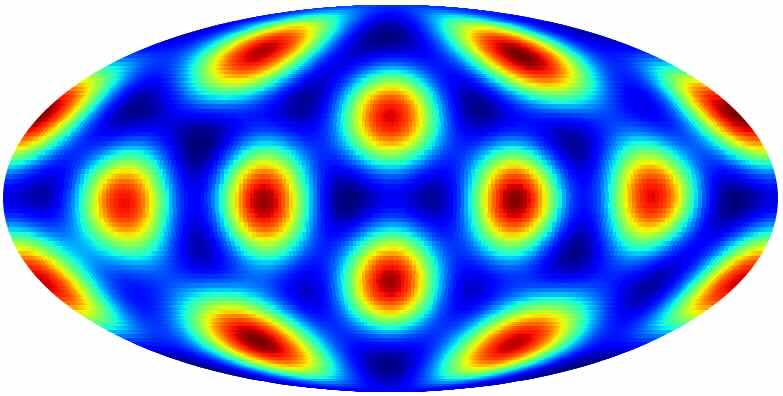}}\hspace{5mm}  % Adjust 5mm as needed
    \subfloat[$RI\text{-}S3W (1\text{-}30)$]{\label{evo_runtime_ri_30_ap}\includegraphics[width=0.2\linewidth]{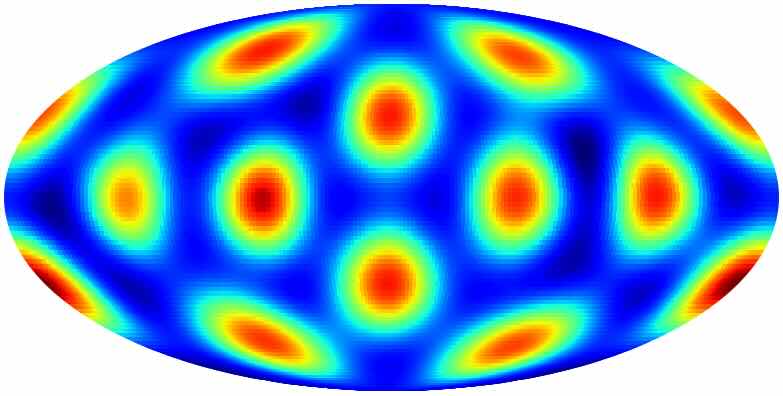}} \hspace{5mm}
    \subfloat[$ARI\text{-}S3W (1\text{-}30)$]{\label{evo_runtime_ari_30_ap}\includegraphics[width=0.2\linewidth]{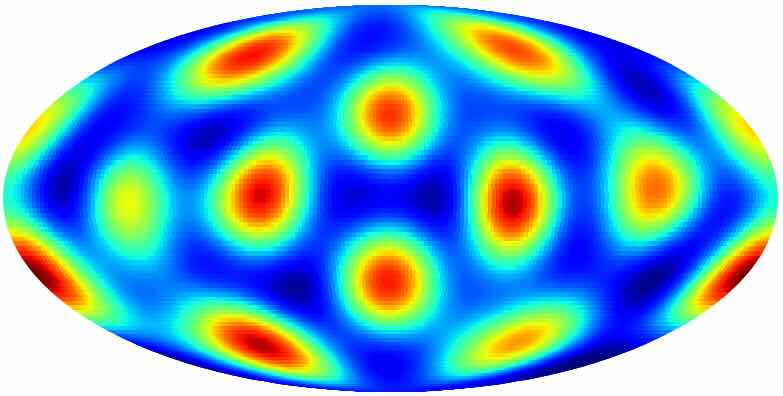}}

    \caption{The Mollweide projections for mini-batch projected gradient descent. $ARI\text{-}S3W$ has pool size of $1000$. $RI\text{-}S3W (1\text{-}30)$ and $ARI\text{-}S3W (1\text{-}30)$ denote $RI\text{-}S3W$ and $ARI\text{-}S3W$ with linear $N_R$ schedule from $1$ to $30$ over $500$ epochs.}
    \label{fig:gf_appendix}
    % \vspace{-10mm}
\end{figure}

\clearpage

\begin{figure}[t]
    % \vspace{-12mm}
    \centering
    \subfloat[$SSW$]{\label{target_ps}\includegraphics[width=0.2\linewidth]{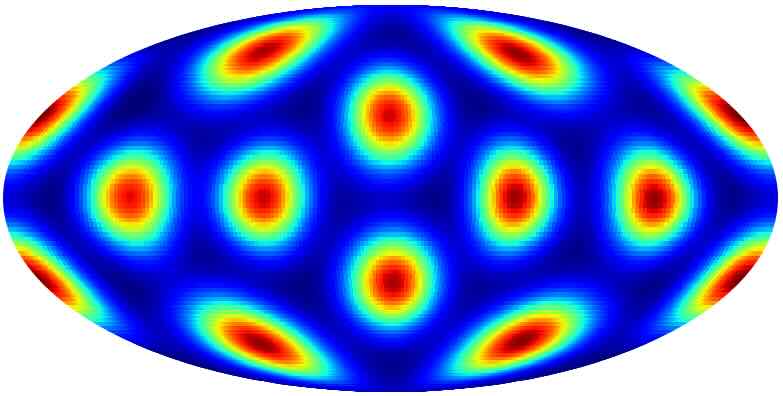}}\hspace{5mm}  % Adjust 5mm as needed
    \subfloat[$S3W$]{\label{evo_runtime_s3w}\includegraphics[width=0.2\linewidth]{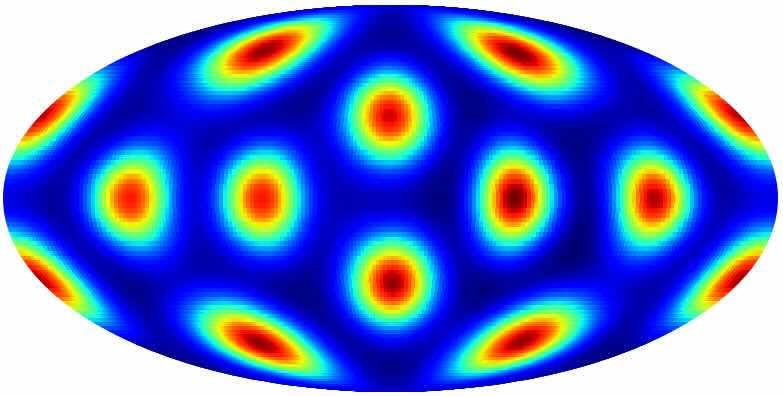}}\hspace{5mm}  % Adjust 5mm as needed
    \subfloat[$RI\text{-}S3W (1)$]
    {\label{evo_runtime_ri_full}\includegraphics[width=0.2\linewidth]{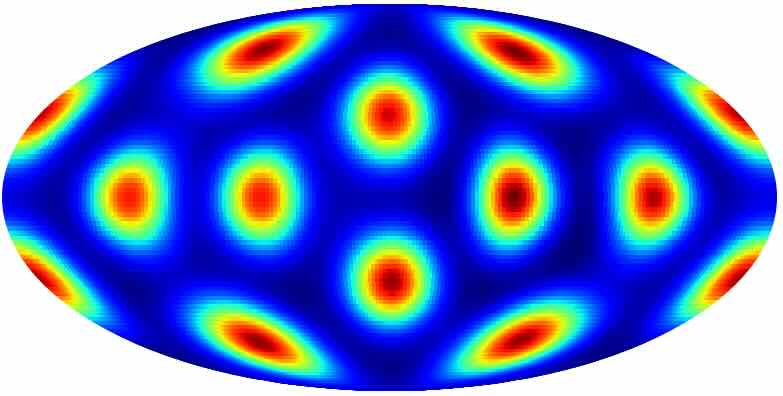}}\hspace{5mm}  % Adjust 5mm as needed
    \subfloat[$RI\text{-}S3W (5)$]{\label{evo_runtime_ri_full_ap}\includegraphics[width=0.2\linewidth]{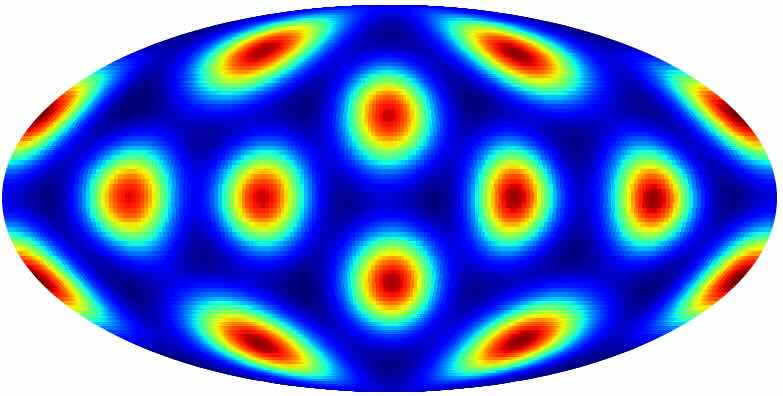}}\\
    
    \subfloat[$RI\text{-}S3W (10)$]{\label{evo_runtime_s3w10_1}\includegraphics[width=0.2\linewidth]{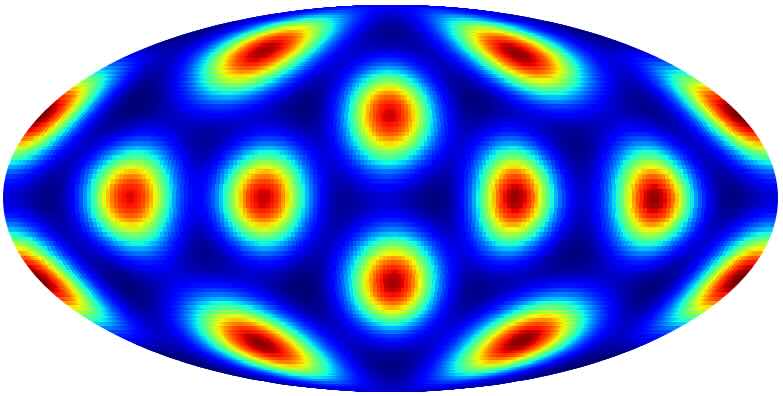}}\hspace{5mm}  % Adjust 5mm as needed
    \subfloat[$ARI\text{-}S3W (30)$]{\label{evo_runtime_s3w30_1}\includegraphics[width=0.2\linewidth]{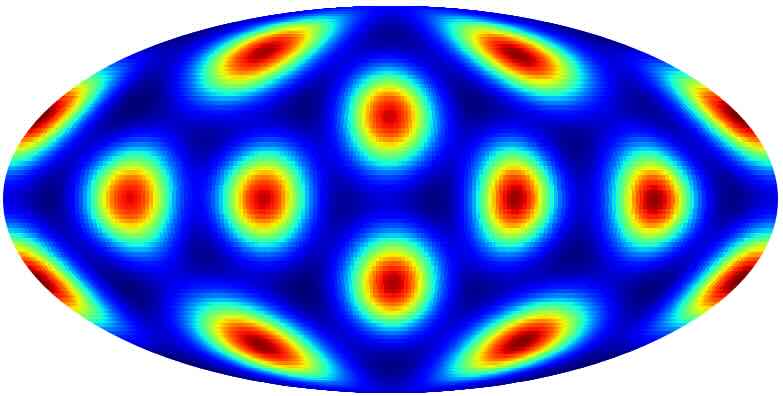}}\hspace{5mm}  % Adjust 5mm as needed
    \subfloat[$RI\text{-}S3W (1\text{-}30)$]{\label{evo_runtime_ri_30_full_ap}\includegraphics[width=0.2\linewidth]{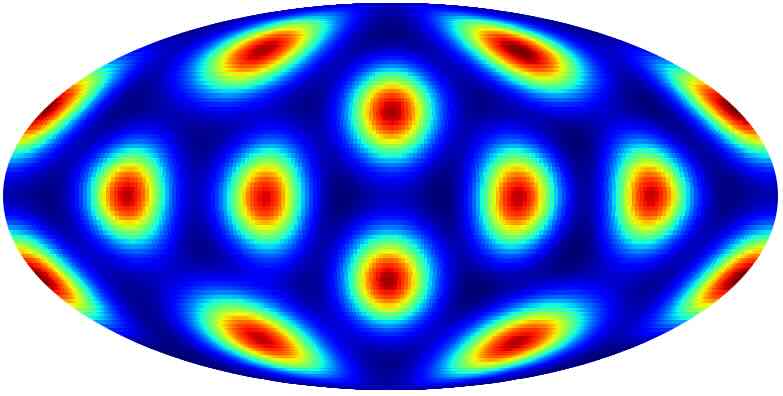}} \hspace{5mm}
    \subfloat[$ARI\text{-}S3W (1\text{-}30)$]{\label{evo_runtime_aris3w_1}\includegraphics[width=0.2\linewidth]{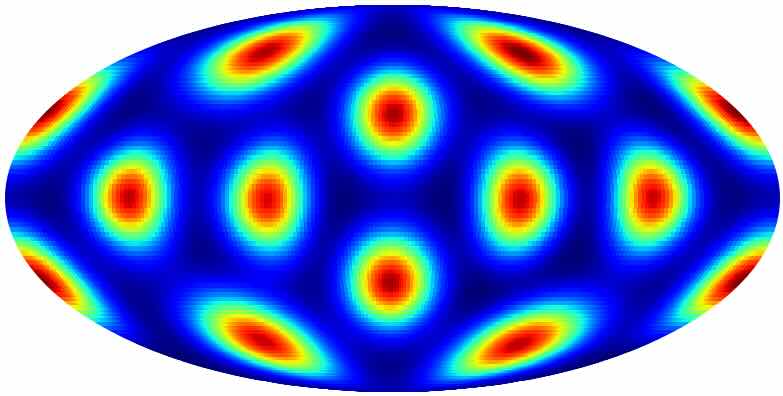}} \\

    \subfloat[$VSW$]{\label{gf_vsw}\includegraphics[width=0.2\linewidth]{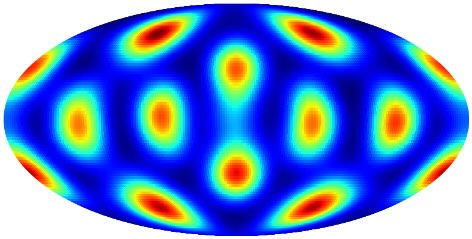}}\hspace{5mm}  % Adjust 5mm as needed
    \subfloat[$SW$]{\label{gf_sw}\includegraphics[width=0.2\linewidth]{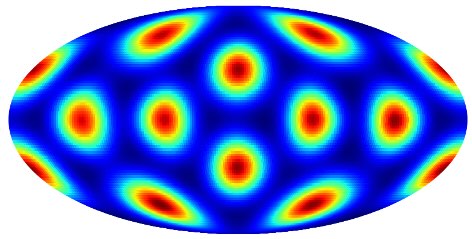}}\hspace{5mm}  % Adjust 5mm as needed    

    \caption{The Mollweide projections for full-batch projected gradient descent. $ARI\text{-}S3W$ has pool size of $1000$. $RI\text{-}S3W (1\text{-}30)$ and $ARI\text{-}S3W (1\text{-}30)$ denote $RI\text{-}S3W$ and $ARI\text{-}S3W$ with linear $N_R$ scheduled from $1$ to $30$ over $500$ epochs.}
    \label{fig:gf_full_appendix}
    % \vspace{-10mm}
\end{figure}

\begin{table}[t]
\centering
\caption{Comparison between different distances as loss for gradient flows.}
\renewcommand{\arraystretch}{1.1}
\begin{tabular}{clccc}
\hline
& Distance & Runtime (s) $\downarrow$ & NLL $\downarrow$ & $\log(W_2)$ $\downarrow$ \\
\hline
\multirow{5}{*}{\rotatebox[origin=c]{0}{Mini-batch}} 
& $SSW$ & 46.10 $\pm$ 0.28 & -284.61 $\pm$ 8.44 & -1.19 $\pm$ 0.05 \\
& $S3W$ & 5.16 $\pm$ 0.62 & -205.13 $\pm$ 17.56 & -1.17 $\pm$ 0.055 \\
& $RI\text{-}S3W$ (1) & 5.50 $\pm$ 0.96 & -254.54 $\pm$ 9.39 & -1.18 $\pm$ 0.06 \\
& $RI\text{-}S3W$ (5) & 8.62 $\pm$ 0.94 & -305.15 $\pm$ 11.79 & -1.22 $\pm$ 0.05 \\
& $RI\text{-}S3W$ (10) & 8.75 $\pm$ 0.61 & -320.74 $\pm$ 8.34 & -2.80 $\pm$ 0.12 \\
& $ARI\text{-}S3W$ (30/1000) & 6.75 $\pm$ 0.12 & \textbf{-343.71 $\pm$ 2.45} & \textbf{-3.02 $\pm$ 0.17} \\ 
& $RI\text{-}S3W$ (1-30) & 7.67 $\pm$ 0.10 & -322.35 $\pm$ 7.69 & -2.96 $\pm$ 0.20 \\
& $ARI\text{-}S3W$ (1-30/1000) & \textbf{4.38 $\pm$ 0.19} & -264.55 $\pm$ 13.31 & -2.84 $\pm$ 0.15 \\ 
\hline\hline
\multirow{1}{*}{\rotatebox[origin=c]{0}{Full-batch}}
& $SW$ & \textbf{0.64 $\pm$ 0.1} & -4891.85 $\pm$ 3.59 & -4.70 $\pm$ 0.22 \\
& $VSW$ & 0.66 $\pm$ 0.01 & -4858.75 $\pm$ 6.37 & -2.32 $\pm$ 0.08 \\
& $RI\text{-}S3W$ (1-30) & 26.90 $\pm$ 0.13 & -5003.21 $\pm$ 1.88 & \textbf{-6.68 $\pm$ 0.06} \\
& $ARI\text{-}S3W$ (1-30/1000) & 23.78 $\pm$ 0.36 & \textbf{-5048.43 $\pm$ 6.16} & -5.40 $\pm$ 0.03 \\ 
% & $SSW$ & 103.85 $\pm$ 0.44 & -134.01 $\pm$ 0.02 & 0.0091 $\pm$ 0.0011 \\
% & $S3W$ & \textbf{4.92 $\pm$ 0.17} & -132.76 $\pm$ 0.45 & 0.0098 $\pm$ 0.0013 \\
% & $RI\text{-}S3W$ (1) & 5.02 $\pm$ 0.12 & -134.19 $\pm$ 0.13 & 0.0092 $\pm$ 0.0012 \\
% & $RI\text{-}S3W$ (5) & 11.15 $\pm$ 0.25 & -134.23 $\pm$ 0.17 & 0.0092 $\pm$ 0.0013 \\
% & $RI\text{-}S3W$ (10) & 17.97 $\pm$ 0.23 & \textbf{-134.25 $\pm$ 0.11} & \textbf{0.0091 $\pm$ 0.0009} \\
% \hline
\end{tabular}
\label{table:gf_comparison}
\end{table}

\clearpage

\begin{figure}[H]
    \centering
    \includegraphics[width=0.90\linewidth]{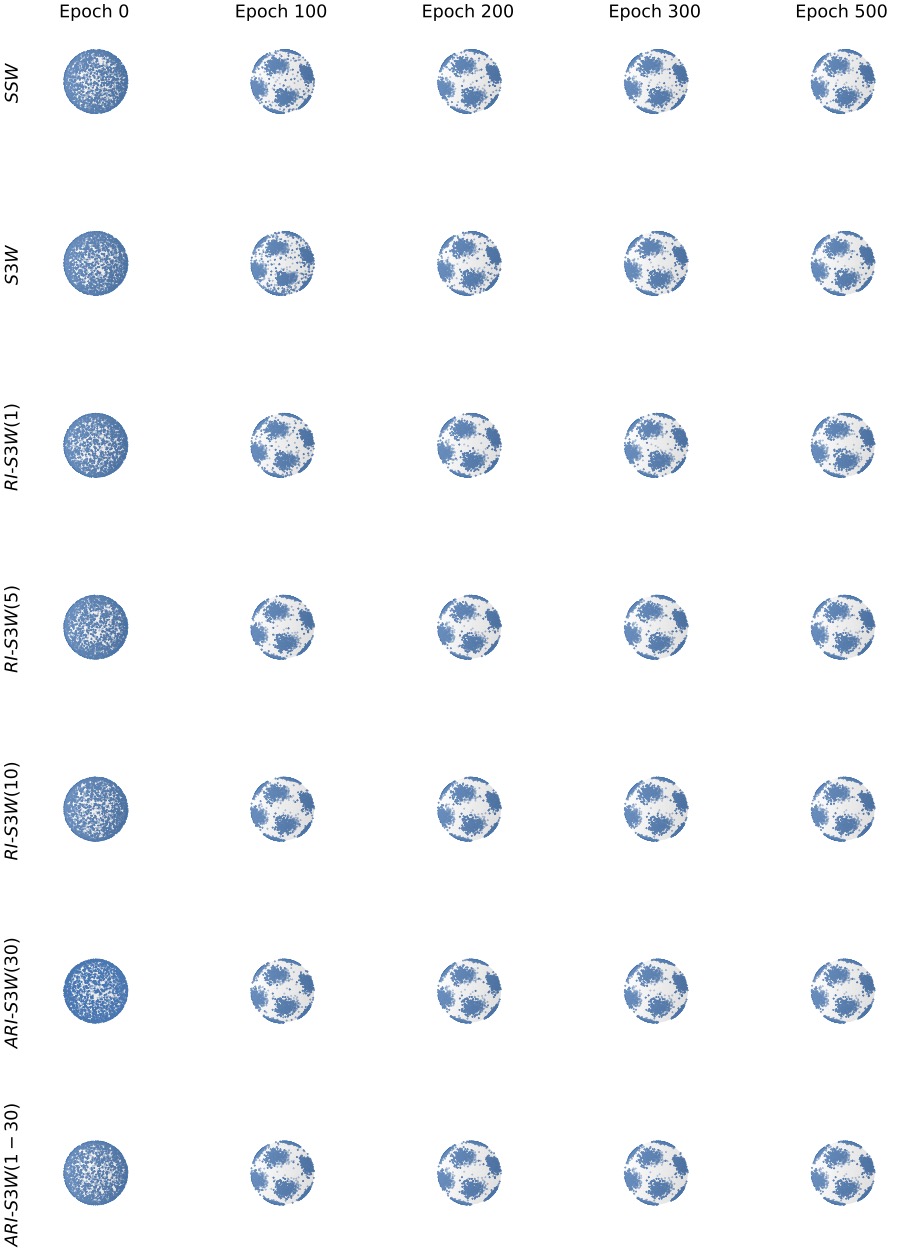}
    % \vspace{-.2in}
    \caption{Particle Evolution for mini-batch projected gradient descent.}
    %\label{fig: angle_sphere_+}
\end{figure}

\clearpage

\subsection{Study: Stability of Metrics w.r.t. $\epsilon$-cap} \label{sec:eps_stability}
While in theory, one only needs to exclude the ``north pole'' from the stereographic projection (SP) operation, the image of points near the north pole under SP can be made arbitrarily large in norm. Numerically, this may result in overflows and make the optimization in our ML applications unstable. Hence, we fix an $\epsilon$-cap for SP which ensures the norm of projected samples are bounded; in particular, any point such that \(x_{d+1} > 1 - \epsilon\) is first mapped to the circle \(x_{d+1} = 1 - \epsilon\) prior to the SP operation. 

In this section, we provide additional analysis of the effect of this $\epsilon$-cap on the proposed metrics. We generate two vMFs (target at the North pole and source at the South pole) with $\kappa=50$ and $N=2048$ samples each. We then compute each of the proposed metrics between the source and target distributions as a function of $\epsilon$. We visualize the results over $100$ runs for each metric in Figure \ref{fig:eps_stability}. We note that it is expected for the behavior of $S3W$ to differ from $RI$-$S3W$ and $ARI$-$S3W$. Importantly, we observe that all of the proposed methods are stable for a wide range of values for $\epsilon$. We observe no notable difference when we vary the number of projections $L \in \{128,512\}$.

\begin{figure}[H]
    \centering
    \includegraphics[width=0.8\linewidth]{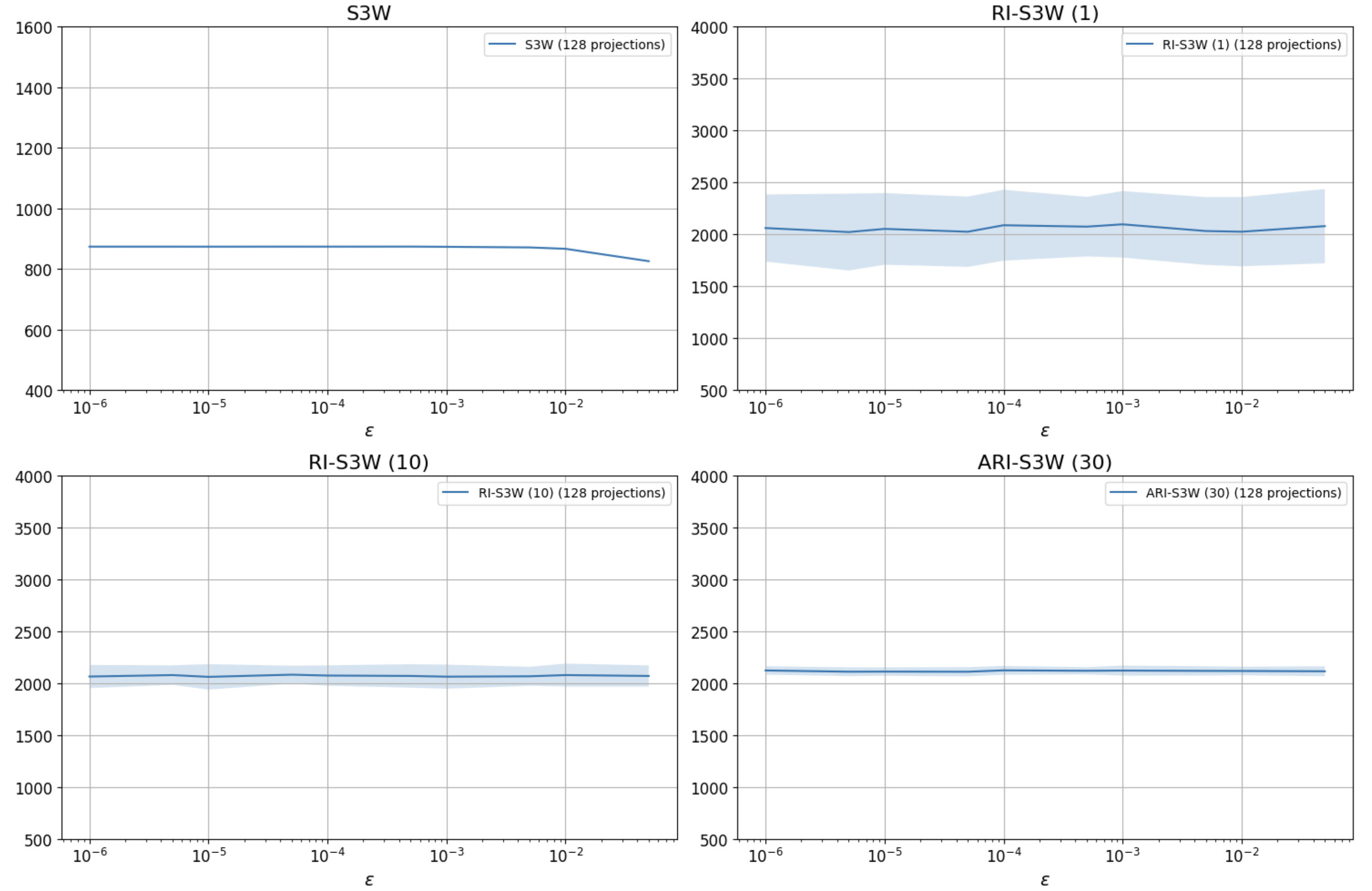}
    \caption{Stability of each metric w.r.t. $\epsilon$. Here, we use $N=2048$ samples of each distribution and fix $L=128$ projections.}
    \label{fig:eps_stability}
\end{figure}

\clearpage
\subsection{Study: Evolution of the $S3W$ Distance}
\label{section:evolution}
\subsubsection{Experiment:  Evolution of Metrics w.r.t. Iteration and Runtime}
\label{subsec:evo_loss_curve}
\textbf{Implementation.} In this experiment, we aim to compare the evolution of different distances when used as a loss for gradient optimization. We randomly initialize a uniform source distribution $\hat{\mu} \sim \text{Unif}(\mathbb{S}^2)$ of $2400$ samples. The objective is to learn a mixture $\hat{\nu}$ of $12$ vMFs similar to Section \ref{subsec:gf}, also consisting of $2400$ samples. We use the Adam optimizer for projected gradient descent (see Section \ref{section:gf}) with a learning rate of $\gamma=0.01$ for all distances. The $SSW$ baseline \cite{bonet2022spherical} will have an additional run with learning rate $\gamma' = 0.05$.

\textbf{Results.} Figure \ref{fig: gf_main} (in the main paper) shows comparative performance for: SSW, $S3W$, $RI\text{-}S3W$ (1 random rotation), $RI\text{-}S3W$ (5 random rotations), $ARI\text{-}S3W$ (30 random rotations, pool size of $1000$). We show the evolution of the $\log(W_2)$ loss w.r.t. the cumulative runtime. We observe that at learning rate of $0.01$, $SSW$ evolves nicely and starts converging around epoch $500$, which is similar in performance with $RI\text{-}S3W$ (1) and $RI\text{-}S3W$ (5), albeit an order of magnitude slower. For a learning rate of $0.05$, $SSW$ begins to display alternating behaviors but converges similarly to the case where learning rate is $0.01$. $S3W$ is the fastest but slightly underperform other distances. $ARI\text{-}S3W$ (30) converges to the lowest loss and twice at fast as $SSW$ for both learning rates. 

\subsubsection{Experiment: Evolution of Metrics w.r.t. \texorpdfstring{$\kappa$}{kappa} for Varying Dimensions}

\begin{figure}[H]
    \vspace{-.3in}
    \centering
    % First row of subfigures
    \hspace*{\fill}
    \subfloat[$\text{KL}(\text{vMF}(\mu,\kappa) \| \text{vMF}(\cdot,0))$]{\includegraphics[width=0.4\columnwidth]{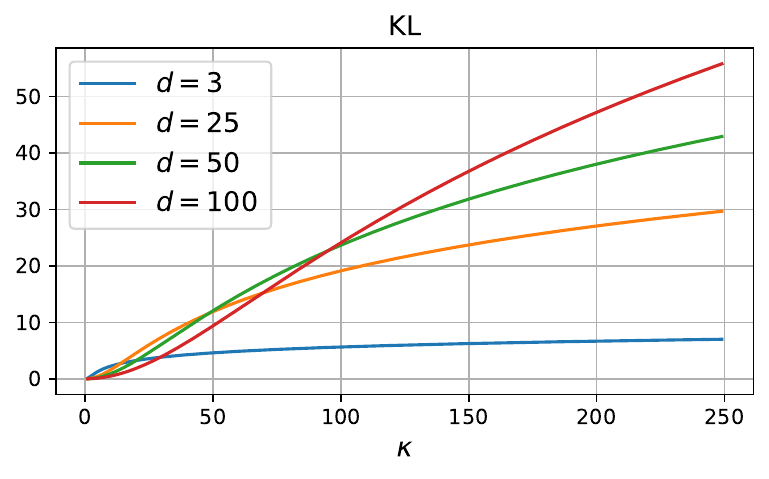}} \hfill 
    \subfloat[$SSW_2(\text{vMF}(\mu,\kappa) \| \text{vMF}(\cdot,0))$]{\includegraphics[width=0.4\columnwidth]{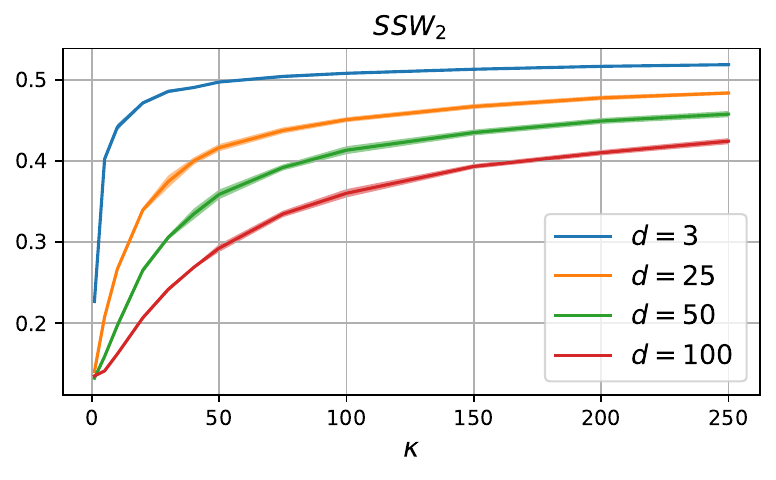}}
    \hspace*{\fill}
    \\
    % Second row of subfigures
    \subfloat[$S3W_2(\text{vMF}(\mu,\kappa) \| \text{vMF}(\cdot,0))$]{\includegraphics[width=0.33\columnwidth]{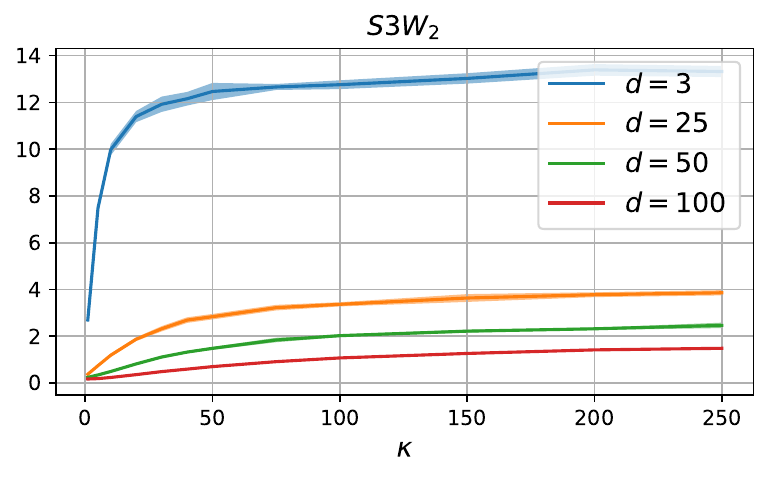}} \hfill 
    \subfloat[$RI\text{-}S3W_2(\text{vMF}(\mu,\kappa) \| \text{vMF}(\cdot,0))$]{\includegraphics[width=0.33\columnwidth]{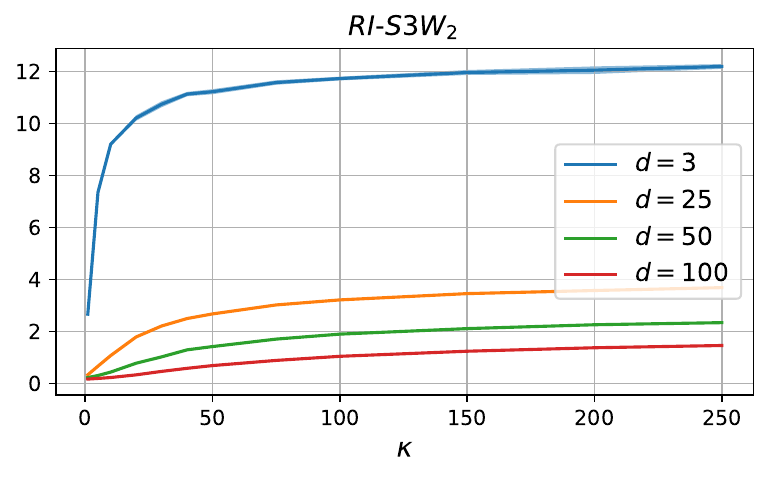}} \hfill
    \subfloat[$ARI\text{-}S3W_2(\text{vMF}(\mu,\kappa) \| \text{vMF}(\cdot,0))$]{\includegraphics[width=0.33\columnwidth]{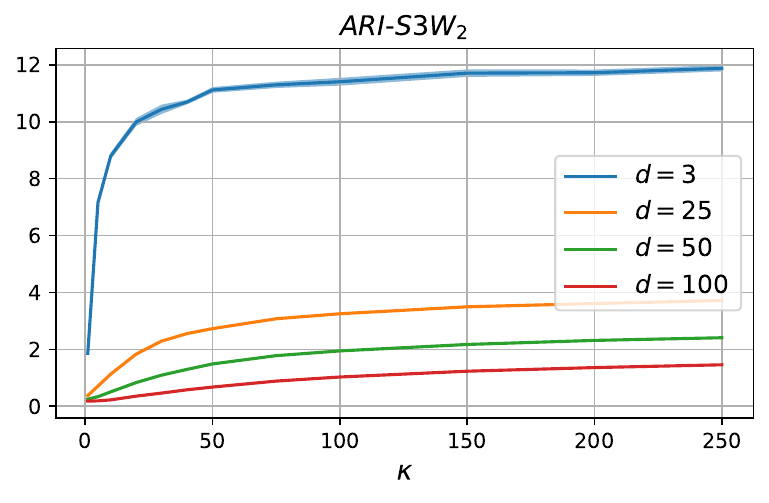}}
    
    \caption{Evolution between $\text{vMF}(\mu, \kappa)$ and $\text{vMF}(\cdot, 0)$ w.r.t. $\kappa$ for varying dimension. We use $L=200$ projections for all sliced metrics, $N_R=100$ rotations for $RI\text{-}S3W$ and $ARI$-$S3W$, and pool size of $1000$ for $ARI$-$S3W$. Each distribution has $500$ samples. For each $\kappa \in \{1,5,10,20,30,40,50,75,100,150,200,250\}$ we average each metric over 10 iterations.}
    \label{fig:evo_d_kappa}
\end{figure}

We illustrate in Figure \ref{fig:evo_d_kappa} the evolution of KL divergence, $SSW$, $S3W$, $RI$-$S3W$, and $ARI\text{-}S3W$ between $\text{vMF}(\mu, \kappa)$ and $\text{vMF}(\cdot, 0)$ w.r.t. $\kappa$ for varying dimension. Just as \cite{bonet2022spherical} found that $SSW$ gets lower with the dimension contrary to KL divergence, we find that $S3W$, $RI$-$S3W$, and $ARI\text{-}S3W$ follow a similar trend.

Here we use the analytic form for the KL divergence between the von Mises-Fisher distribution and the uniform distribution on $\mathbb{S}^d$ as derived in \cite{davidson2018hyperspherical, xu2018spherical}:
\begin{equation}
\begin{aligned}
    \mathrm{KL}\big(\mathrm{vMF}(\mu,\kappa)||\mathrm{vMF}(\cdot,0)\big) &= \kappa \frac{I_{(d+1)/2}(\kappa)}{I_{(d+1)/2-1}(\kappa)} + \left(\frac{d+1}{2}-1\right)\log \kappa \\
    &\quad - \frac{d+1}{2} \log(2\pi) - \log I_{(d+1)/2-1}(\kappa) \\
    &\quad + \frac{d+1}{2}\log \pi + \log 2 -\log \Gamma\left(\frac{d+1}{2}\right),
\end{aligned}
\end{equation}

where $I_v(\cdot)$ is the modified Bessel function of the first kind with order $v$, and $\Gamma(\cdot)$ is the gamma function.

\clearpage
\subsubsection{Experiment: Evolution of Metrics Between Rotated vMFs Distributions}

Similar to \cite{bonet2022spherical}, we also compare the evolution of $SSW$, $S3W$, $RI$-$S3W$, and $ARI\text{-}S3W$ between a fixed vMF distribution and the rotation of this distribution along a great circle. That is, we compute each metric between $\text{vMF}(\mu_0, \kappa_0)$ and $\text{vMF}((\cos\theta, \sin\theta,0,\dots), \kappa_0)$ for $\theta\in\{\frac{k\pi}{6}\}_{k=0}^{12}$, where $\mu_0 = (1,0,0,\dots)$ and $\kappa_0 = 10$. We observe similar behavior between all metrics, with each distance being maximal between $\text{vMF}(\mu_0, \kappa_0)$ and $\text{vMF}(-\mu_0, \kappa_0)$ (corresponding to $\theta = \pi$).

\begin{figure}[H]
    \vspace{-0.15in}
    \centering
    \hspace*{\fill}
    \subfloat[$SSW$]{\includegraphics[width=0.4\columnwidth]{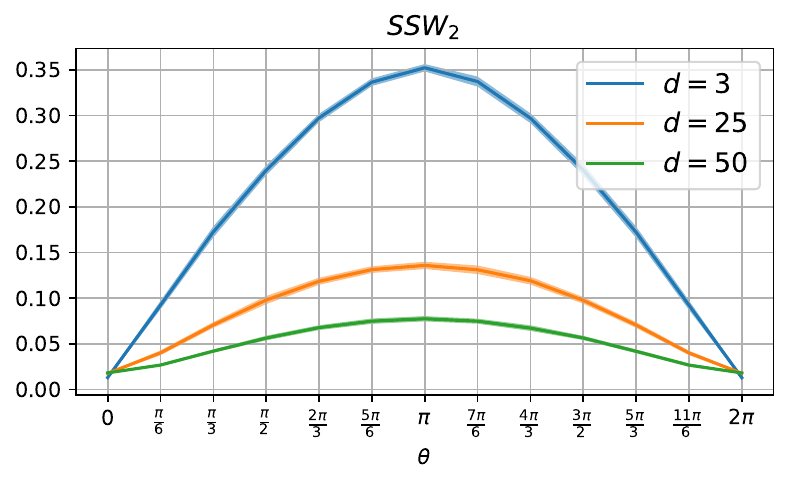}} \hfill 
    \subfloat[$S3W$]{\includegraphics[width=0.4\columnwidth]{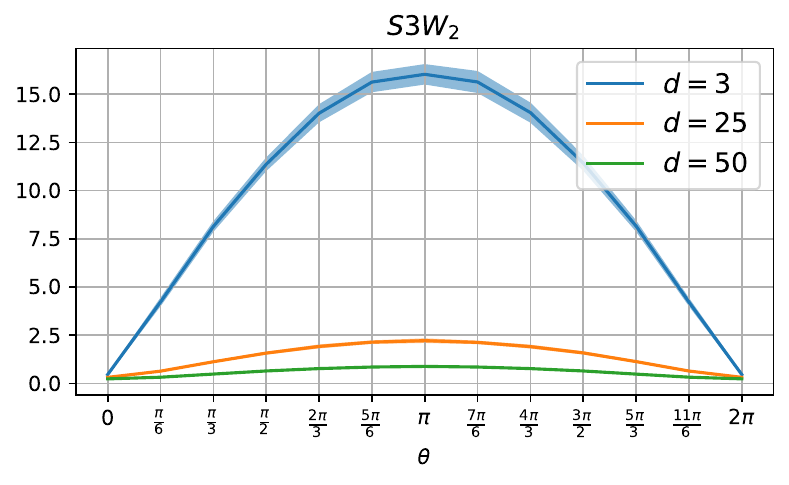}} 
    \hspace*{\fill}
    \\
    \hspace*{\fill}
    \subfloat[$RI\text{-}S3W$]{\includegraphics[width=0.4\columnwidth]{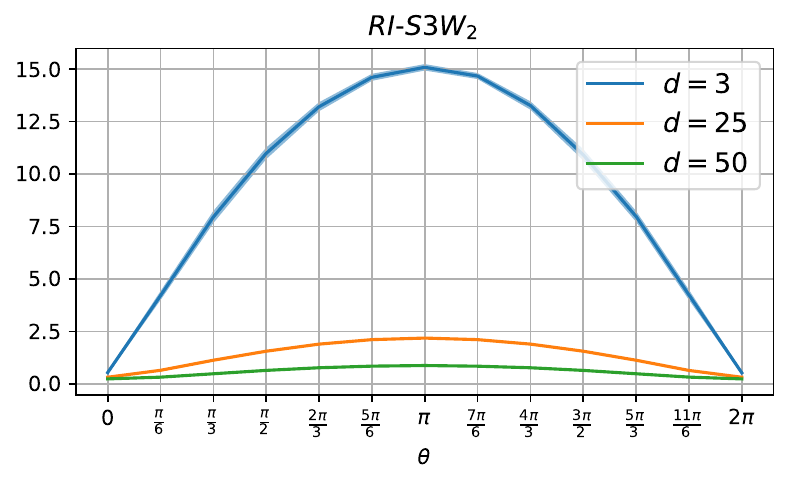}} \hfill  
    \subfloat[$ARI\text{-}S3W$]{\includegraphics[width=0.4\columnwidth]{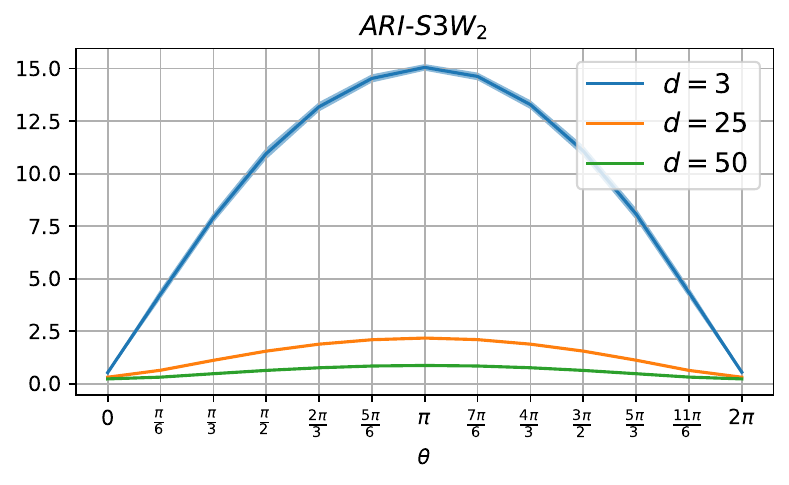}} \hspace*{\fill}
    \caption{Evolution between vMF distributions on $\mathbb{S}^{d-1}$. Here we use $L=200$ projections for all metrics, $N_R=100$ rotations for $RI\text{-}S3W$ and $ARI$-$S3W$, and a pool size of $1000$ for $ARI$-$S3W$. We average each metric over 100 iterations and use 500 samples for each distribution.} \hfill
    \label{fig:evo_angle}
\end{figure}

\subsubsection{Experiment: Evolution of \texorpdfstring{$S3W$}{S3W} w.r.t. Number of Slices and Number of Rotations}

We perform several experiments to understand how $S3W$, $RI\text{-}S3W$, and $ARI$-$S3W$ evolve w.r.t. the number of slices used, the number of random rotations used, and the pool size used, respectively. In each of the following experiments, we measure the distance between a source uniform distribution $\text{vMF}(\cdot, 0)$ and a target von Mises-Fisher distribution $\text{vMF}(\mu, \kappa)$, where we sample 500 points from each distribution.

Figure \ref{fig:evol_projections} demonstrates that beyond $L=100$ projections, variance becomes negligible across different dimensions and values of $\kappa$. Figure \ref{fig:evol_rot} demonstrates that $RI\text{-}S3W$ is generally stable for $R=10$ rotations across varying dimensions and $\kappa$. Finally, Figure \ref{fig:evol_pool} demonstrates that variance is stable beyond a pool size of $20$. In all plots, we average the metric over 20 iterations.  

\begin{figure}[H]
    \centering
    % \vspace{-5mm}
    \hspace*{\fill}
    \subfloat{\includegraphics[width=0.45\columnwidth]{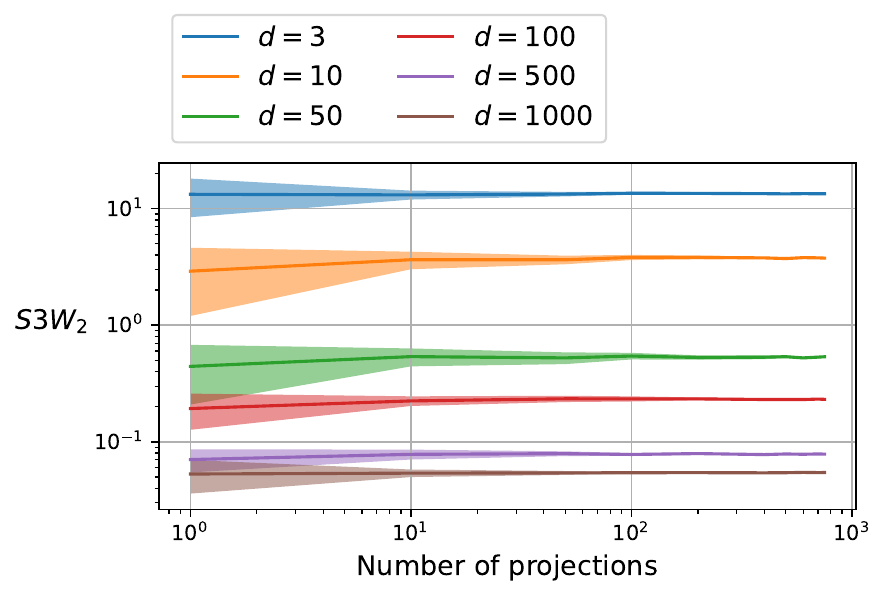}} \hfill 
    \subfloat{\includegraphics[width=0.45\columnwidth]{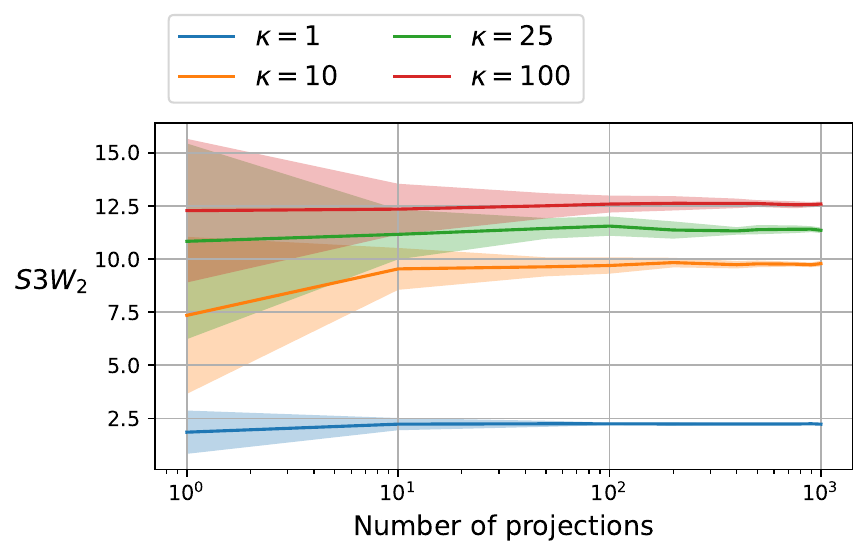}}
    \hspace*{\fill}
    \caption{Evolution of $S3W$ between the source and target vMFs w.r.t. number of projections used. For the plot on the left, we fix $\kappa = 10$ for the target distribution, and for the plot on the right we use $d=3$.} \hfill
    \label{fig:evol_projections}
    \vspace{-10mm}
\end{figure}

\begin{figure}[H]
    \centering
    \hspace*{\fill}
    \subfloat{\includegraphics[width=0.45\columnwidth]{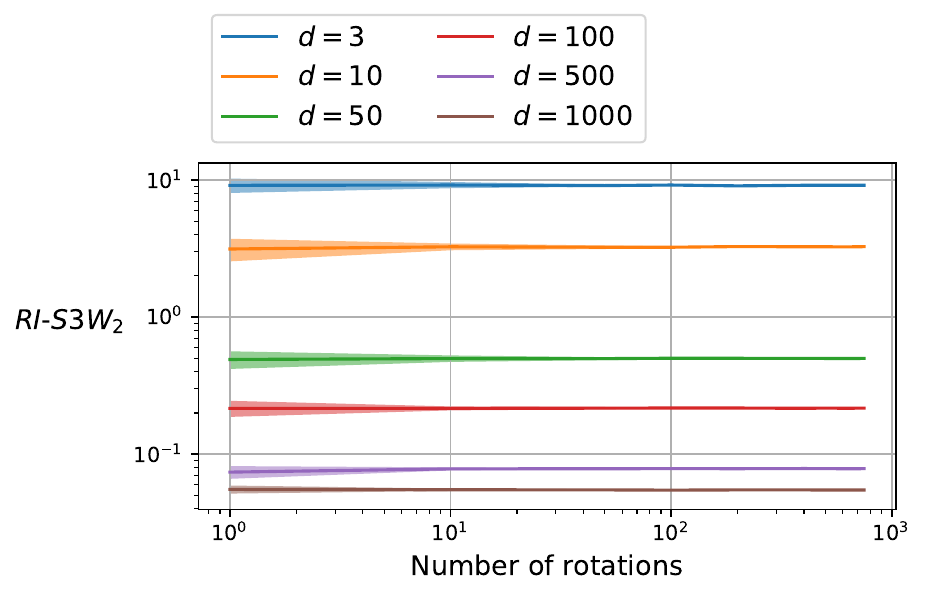}} \hfill 
    \subfloat{\includegraphics[width=0.45\columnwidth]{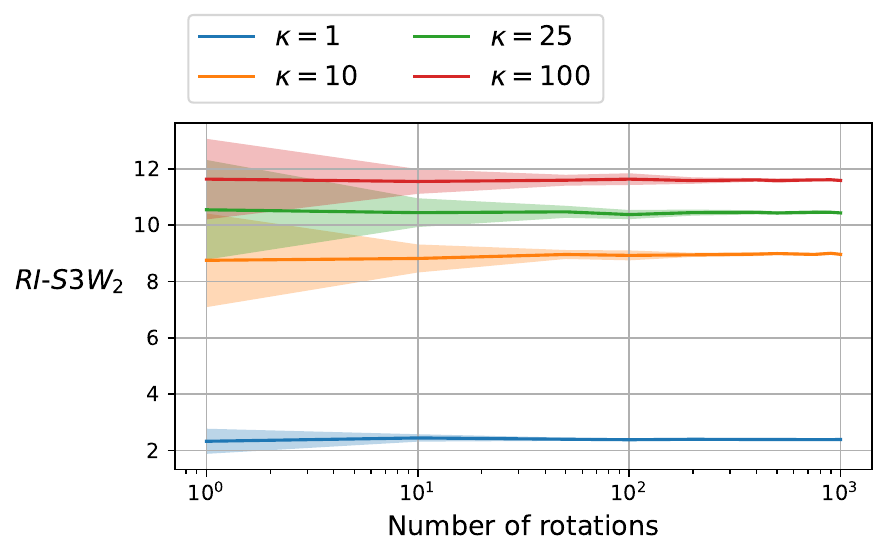}}
    \hspace*{\fill}
    \caption{Evolution of $RI$-$S3W$ between the source and target vMFs w.r.t. number of rotations used. We fix $L=10$ projections. For the plot on the left, we fix $\kappa = 10$ for the target distribution, and for the plot on the right we use $d=3$.} \hfill
    \label{fig:evol_rot}
    \vspace{-10mm}
\end{figure}

\begin{figure}[H]
    \centering
    \hspace*{\fill}
    \subfloat{\includegraphics[width=0.45\columnwidth]{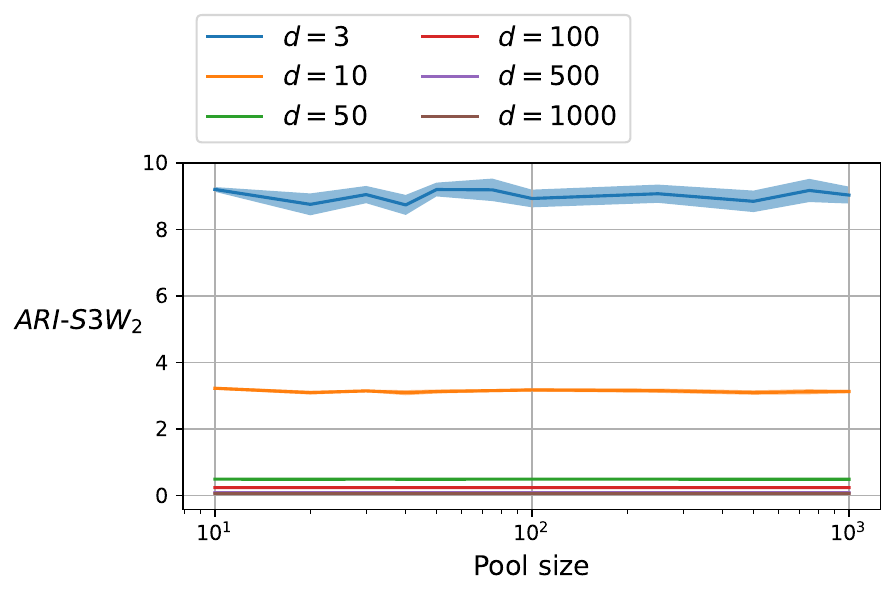}} \hfill 
    \subfloat{\includegraphics[width=0.45\columnwidth]{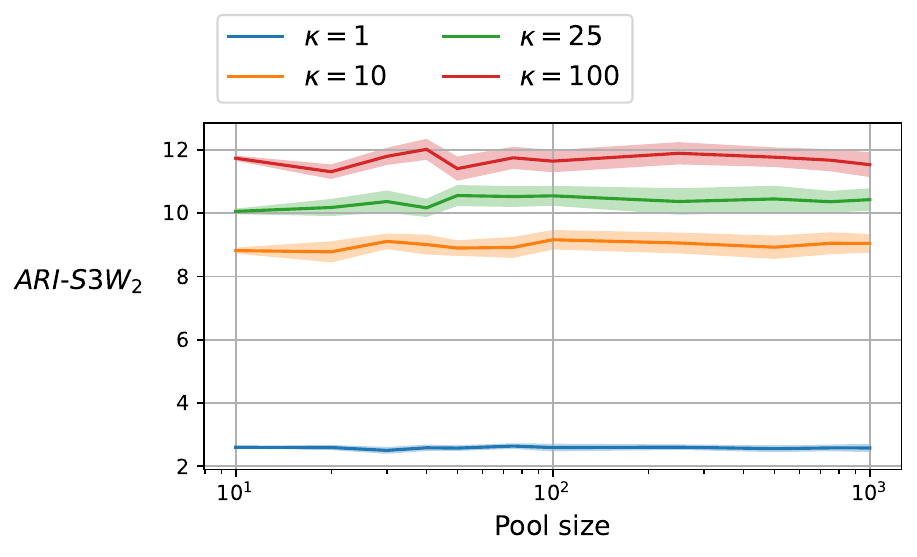}}
    \hspace*{\fill}
    \caption{Evolution of $ARI$-$S3W$ between the source and target vMFs w.r.t. pool size used. We fix $L=10$ projections and $N_R=10$ rotations. For the plot on the left, we fix $\kappa = 10$ for the target distribution, and for the plot on the right we use $d=3$.} \hfill
    \label{fig:evol_pool}
\end{figure}

\clearpage
\subsubsection{Experiment: Sample Complexity of Metrics}

We perform a series of experiments to better understand the sample complexity of $S3W$, $RI$-$S3W$, and $ARI$-$S3W$. In application, since we employ empirical probability measures to approximate their continuous versions, we seek to understand how each of the proposed metrics evolves as we vary the number of supports in the empirical distributions.

In each of the following experiments, we measured the proposed distances between $N$ samples from two von Mises-Fisher distributions, centered at the north and south pole of $\mathbb{S}^d$ with $\kappa = 10$. We compute the distance as a function of the number of samples for $d \in \{3, 10, 50, 100, 500, 1000\}$, repeating each computation over 50 iterations. The resulting average and standard deviation are visualized in Figure \ref{fig:sample_complexity}.

\begin{figure}[H]
    \centering
    \hspace*{\fill}
    \subfloat[$S3W$]{\includegraphics[width=0.33\columnwidth]{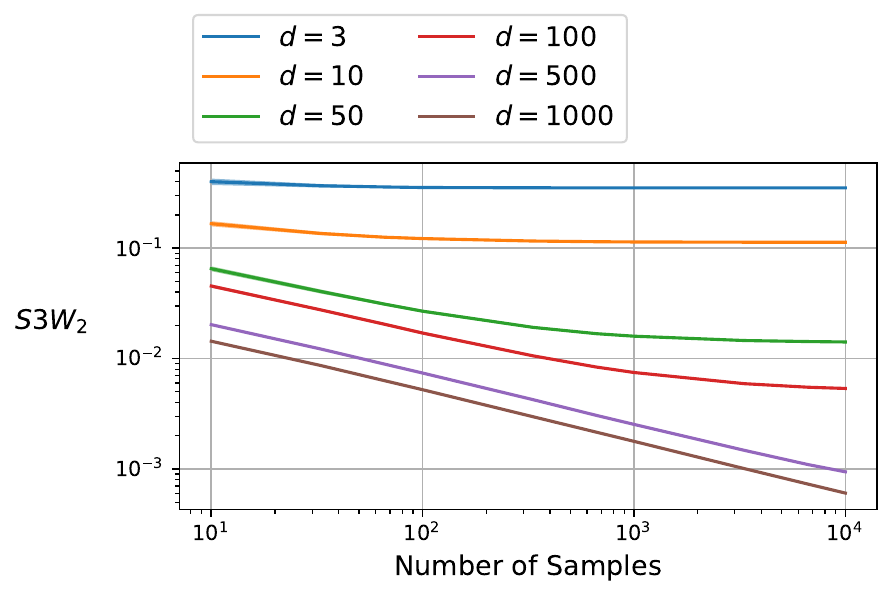}}
    \subfloat[$RI\text{-}S3W$]{\includegraphics[width=0.33\columnwidth]{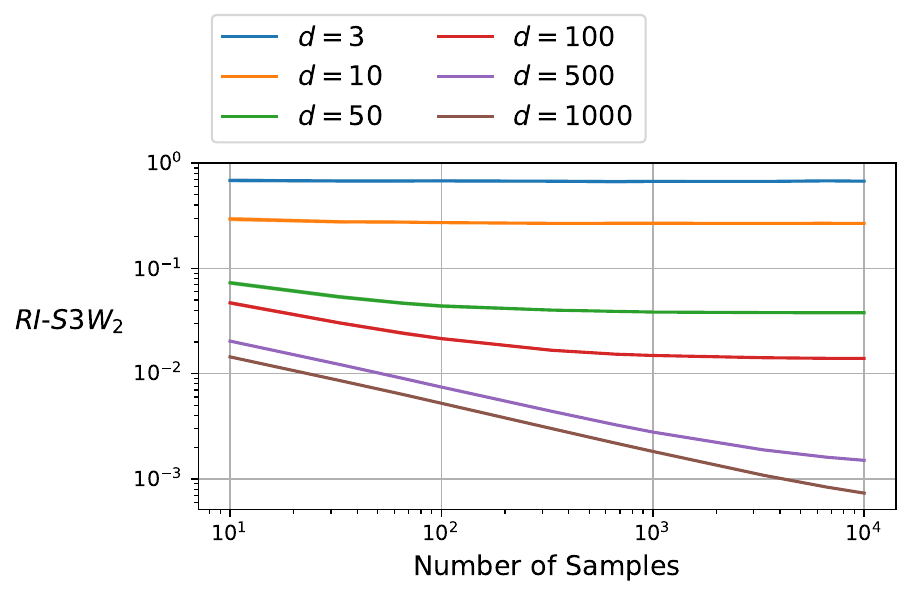}} 
    \hspace*{\fill}
    \subfloat[$ARI\text{-}S3W_2$]{\includegraphics[width=0.33\columnwidth]{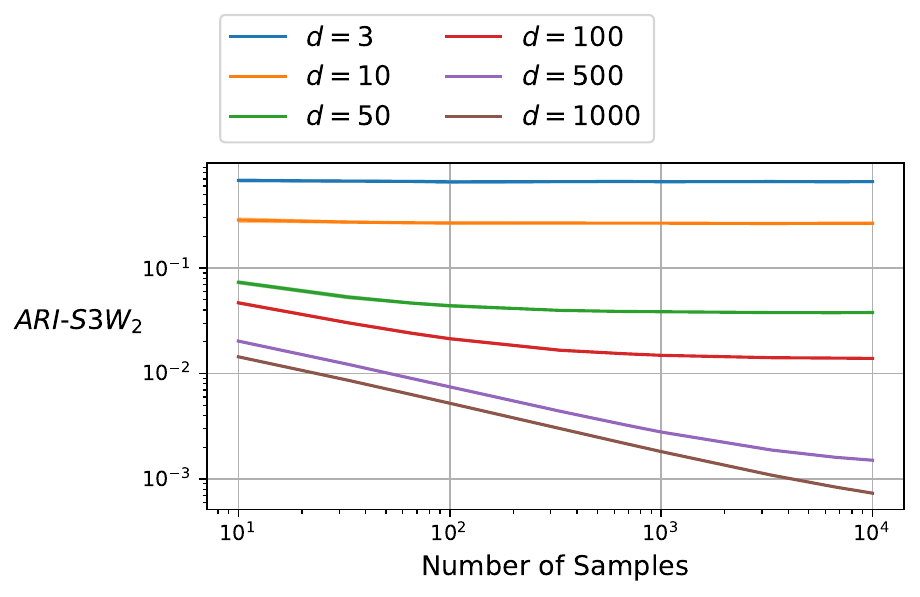}}
    \hspace*{\fill}
    
    \caption{Sample complexity of $S3W$, $RI$-$S3W$, and $ARI$-$S3W$. Distances are computed between two vMF distributions on $\mathbb{S}^d$. We use $L=1000$ projections for all metrics, $N_R=10$ rotations for $RI$-$S3W$ and $ARI$-$S3W$, and a pool size of 100 for $ARI$-$S3W$. Each distance is measured over 50 iterations.}
    \label{fig:sample_complexity}
\end{figure}

\clearpage
\subsection{Study: Runtime Analysis}
\label{section:runtime}
In this section, we provide additional runtime analysis of the proposed method w.r.t. varying parameters. Figure \ref{fig:runtime} illustrates a comparative analysis of the runtime for calculating $S3W$ and $RI\text{-}S3W$  as we vary the number of projections, samples, and rotations, across different dimensions $d$ of the data. In each of these experiments, we measure the distance between a source uniform distribution $\text{vMF}(\cdot, 0)$ and a target von Mises-Fisher distribution $\text{vMF}(\mu, 10)$. The figure indicates a linear or quasi-linear relationship between the runtime and each parameter. Moreover, in Figure \ref{fig:pool_gen_time} we illustrate the time required to generate a pool of rotation matrices in $\mathrm{SO}(d+1)$ as we vary the size of the pool and the dimension $d$. As expected, the runtime is linear in the size of rotation pool.

\begin{figure}[H]
    \centering
    \hspace*{\fill}
    \subfloat[]{\label{fig:runtime_a}\includegraphics[width=0.33\columnwidth]{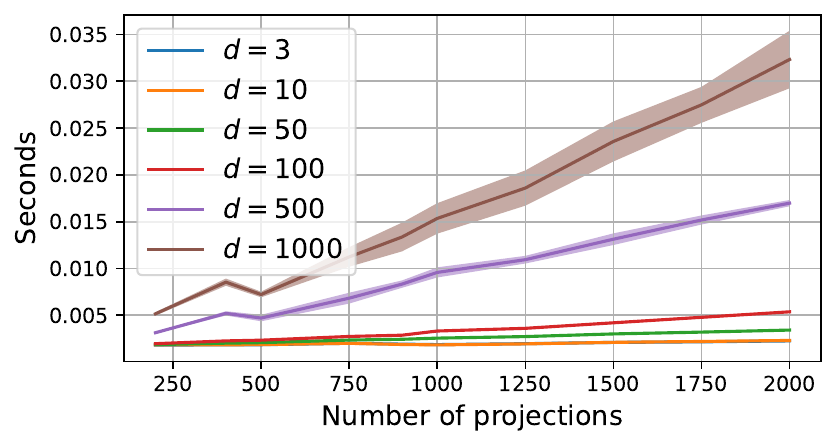}} \hfill 
    \subfloat[]{\label{fig:runtime_b}\includegraphics[width=0.33\columnwidth]{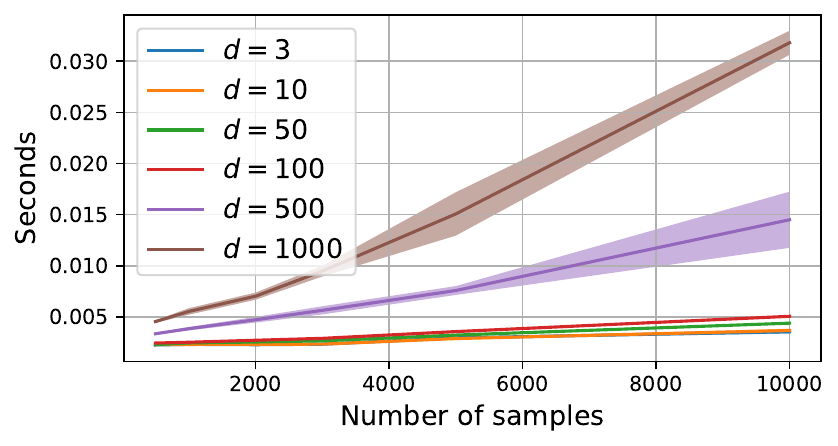}} \hfill
    \subfloat[]{\label{fig:runtime_c}\includegraphics[width=0.33\columnwidth]{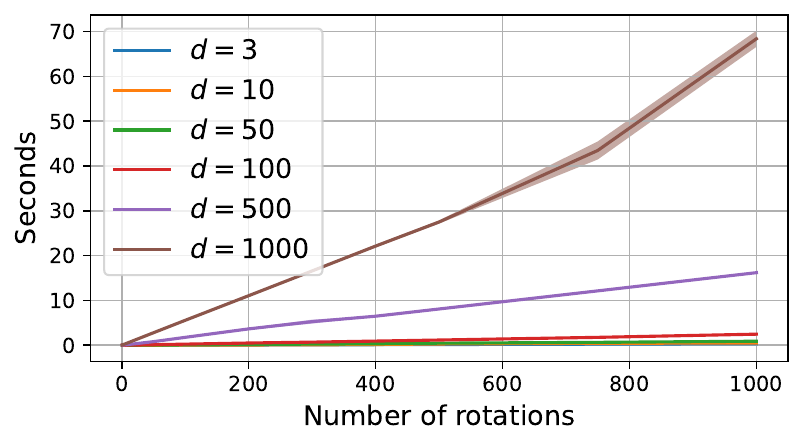}} \hfill
    \caption{Runtime w.r.t. varying number of projections, samples, and rotations, respectively. In \ref{fig:runtime_a}, \ref{fig:runtime_c} we use $N=500$ samples of each distribution; in \ref{fig:runtime_b},\ref{fig:runtime_c} we fix $L=100$ projections; and in \ref{fig:runtime_a},\ref{fig:runtime_b} we do not use any rotations.}
    \label{fig:runtime}
\end{figure}

\begin{figure}[H]
    \centering
    \includegraphics[width=0.33\columnwidth]{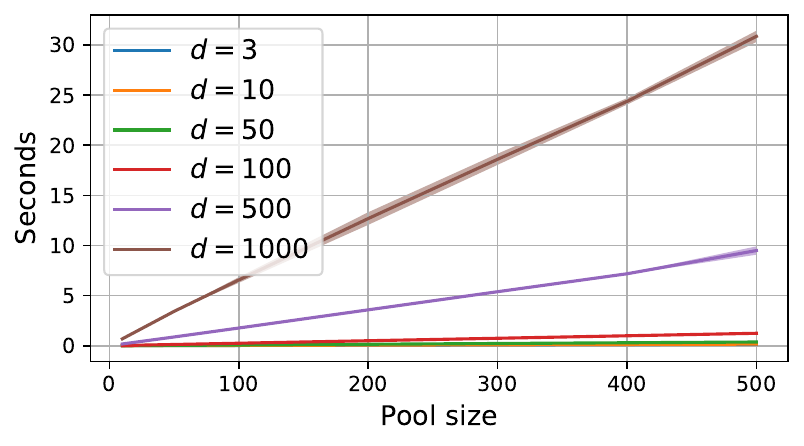}
    \caption{Runtime to generate pool of rotation matrices of varying sizes and dimension.}
    \label{fig:pool_gen_time}
\end{figure}

\subsection{Task: Density Estimation with Normalizing Flows}
\label{section:de}
\subsubsection{Background Overview}
\textbf{Normalizing flows} offer a framework for constructing complex probability distributions from simple ones through invertible transformations. It works by restricting the model class such that the likelihood function can be evaluated via simple sampling. Suppose we have access to the samples $\{x_i \sim \mu\}_{i=1}^n$ where $\mu$ is our target distribution. Let $p_Z$ be a base distribution over the latent $z$ (i.e. a Gaussian). We aim to find an invertible transformation $T$ such that the pushforward $T_{\#}p_{\mu}=p_Z$. Then, the density $p_{\mu}$ can be obtained through the change of variable formula

\begin{equation}
    p_{\mu}(x) = p_Z(T(x))|\det J_{T}(x)|,
\end{equation}

where $|\text{det}J_{T}(x)|$ denotes the determinant of the Jacobian of $T$ at $x$.

More concretely, consider the dataset $\mathcal{D}_{\text{train}} = \{x_1,...,x_n\}$ sampled i.i.d from the empirical measure $\hat{\mu}=\frac{1}{n}\sum_{i=1}^{n}\delta_{x_i}$ where $\delta_{x_i}$ is the Dirac delta at $x_i$. The log likelihood for training the neural network is given by

\begin{equation}
    \log p(\mathcal{D_{\text{train}}}|w) = \sum_{i=1}^n \log p_X(x_i|w) = \sum_{i=1}^n \left \{ \log p_Z(g(x_i,w)) + \log |\det J(x_i)| \right \}
\end{equation}

Here, $T^{-1}$ represents the inverse transformation of $T$, and $|\det J(x_i)|$ is the determinant of the Jacobian of $T^{-1}$ at $x_i$.

\noindent\textbf{Real NVP and Stereographic Projection-based Normalizing Flows with Real NVP:}

Real NVP (Real-valued Non-Volume Preserving) \cite{dinh2016density} uses a series of invertible layers to transform the base distribution into a more complex one. Each layer, known as a coupling layer, is designed such that its Jacobian is easy to compute. Specifically, given an input vector $x \in \mathbb{R}^m$, the coupling layer splits it into two parts $x_1 \in \mathbb{R}^{m_1}$ and $x_2 \in \mathbb{R}^{m_2}$ such that $m_1 + m_2 = m$. The transformation applied to $x_2$ is then conditioned on $x_1$ and defined by

\begin{equation}
    x'_2 = \text{exp}(s(x_1)) \odot x_2 + t(x_1),
\end{equation}

where $s: \mathbb{R}^{m_1} \rightarrow \mathbb{R}^{m_2}$ and $t: \mathbb{R}^{m_1} \rightarrow \mathbb{R}^{m_2}$ are scale and translation functions realized through neural networks. The operation $\odot$ denotes the Hadamard product. The key advantage of Real NVP is the traceability of both its inverse and the Jacobian determinant, that is

\begin{equation}
    \log |\det J_T(x)| = \sum s(x_1)
\end{equation}

which makes it a good candidate for high-dimensional density estimation.

To generalize to $\mathbb{S}^d$, \cite{gemici2016normalizing} proposes using a transformation $T$ composed of a stereographic projection $\phi$ (see section \ref{sec:background}) followed by a real NVP flow $f$, denoted as $T=f\circ \phi$. The log density of the target distribution under this transformation is given by

\begin{equation}
\log p(x) = \log p_Z(z) + \log |\det J_f(z)| - \frac{1}{2} \log |\det (J_T \phi^{-1} J_{\phi^{-1}}(\phi(x)))|,
\end{equation}

where $p_Z$ is the density of a prior on $\mathbb{R}^d$, typically a standard Gaussian, and $J_f(z)$ is the Jacobian of the Real NVP flow. The log det term accounts for the change in volume due to the stereographic projection $\phi$ and its inverse $\phi^{-1}$.

\noindent\textbf{Exponential Map and Exponential Map Normalizing Flows:} 

Exponential maps serve as a tool for mapping data residing on curved spaces. In essence, it maps a vector in the tangent space of a point on a manifold $\mathcal{M}$ to the manifold itself. Specifically, for a point $x\in \mathcal{M}$ and a tangent vector $v_x$, the exponential map $\text{exp}_x(v)$ traces a geodesic (i.e. shortest path) on $\mathcal{M}$. For the spherical manifold $\mathbb{S}^d$, it is given by

\begin{equation}
\label{eqn:exp_map}
    \text{exp}_x(v) = x \cos(|v|) + \frac{v}{|v|} \sin(|v|).
\end{equation}

Inspired by this, \cite{rezende2020normalizing} introduces the exponential map normalizing flows, which ensures that the construction of transformations happens directly on the manifold, avoiding the need for intermediate, non-diffeomorphic set mappings to Euclidean spaces (i.e. via the Stereographic projection). In this frame work, the transformation $T$ is defined by applying the exponential map to the projected gradient of a scalar field $\phi(x)$, that is

\begin{equation}
    T(x) = \text{exp}_x(\text{Proj}_x(\nabla \phi(x)))
\end{equation}

where $\text{Proj}_x$ projects $\nabla\phi(x)$ onto the tangent space at $x$, $T_x\mathbb{S}^d$. The scalar field $\phi(x)$ is a convex combination of simple functions:

\begin{equation}
\phi(x) = \sum_{i=1}^K \alpha_i \beta_i e^{\beta_i(x^T \mu_i - 1)},
\end{equation}

with constraints $\alpha_i \geq 0$, $\sum_{i=1}^K \alpha_i \leq 1$, each $\beta_i > 0$, and $\mu_i \in \mathbb{S}^d$. The density update is then computed as 

\begin{equation}
p(T(x)) = \frac{\pi(x)}{\sqrt{\det(E(x)^T J_T(x)^T J_T(x) E(x))}},
\end{equation}

where $J_T(x)$ is the Jacobian of $T$ at $x$, and $E(x)$ forms an orthonormal basis of $T_x\mathbb{S}^d$.

\subsubsection{Experiment: Estimating Density on the Earth Datasets}

\begin{wraptable}{r}{5.5cm} 
    \centering
    % Adjust vertical space above the table if needed
    \vspace{-5mm}
    \caption{The Earth datasets.}
    \label{tab:earth_datasets}
    \vspace{-2mm}
    \small
    \begin{tabular}{@{}lccc@{}}
        \toprule
        & \textbf{Quake} & \textbf{Flood} & \textbf{Fire} \\ 
        \midrule
        Train & 4284 & 3412 & 8966 \\
        Test & 1836 & 1463 & 3843 \\
        Total & 6120 & 4875 & 12809 \\
        \bottomrule
    \end{tabular}
    % Adjust vertical space below the table if needed
    % \vspace{-5mm}
\end{wraptable}

Our main focus is on $\mathbb{S}^2$ for the purpose of demonstration and visualization. Similar to \cite{bonet2022spherical}, we use the three datasets introduced by \cite{mathieu2020riemannian}, representing the earth's surface as a perfect spherical manifold: volcano eruptions \cite{NOAA2020earthquake}, earthquakes \cite{NOAA2020earthquake}, and floods \cite{Brakenridge2017FloodArchive}. Our objective is to learn a normalizing flows model capable of reporting exact density at any point on the sphere, so that we can predict the likelihood of future events based on past data. 

\begin{figure}[H]
    \centering
    \hspace*{\fill}
    \subfloat[Earthquake]{\includegraphics[width=0.32\columnwidth]{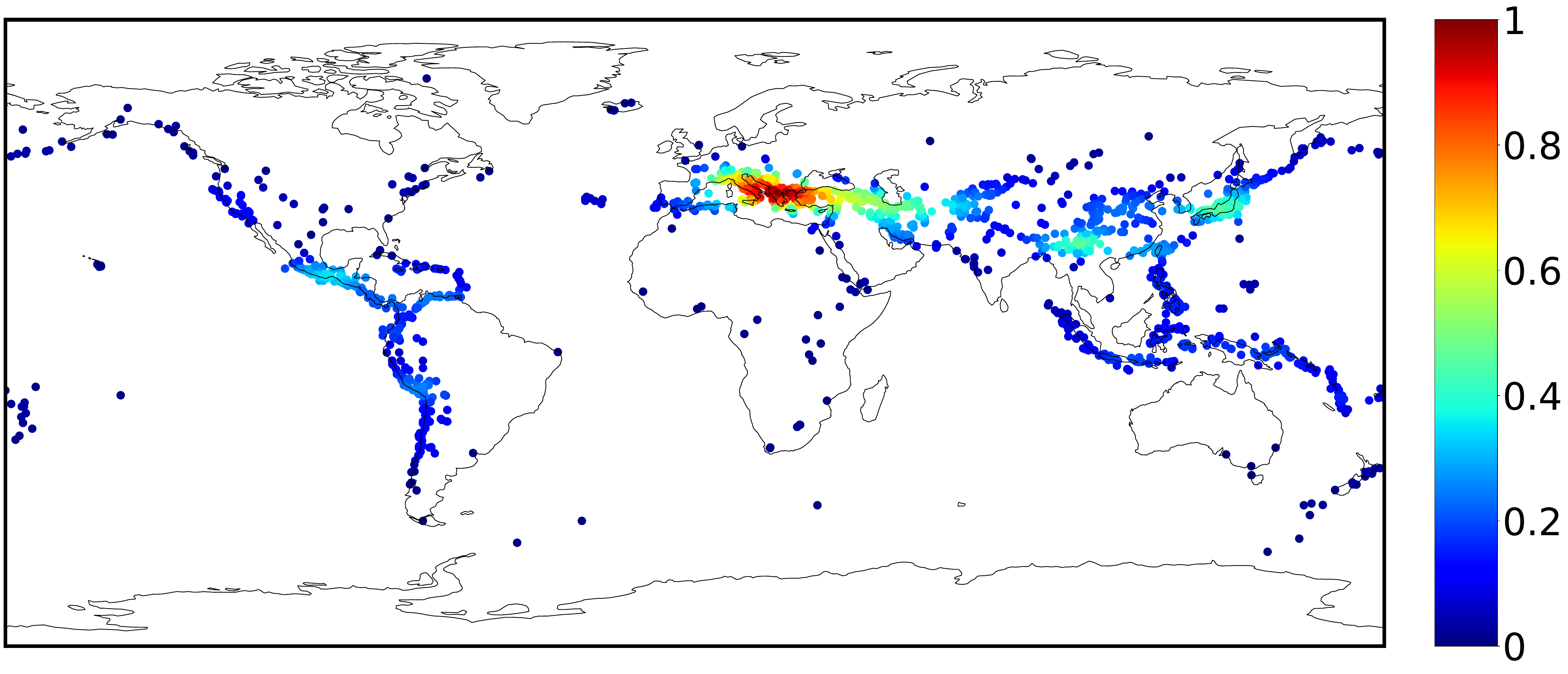}} \hfill 
    \subfloat[Flood]{\includegraphics[width=0.32\columnwidth]{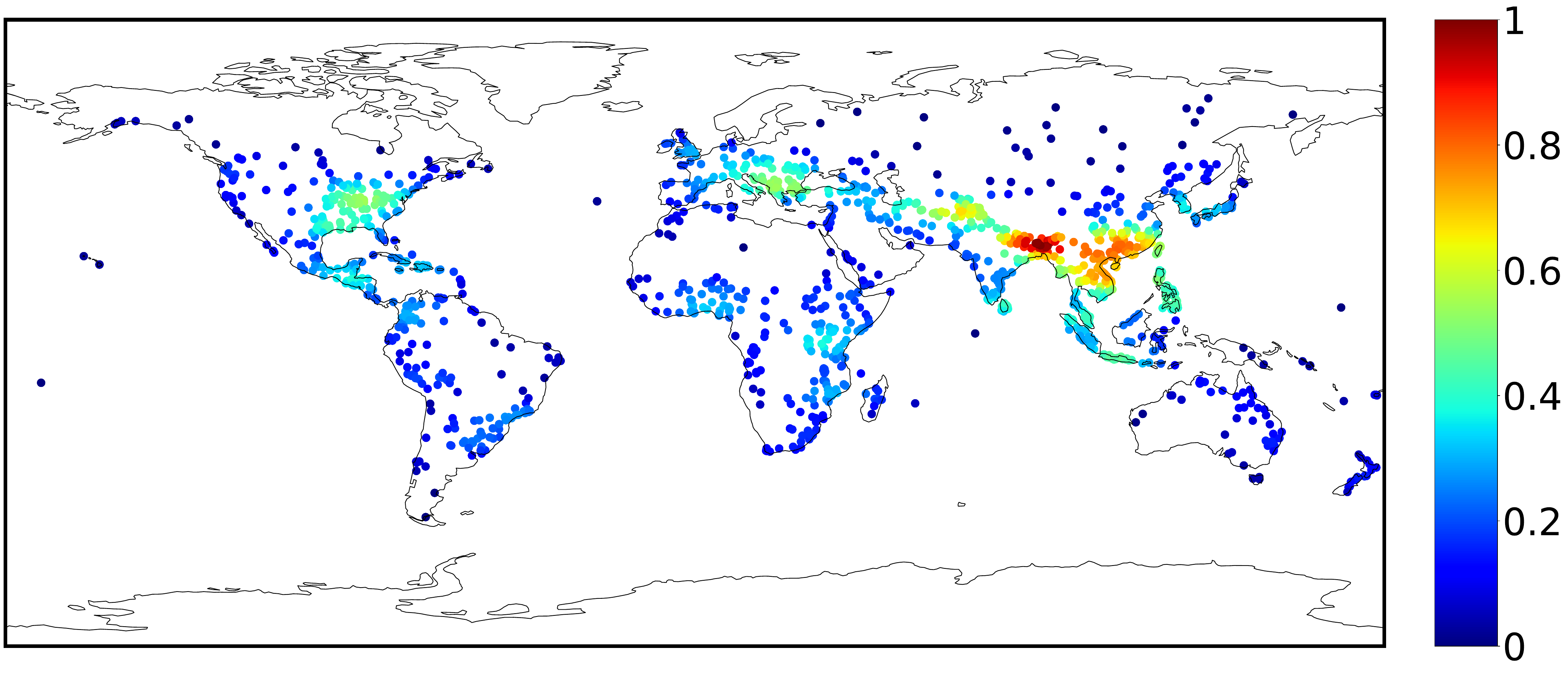}} \hfill
    \subfloat[Fire]{\includegraphics[width=0.32\columnwidth]{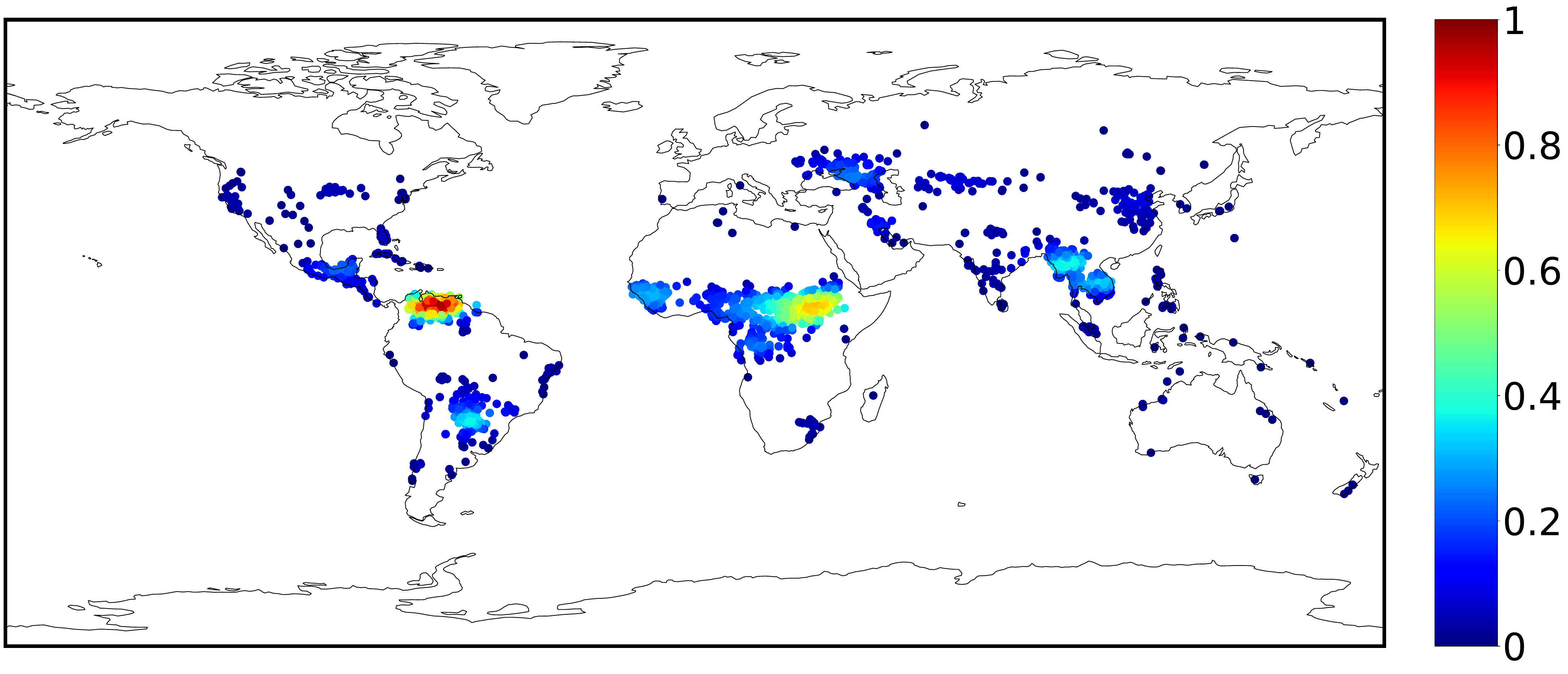}} \hfill
    \caption{Groundtruth as estimated with KDE (bandwidth $0.1$) using test data. Likelihoods are min-max normalized, mapping the smallest value to $0$ and the largest to $1$.} 
    % \vspace{-5mm}
\end{figure}

Specifically, we demonstrate two approaches for training a normalizing flow model on spherical data. The first of which follows \cite{gemici2016normalizing}, stereographically projecting the data from $\mathbb{S}^2$ onto $\mathbb{R}^2$ and then training a Real NVP model on $\mathbb{R}^2$. The second utilizes spherical exponential maps \cite{rezende2020normalizing} to train a normalizing flow natively on $\mathbb{S}^2$, using both $SW$ and $SSW$ \cite{bonet2022spherical} as baseline losses and comparing with our proposed S3W. 

% \cite{gemici2016normalizing} also employs the Stereographic Projection. However, they learn the NF model by projecting the manifold $\mathbb{S}^d$ to $\mathbb{R}^d$, transform the density, and then project it back to $\mathbb{S}^d$. We replicate their experiment and show NLL results in (table), comparing with SSW, SW, the $S3W$ variants (ours).

\textbf{Implementation.}
For exponential map normalizing flows, we construct the model using $48$ radial blocks, each consisting of $100$ components, totaling $24000$ trainable parameters, similar to \cite{bonet2022spherical}. Each block in the flow applies a transformation defined by a potential function $V$. The model projects the gradient of the potential function onto the sphere's tangent space and then maps it back to the sphere using the spherical exponential map. 

For Stereographic Projection-based normalizing flows (Stereo-NF), we first project the manifold $\mathbb{S}^2$ onto $\mathbb{R}^2$ using the stereographic projection. Then, we train a Real NVP model, consisting of $10$ Real NVP blocks with scaling and translation parameterized by an MLP ($10$ layers of $25$ neurons each, $\text{LeakyReLU}_{0.2}(\cdot)$ activation). This setup follow \cite{bonet2022spherical} and has $27520$ parameters. We train the models for $20000$ epochs each using full batch gradient descent via the Adam optimizer. For exponential map NF, we use the learning rate of $0.1$ for $SSW$ and $SW$ as in the original setup and $0.05$ for our $S3W$and its variants. For Stereo-NF, we use a learning rate of $0.001$.

\textbf{Results.} We provide normalized density maps of events, with the color scale representing the relative density on a scale from $0$ to $1$. We note that this normalization scheme does not change the underlying distribution but allows for an interpretable visualization, showing the relative differences in densities on a fixed scale. We observe that our proposed $S3W$ and its variants put relative density mass more accurately. 

Table \ref{table:earth_density_nll} provides the Negative Log-likelihood values for different distances. Despite having a higher parameter count, the increase in model complexity does not necessarily translate to better performance. This shows that estimating density natively on the sphere is superior for spherical data per this metric.

\clearpage

% \begin{figure}[]
%     \hspace*{\fill}
%     \subfloat[Earthquake]{\includegraphics[width=0.32\columnwidth]{Figures/DensityEstimation/KDE_Quakes.pdf}} \hfill 
%     \subfloat[Flood]{\includegraphics[width=0.32\columnwidth]{Figures/DensityEstimation/KDE_Flood.pdf}} \hfill
%     \subfloat[Fire]{\includegraphics[width=0.32\columnwidth]{Figures/DensityEstimation/KDE_Fire.pdf}} \hfill
%     \caption{Groundtruth as estimated with KDE (bandwidth $0.1$) using test data.} \hfill
%     \label{fig:swvi_power_spherical}
%     \vspace{-5mm}
% \end{figure}

\begin{figure}[H]
\centering
\subfloat[Earthquake]{\includegraphics[width=0.32\columnwidth]{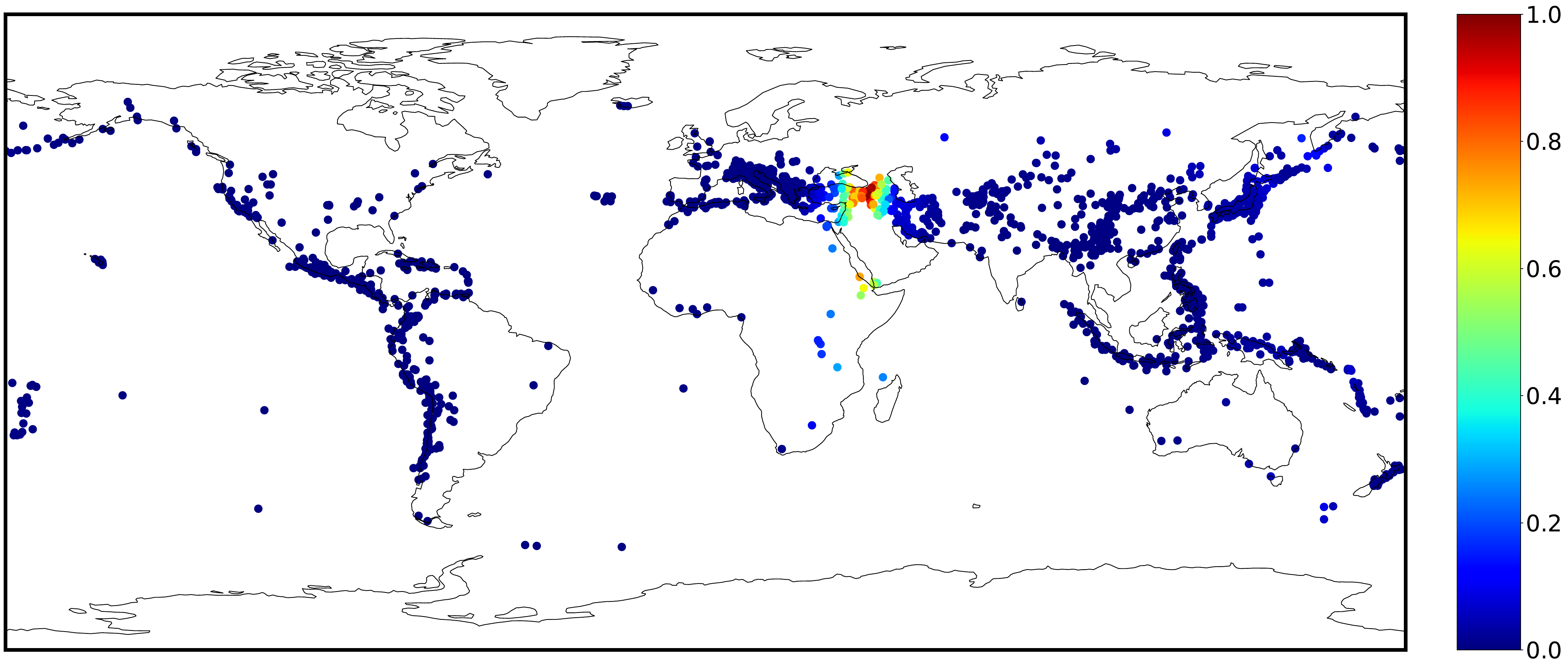}}
\subfloat[Flood]{\includegraphics[width=0.32\columnwidth]{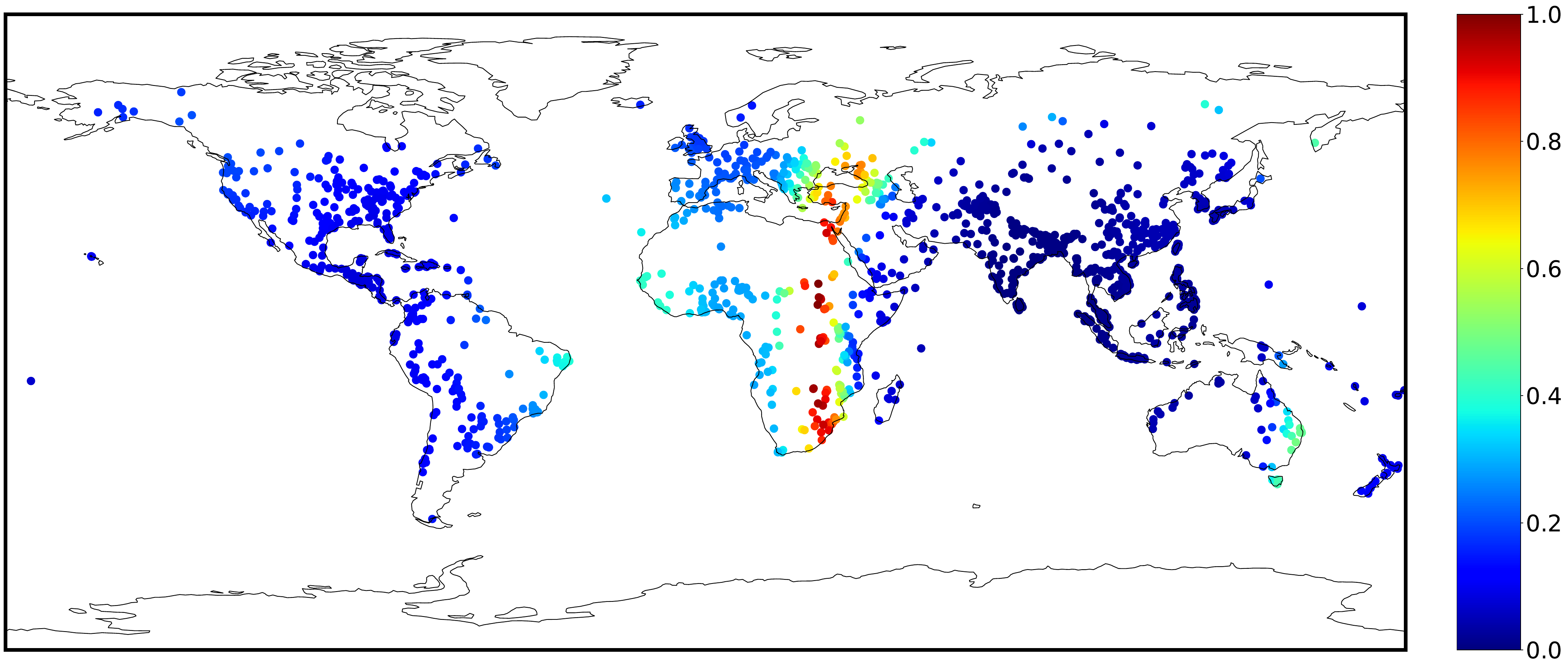}}
\subfloat[Fire]{\includegraphics[width=0.32\columnwidth]{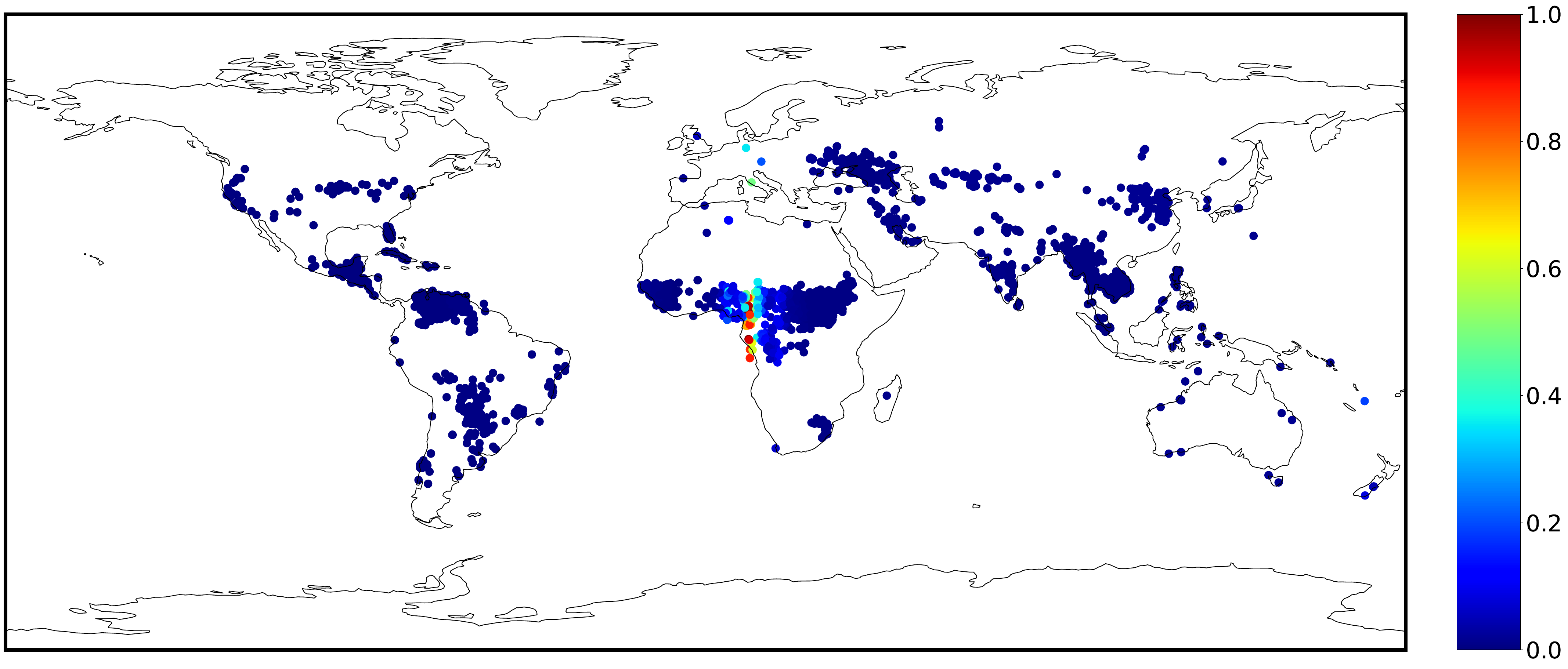}}
\caption{Stereo + RealNVP.}
\vspace{-.2in}
\end{figure}

\begin{figure}[H]
\centering
\includegraphics[width=0.32\columnwidth]{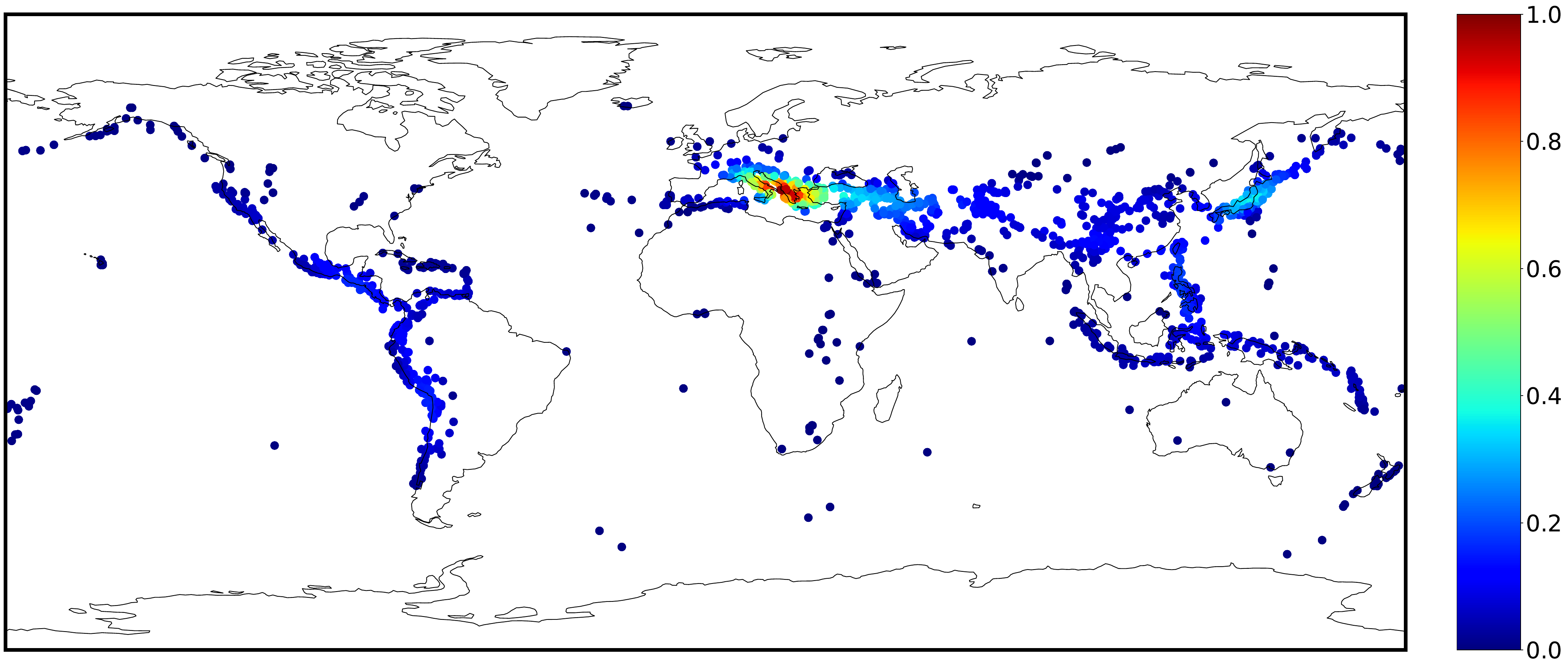}
\includegraphics[width=0.32\columnwidth]{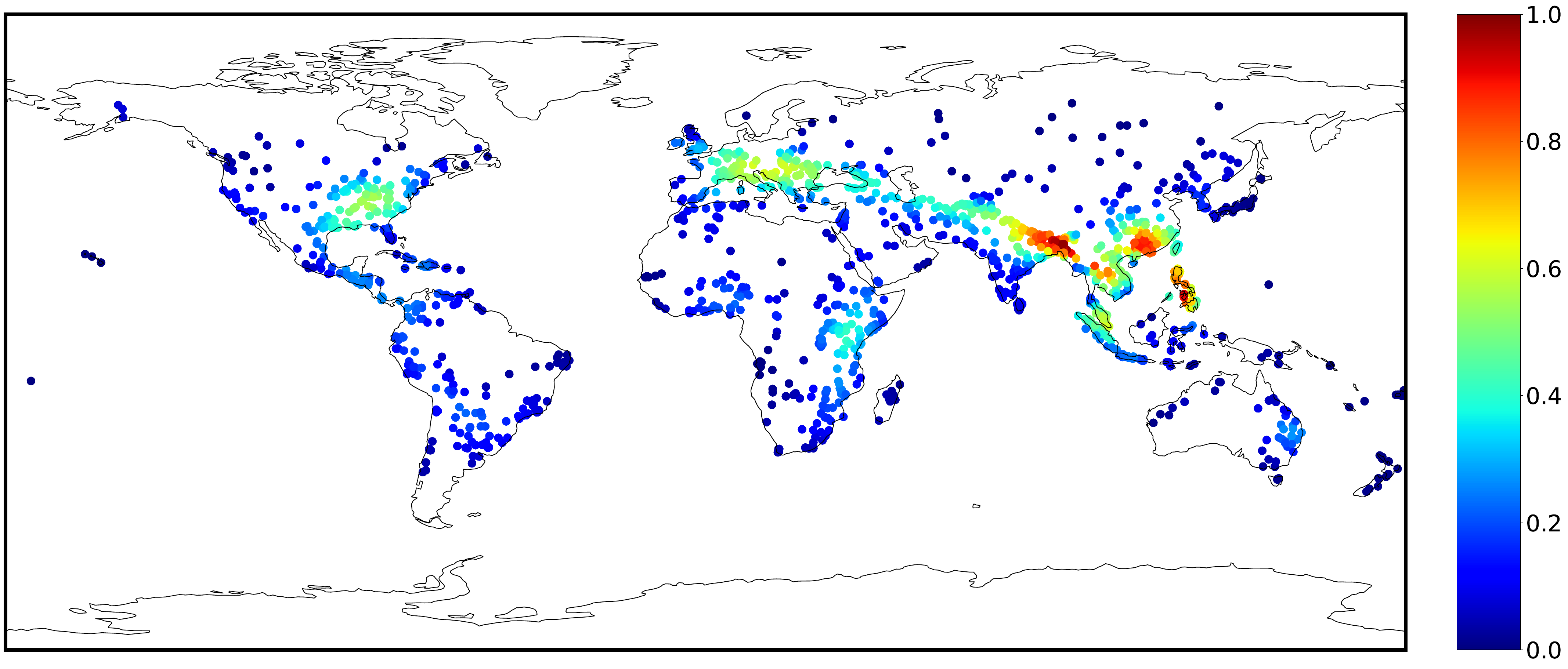}
\includegraphics[width=0.32\columnwidth]{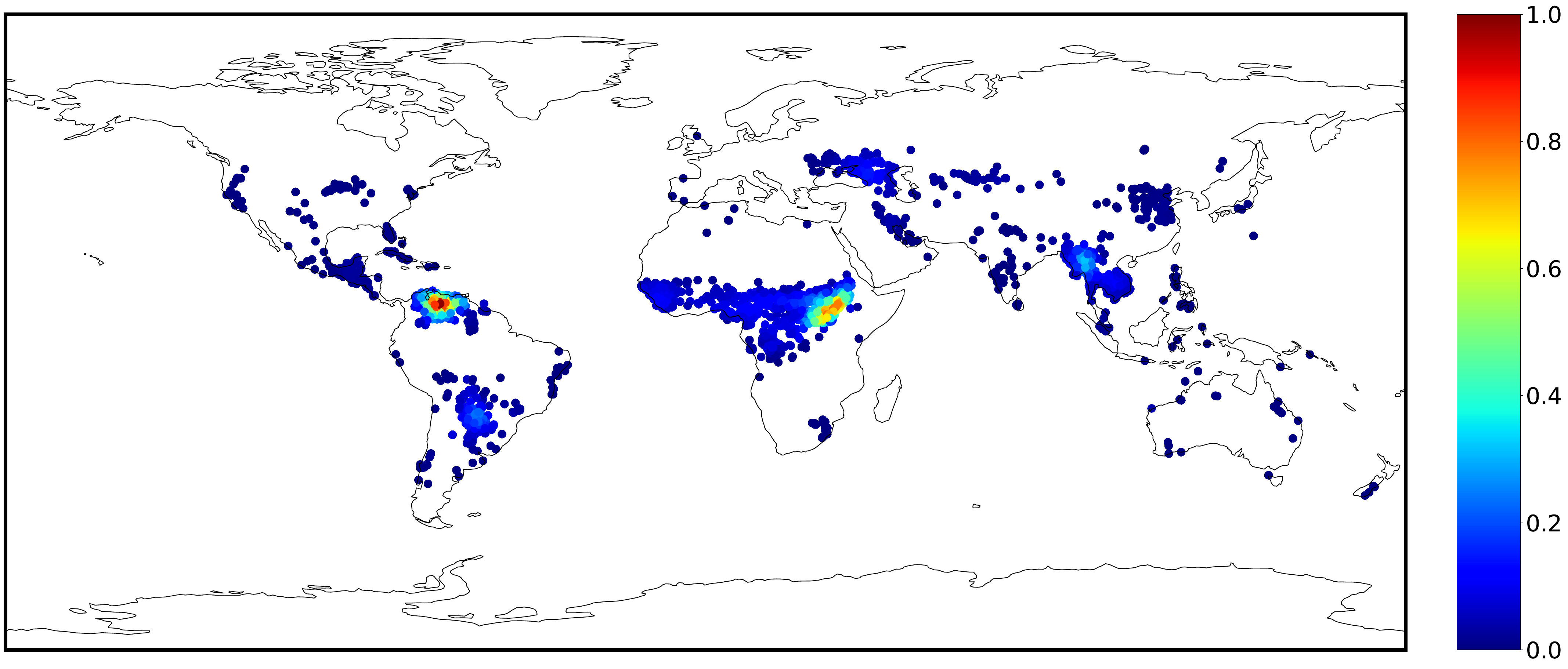}
\caption{Sliced-Wasserstein ($SW$).}
\vspace{-.2in}
\end{figure}

\begin{figure}[H]
\centering
\includegraphics[width=0.32\columnwidth]{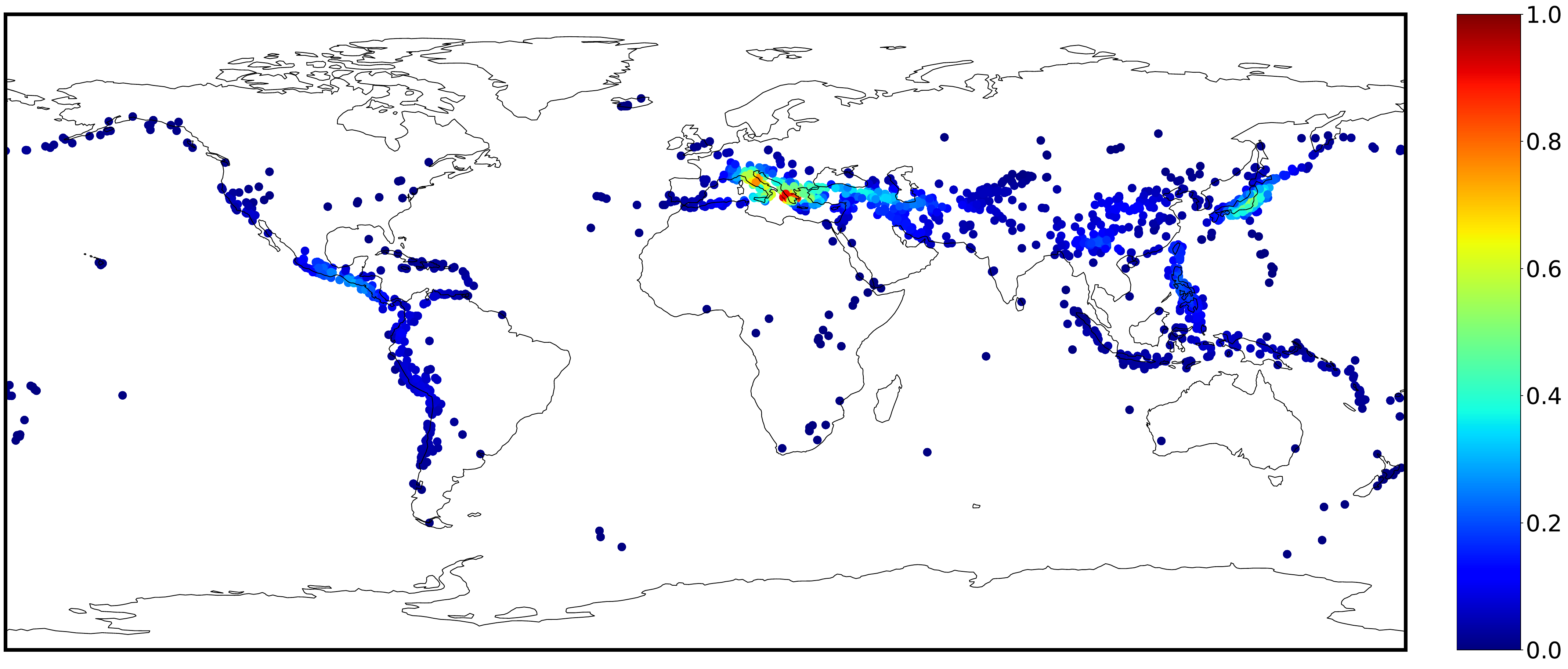}
\includegraphics[width=0.32\columnwidth]{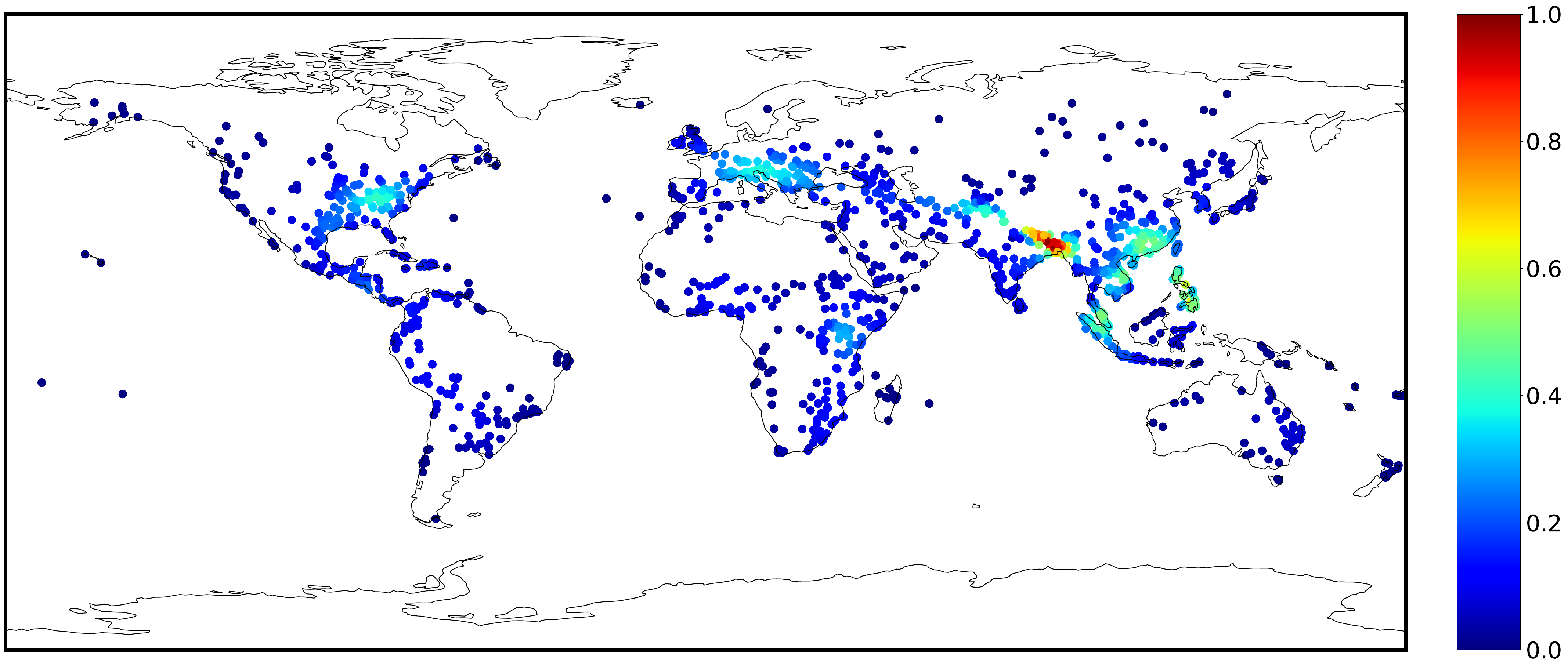}
\includegraphics[width=0.32\columnwidth]{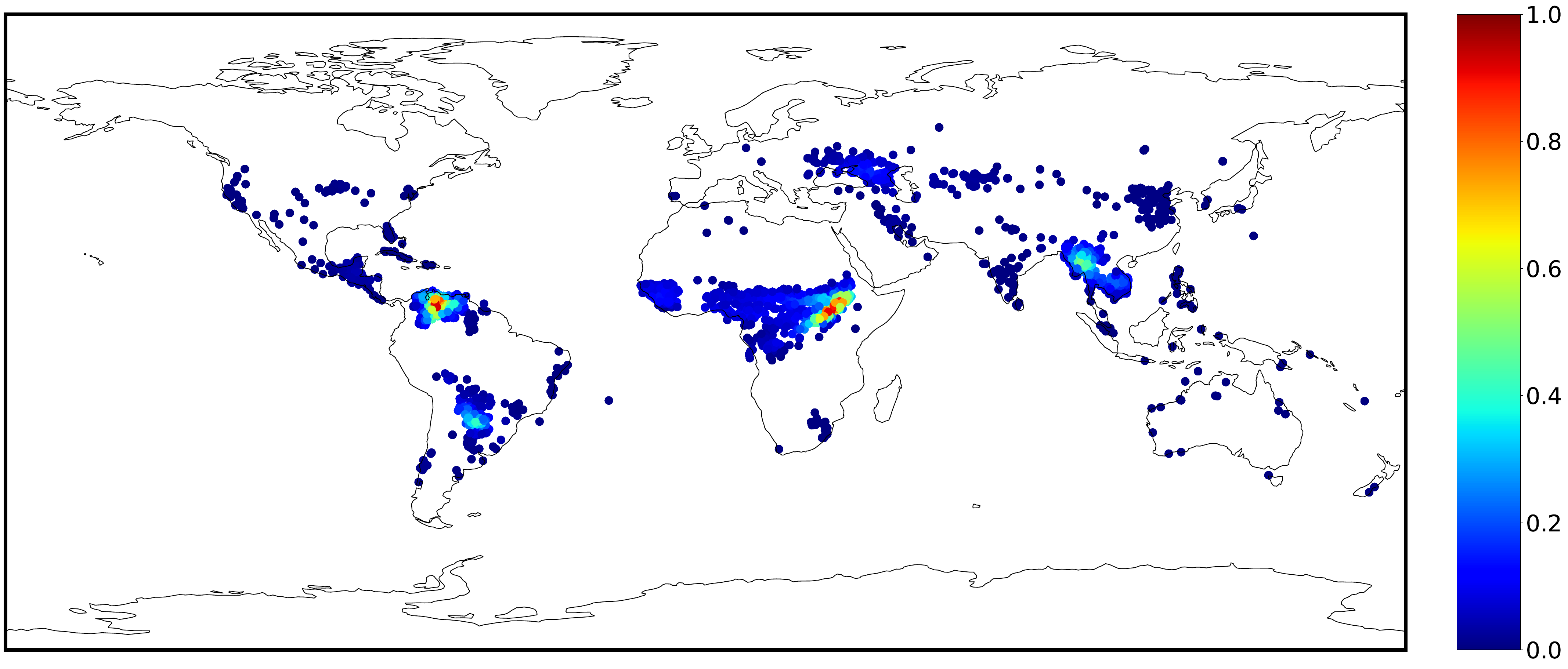}
\caption{Spherical Sliced-Wasserstein ($SSW$).}
\vspace{-.2in}
\end{figure}

\begin{figure}[H]
\centering
\includegraphics[width=0.32\columnwidth]{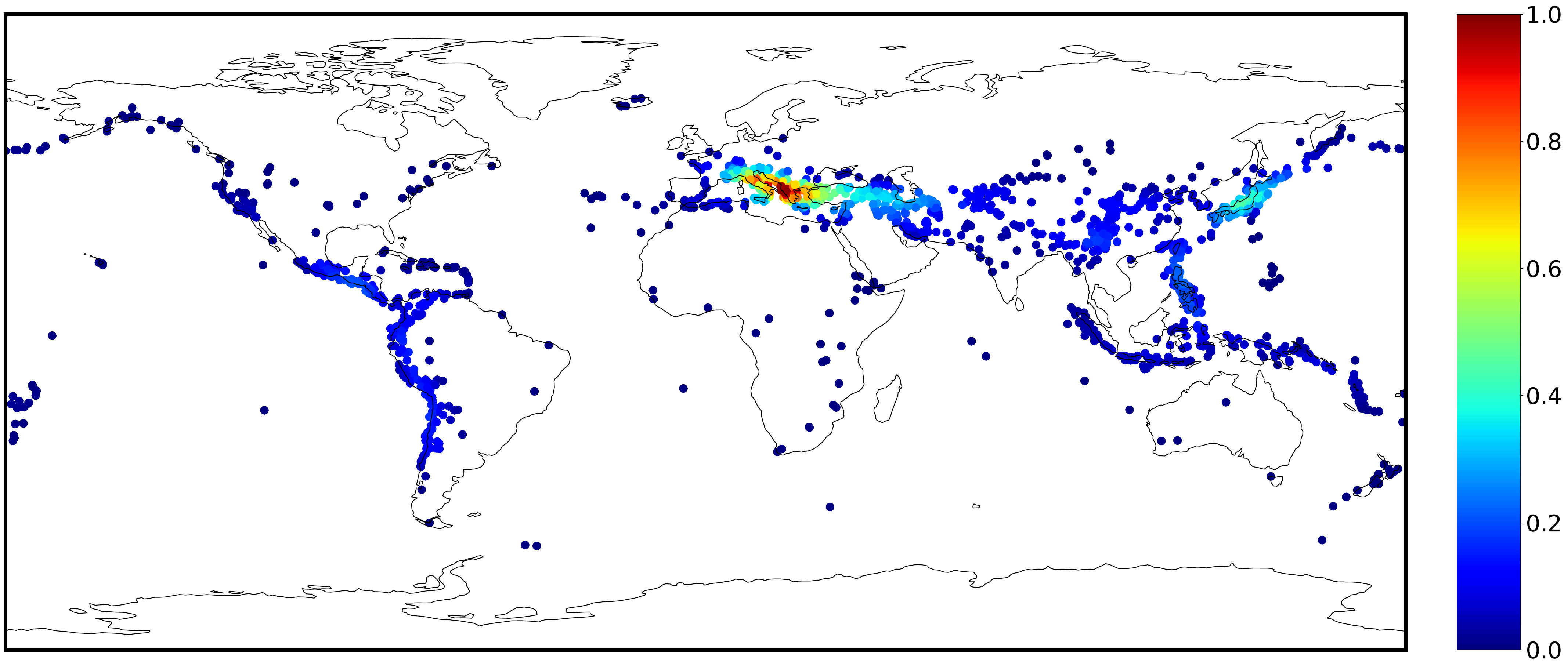}
\includegraphics[width=0.32\columnwidth]{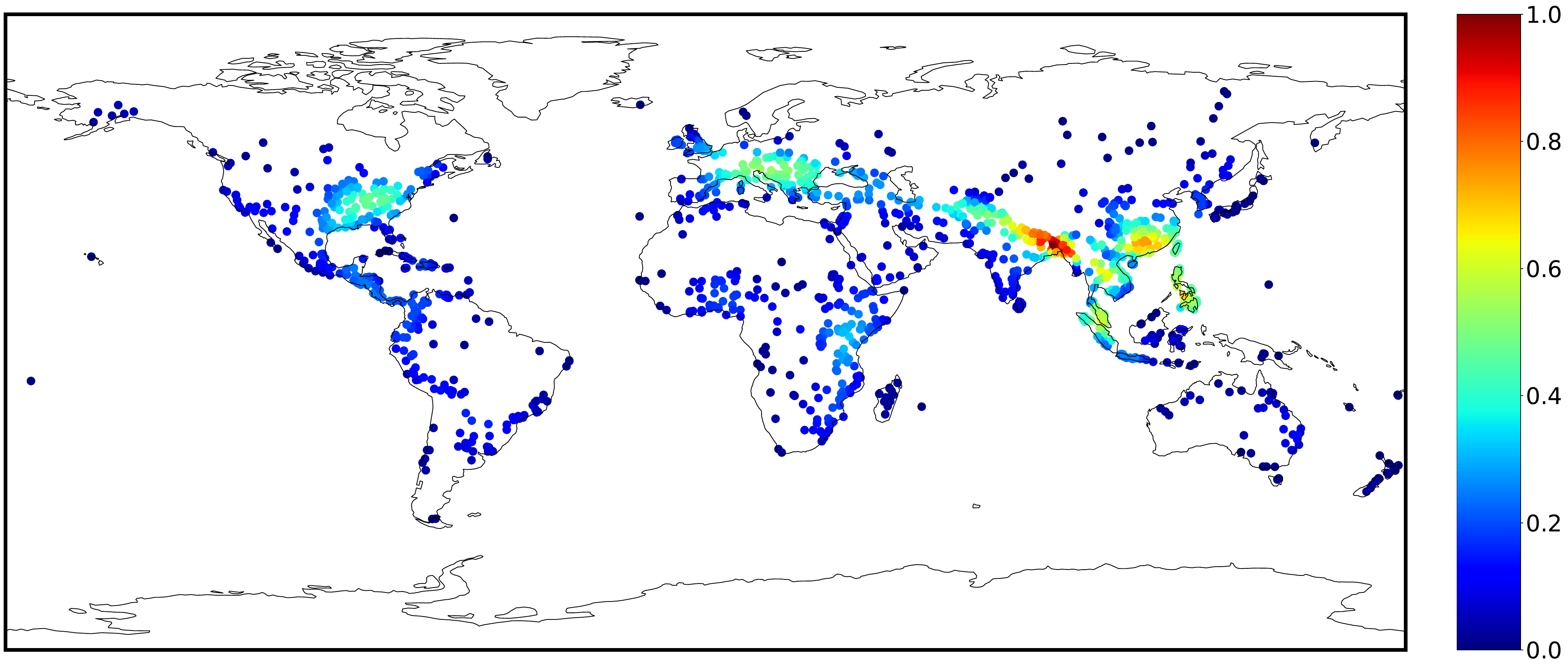}
\includegraphics[width=0.32\columnwidth]{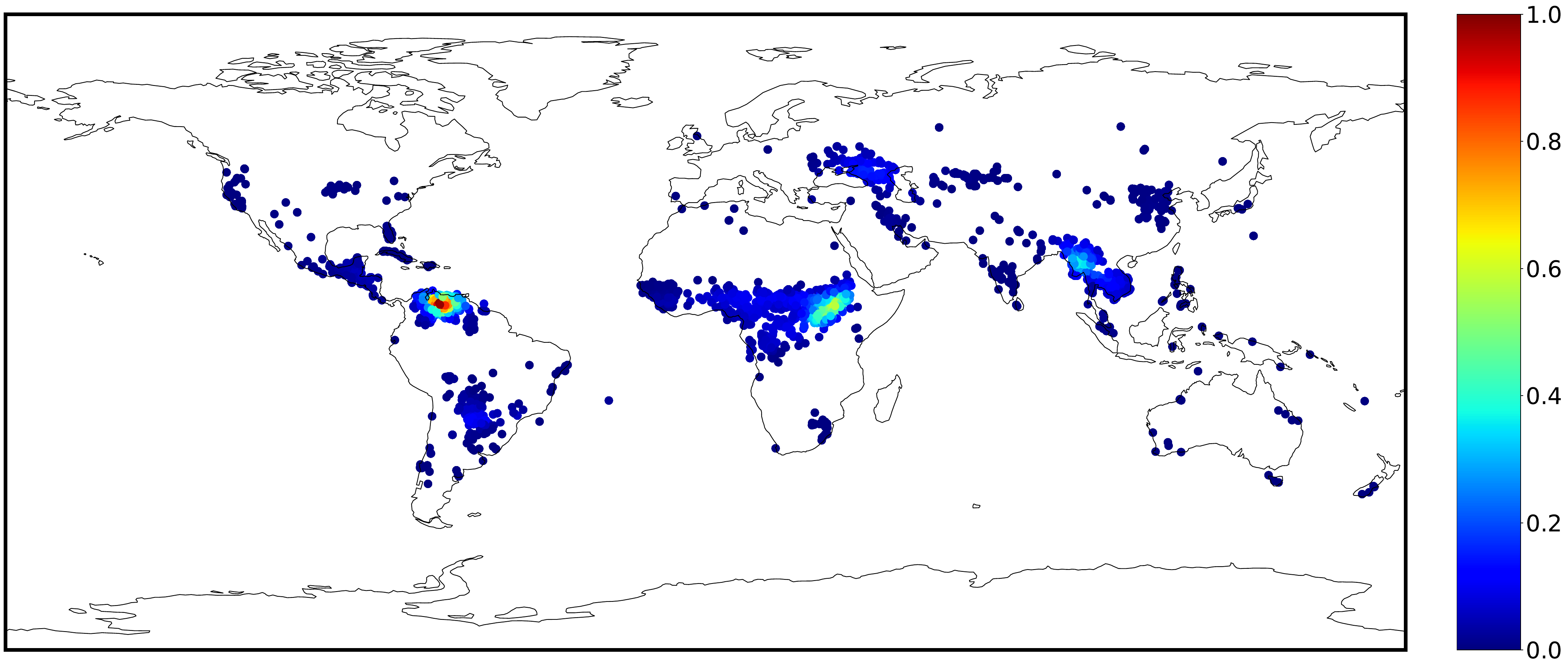}
\caption{Stereographic Spherical Sliced-Wasserstein ($S3W$).}
\vspace{-.2in}
\end{figure}

\begin{figure}[H]
\centering
\includegraphics[width=0.32\columnwidth]{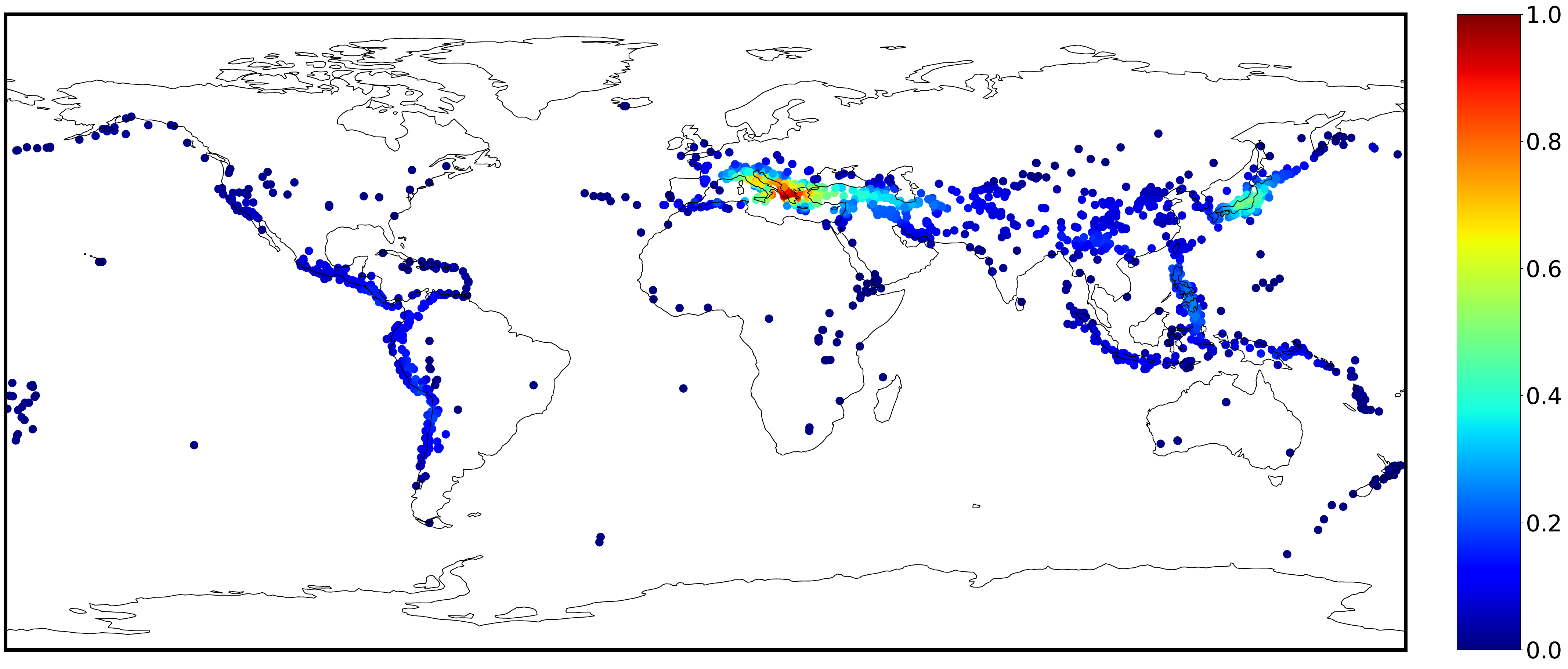}
\includegraphics[width=0.32\columnwidth]{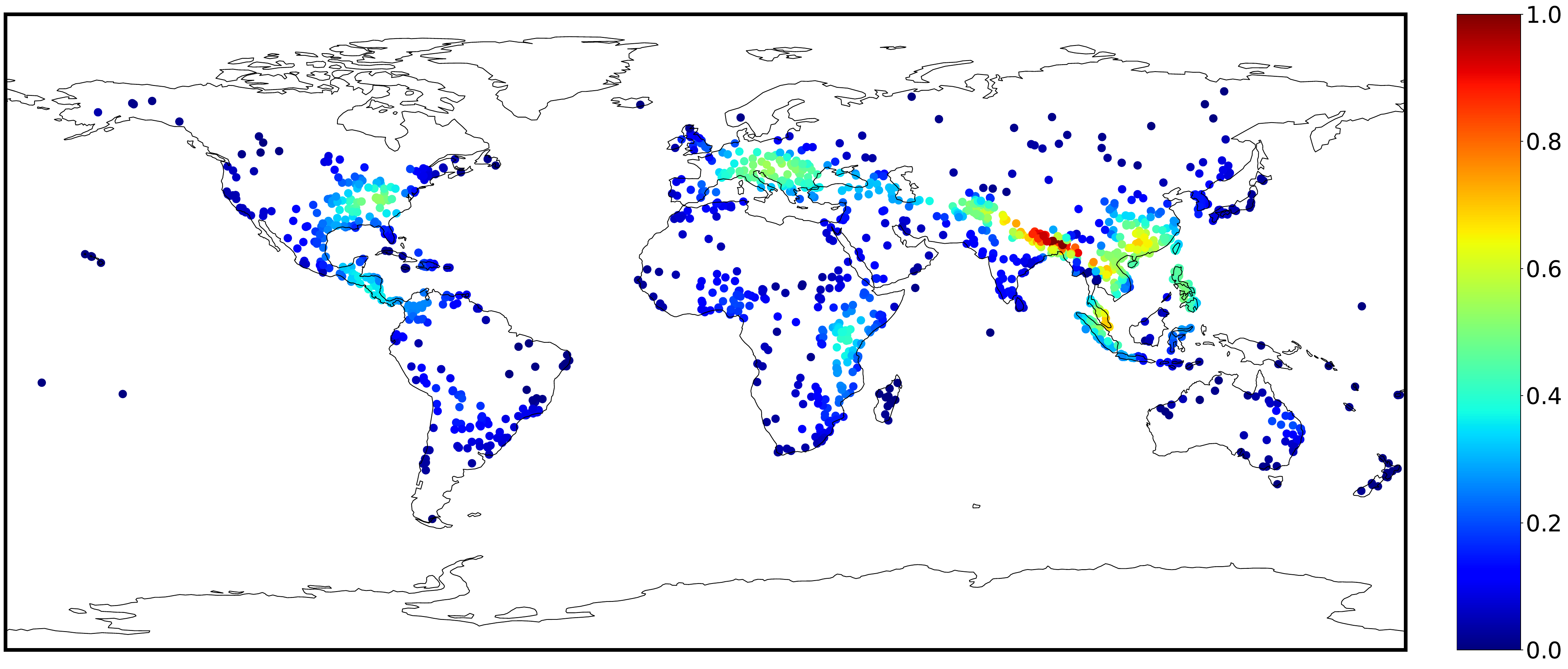}
\includegraphics[width=0.32\columnwidth]{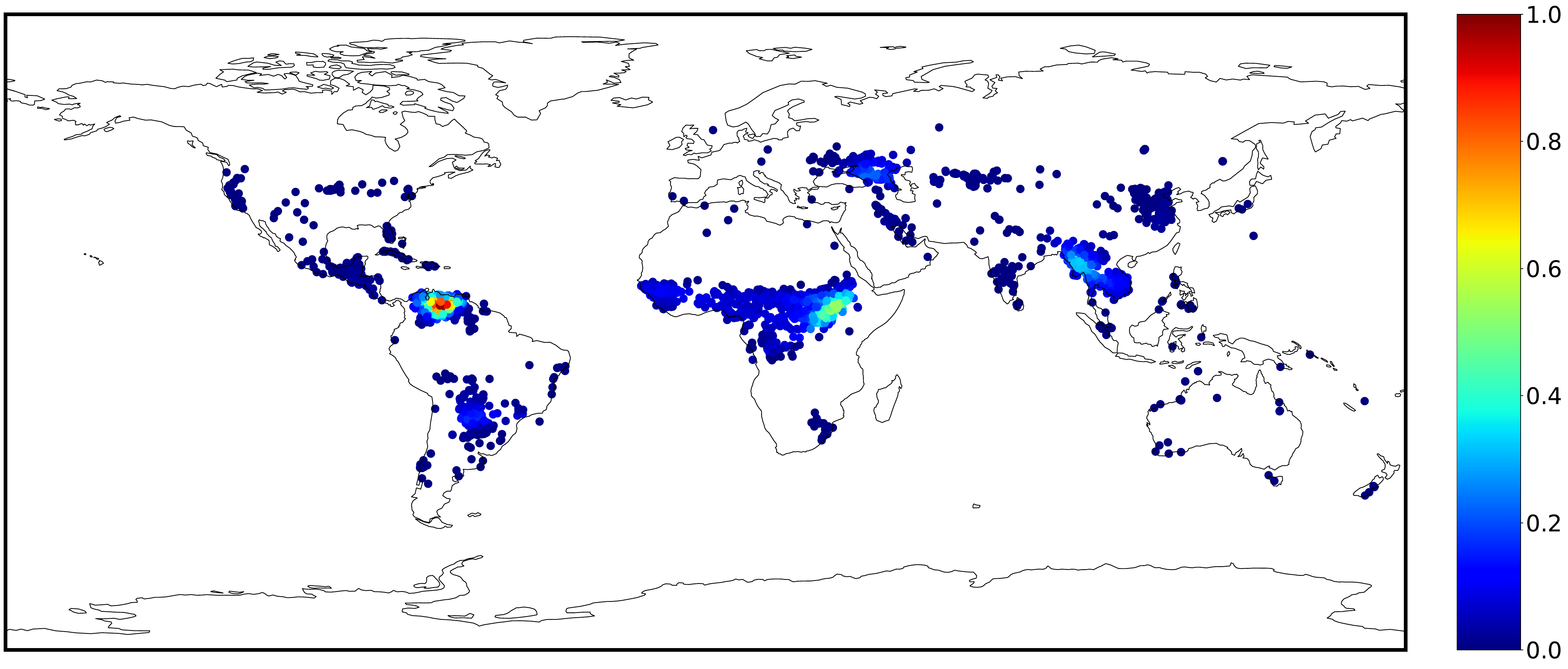}
\caption{Rotation Invariant $S3W$ ($RI$-$S3W$) with $1$ random rotation.}
\vspace{-.2in}
\end{figure}

\begin{figure}[H]
\centering
\includegraphics[width=0.32\columnwidth]{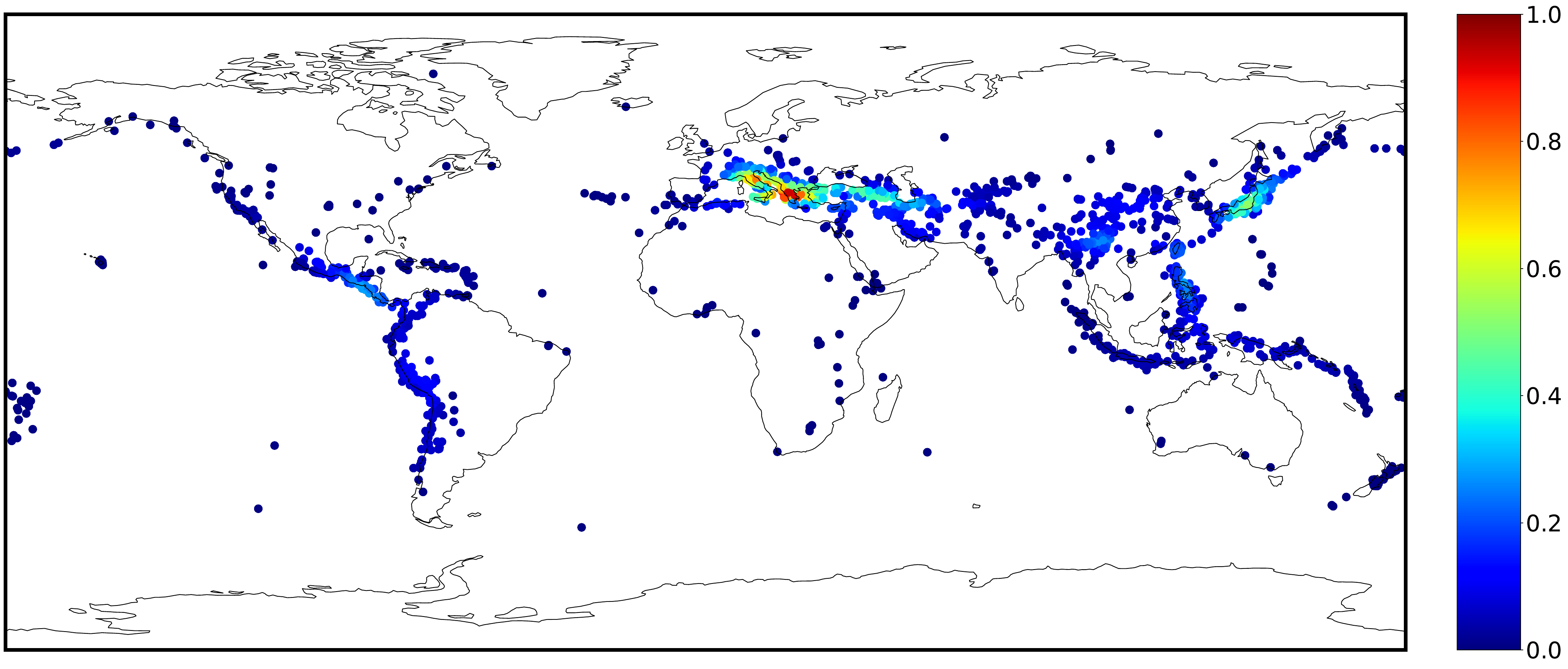}
\includegraphics[width=0.32\columnwidth]{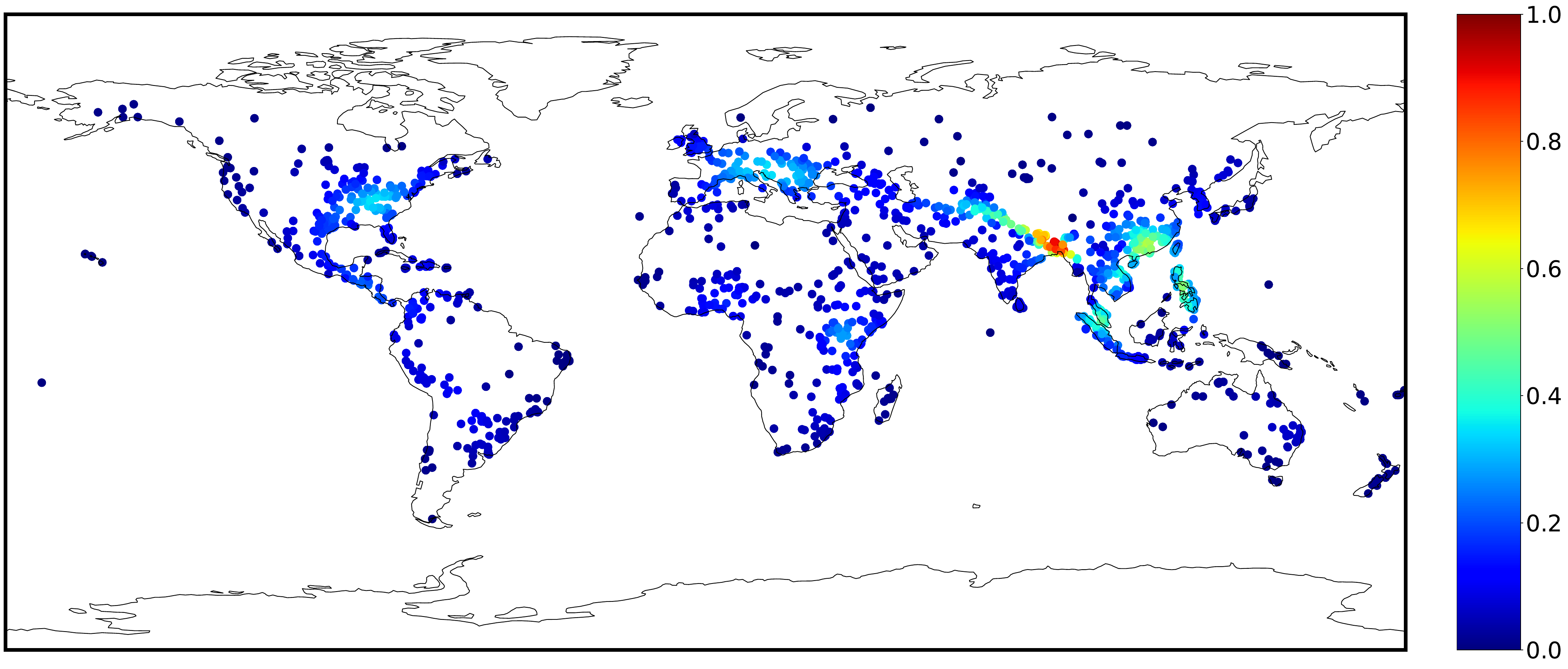}
\includegraphics[width=0.32\columnwidth]{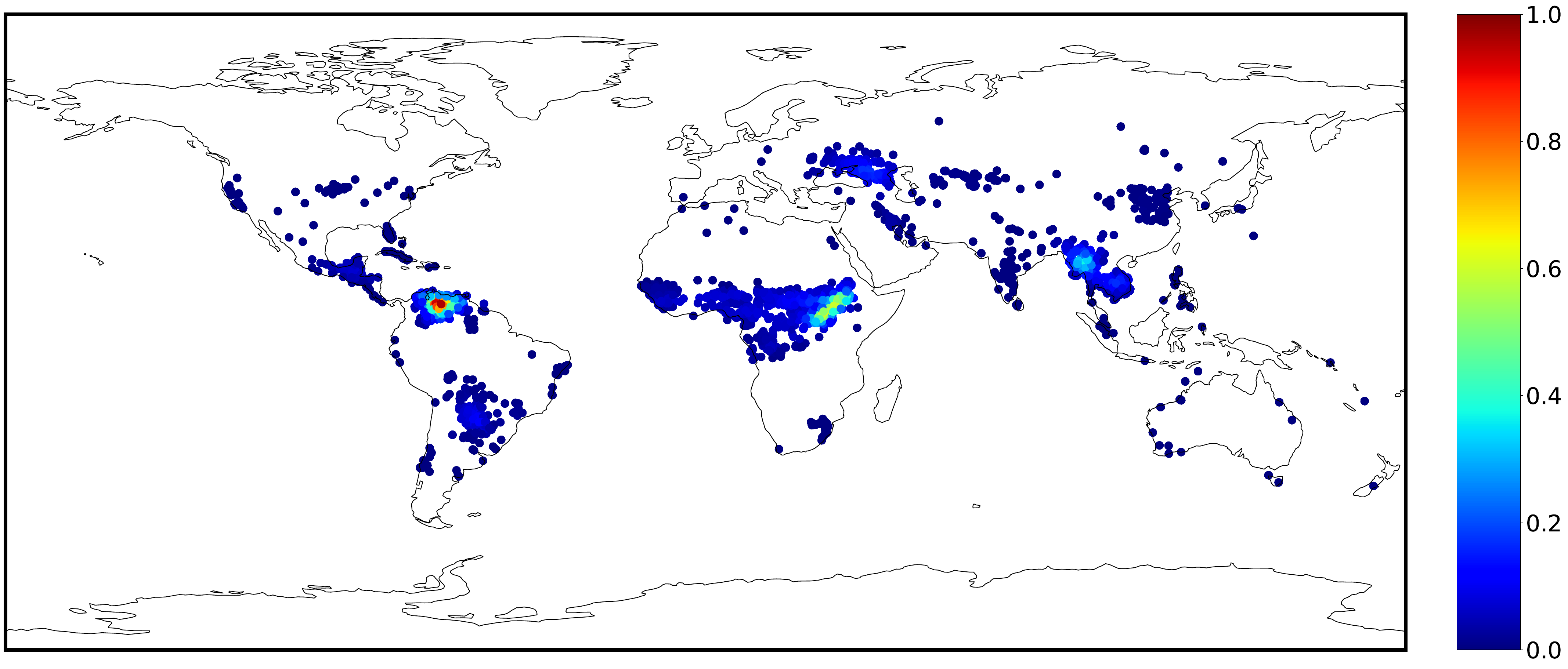}
\caption{Amortized Rotation Invariant $S3W$ ($ARI$-$S3W$) with $50$ random rotations, pool size of $1000$.}
\end{figure}

\clearpage

\subsection{Task: Self-Supervised Representation Learning on the Sphere}
\label{section:ssl}
The self-supervised learning (SSL) paradigm derives learning signals from its own training data without explicit labeling. Among various SSL approaches, contrastive learning-based methods have gained significant popularity due to their effectiveness. Previous works have found that constraining feature vectors to the hypersphere (i.e., having unit norm) offers additional benefits \cite{caron2020unsupervised, wu2018unsupervised}. In what follows, we provide a brief overview of (spherical) contrastive learning -- currently one of the most potent approaches in SSL -- and then outline the details of our experiment in Section \ref{subsubsection:ssl_exp}.
\subsubsection{Background Overview}

\textbf{Contrastive Learning with InfoNCE} in the SSL paradigm is fundamentally about distinguishing between similar (positive) and dissimilar (negative) pairs of data points. The objective is to minimize the distance between embeddings of similar data points while maximizing the distance between embeddings of dissimilar points. Let $p_{data}(\cdot)$ be the data distribution over $\mathbb{R}^m$ and $p_{pos}(\cdot, \cdot)$ the distribution of positive pairs over $\mathbb{R}^m \times \mathbb{R}^m$. Assuming symmetry, i.e., $\forall x,y: p_{pos}(x,y) = p_{pos}(y,x)$, and that the marginal distribution matches the data distribution, i.e., $\forall x: \int p_{pos}(x,y)dy = p_{data}(x)$, the popular InfoNCE \cite{oord2018representation} objective can be described as

\begin{equation}
\mathcal{L}{\small \text{contrastive}} = -\mathbb{E}{(x, y) \sim p_{\text{pos}}} \left[ \log \frac{\exp(\text{sim}(f(x), f(y)) / \tau)}{\sum_{k \in \text{Neg}(x)} \exp(\text{sim}(f(x), f(x_k)) / \tau)} \right],
\end{equation}

where $f(x)$ and $f(y)$ denote the embedding of the positive pairs, $\text{sim}(\cdot, \cdot)$ a similarity measure between two embeddings (e.g., the dot product), $\tau$ the temperature, and $\text{Neg}(x)$ the set of negative samples for $x$.

\noindent\textbf{Spherical Contrastive Loss as Alignment and Uniformity:}

\cite{wang2020understanding} provides valuable insights into the mechanics of contrastive learning by decomposing the contrastive loss into two components

\begin{equation}
    \mathcal{L}_{\text{align}}(f; \alpha) = - \mathbb{E}_{(x,y) \sim p_{\text{pos}}} \left[ \|f(x) - f(y)\|_{\alpha}^2 \right], \quad \alpha > 0.
\end{equation}
and 
\begin{equation}
    \mathcal{L}_{\text{uniform}}(f; t) = \log \mathbb{E}_{x, y \, \text{i.i.d.} \sim p_{\text{data}}} \left[ \exp(-t \|f(x) - f(y)\|_2^2) \right], t > 0.
\end{equation}

The overall objective is therefore
\begin{equation}
    \mathcal{L}_{\text{contrastive}}(f; \alpha, \beta) = \alpha \mathcal{L}_{\text{align}} + \beta \mathcal{L}_{\text{uniform}}
\end{equation}

where $\alpha, \beta$ denotes the component weight for $\mathcal{L}_{\text{align}}$ and $\mathcal{L}_{\text{uniform}}$, respectively. The alignment component $\mathcal{L}_{\text{align}}$ forces positive spherical views to be similar. The uniformity component $\mathcal{L}_{\text{uniformity}}$ encourages all spherical views to be spread out (i.e. uniformly distributed) on the hypersphere, which prevents representation collapse.

\subsubsection{Experiment: $S3W$-Based Uniformity Loss}
\label{subsubsection:ssl_exp}
To demonstrate the efficacy of our proposed $S3W$ in SSL, we follow the same setup as described in \cite{bonet2022spherical}. Specifically, they propose to replace the Gaussian kernel in $\mathcal{L}_{\text{uniform}}$, which operates on pairwise instances, with the distributional distance SSW, and optimize the following objective

\begin{equation}
    \mathcal L_{\text{SSW-SSL}} = \underbrace{\frac1n\sum_{i=1}^n \lVert z^A_i- z^B_i\rVert^2_2}_{\text{Alignment loss}} + \frac\lambda2\big(\underbrace{SSW^2_2(z^A, \nu) + SSW^2_2(z^B, \nu)}_{\text{Uniformity loss}}\big),
    \label{eq:ssw_ssl}
\end{equation}

where $z^A,z^B\in\mathbb R^{n\times (d+1)}$ are the spherical representations of two views (i.e. augmentation) of the same images, $\nu =\text{Unif}(\mathbb{S}^{d})$ is the uniform distribution on the hypersphere and $\lambda>0$ is the regularization coefficient. They achieved comparable performance with \cite{wang2020understanding, chen2020simple} while benefiting from the subquadratic $O(Ln(d + \log n))$.

We will demonstrate the effectiveness of $S3W$ in terms of performance and runtime by modifying Eq. \ref{eq:ssw_ssl} and instead optimize
\begin{equation}
    \mathcal{L}_{\text{S3W-SSL}} = \frac{1}{n}\sum_{i=1}^n \lVert z^A_i - z^B_i \rVert^2_2 + \frac{\lambda}{2} \left(\text{S3W}_2(z^A, \nu) + \text{S3W}_2(z^B, \nu)\right).
\end{equation}
We use analogous objectives for $RI$-$S3W$ and $ARI$-$S3W$.

% and 
% \begin{equation}
%     \mathcal{L}_{\text{RI-S3W-SSL}} = \frac{1}{n}\sum_{i=1}^n \lVert z^A_i - z^B_i \rVert^2_2 + \frac{\lambda}{2} \left(\text{RI-S3W}_2(z^A, \nu) + \text{RI-S3W}_2(z^B, \nu)\right).
% \end{equation}

\textbf{Implementation.}
Similar to \cite{bonet2022spherical}, we use a pretrained ResNet18 encoder \cite{He2015DeepRL} on CIFAR-10 \cite{Krizhevsky2009LearningML} with $1024$ dimensional penultimate features, projected and $\ell_2-$normalized to be unit vectors on $\mathbb{S}^{d}$. The models are pretrained for $200$ epochs with minibatch SGD (momentum $0.9$, weight decay $0.001$, initial learning rate $0.05$). We select the batch size to be $512$ samples and use the standard random augmentation set consisting of random crop, horizontal flipping, color jittering, and gray scale transformation, as done in \cite{bonet2022spherical, wang2020understanding}. 

For evaluation, we fit a linear classifier on the pretrained representations and report the test accuracy. We report the accuracy both when this linear classifier is fit on the encoded feature representations and when it is fit on the projected features in $\mathbb{S}^d$. We train this linear layer for $100$ epochs with the Adam optimizer (learning rate $0.001$, weight decay of $0.2$ at epochs $60$ and $80$). For comparison, we also test the hypersphere method in \cite{wang2020understanding} and SimCLR \cite{chen2020simple} to aid in evaluating the results. We also train a fully supervised model by training the encoder and linear classifier jointly with cross entropy loss, in order to serve as a baseline for performance comparison.

We use $L=200$ projections for all sliced distances, $N_R = 5$ rotations for $RI$-$S3W$ and $ARI$-$S3W$, and a pool size of $100$ for $ARI$-$S3W$. We first perform a series of experiments with $d=9$. In this setting, we let the regularization coefficient $\lambda = 0.5$ for $S3W$, $RI$-$S3W$, and $ARI$-$S3W$; $\lambda = 20.0$ for $SSW$; and $\lambda = 1.0$ for $SW$. These results are found in the main paper Table \ref{table:ssl_dim_10}. We also perform a series of experiments with $d=2$ in order to visualize the quality of the learned representations. In this setting, we instead use a regularization coefficient of $\lambda = 0.1$ for $S3W$, $RI$-$S3W$, and $ARI$-$S3W$. These visualizations for all the considered methods are provided in Figure \ref{fig:ssl_dim_3}.

\begin{figure}
    % \vspace{-0.5in}
    \centering
    \subfloat[Supervised]{\includegraphics[width=0.33\columnwidth]{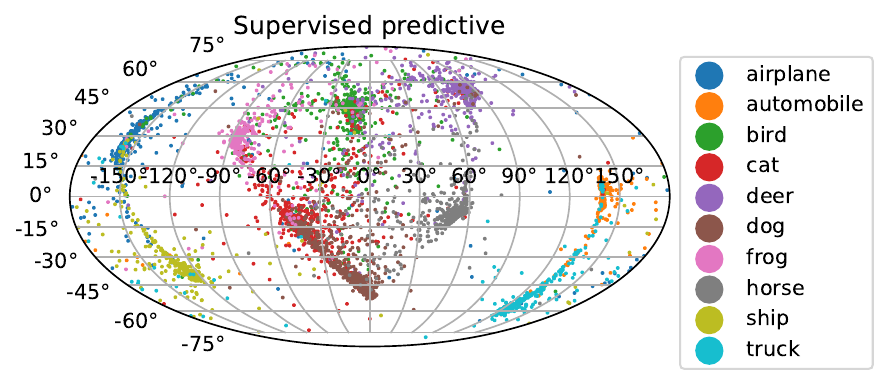}} \hfill
    \subfloat[SimCLR \cite{chen2020simple}]{\includegraphics[width=0.33\columnwidth]{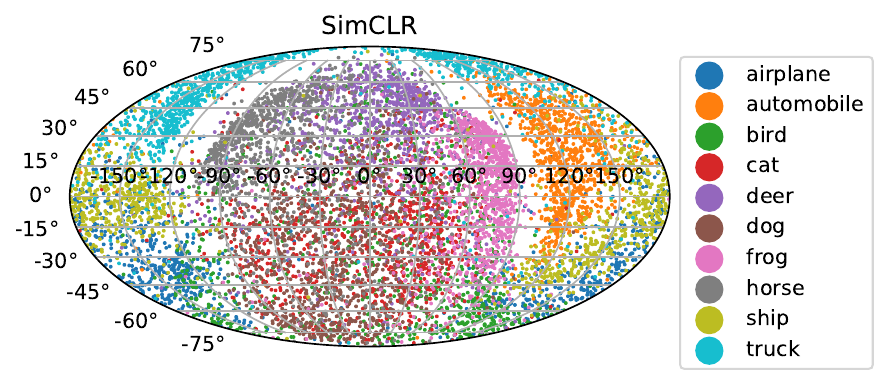}} \hfill
    \subfloat[\cite{wang2020understanding}]{\includegraphics[width=0.33\columnwidth]{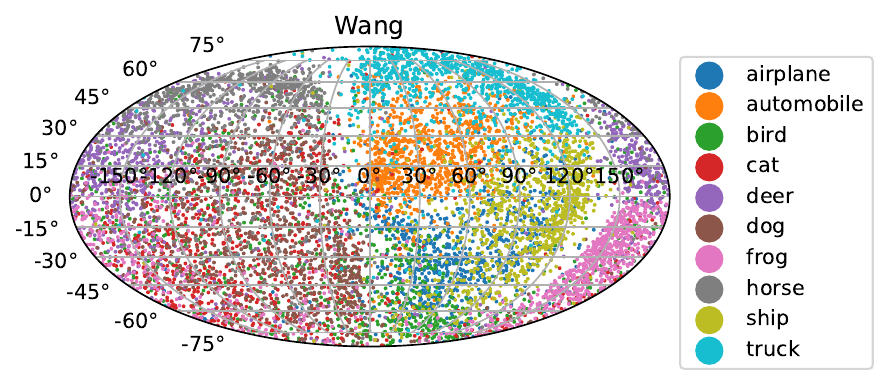}} 
    \\
    \subfloat[$SW$]{\includegraphics[width=0.33\columnwidth]{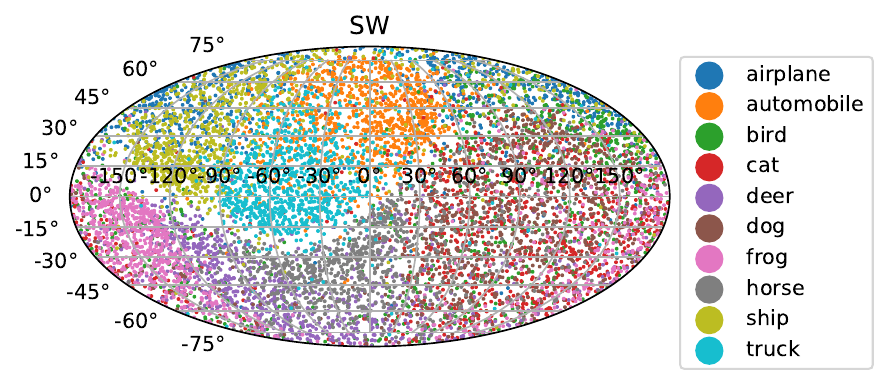}} \hfill
    \subfloat[$SSW$]{\includegraphics[width=0.33\columnwidth]{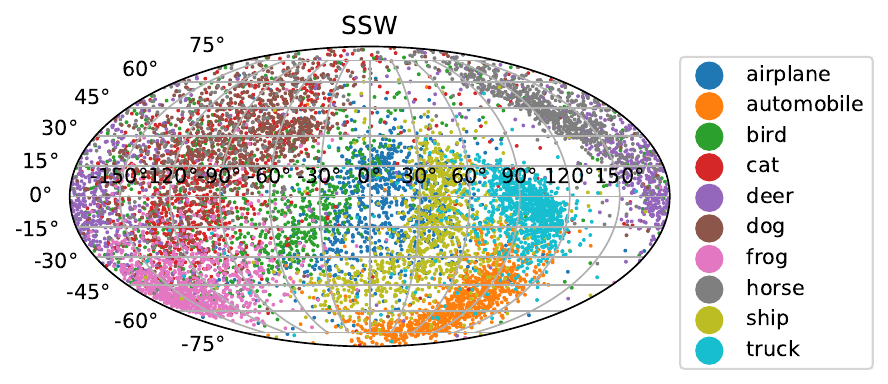}} \hfill
    \subfloat[$S3W$]{\includegraphics[width=0.33\columnwidth]{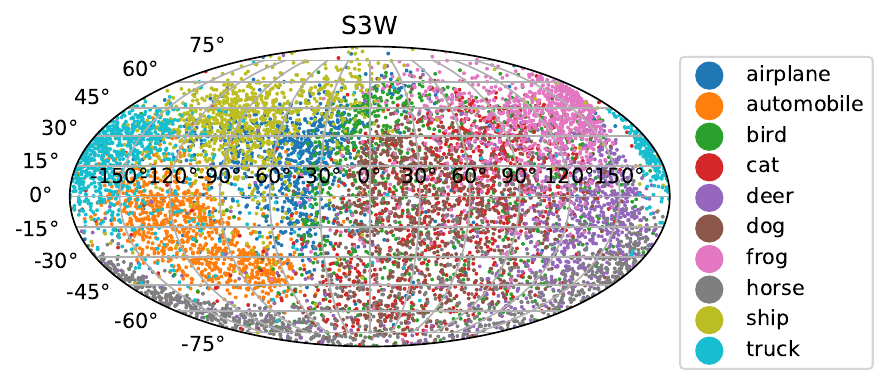}}
    \\
    \hspace*{\fill}
    \subfloat[$RI$-$S3W$]{\includegraphics[width=0.33\columnwidth]{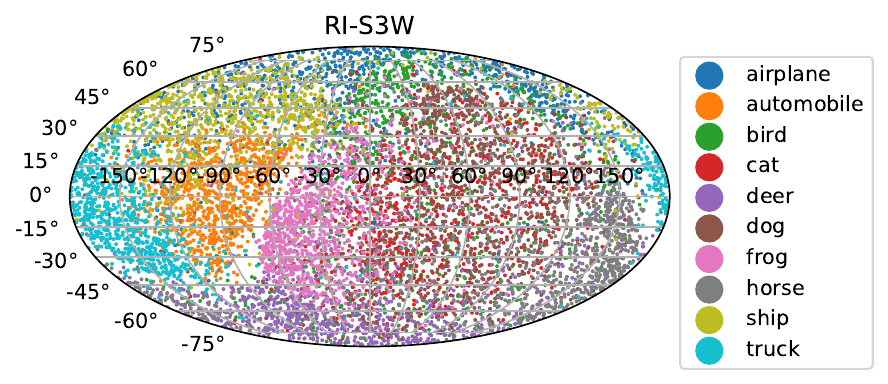}} \hfill
    \subfloat[$ARI\text{-}S3W$]{\includegraphics[width=0.33\columnwidth]{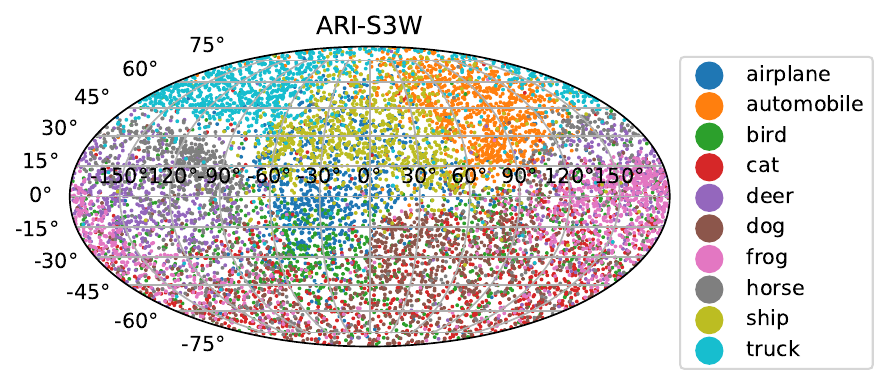}}
    \hspace*{\fill}

    \caption{Evolution between the source and target vMFs, $500$ samples each. For all sliced-Wasserstein variants, we use $200$ projections, and for RI-S3W, we use $100$ random rotations.}
    \label{fig:ssl_dim_3}
\end{figure}

\newpage
\subsection{Task: Sliced-Wasserstein Variational Inference on the Sphere}
\label{section:vi}

\subsubsection{Background Overview}
\textbf{Variational Inference} turns the Bayesian inference problem into a variational one, benefiting from the rich optimization literature and reduced computational costs. The usual goal is to approximate an unnormalized distribution by minimizing the KL divergence between a tractable density and the true posterior. Formally, let $p(\cdot| x)$ denote the target posterior and and $q \in \mathcal{Q}$ the approximate posterior from a family of tractable distributions. The standard objective is to minimize  

\begin{equation}
    \min_{q\in\mathcal{Q}}\ \mathrm{KL}(q||p(\cdot|x)) =  \int q(Z) \log\left(\frac{q(Z)}{p(Z|x)}\right) dZ = \mathbb{E}_q[\log\left(\frac{q(Z)}{p(Z|x)}\right)].
\end{equation}
%The Bayesian Inference landscape fundamentally revolves around estimating the posterior distribution $p(\cdot|x)$ as new evidence emerges. Given that exact inference is often infeasible, various approximate methods have been developed (i.e. MCMC, ABC...). Among those, 

\noindent\textbf{Sliced-Wasserstein Variational Inference (on the sphere):}

The non-metric nature of KL, namely its asymmetry and failure to satisfy the triangle inequality, can result in undesirable behaviors. As an alternative, \cite{Yi2022SlicedWV} introduces the sliced Wasserstein variational inference, optimized via MCMC without optimization or requiring a tractable $\mathcal{Q}$ family. Hence, this can be armortized with deep neural networks. In spherical data contexts, \cite{bonet2022spherical} recommends replacing $SW$ with $SSW$ and employs the Geodesic Langevin Algorithm (GLA) (see \ref{subsubsec:gla}).

To perform SWVI, we first select  a sampler \( q_\theta \) for the sphere. At each iteration $k$, we draw a batch of $N$ samples $ \{z_0^i\}_{i=1}^N$from $q_\theta$. These samples are then propagated using \( \ell \) MCMC steps on the sphere to obtain a new set of samples \( \{z_\ell^j\}_{j=1}^N \). The updates are constrained to the sphere by employing the exponential map \( \exp_{x_k} \) or normalization of the projection onto the tangent space \( \text{Proj}_{x_k} \) at each point $x_k  \in \mathbb{S}^d $. We then compute sliced Wasserstein distance (i.e. SSW, S3W) between the empirical $ \hat{\mu}_0 $ and $\hat{\mu}_\ell $, which is then used to compute gradients w.r.t. $\theta$. 

\textbf{Unadjusted Langevin Algorithm (ULA)} is used to sample from a probability distribution with density proportional to $e^{-V(x)}$ where $V(x)$ is a potential function (i.e. log density) and $x \in \mathbb{R}^d$. The update rule simply combines a deterministic gradient step with a stochastic noise term. That is

\begin{equation}
    x_{t+1} = x_t - \gamma \nabla V(x_t) + \sqrt{2\gamma} Z, \quad Z \sim \mathcal{N}(0, I)
\end{equation}

\textbf{Geodesic Langevin Algorithm (GLA)} \label{subsubsec:gla} \cite{wang2020fast} adapts ULA to a Riemannian manifold like $\mathbb{S}^d$ by modifying the update rule to respect the manifold structure. Similar to \ref{subsubsec:gd}, we first replace the Euclidean gradient with the Riemannian gradient, as follows

\begin{equation}
    x_{t+1} = \text{Retr}_{x_t}\left( -\gamma \nabla_{\text{Riemannian}} V(x_t) + \sqrt{2\gamma} Z \right), \quad Z \sim \mathcal{N}(0, I),
\end{equation}

Here, $\mathrm{Retr}_{x_t}$ is a retraction. If we select $\mathrm{Retr}_{x_t}= \exp_{x_t}$ (see Eq. \ref{eqn:exp_map}) then we arrive at the Geodesic Langevin Algorithm (GLA) whose update rule is simply

\begin{equation}
    x_{t+1} = \exp_{x_t}\left( -\gamma \nabla_{\text{Riemannian}} V(x_t) + \sqrt{2\gamma} Z \right), \quad Z \sim \mathcal{N}(0, I)
\end{equation}

where $\nabla_{\text{Riemannian}} V(x_t) = \nabla V(x_t) - \langle \nabla V(x_t), x_t \rangle x_t$.

In our experiments, since $\mathbb{S}^d$ has constant curvature and the step size is small, we use normalization as the retraction for simplicity. The update rule is hence given by:

\begin{equation}
\label{eqn:ula_normalized}
    x_{t+1} = \frac{\tilde{x}_{t+1}}{\|\tilde{x}_{t+1}\|}, \quad \text{where} \quad \tilde{x}_{t+1} = x_t - \gamma \nabla_{\text{Riemannian}} V(x_t) + \sqrt{2\gamma} Z,
\end{equation}

where \( Z \sim \mathcal{N}(0, I) \) is the Gaussian noise term. Putting it all together, we arrive at

\begin{equation}
    x_{t+1} = \frac{x_t - \gamma \left(\nabla V(x_t) - \langle \nabla V(x_t), x_t \rangle x_t\right) + \sqrt{2\gamma} Z}{\left\| x_t - \gamma \left(\nabla V(x_t) - \langle \nabla V(x_t), x_t \rangle x_t\right) + \sqrt{2\gamma} Z \right\|}.
\end{equation}

\textbf{Power Spherical Distribution} \cite{de2020power} provides a parametric alternative to vMFs (refer to H2). It demonstrates greater stability and has a univariate marginal with a closed-form expression for both its CDF and its inverse. This enables fast and efficient sampling, in contrast to the vMF, which depends on rejection sampling. The PDF of the Power Spherical is defined as

\begin{equation}
    p_X(x)(x; \mu, \kappa) \propto (1 + \mu^{\top}x)^{\kappa},
\end{equation}

where $\mu \in \mathbb{S}^d$, $\kappa > 0$, for all $x\in \mathbb{S}^d$. We sample by drawing $Z \sim Beta(\frac{d-1}{2} + \kappa, \frac{d-1}{2})$ and $v \sim Unif(\mathbb{S}^{d-1})$. These samples are then used to construct $T=2Z-1$ and subsequently a vector $Y=[T, v^{\top}\sqrt{1-T^2}]^{\top}$. The final sample is obtained by applying a Householder\footnote{a linear operation $H(v)=v -2u(u^{\top}v)$ reflecting $v$ across the hyperplane perpendicular to $u$.} reflection about the mean direction $\mu$ to $Y$.

\textbf{Effective Sample Size (ESS)} \label{background:ess} estimates the equivalent number of independent samples that would provide the same amount of information as the correlated samples generated by an MCMC process \cite{doucet2001sequential}. For a chain of $n$ samples, the autocorrelation at lag\footnote{the number of time steps separating sequential data points in a time series, used for measuring temporal correlation.} $t$ ($t\geq0$) is given by

\begin{equation}
    \rho_t = \frac{1}{\sigma^2} \int (\theta(n)-\mu)(\theta(n+t)-\mu)p(\theta) d\theta.
\end{equation}

The ESS is then calculated from the total number of samples, $N$, adjusted for the sum of autocorrelation at all lags

\begin{equation}
    ESS = \frac{N}{1+2\sum_{t=1}^{\infty}\rho_t}
\end{equation}

Direct calculation for the ESS is often impractical. \cite{rezende2020normalizing} proposes estimating it using importance sampling weights $w_s$. That is

\begin{equation}
    \mathrm{ESS} = \frac{\mathrm{Var}_{Unif}(e^{-\beta u(X)})}{\mathrm{Var}_q\left(\frac{e^{-\beta u(X)}}{q_\eta(X)}\right)} \approx \frac{\left(\sum_{s=1}^S w_s\right)^2}{\sum_{s=1}^S w_s^2},
\end{equation}

where $w_s=e^{-\beta u(x_s)/q_\eta(x_s)}$. Here, in the context of normalizing flows (see experiments \ref{exp:vi_powerspherical}, \ref{exp:vi_vmfs}), it is reported as a percentage of the sample size. If the ESS can be estimated reliably, then higher ESS means the flow matches the target better. 

\subsubsection{Experiment 1: Learning a Power Spherical Distribution}
\label{exp:vi_powerspherical}

\textbf{Implementation.}
Following the setup described in \cite{bonet2022spherical}, we set initial Power Spherical parameters $\mu=(1,1,1)$ and $\kappa=0.1$ for source,$\mu=(0,1,0)$ and $\kappa=10$ for target. We perform $2000$ Riemannian Gradient Descent \cite{absil2008optimization} steps and then $20$ steps of GLA (step size $0.001$) with $1000$ projections (learning rate $2$). We select $K=2000$ steps with $N=500$ particles.

We report runtimes and parameter convergence behaviors for all the distances. Although they all converge similarly, the runtimes differ drastically. For SSW, we report an average of $215.5$ seconds over $5$ independent runs. For $S3W$, $RI\text{-}S3W$ (1), $RI\text{-}S3W$ (5), $RI\text{-}S3W$ (10), $ARI\text{-}S3W$ ($10$ rotations, pool size $1000$), they are $55.3$, $57.3$, $63.6$, $73.1$, $63.04$ seconds , respectively.

\textbf{Results.}

\begin{figure}[h!]
    \centering
    \hspace*{\fill}
    \subfloat[$SSW$]{\includegraphics[width=0.6\linewidth]{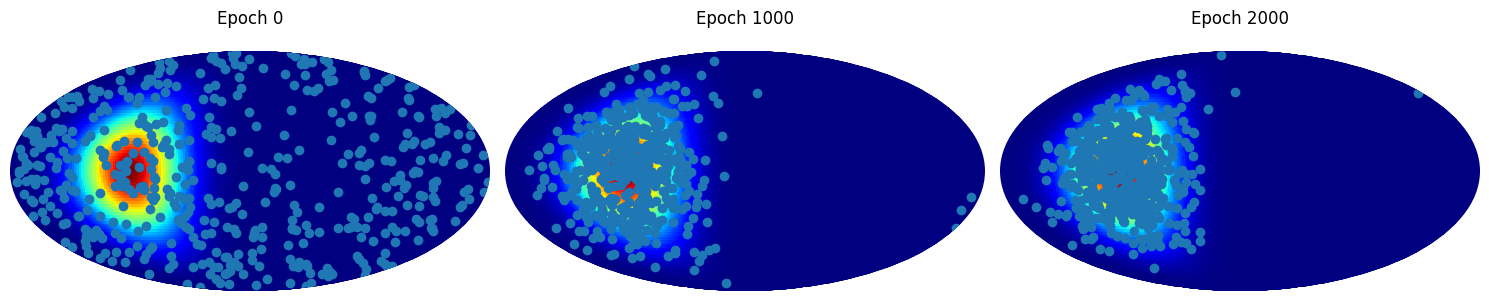}}
    \hspace*{\fill}
    \\
    \hspace*{\fill}
    \subfloat[$S3W$]{\includegraphics[width=0.6\linewidth]{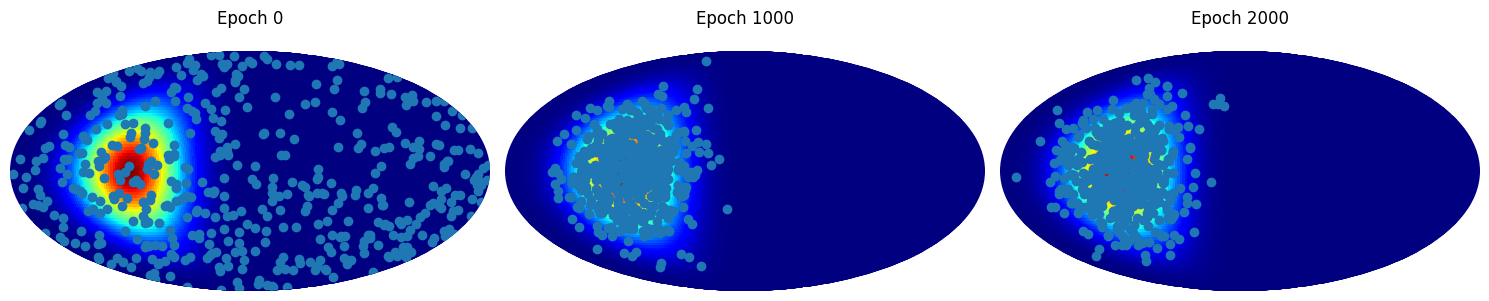}}
    \hspace*{\fill}
    \\
    \hspace*{\fill}
    \subfloat[$RI\text{-}S3W (1)$]{\includegraphics[width=0.6\linewidth]{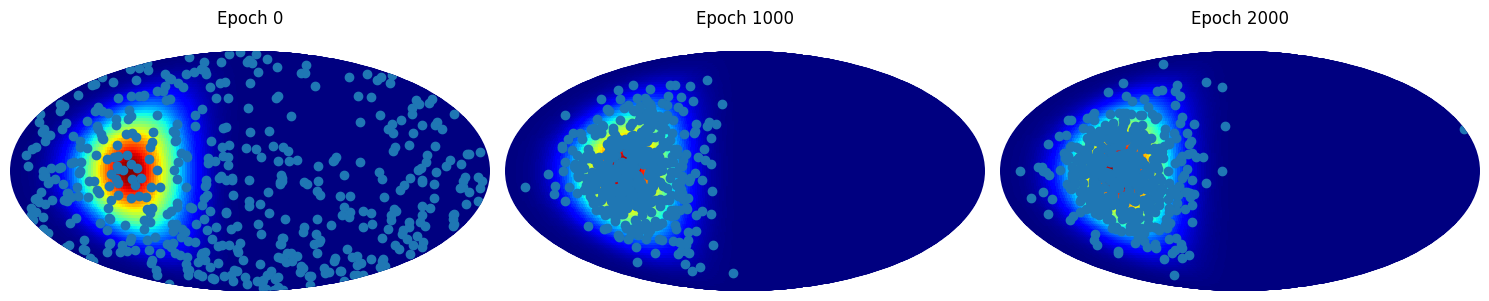}} 
    \hspace*{\fill}
    \\
    \hspace*{\fill}
    \subfloat[$RI\text{-}S3W (5)$]{\includegraphics[width=0.6\linewidth]{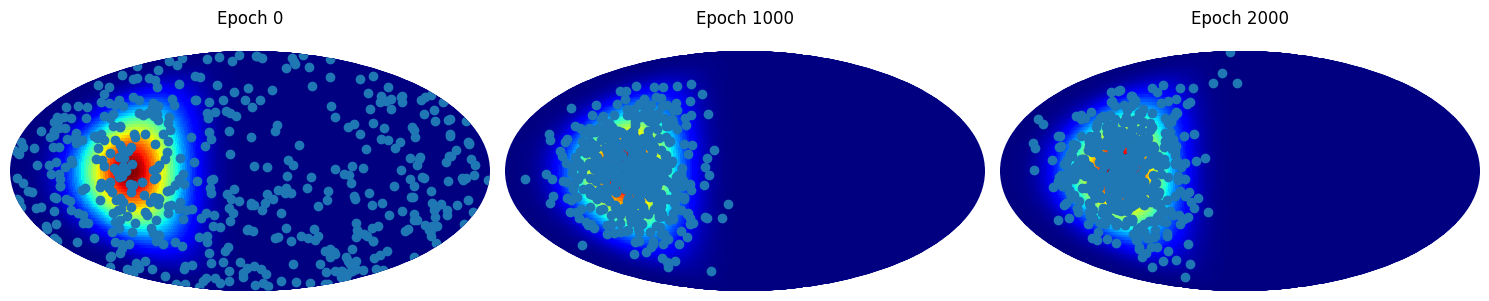}} 
    \hspace*{\fill}
    \\
    \hspace*{\fill}
    \subfloat[$RI\text{-}S3W (5)$]{\includegraphics[width=0.6\linewidth]{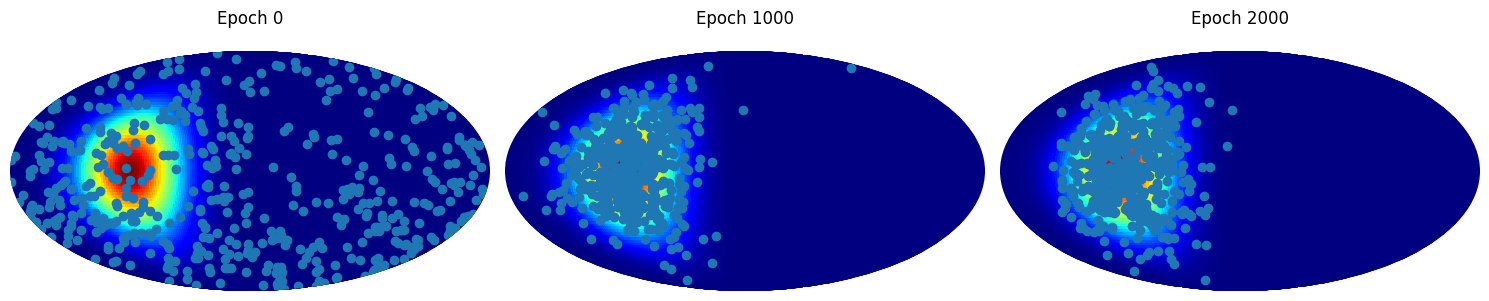}} 
    \hspace*{\fill}
    \\
    \subfloat[$RI\text{-}S3W (10)$]
    {\includegraphics[width=0.6\columnwidth]{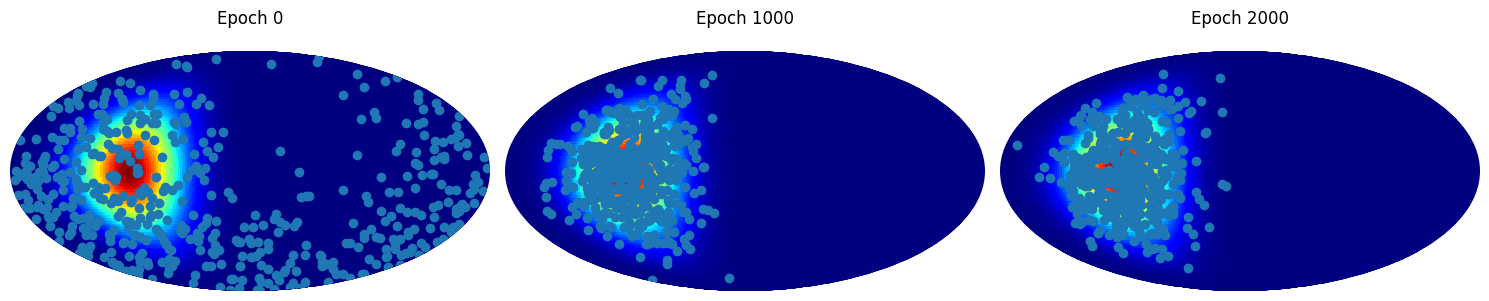}} \hfill 
    \caption{Learning a mixture of $4$ von Mises-Fisher distributions. To obtain the Mollweide projections, we perform Kernel Density Estimation with the Scott adaptive bandwidth.}
    \label{fig:rociofix}
    \label{fig:swvi_power_spherical1}
\end{figure}

\subsubsection{Experiment 2: Learning a mixture of $4$ von Mises-Fisher distributions}
\label{exp:vi_vmfs}

% We train an exponential map normalizing flows $f_{\mu_T}(x;T_{exp})$, composed of $N=6$ blocks and $5$ components, to learn a mixture of vMFs, similar to \cite{rezende2020normalizing} and \cite{bonet2022spherical}. The optimization is run for $10000$ iterations with $20$ GLA steps each, using the Adam optimizer (learning rate $0.01$). We provide qualitative results on Figure \ref{fig:mcmc_vmfs}, showing comparable performance with $SSW$ \cite{bonet2022spherical} while benefitting from much lower computational costs. For quantitative results, we report the effective sample size (ESS) estimated by \cite{rezende2020normalizing} as
\textbf{Implementation.} We perform the Geodesic Langevin Algorithm (GLA) (see \ref{subsubsec:gla}) to generate samples that track the target distribution and used to guide our variational model, which is an exponential map normalizing flows $f_{\mu_T}(x;T_{exp})$ ($N=6$ blocks, $5$ components). The objective of $f_{\mu_T}$ is to transform uniform noise on the sphere into a close approximation of the target, which is a mixture of $4$ von Mises-Fisher distributions. The parametric target has the mean directions $\mu_1 = (1.5, 0.7 + \frac{\pi}{2}), \mu_2 = (1.0, -1.0 + \frac{\pi}{2}), \mu_3 = (5.0, 0.6 + \frac{\pi}{2}), \mu_4 = (4.0, -0.7 + \frac{\pi}{2})$ and uniform concentration parameters $\kappa=10$, equally weighted. The optimization is run for $10000$ iterations with $20$ GLA steps each, using the Adam optimizer (learning rate of $0.01$, matching the baseline).

\textbf{Results.} Figure \ref{fig:4vmfs} demonstrates the qualitative results for our experiments. While all models are capable of closely matching the target, it appears that $ARI\text{-}S3W$ (30), $RI\text{-}S3W$ (10), $RI\text{-}S3W$ (5) outperform $RI\text{-}S3W (1)$ , which fares better than $SSW$ and $S3W$. 

\begin{figure}[H]
    \centering
    \hspace*{\fill}
    % \label{subfig:4vmfs_target}
    \subfloat[Target Density]{\includegraphics[width=0.22\linewidth]{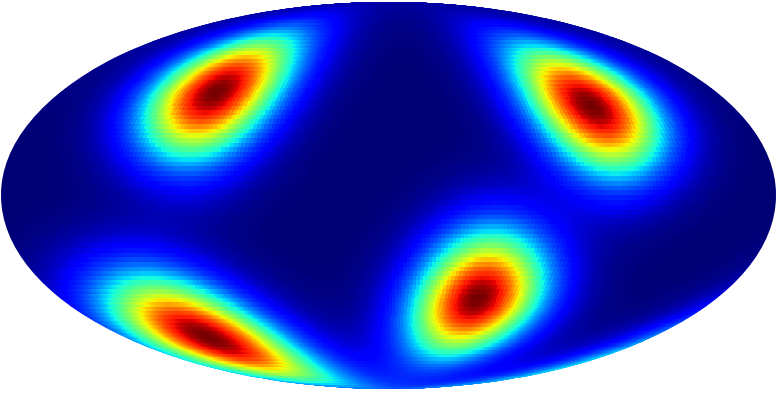}} \hspace{2mm} 
    \subfloat[Target Empirical]{\includegraphics[width=0.22\linewidth]{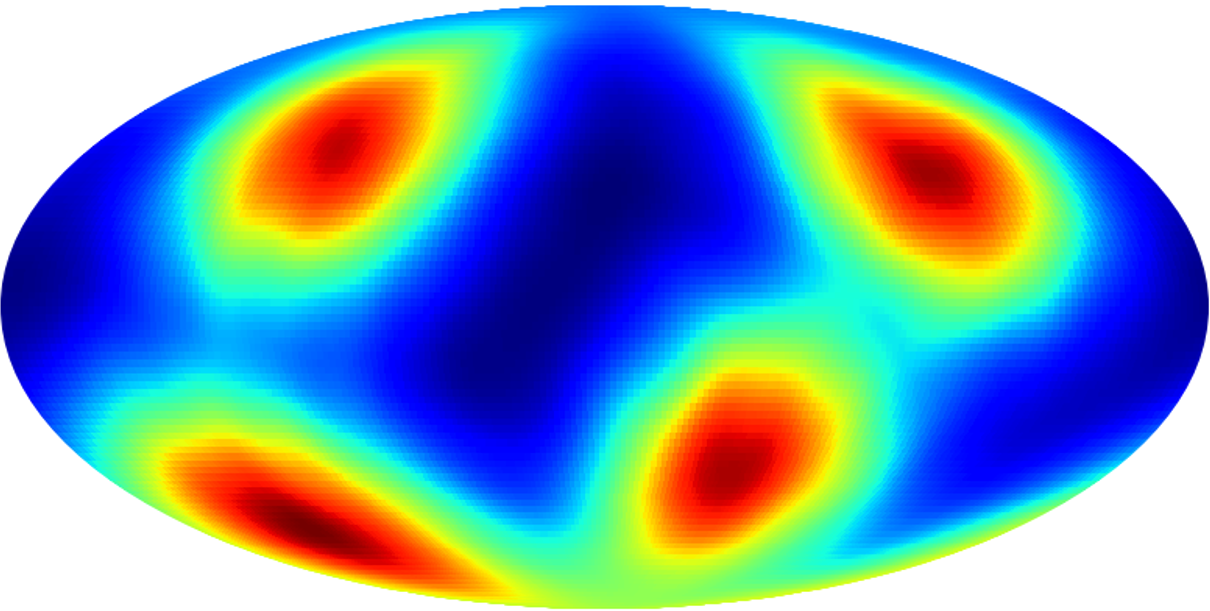}} \hspace{2mm} 
    \subfloat[$SSW$]{\includegraphics[width=0.22\columnwidth]{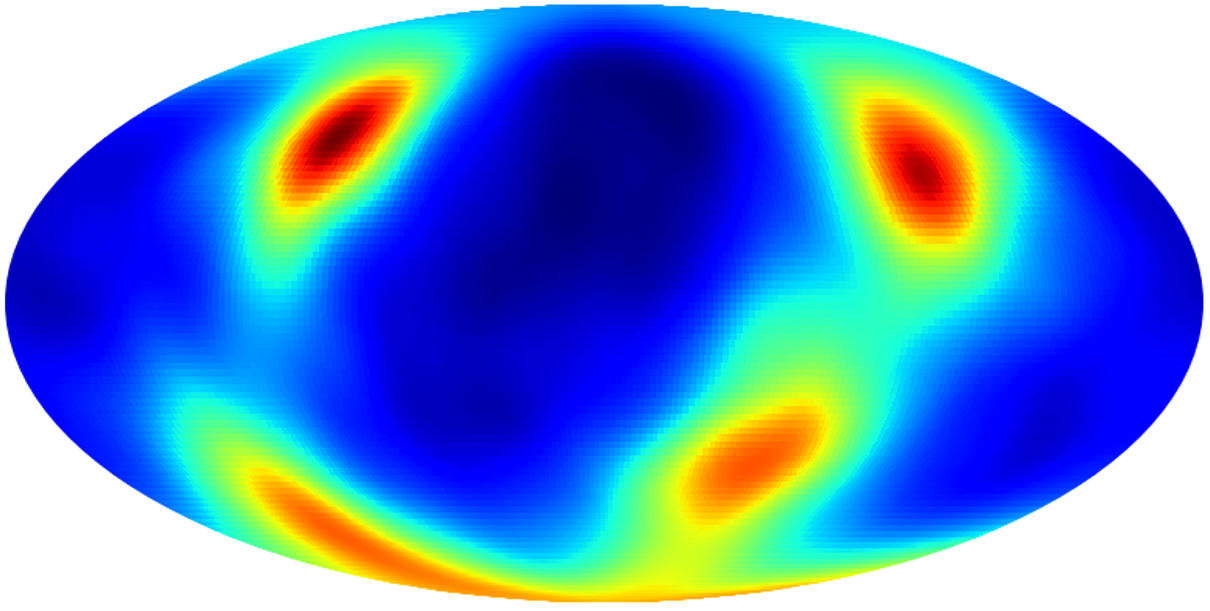}} \hspace{2mm} 
    \subfloat[$S3W$]{\includegraphics[width=0.22\columnwidth]{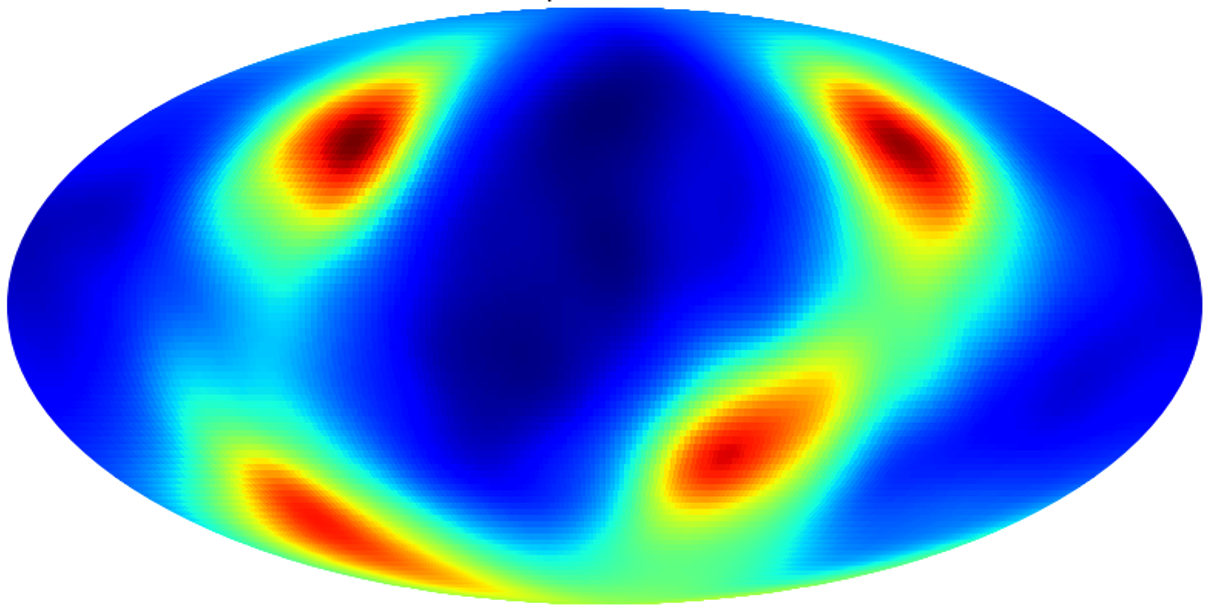}}  
    \hspace*{\fill}
    \\
    \hspace*{\fill}
    \subfloat[$RI\text{-}S3W$ ($1$)]{\includegraphics[width=0.22\columnwidth]{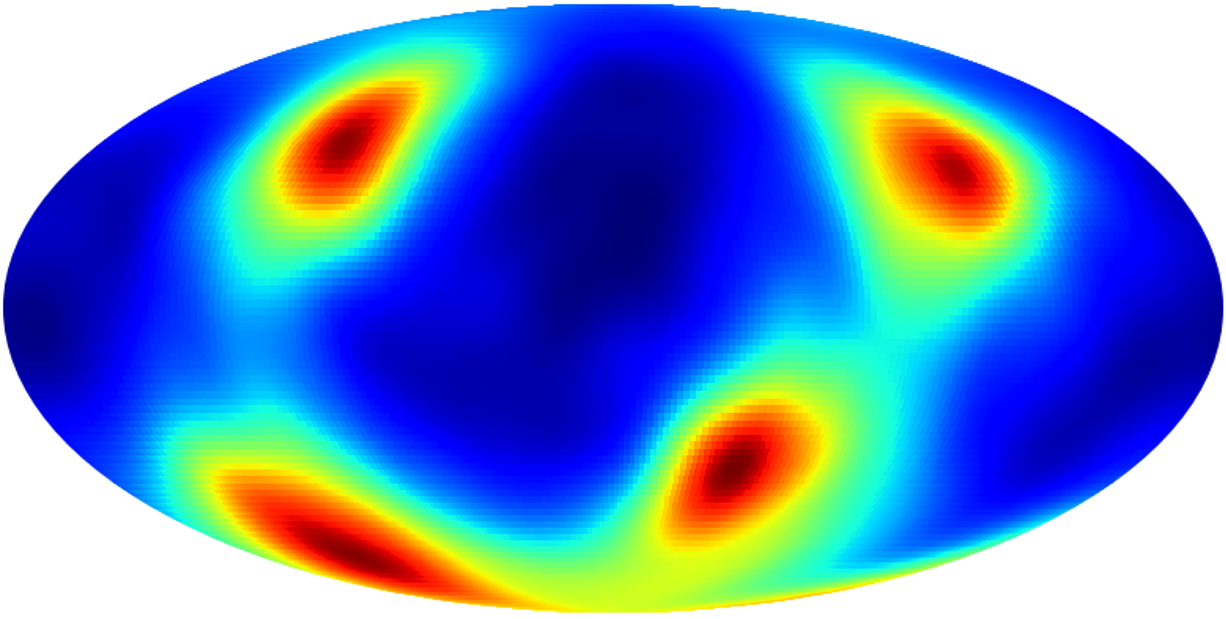}}  \hspace{2mm} 
    \subfloat[$RI\text{-}S3W$ ($5$)]
    {\includegraphics[width=0.22\columnwidth]{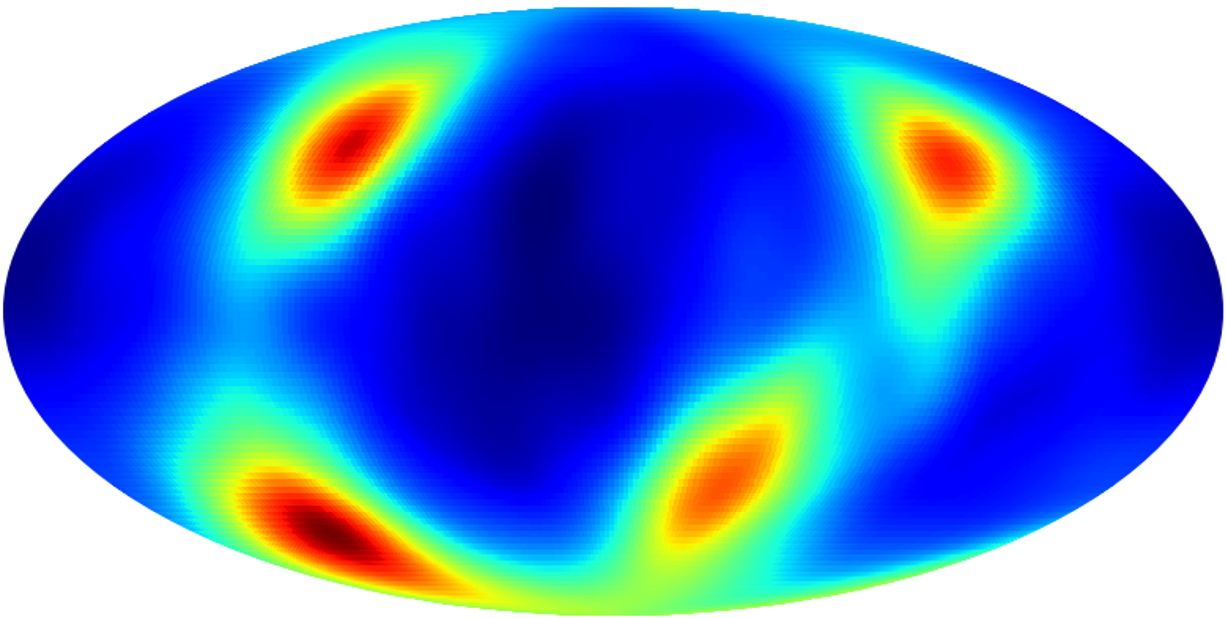}}  \hspace{2mm} 
    \subfloat[$RI\text{-}S3W$ ($10$)]{\includegraphics[width=0.22\columnwidth]{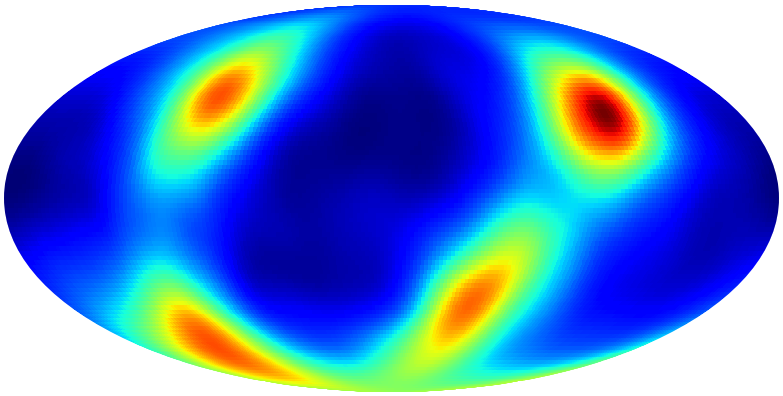}}  \hspace{2mm} 
    \subfloat[$ARI\text{-}S3W$ ($30$)]{\includegraphics[width=0.22\columnwidth]{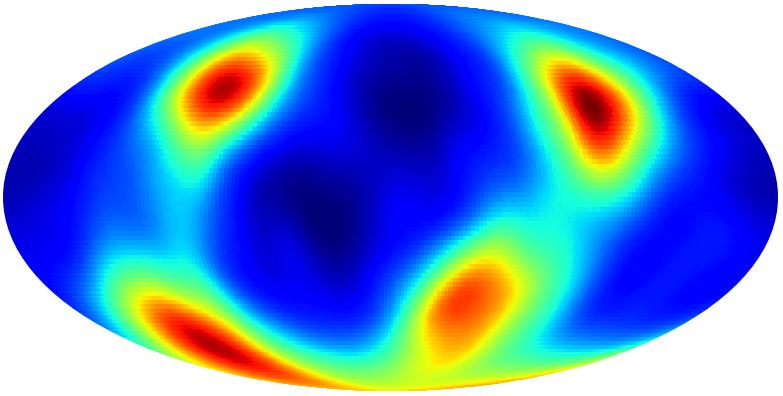}} \hspace*{\fill} 
    \caption{Learning a mixture of $4$ von Mises-Fisher distributions. To obtain the Mollweide projections, we perform KDE with the Scott adaptive bandwidth. For $ARI\text{-}S3W$, we use pool size of $1000$.}
    \label{fig:4vmfs}
    \vspace{-10mm}
\end{figure}

\begin{figure}[H]
    \centering
    \hspace*{\fill}
    \subfloat[]{\includegraphics[width=0.42\columnwidth]{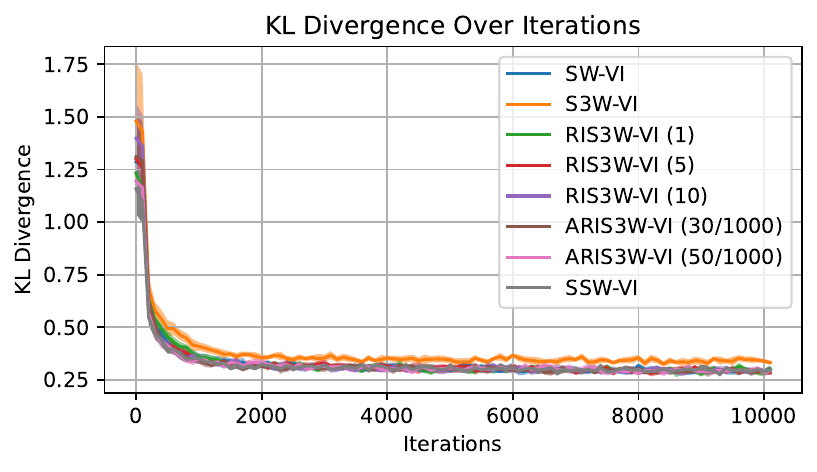}} \hfill 
    \subfloat[]{\includegraphics[width=0.42\columnwidth]{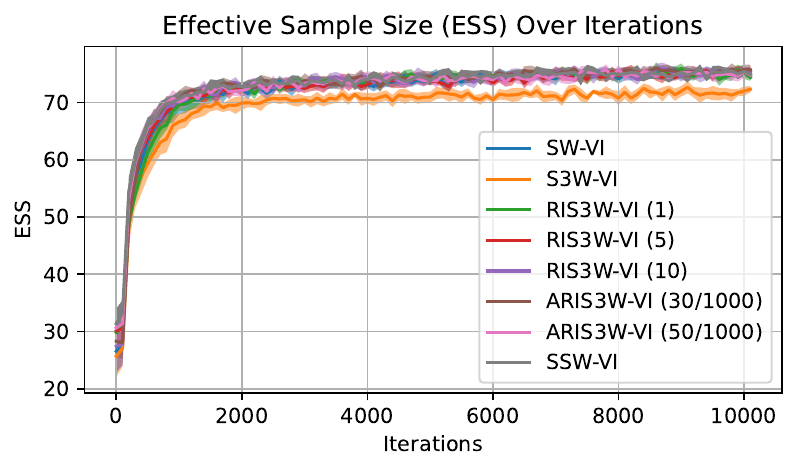}}
    \caption{Evolution between the source and target vMFs, $500$ samples each. For all sliced-Wasserstein variants, we use $200$ projections, and for $ARI\text{-}S3W$, we use the pool size of $1000$ random rotations.} \hfill
    \label{fig:kl_ess}
\end{figure}
\vspace{-5mm} 

To compare the performance of the different distances more rigorously, we use the KL Divergence and the Effective Sample Size (ESS, see \ref{background:ess}), similar to \cite{bonet2022spherical}. We perform $10$ runs per distance, and also add $ARI\text{-}S3W (50)$ for comparison. Figure \ref{fig:kl_ess} shows the convergence trend across different distances. The KL Divergence plot shows most models (with different distances) similarly and rapidly converge to approximating the target distribution. $S3W$ performs slightly worse than the rest while $SSW$ is on par with other variants of $S3W$, in contrast to our qualitative finding based on figure \ref{fig:4vmfs}. The ESS plot demonstrates sampling efficacy. Again, similar trend is observed, with $S3W$ slightly underperforming other distances. We note that greater ESS values signal more independence among drawn samples.

Additionally, we tried the learning rate of $0.001$ (not shown) for all distances and observed a slightly different trend. $SSW$ still outperforms $S3W$ and is on par with $RI\text{-}S3W (1)$, while underperforming the rest by a small margin.

\subsection{Task: Generative Modeling with Sliced-Wasserstein Autoencoder (SWAE)}
\label{section:swae}
Autoencoders (AEs) map an input \( x \in \mathbb{R}^d \) back to itself through an intermediate latent space and play a central role in discovering low-dimensional representations useful for downstream tasks. Deterministic AEs focus on point estimates in the latent space, while their probabilistic counterparts extend this concept to probability distributions. A common requirement for both is regularization to avoid the trivial solutions. SWAEs \cite{kolouri2016sliced}, a type of probabilistic AE, offer a simple yet effective approach to regularizing the latent space. In this section, we show the effectiveness of our proposed $S3W$ method when the SWAE latents are constrained to the unit sphere.

\subsubsection{Experiment: $S3W$-Based Latent Regularization Loss}

\begin{algorithm}[ht]
\caption{S3W-AE}
\label{alg:s3wae}
\begin{algorithmic}
\REQUIRE Reg. coefficient $\lambda$, $L$ projections, encoder $f_{\theta}$, decoder $g_{\eta}$, dimension $d$ of the hypersphere.

\WHILE{$f_{\theta}$ and $g_{\eta}$ have not converged}
\STATE Sample ${x_1,...,x_M}$ from the training set (i.e., $p_X$).
\STATE Encode the samples: ${z_1,...,z_M} \leftarrow {f_{\theta}(x_1),...,f_{\theta}(x_M)}$.
\STATE Normalize encoded samples to $\mathbb{S}^d$: ${\hat{z}_1,...,\hat{z}_M} \leftarrow \text{Normalize}({z_1,...,z_M})$.
\STATE Sample ${\tilde{z}_1,...,\tilde{z}_M}$ from $q_Z$ and normalize to $\mathbb{S}^d$.
\STATE Use algorithm \ref{alg:s3w} or its variants to calculate $\mathcal{L}_{\text{S3W}}^{(l)}(\{\tilde{z}_m\},\{\hat{z}_m\})$.
\STATE Compute $\hat{\mathcal{L}}_{\text{S3W}} = \frac{1}{L} \sum_{l=1}^L \mathcal{L}_{\text{S3W}}^{(l)}$.
\STATE Update $\theta$ and $\eta$ by gradient descending
\begin{equation*}
\hat{\mathcal{L}} = \sum_{m=1}^M c(x_m, g_{\eta}(\hat{z}_m)) + \lambda \cdot \hat{\mathcal{L}}_{\text{S3W}}
\end{equation*}
\ENDWHILE
\end{algorithmic}
\end{algorithm}

We perform experiments with both MNIST and CIFAR-10. The details of the architectures used in our experiments can be found in Table \ref{tab:swae_arch}. In all experiments, we train the networks with the Adam optimizer (learning rate of $10^{-3}$) with a batch size of $500$ over $100$ epochs. The procedure is described in Algorithm \ref{alg:s3wae}

For all experiments, we use the standard binary cross entropy (BCE) loss as our reconstruction loss, with a tradeoff parameter $\lambda$. For CIFAR-10, we use $\lambda = 10^{-3}$ for $SW$, $S3W$, $RI\text{-}S3W$, and $ARI\text{-}S3W$, and $\lambda = 10$ for $SSW$. We use $L=100$ projections for all sliced methods, $N_R = 5$ rotations, and a pool size of $100$ random rotations. Moreover, we use a vMF mixture with 10 components for the latent prior. These results are given in Table \ref{table:main_swae}.

\begin{wraptable}{r}{6cm} 
    \centering
    \small
    \caption{\small FID scores for SWAE models.}
    \vspace{-2mm}
    \begin{tabular}{lc}
        \hline
        Method & FID $\downarrow$ \\
        \hline
        $SW$ & 74.2301 $\pm$ 2.7727 \\
        $SSW$ & 75.0354 $\pm$ 2.8609 \\
        $S3W$ & 75.0717 $\pm$ 2.6026 \\
        $RI$-$S3W$ (1) & 73.5115 $\pm$ 3.4937 \\
        $RI$-$S3W$ (10) & 70.2262 $\pm$ 3.4730 \\
        $ARI$-$S3W$ (30) & \textbf{69.5562 $\pm$ 1.5386} \\
        \hline
    \end{tabular}
    % \vspace{-3mm}
    \label{table:fid_scores}
    % \vspace{-2mm}
\end{wraptable}

For MNIST, we provide a comprehensive evaluation of the considered methods as we vary the regularization parameter $\lambda$. These results can be found in Table \ref{table:swae_mnist_uniform}. In these experiments, we once again use $L=100$ projections for all sliced methods and a pool size of $100$ for $ARI$-$S3W$. Here, we use the uniform distribution $\text{Unif}(\mathbb{S}^2)$ as our latent prior. Lastly, we repeat a set of experiments for each distance using the MNIST dataset and compare the Fr\'echet Inception Distance (FID) \cite{Heusel2017GANsTB} of each resulting model in Table \ref{table:fid_scores}. We compute the FID using 10000 samples and report the mean and standard deviation over 5 independent runs. Moreover, we visualize samples generated by each model in Figure \ref{fig:swae_samples}. For these experiments, we again use $L=100$ projections for all methods and a pool size of $100$ for $ARI$-$S3W$, as well as a uniform prior on $\mathbb{S}^2$. We use $\lambda = 10^{-3}$ for all models for fair comparison. 

\begin{figure}[!h]
    \centering
    \subfloat[Autoencoder $(\lambda=1)$]{\label{fig:aep4}\includegraphics[width=0.49\columnwidth]{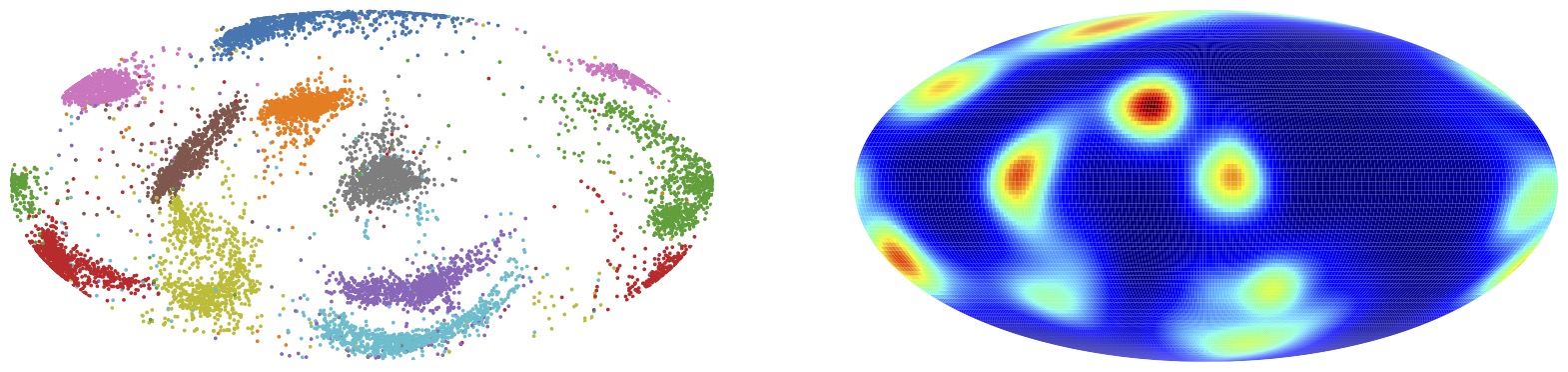}} \hfill
    \subfloat[$SW (\lambda=1000)$]{\label{fig:aep_sw}\includegraphics[width=0.49\columnwidth]{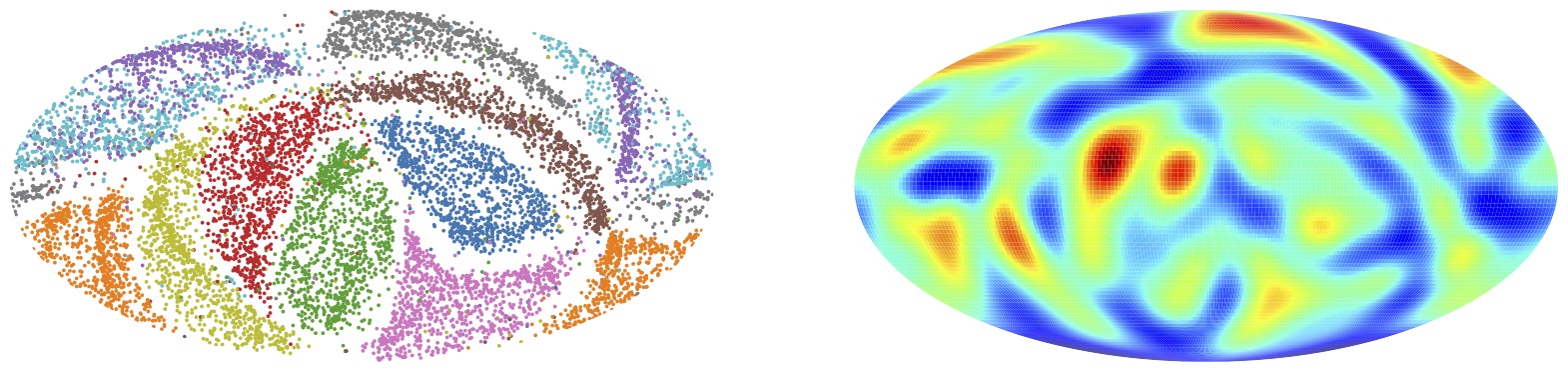}} \\
    \subfloat[$SSW (\lambda=1000)$]{\label{fig:aep_ssw_beta_1e3}\includegraphics[width=0.49\columnwidth]{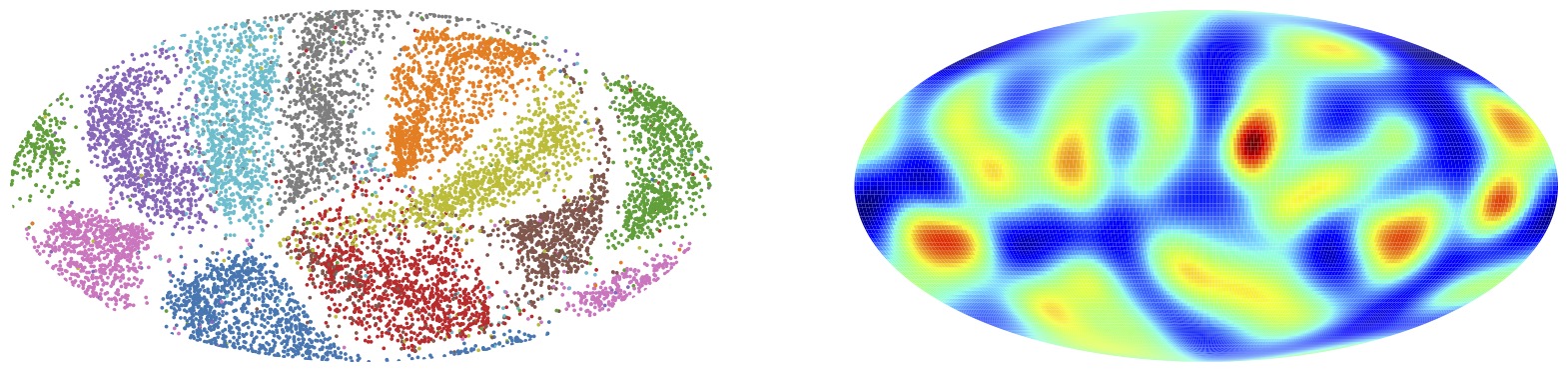}} \hfill
    \subfloat[$S3W (\lambda=1)$]{\label{fig:aep_s3wd_beta_1e0_4}\includegraphics[width=0.49\columnwidth]{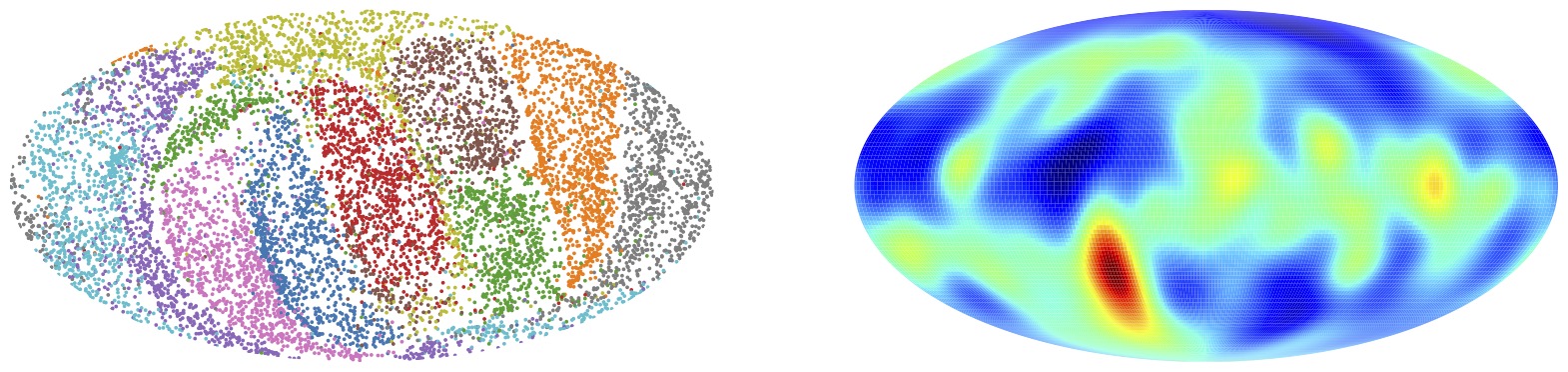}} \\
    \subfloat[$RI\text{-}S3W$ ($1$ rotations, $\lambda=1$)]{\label{fig:aep_ri1_beta_1e0_4}\includegraphics[width=0.49\columnwidth]{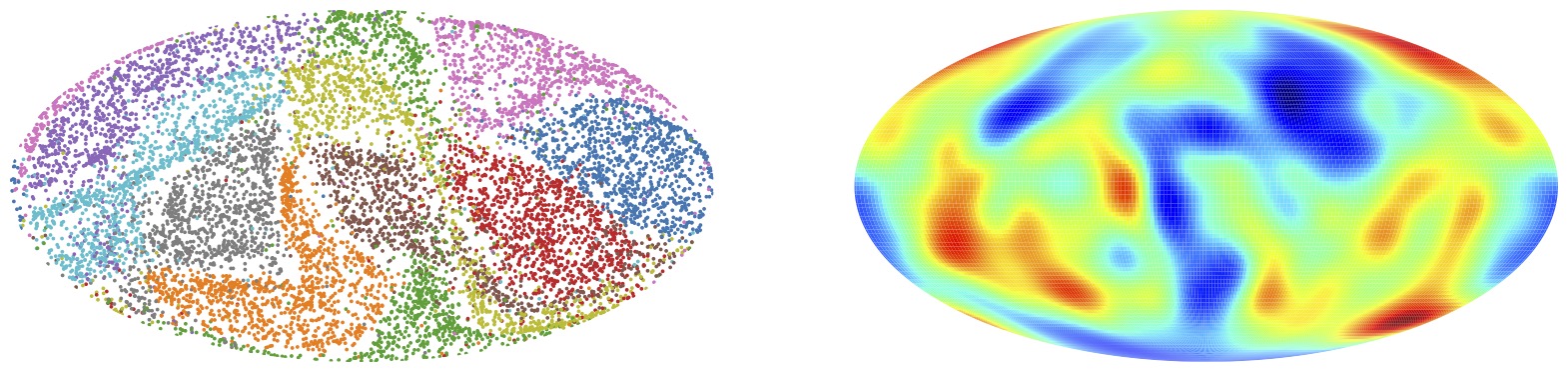}} \hfill
    \subfloat[$RI\text{-}S3W$ ($10$ rotations, $\lambda=1$)]{\label{fig:aep_ri10_beta_1e0}\includegraphics[width=0.49\columnwidth]{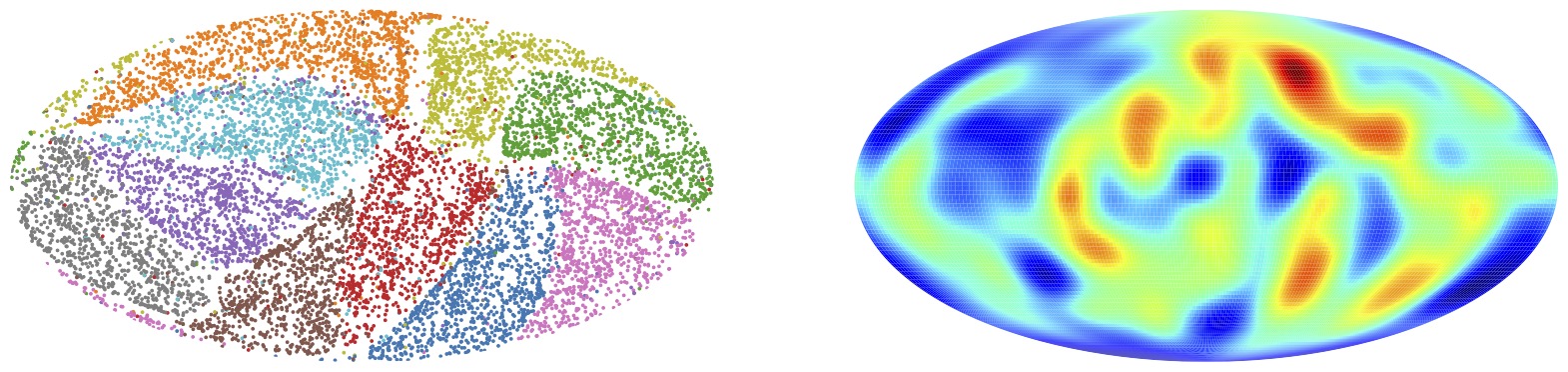}} \\
    \subfloat[$ARI\text{-}S3W$ ($10$ rotations, $\lambda=1$)]{\label{fig:aep_ari10_beta_1e0_4}\includegraphics[width=0.49\columnwidth]{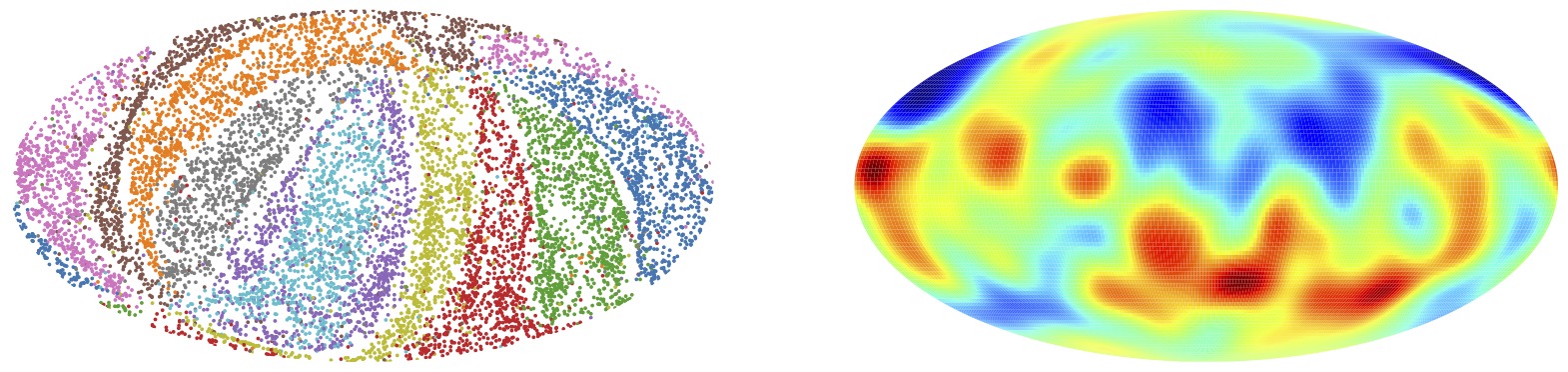}} \hfill
    \subfloat[$ARI\text{-}S3W$ ($30$ rotations, $\lambda=1$)]{\label{fig:aep_ari30_beta_1e0_1}\includegraphics[width=0.49\columnwidth]{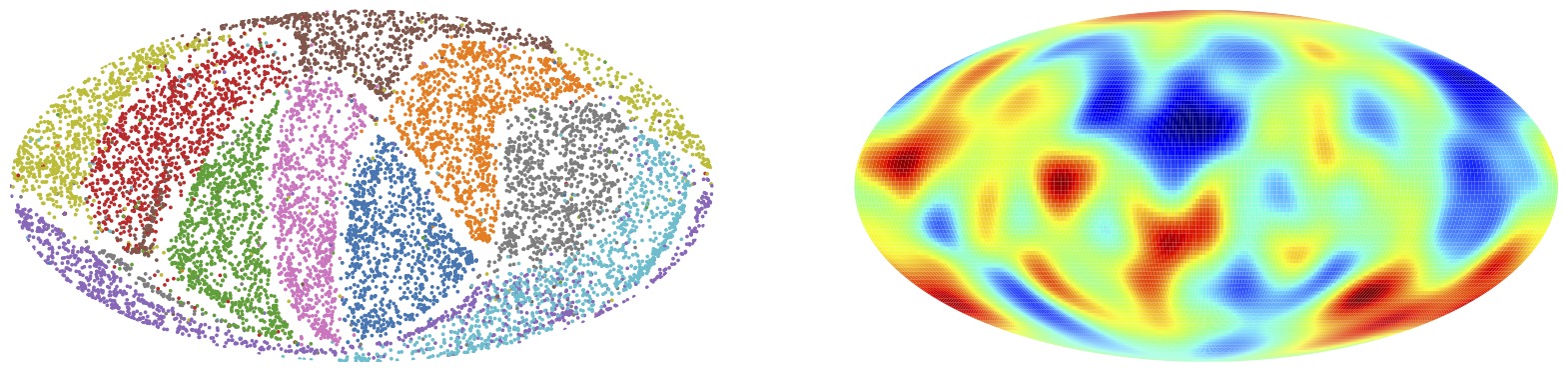}}
    \caption{Latent space visualization (MNIST).}
    \label{fig:models}
\end{figure}

\begin{table}[h]
\centering
\setlength{\tabcolsep}{5pt} % set the space between the text and the cell borders
\renewcommand{\arraystretch}{1.2} % to increase the space between rows
\begin{tabular}{|c|c|c|}
\hline
\textbf{Dataset} & \textbf{Encoder} & \textbf{Decoder} \\ \hline

\multirow{8}{*}{\rotatebox[origin=c]{90}{MNIST}} & \(x\in \mathbb{R}^{28\times 28} \to \mathrm{Conv2d}_{32} \to \mathrm{ReLU}\) & \(s\in\mathbb{S}^{2} \to \mathrm{FC}_{512} \to \mathrm{FC}_{512} \to \mathrm{ReLU}\) \\
 & \( \to \mathrm{Conv2d}_{32} \to \mathrm{ReLU}\) & \(\to \mathrm{Reshape(128 \times 2 \times 2)} \) \\
 & \( \to \mathrm{Conv2d}_{64} \to \mathrm{ReLU}\) & \(\to \mathrm{Conv2dT}_{128} \to \mathrm{ReLU}\)  \\
 & \(\to \mathrm{Conv2d}_{64} \to \mathrm{ReLU}\) & \(\to \mathrm{Conv2dT}_{64} \to \mathrm{ReLU}\)  \\
 & \(\to \mathrm{Conv2d}_{128} \to \mathrm{ReLU}\) & \(\to \mathrm{Conv2dT}_{64} \to \mathrm{ReLU}\) \\
 & \(\to \mathrm{Conv2d}_{128}\) & \(\to \mathrm{Conv2dT}_{32} \to \mathrm{ReLU}\) \\
 & \(\to \mathrm{Flatten} \to \mathrm{FC}_{512} \to \mathrm{ReLU}\) & \(\to \mathrm{Conv2dT}_{32} \to \mathrm{ReLU}\)\\
 & \(\to \mathrm{FC}_{3} \to \ell_2 \text{ normalization} \to s \in \mathbb{S}^{2}\) & \(\to \mathrm{Conv2dT}_{1} \to \mathrm{Sigmoid}\) \\
\hline

\multirow{8}{*}{\rotatebox[origin=c]{90}{CIFAR-10}} & \(x\in \mathbb{R}^{3\times 32 \times 32} \to \mathrm{Conv2d}_{32} \to \mathrm{ReLU}\) & \(s\in\mathbb{S}^{2} \to \mathrm{FC}_{512} \to \mathrm{FC}_{2048} \to \mathrm{ReLU}\) \\
 & \( \to \mathrm{Conv2d}_{32} \to \mathrm{ReLU}\) & \(\to \mathrm{Reshape(128 \times 4 \times 4)} \) \\
 & \( \to \mathrm{Conv2d}_{64} \to \mathrm{ReLU}\) & \(\to \mathrm{Conv2dT}_{128} \to \mathrm{ReLU}\)  \\
 & \(\to \mathrm{Conv2d}_{64} \to \mathrm{ReLU}\) & \(\to \mathrm{Conv2dT}_{64} \to \mathrm{ReLU}\)  \\
 & \(\to \mathrm{Conv2d}_{128} \to \mathrm{ReLU}\) & \(\to \mathrm{Conv2dT}_{64} \to \mathrm{ReLU}\) \\
 & \(\to \mathrm{Conv2d}_{128}\) & \(\to \mathrm{Conv2dT}_{32} \to \mathrm{ReLU}\) \\
 & \(\to \mathrm{Flatten} \to \mathrm{FC}_{512} \to \mathrm{ReLU}\) & \(\to \mathrm{Conv2dT}_{32} \to \mathrm{ReLU}\)\\
 & \(\to \mathrm{FC}_{3} \to \ell_2 \text{ normalization} \to s \in \mathbb{S}^{2}\) & \(\to \mathrm{Conv2dT}_{3} \to \mathrm{Sigmoid}\) \\
\hline
\end{tabular}
\caption{Architecture of the Encoder and Decoder for MNIST and CIFAR datasets.}
\label{tab:swae_arch}
\end{table}

\begin{table}[h]
\centering
\begin{minipage}{.85\linewidth}
    \caption{SWAE with different regularizations on the MNIST dataset with uniform prior for $d=2$. Here we compare among the Sliced Wasserstein ($SW$), Spherical Sliced Wasserstein ($SSW$) \cite{bonet2022spherical}, $S3W$, $RI\text{-}S3W$ ($1$ rotation), $RI\text{-}S3W$ ($10$ rotations), $ARI\text{-}S3W$ ($10$ rotations), $ARI\text{-}S3W$ ($30$ rotations, pool size of $100$). The metrics are $log_{10}(W_2)$ and Cross Entropy (CE). For $SSW$, we include a wide range of regularization parameter $\lambda=\{0.001, 0.01, 0.1, 1, 10, 100, 1000\}$ for thorough comparison. For other methods, we include either $\lambda=\{0.001, 0.01, 0.1, 1\}$ or $\lambda=\{1, 10, 100, 1000\}$.}
    \label{table:swae_mnist_uniform}
\end{minipage} 
% \captionsetup{skip=15pt}
\renewcommand{\arraystretch}{1.2}
\footnotesize
\begin{tabular}{|p{2.2cm}|>{\raggedright\arraybackslash}p{1.2cm}|>{\centering\arraybackslash}p{2.8cm}|>{\centering\arraybackslash}p{2.8cm}|>{\centering\arraybackslash}p{2.8cm}|}
\hline
\rowcolor[gray]{0.99}
Method & $\lambda$ & $\log(W_2)$ $\downarrow$ & CE $\downarrow$ & Time(s/ep.)\\ 
\hline
Supervised AE & $1$ & $-1.8438 \pm 0.1667$ & $0.1585 \pm 0.0018$ & $5.2500 \pm 0.0502$ \\ 
\hline\hline
\multirow{4}{*}{$SW$} 
& $1$ & $-1.1294 \pm 0.0102$ & $0.1619 \pm 0.0608$ & $5.2113 \pm 0.1191$ \\
& $10$ & $-1.1827 \pm 0.0096$ & $0.1652 \pm 0.0631$ & $5.4121 \pm 0.9013$ \\
& $100$ & $-1.6033 \pm 0.0148$ & $0.1641 \pm 0.0631$ & $5.7512 \pm 0.3019$ \\
& $1000$ & $-1.7856 \pm 0.0296$ & $0.1725 \pm 0.0668$ & $6.1400 \pm 0.4853$ \\
\hline\hline
\multirow{7}{*}{$SSW$} 
& $0.001$ & $-0.7585 \pm 0.0020$ & $0.1620 \pm 0.0615$ & $16.1380 \pm 1.2041$ \\ 
& $0.01$ & $-0.8523 \pm 0.0097$ & $0.1619 \pm 0.0623$ & $14.4800 \pm 1.4353$ \\
& $0.1$ & $-0.8555 \pm 0.0119$ & $0.1609 \pm 0.0610$ & $12.4125 \pm 0.0844$ \\
& $1$ & $-0.9393 \pm 0.0123$ & $0.1681 \pm 0.0647$ & $12.1478 \pm 1.462$ \\
& $10$ & $-0.9168 \pm 0.0099$ & $0.1603 \pm 0.0604$ & $13.3211 \pm 0.2393$ \\
& $100$ & $-1.1430 \pm 0.0119$ & $0.1598 \pm 0.0608$ & $12.3703 \pm 0.2108$ \\
& $1000$ & $-1.3658 \pm 0.0198$ & $0.1624 \pm 0.0616$ & $12.9221 \pm 0.1441$ \\
\hline\hline
\multirow{4}{*}{$S3W$} 
& $0.001$ & $-1.4883 \pm 0.0298$ & $0.1649 \pm 0.0634$ & $6.0980 \pm 0.1298$ \\
& $0.01$ & $-1.6289 \pm 0.0230$ & $0.2013 \pm 0.0706$ & $6.1580 \pm 0.0264$ \\
& $0.1$ & $-1.7331 \pm 0.0543$ & $0.2262 \pm 0.0643$ & $6.1620 \pm 0.0117$ \\
& $1$ & $-1.7446 \pm 0.0403$ & $0.2188 \pm 0.0669$ & $6.1400 \pm 0.0303$ \\
\hline\hline
\multirow{4}{*}{$RI\text{-}S3W (1)$} 
& $0.001$ & $-1.4061 \pm 0.0098$ & $0.1648 \pm 0.0624$ & $5.8920 \pm 0.0546$ \\
& $0.01$ & $-1.6955 \pm 0.0422$ & $0.1655 \pm 0.0422$ & $5.9120 \pm 0.0286$ \\
& $0.1$ & $-1.8754 \pm 0.0384$ & $0.2097 \pm 0.0721$ & $5.8780 \pm 0.0293$ \\
& $1$ & $-1.8308 \pm 0.0618$ & $0.2315 \pm 0.0699$ & $5.9100 \pm 0.0460$ \\
\hline\hline
\multirow{4}{*}{$RI\text{-}S3W (10)$} 
& $0.001$ & $-1.6410 \pm 0.0205$ & $0.1599 \pm 0.0613$ & $5.7860 \pm 0.4198$ \\
& $0.01$ & $-1.6724 \pm 0.0256$ & $0.1642 \pm 0.0644$ & $5.9725 \pm 0.0286$ \\
& $0.1$ & $-1.7934 \pm 0.0401$ & $0.1784 \pm 0.0679$ & $5.9900 \pm 0.0374$ \\
& $1$ & $-1.7332 \pm 0.0437$ & $0.2178 \pm 0.0718$ & $6.0040 \pm 0.0162$ \\
\hline\hline
\multirow{4}{*}{$ARI\text{-}S3W (10)$} 
& $0.001$ & $-1.5523 \pm 0.0347$ & $0.1649 \pm 0.0638$ & $6.2960 \pm 0.0372$ \\
& $0.01$ & $-1.5835 \pm 0.0311$ & $0.1572 \pm 0.0018$ & $6.3120 \pm 0.0271$ \\
& $0.1$ & $-1.6448 \pm 0.0381$ & $0.2141 \pm 0.0698$ & $6.3240 \pm 0.0242$ \\
& $1$ & $-1.7876 \pm 0.0586$ & $0.2273 \pm 0.0688$ & $6.1980 \pm 0.0970$ \\
\hline\hline
\multirow{4}{*}{$ARI\text{-}S3W (30)$} 
& $0.001$ & $-1.6453 \pm 0.0154$ & $0.1952 \pm 0.0176$ & $5.9540 \pm 0.5100$ \\
& $0.01$ & $-1.6290 \pm 0.0126$ & $0.1756 \pm 0.0647$ & $6.1600 \pm 0.0253$ \\
& $0.1$ & $-1.7889 \pm 0.0458$ & $0.2305 \pm 0.0631$ & $6.1625 \pm 0.0286$ \\
& $1$ & $-1.8890 \pm 0.0506$ & $0.2263 \pm 0.0637$ & $6.1840 \pm 0.0185$ \\
\hline
\end{tabular}
\end{table}

\begin{figure}[!h]
    \centering
    \subfloat[$SW$]{\label{fig:sample_sw}\includegraphics[width=0.2\columnwidth]{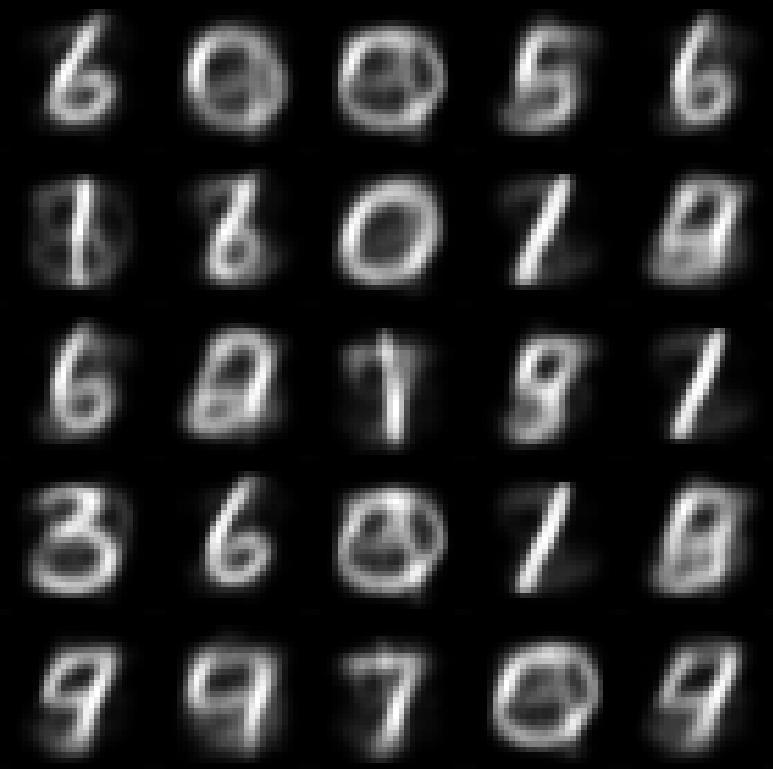}} \hspace{0.2in}
    \subfloat[$SSW$]{\label{fig:sample_ssw}\includegraphics[width=0.2\columnwidth]{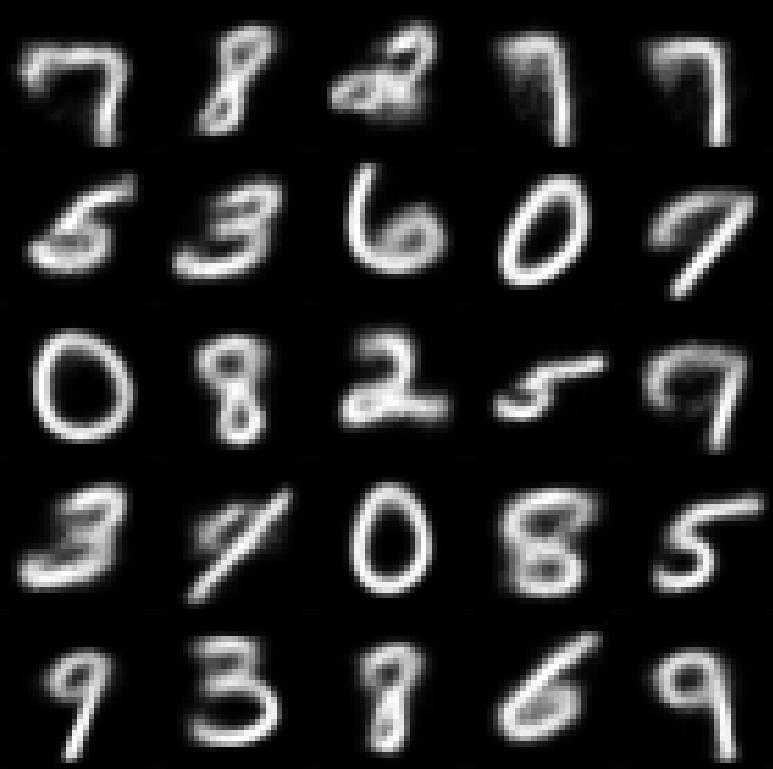}} \hspace{0.2in}
    \subfloat[$S3W$]{\label{fig:sample_s3w}\includegraphics[width=0.2\columnwidth]{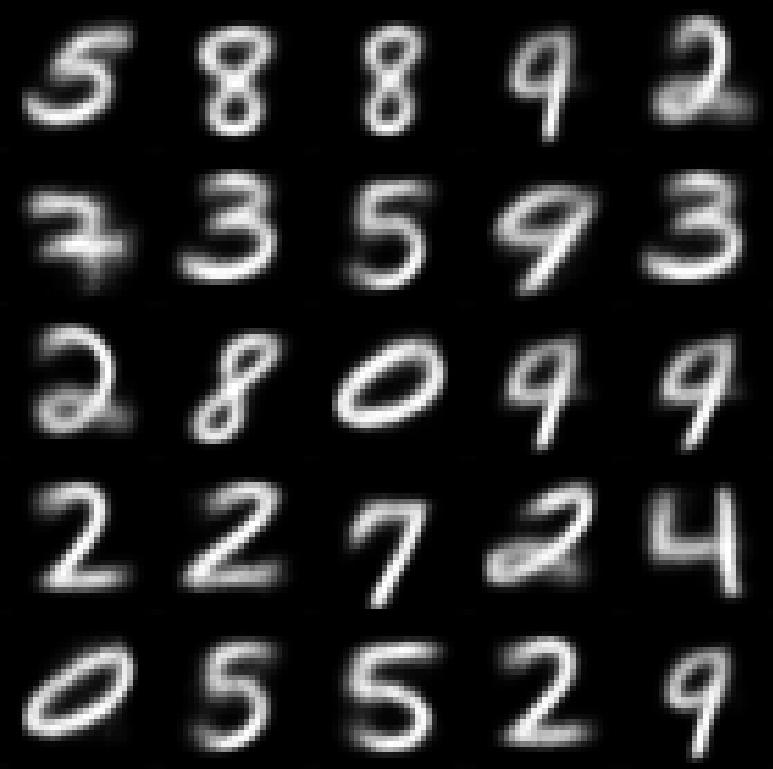}} \\
    \subfloat[$RI$-$S3W$ (1)]{\label{fig:sample_ris3w1}\includegraphics[width=0.2\columnwidth]{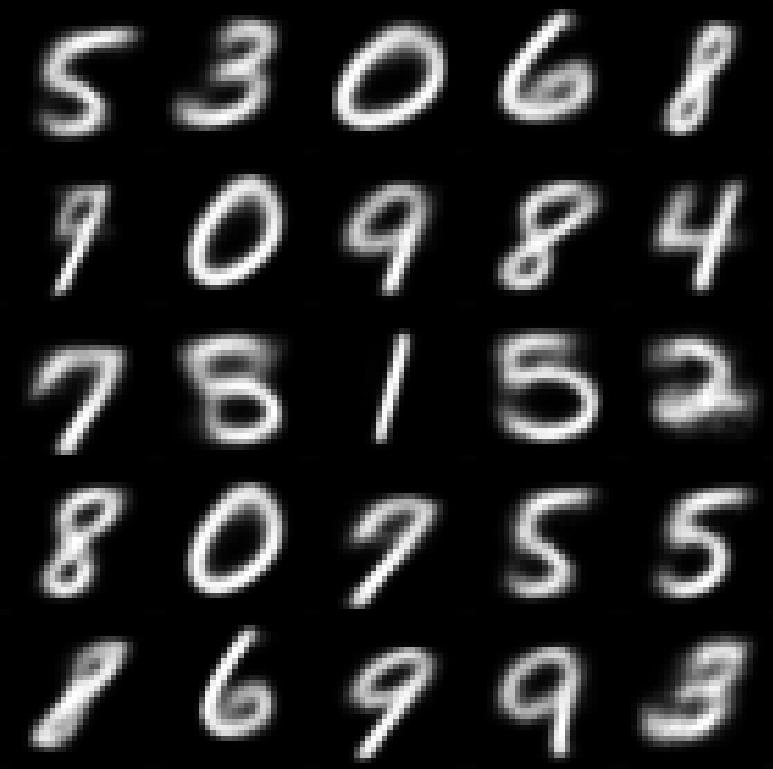}} \hspace{0.2in}
    \subfloat[$RI$-$S3W$ (10)]{\label{fig:sample_ris3w10}\includegraphics[width=0.2\columnwidth]{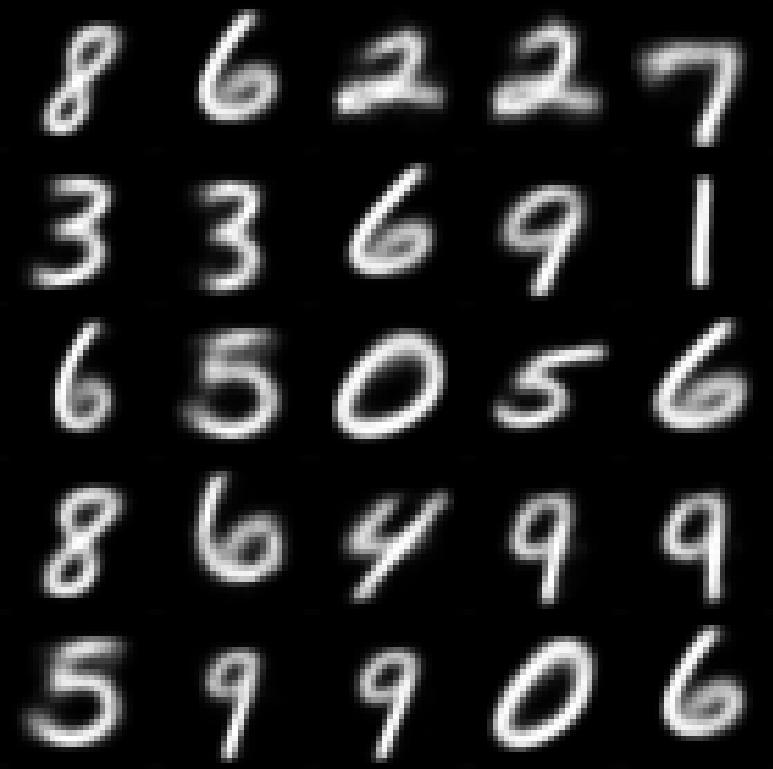}} \hspace{0.2in}
    \subfloat[$ARI$-$S3W$ (30)]{\label{fig:sample_aris3w30}\includegraphics[width=0.2\columnwidth]{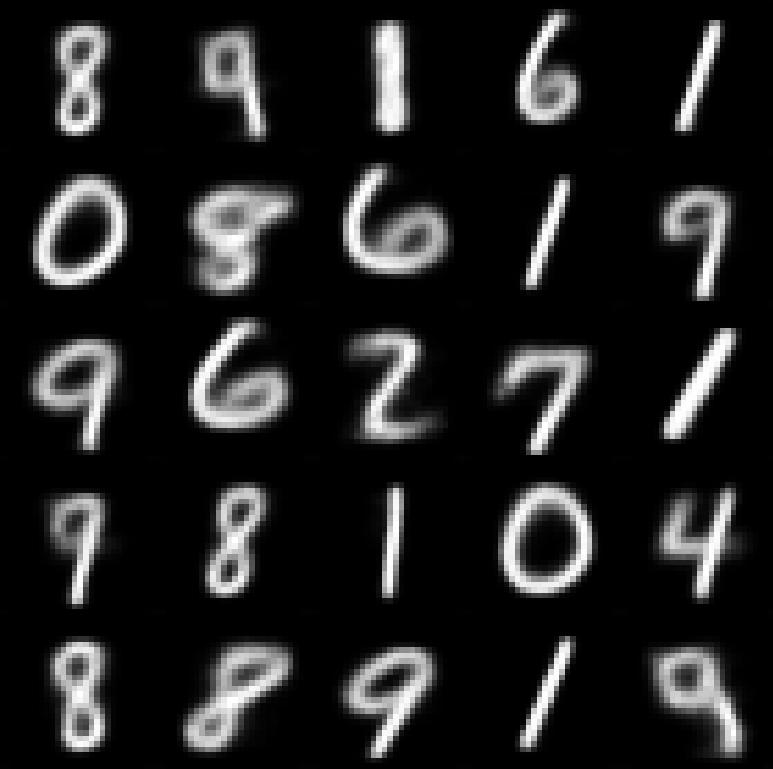}}
    \caption{Samples generated by SWAE models with a uniform prior on $\mathbb{S}^2$.}
    \label{fig:swae_samples}
\end{figure}

%%%%%%%%%%%%%%%%%%%%%%%%%%%%%%%%%%%%%%%%%%%%%%%%%%%%%%%%%%%%%%%%%%%%%%%%%%%%%%%
%%%%%%%%%%%%%%%%%%%%%%%%%%%%%%%%%%%%%%%%%%%%%%%%%%%%%%%%%%%%%%%%%%%%%%%%%%%%%%%

\end{document}